%% file: main.tex
\documentclass[journal]{IEEEtran} 

\usepackage[utf8]{inputenc}

\usepackage{pgf}
\usepackage{amsmath}
\usepackage{amsfonts}
\usepackage{amssymb}
\usepackage{gensymb}
\usepackage{xcolor}
\usepackage{graphicx}
\usepackage{caption}
\usepackage{subcaption}
\usepackage{enumitem}
\usepackage[ruled,vlined]{algorithm2e}
\usepackage{amsthm}

\theoremstyle{definition}

\newtheorem{definition}{Definition}
\newtheorem{property}{Property}
\newtheorem{proposition}{Proposition}
\newtheorem{problem}{Problem}

\theoremstyle{remark}
\newtheorem*{remark}{Remark}

\usepackage{tikz}
\usetikzlibrary{positioning}
\usepackage[normalem]{ulem}
\usepackage{cite}
\usepackage{color,soul}

\usepackage{layouts}
\usepackage[font={footnotesize,it}]{caption}

\usepackage{tikz}
\usetikzlibrary{decorations.pathreplacing,calc}
\newcommand{\tikzmark}[1]{\tikz[overlay,remember picture] \node (#1) {};}

\newcommand{\added}[1]{#1}

\title{
Creating Star Worlds\\
\Large \added{Reshaping the Robot Workspace for Online Motion Planning}
}

\author{Albin Dahlin and Yiannis Karayiannidis 
\thanks{This work has been supported by Chalmers AI Research Centre (CHAIR) and AB Volvo through the project AiMCoR.}

\thanks{A. Dahlin is  with the Department of Electrical Engineering, Chalmers University of Technology, SE-412 96 Gothenburg, Sweden 
        {\tt\small albin.dahlin@chalmers.se}}%
        \thanks{Y. Karayiannidis is with the Department of Automatic Control, Lund University, Sweden {\tt\small  yiannis@control.lth.se}. The author is a member of the ELLIIT Strategic Research Area at Lund University. }
}

\begin{document}

\maketitle

\begin{abstract}
    Motion planning methods like navigation functions and harmonic potential fields provide (almost) global convergence and are suitable for obstacle avoidance in dynamically changing environments due to their reactive nature. A common assumption in the control design is that the robot operates in a disjoint star world, i.e. all obstacles are strictly starshaped and mutually disjoint. However, in real-life scenarios obstacles may intersect due to expanded obstacle regions corresponding to robot radius or safety margins. To broaden the applicability of aforementioned reactive motion planning methods, we propose a method to reshape a workspace of intersecting obstacles into a disjoint star world. The algorithm is based on two novel concepts presented here, namely admissible kernel and starshaped hull with specified kernel, which are closely related to the notion of starshaped hull. 
    The utilization of the proposed method is illustrated with examples of a robot operating in a 2D workspace using a harmonic potential field approach in combination with the developed algorithm.
\end{abstract}


\section{Introduction}
One of the central problems in robotics is to plan the motion of the robot to a desired goal state while avoiding collisions with obstacles \cite{lavalle_11}. Motion planning remains an active topic of research as modern robots, such as collaborative robots, mobile robots, unmanned aerial vehicles or even combined robotic systems for mobile manipulation, are intended for use in dynamically changing environments. Traditionally, motion planning problems have been solved considering static maps, but when considering changes in the environments and moving obstacles the robot must constantly adjust its planned path to avoid crashes. A common method to tackle such problems is to construct closed form control laws formulated as dynamical systems that provide stability and convergence guarantees in combination with instant reactivity to changing environments. Specifically, artificial potential fields, introduced in \cite{khatib_85}, based on scalar functions, guiding the robot with an attractive force to the goal and repulsive forces from the obstacles, have become popular \cite{ginesi_etal_19,stavridis_etal_17}. 
However, a drawback of the additive potential field methods is the possible existence of local minimum other than the goal point, i.e. the robot could get stuck at a position away from the goal. To address this issue, navigation functions \cite{rimon_koditschek_92,conn_kam_98,loizou_11_2,paternain_etal_18, hacohen_etal_19, loizou_rimon_21} have emerged which are a special subclass of potential functions designed to be bounded. Other approaches providing (almost) global convergence are based on harmonic potential fields  \cite{connolly_etal_90, feder_slotine_97, daily_bevly_08, huber_etal_19, huber_etal_22}. \added{}

A repeated assumption \cite{daily_bevly_08,conn_kam_98,huber_etal_19,paternain_etal_18,rimon_koditschek_92,loizou_11_2,hacohen_etal_19,loizou_rimon_21} enabling the proof of (almost) global convergence is the premise of disjoint obstacles. However, in cluttered dynamic environments, closely positioned obstacles may be seen as having intersecting regions, violating the assumption of disjoint regions. This occurs since the obstacles need to be inflated by the robot radius due to a point mass modelling, and since it is common to add extra safety margin to prevent the robot touching the obstacles during obstacle circumvention or to address uncertainties in obstacle positions. To preserve the convergence properties, intersecting obstacles must therefore be combined into a single obstacle.
Additionally, the obstacle shapes are restricted to be topological disks. In most cases they are considered as ellipses or Euclidean disks, but a more general premise is to consider all \textit{strictly starshaped sets} (see example in Fig. \ref{fig:starshaped_set}) - sets where there exists a point, a \textit{kernel} point, where all rays emanating from the kernel point cross the boundary once and only once. 
While methods like \cite{huber_etal_19} can operate directly in star worlds (a workspace of strictly starshaped obstacles), navigation functions can be used in star worlds by defining diffeomorphic transformations of the obstacles to Euclidean disks \cite{rimon_koditschek_92} or to points \cite{loizou_17}.

\begin{figure}
    \centering
    \begin{subfigure}[t]{0.49\linewidth}
        \begin{center}  \input{figures/obstacle_representation.pgf}
        \end{center}
         \caption{Example of a (concave) strictly starshaped set with kernel shown in blue.}
         \label{fig:starshaped_set}
    \end{subfigure}
    \hfill
    \begin{subfigure}[t]{0.49\linewidth}
        \centering
        \includegraphics[width=\linewidth,trim={6.5cm 15.5cm 5.5cm 7.5cm},clip]{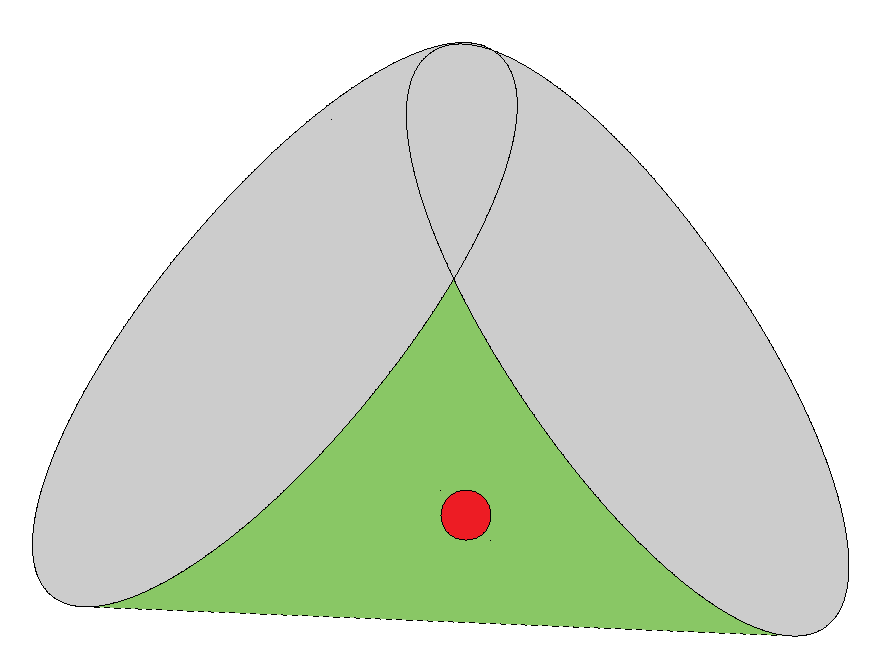}
        \caption{Robot position, shown as red circle, is included in the convex hull, shown in green, of the intersecting obstacles.}
        \label{fig:ch_include_robot}
    \end{subfigure}
    \caption{}
    \label{fig:intro_examples}
\end{figure}

To apply the aforementioned motion planning methods in practice with preserved convergence properties, a star world representation with mutually disjoint obstacles is thus needed. The obstacles in this star world should fully cover the original obstacles to ensure obstacle avoidance.
\added{In this paper, we consider the problem of reshaping a workspace of possibly intersecting obstacles into a workspace of disjoint strictly starshaped obstacles. The modification should be applicable to a generic environment of intersecting obstacles. In this way, it can be used online as a preprocessing step to reactive motion planning methods able to cope with online topological and geometrical changes of the obstacles, e.g. \cite{huber_etal_22}, as exemplified in Fig. \ref{fig:block_soads}.}

\begin{figure}
    \centering
    \begin{tikzpicture}
        \node [draw,
        text width=6em, fill=blue!20, text centered,
        minimum height=3em, rounded corners
        ]  (alg) at (0,0) {Workspace\\ Modification};
        \node [draw,
        text width=6em, fill=orange!20, text centered,
        minimum height=3em, rounded corners,
            right=2cm of alg
        ] (soads) {Motion Planner};
        \node [draw,
        text width=6em, fill=orange!20, text centered,
        minimum height=3em, rounded corners,
            below right= 1cm and -0.25cm of alg
        ]  (workspace) {Workspace};
        \node [left=.5cm of alg](input){};
        \draw[-stealth] (input.center) -- (alg.west);
        \draw[-stealth] (alg.east) -- (soads.west) 
            node[midway,above]{$\mathcal{O}^{\star}, x, x_g$};
        \draw[-] (soads.east) -- ++ (.5,0) 
            node[](output){}node[midway,above]{$\dot{x}$};
        \draw[-stealth] (output.center) |- (workspace.east);
        \draw[-] (workspace.west) -| (input.center) 
            node[near start,above]{$\mathcal{O}, x, x_g$};
        \end{tikzpicture}
    \caption{A set of disjoint starshaped obstacles, $\mathcal{O}^{\star}$, is generated given a workspace with obstacles, $\mathcal{O}$, the robot position, $x$, and the goal position, $x_g$. $\mathcal{O}^{\star}$ is then used when deciding on the movement for the robot, $\dot{x}$, with the motion planner.}
    \label{fig:block_soads}
\end{figure}
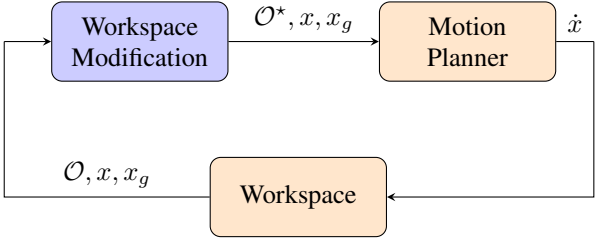

The naive approach for ``starifying'' a non-starshaped region is to make use of the convex hull since all convex sets are starshaped.
However, using the convex hull introduces a more conservative enclosing than needed and 
it can result in expanding an obstacle such that the robot position is in the obstacle interior, as in Fig. \ref{fig:ch_include_robot}, which could inhibit a sound motion planning process.
Instead, forest of stars was presented in \cite{rimon_koditschek_92} as a way to transform a set of intersecting starshaped obstacles into a set of disjoint starshaped obstacles. However, it is restricted to a specific structure where all intersecting obstacles are ordered in a parent-child relation. In particular, a parent must contain a kernel point of every child and no children with same parent are allowed to intersect. \added{No solution is hence provided for the situations where several obstacles share an intersecting region as in Fig. \ref{fig:hull_kernel_intersecting}. Moreover, the transformation relies on correct tuning of scene specific parameters.}
While these may be reasonable restrictions when modelling a static workspace offline, it cannot be presumed to hold in a dynamic environment with moving obstacles. 
\added{
Other approaches like \cite{vlantis_etal_18} and \cite{loizou_14} propose transformations in more general workspaces, without the restrictions needed for forest of stars, but assume static obstacles in predefined scenes. Moreover, these transform the workspace directly to a point world such that planning methods operating directly in star worlds, e.g. \cite{huber_etal_19}, are not applicable.}

As an alternative, the mathematical concept of starshaped hull introduced in \cite{beltagy_araby_00}, defined analogously to the convex hull \footnote{The starshaped hull of a set $A$ with respect to some point $x$ is the smallest set which fully contains $A$ and where $x$ is a kernel point. See Sec. \ref{sec:starshaped_hull} for a more detailed explanation.}, is reasonable to use in order to guarantee starshaped obstacle representation \added{as mentioned in \cite{huber_etal_22}}. A straightforward use of the starshaped hull to generate obstacle regions does not consider excluding points from the expanded obstacle and, as stated above for the convex hull, may be problematic. Moreover, by nature of the starshaped hull, the resulting set is not in general strictly starshaped and a direct use of the starshaped hull does not in general result in a star world.

In this paper, we present a method to reshape a workspace of possibly intersecting obstacles into a workspace of disjoint strictly starshaped obstacles. The method is based on two concepts, both introduced here: admissible kernel and starshaped hull with specified kernel. The admissible kernel enables excluding points of interest from the starshaped hull, i.e. obstacles can be modelled as starshaped regions which do not contain the robot nor goal position. The starshaped hull with specified kernel expands on the idea of starshaped hull and will prove to be useful in generating strictly starshaped sets.
Additionally, some general properties of the starshaped hull are established and are instrumental in the design of the proposed algorithm. \added{In contrast to forest of stars, no specific structure of the obstacle intersection is required for the proposed method. This allows for online mapping to disjoint star worlds of more generic environments.}

First, in Sec. \ref{sec:preliminaries}, the notation used and brief theory of starshaped sets is presented, and in Sec. \ref{sec:problem} the problem formulation is stated. In Sec. \ref{sec:starshaped_hull}, some important properties of the starshaped hull is presented, followed by the definition of the admissible kernel and a definition of the starshaped hull with specified kernel. 
An algorithm to reshape a workspace containing intersecting obstacles into a workspace with disjoint starshaped obstacles is presented in Sec. \ref{sec:forming_star_obstacles} and in Sec. \ref{sec:examples} some examples are provided where the algorithm is used in combination with a motion planner. Finally, in Sec. \ref{sec:conclusions}, conclusions are drawn.

\section{Preliminaries}
\label{sec:preliminaries}
\subsection{Mathematical notation}
The closed line segment from point $a$ to point $b$ is denoted as $l[a,b]$, the line through $a$ and $b$ is denoted as $l(a,b)$ and the vector from $a$ to $b$ is denoted $\overrightarrow{ab}$. The ray emanating from point $a$ in the direction of $\overrightarrow{bc}$ is denoted as $r(a,\overrightarrow{bc})$. \added{Given two sets $A$ and $B$, the notation $A\subset B$ is used to indicate that $A$ is a subset of $B$.}
The interior, the boundary and the exterior of a set $A\subset \mathbb{R}^n$ are denoted by $\textnormal{int}A$, $\partial A$ and $\textnormal{ext}A$, respectively. The convex hull of $A$ is denoted $CH(A)$ \added{and the cardinality of $A$ is denoted $|A|$}.
Given a closed convex set $A$ and an exterior point $x \in \textnormal{ext}A$, a point $a\in \partial A$ is called a tangent point of $A$ through $x$ if the ray emanating from $x$ in direction of $\overrightarrow{xa}$ does not intersect the interior of $A$. That is, $r(x,\overrightarrow{xa}) \cap \textnormal{int}A = \emptyset$. The set of all tangent points of $A$ through $x$ is denoted by $\mathcal{T}_A(x)$. For any interior or boundary point $x \in A$, $\mathcal{T}_A(x)$ is defined as the empty set. 
The linear cone containing all rays between two rays, $r_1$ and $r_2$, emanating from the same point, in counterclockwise (CCW) orientation, is denoted by $C_\angle(r_1,r_2)$. 

\subsection{Starshaped sets}
A set $A\subset \mathbb{R}^n$ is \textit{starshaped with respect to $x$} if for every point $y\in A$ the line segment $l[x,y]$ is contained in $A$. The set $A$ is said to be \textit{starshaped} if it is starshaped with respect to (w.r.t.) some point $x$, i.e. $\exists x$ s.t. $l[x,y] \subset A, \forall y \in A$. The set of all such points is called the \textit{kernel of $A$} (show in blue for the example in Fig. \ref{fig:starshaped_set}) and is denoted $\textnormal{ker}(A)$, i.e. $\textnormal{ker}(A) = \{x\in A : l[x,y] \subset A, \forall y\in A\}$. The kernel of $A$ is a convex set and the set $A$ is convex if and only if $\textnormal{ker}(A) = A$. 

The set $A$ is \textit{strictly starshaped with respect to $x$} if it is starshaped w.r.t. $x$ and any ray emanating from $x$ crosses the boundary only once, i.e. $r(x,\overrightarrow{xy}) \cap \partial A = \{y\}, \forall y\in \partial A$. We say that $A$ is strictly starshaped if if it is strictly starshaped w.r.t. some point.
For a thorough survey on the theory of starshaped sets, see \cite{hansen_etal_20}.

\section{Problem Formulation}
\label{sec:problem}
In this work, we consider a robot operating in the Euclidean space or Euclidean plane containing a collection of possibly intersecting obstacle regions, $\mathcal{O}$. Depending on dimensionality, the robot workspace, $\mathcal{W}$, and each obstacle region, $\mathcal{O}_i$, are
\begin{enumerate}[label=\alph*)]
    \item $\mathcal{W}=\mathbb{R}^3$ and $\mathcal{O}_i\subset \mathbb{R}^3$ is a convex set,
    \item $\mathcal{W}=\mathbb{R}^2$ and $\mathcal{O}_i\subset \mathbb{R}^2$ is a convex set or a polygon.
\end{enumerate}
Note that by allowing intersecting regions, the formulation does not restrict the obstacles to be convex since a single concave obstacle can be modelled as a combination of convex regions, e.g. a human/robot modelled as a kinematic chain of ellipsoids. In addition to this, any kind of (starshaped or non-starshaped) polygon shape is included in the formulation in $\mathbb{R}^2$.
Obviously, the scenario when multiple obstacles are closely located such that their regions intersect when introducing margins, e.g. to adjust for robot radius, as in Fig. \ref{fig:workspace_example}, is also considered. 

\begin{figure}
    \centering
    \includegraphics[width=\linewidth,trim={0 14cm 0 0},clip]{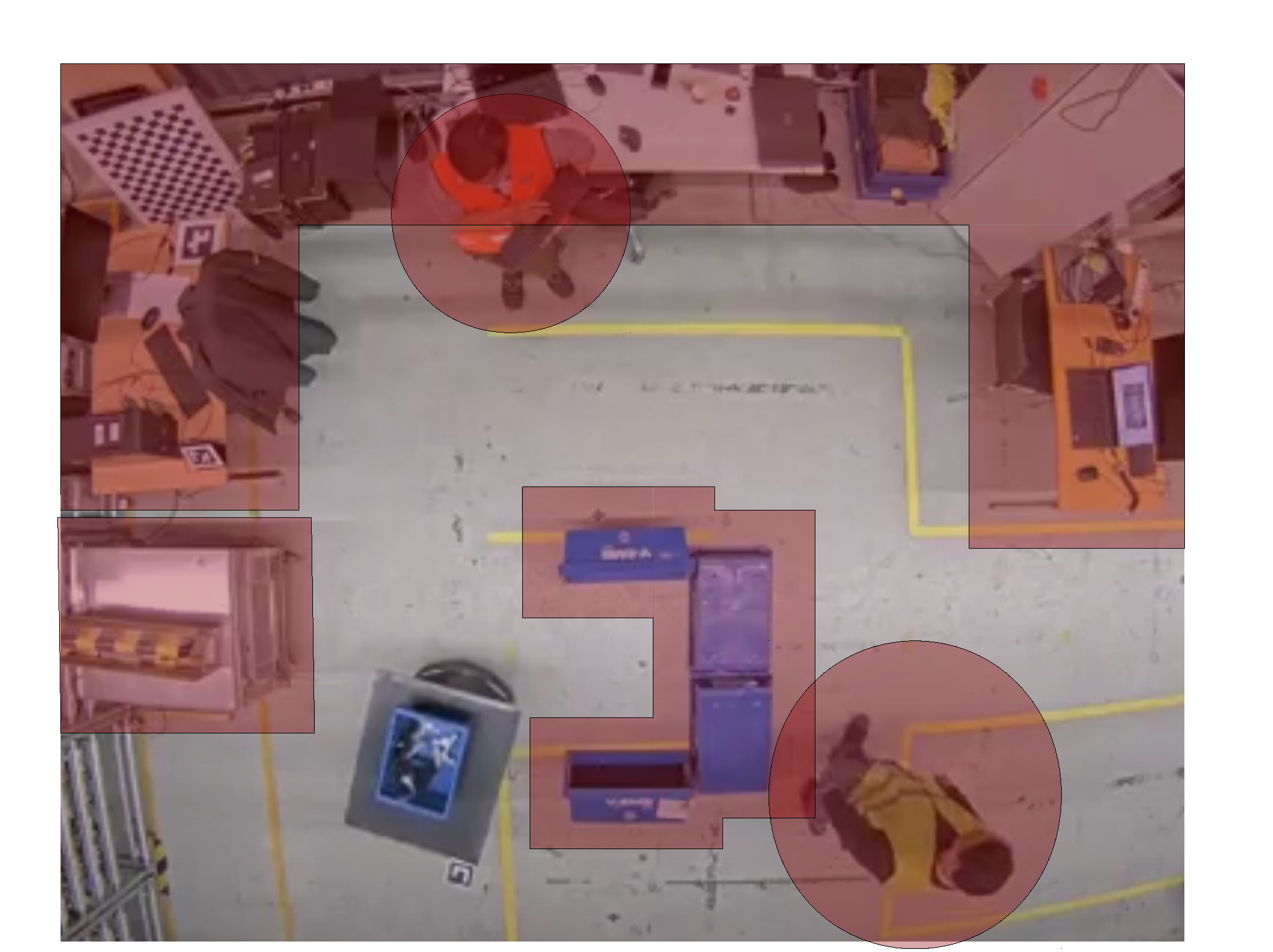}
    \caption{Example of a scenario in $\mathbb{R}^2$ containing two humans modelled as disk obstacles and three static obstacles modelled as polygons, one convex and two non-starshaped. Both non-starshaped polygons intersect with the disks due to introduced safety margins of the obstacles.}
    \label{fig:workspace_example}
\end{figure}

The \textit{free configuration space} is 
\begin{equation}
    \mathcal{F} = \mathcal{W} \setminus \bigcup_{\mathcal{O}_i\in\mathcal{O}}\mathcal{O}_i
\end{equation}
and similar to \cite{rimon_koditschek_92} we define a star world as follows\footnote{The definition of a star world in \cite{rimon_koditschek_92} in fact correspond to what we call a disjoint star world and does not include scenarios with intersecting obstacles.}:
\begin{definition}
A free configuration space, $\mathcal{F}$, all of whose obstacles are strictly starshaped sets is a \textit{star world}.
\end{definition}
To distinguish between the scenario with intersecting obstacles and with mutually disjoint obstacles, we will call a star world where all obstacles are mutually disjoint a \textit{disjoint star world} and a star world where two or several obstacles intersect an \textit{intersecting star world}.

The objective is to create a disjoint star world in the free configuration space, $\mathcal{F^{\star}}\subset \mathcal{F}$, such that motion planning methods operating in star worlds can be applied with guaranteed convergence to a goal position, $x_g$. For sound motion planning, the robot position, $x$, should remain in the free set, and for convergence to the goal, the goal should remain in the free set, i.e. $x\in\mathcal{F}^{\star}$ and $x_g\in\mathcal{F}^{\star}$. In cases when no such $\mathcal{F}^{\star}$ exists, e.g. when the robot and/or goal are fully surrounded by several intersecting obstacles, the condition of disjoint obstacles is relaxed and an intersecting star world containing $x$ and $x_g$ should be created. The problem is stated as follows:
\begin{problem}
\label{problem}
Consider a robot workspace, $\mathcal{W}$, and a collection of obstacle regions, $\mathcal{O}$, that are either
\begin{enumerate}[label=\alph*)]
    \item $\mathcal{W}=\mathbb{R}^3$ and $\mathcal{O}_i\subset \mathbb{R}^3$ is a convex set,
\end{enumerate}
or
\begin{enumerate}[label=\alph*)]
\setcounter{enumi}{1}
    \item $\mathcal{W}=\mathbb{R}^2$ and $\mathcal{O}_i\subset \mathbb{R}^2$ is a convex set or a polygon.
\end{enumerate}
Given the current position of a robot, $x\in\mathcal{F}$, and a goal position $x_g\in\mathcal{F}$, construct a disjoint star world $\mathcal{F}^{\star} \subset \mathcal{F}$ such that $x\in\mathcal{F}^{\star}$ and $x_g\in\mathcal{F}^{\star}$. That is, construct a collection of obstacles, $\mathcal{O}^{\star}$, such that

\begin{subequations}
\begin{align}
    &\bigcup_{\mathcal{O}_i\in\mathcal{O}}\mathcal{O}_i \subset \bigcup_{\mathcal{O}_i^{\star}\in\mathcal{O}^{\star}}\mathcal{O}_i^{\star} \label{eq:obstacle_inclusion}\\ 
    &\mathcal{O}_i^{\star} \textnormal{ is strictly starshaped},\ \forall \mathcal{O}_i^{\star}\in \mathcal{O}^{\star} \label{eq:obstacle_proper_starshaped}\\ 
        &x \not\in \mathcal{O}_i^{\star},\ \forall \mathcal{O}_i^{\star}\in \mathcal{O}^{\star} \label{eq:obstacle_robot_exclusion}\\ 
    &x_g \not\in \mathcal{O}_i^{\star},\ \forall \mathcal{O}_i^{\star}\in \mathcal{O}^{\star} \label{eq:obstacle_goal_exclusion}\\
    &\mathcal{O}_i^{\star} \cap \mathcal{O}_j^{\star} = \emptyset,\ \forall \added{i\neq j} 
    \label{eq:obstacle_disjoint}.
\end{align}
\end{subequations}

If no such $\mathcal{F}^{\star}$ exists, construct a star world $\mathcal{F}^{\star} \subset \mathcal{F}$ such that $x\in\mathcal{F}^{\star}$ and $x_g\in\mathcal{F}^{\star}$. That is, construct a collection of obstacles, $\mathcal{O}^{\star}$, satisfying \eqref{eq:obstacle_inclusion}-\eqref{eq:obstacle_goal_exclusion}.
\end{problem}

\section{Starshaped hull}
\label{sec:starshaped_hull}

A definition of the starshaped hull was provided in \cite{beltagy_araby_00}. In this work, it is defined with a minor modification as follows:
\begin{definition}
\label{d:st_hull}
Let $A\subset \mathbb{R}^n$ and $x\in \mathbb{R}^n$. The \textit{starshaped hull of $A$ with respect to $x$}, denoted $SH_x(A)$, is the smallest starshaped set with respect to $x$ containing $A$.
\end{definition}
In comparison, the starshaped hull is now well-defined for any $x\in \mathbb{R}^n$ and not solely for $x\in A$. The measure is conventionally considered in terms of Lesbegue measure. 
Proposition 3.2 in \cite{beltagy_araby_00} still holds for the adjusted definition and we have 
\begin{equation}
    SH_x(A) = \bigcup_{y\in A}l[x,y].
\end{equation}
We will interchangeably refer to the starshaped hull as a set and as an operation, i.e. the smallest starshaped set w.r.t $x$ containing $A$ vs the generation of this set.
Here we provide some important properties of the starshaped hull which will be used in the subsequent sections.
\begin{property}
\label{p:st_hull}
Let $A\subset \mathbb{R}^n$ and let $\mathcal{B}$ be a collection of sets $B\subset \mathbb{R}^n,\ \forall B\in \mathcal{B}$.
\begin{enumerate}[label=\alph*.]
    \item $SH_x(A) = A \Leftrightarrow A$ is starshaped and $x \in \textnormal{ker}(A)$ \label{p:hull_star_kernel}
    \item $SH_x(A) \subset CH(A),\ \forall x \in CH(A)$ \label{p:hull_ch_bound}
    \item $\displaystyle SH_x\left(\bigcup_{B\in\mathcal{B}} B\right) = \bigcup_{B\in\mathcal{B}} SH_x(B)$ \label{p:hull_union}
\end{enumerate}
\end{property}
\begin{proof}
See Appendix \ref{a:proof_p_st_hull}.
\end{proof}

As a consequence of Property \ref{p:st_hull}\ref{p:hull_star_kernel} we have $SH_x(A) \neq A$ for any $x\not\in \textnormal{ker}(A)$. That is, the starshaped hull of any set is a strict superset unless the set is starshaped and the hull is generated w.r.t. a kernel point. 
Property \ref{p:st_hull}\ref{p:hull_ch_bound} ensures that the starshaped hull w.r.t. any point $x\in CH(A)$ provides a less (or at most equally) conservative enclosing of the set $A$ compared to the convex hull. In particular, any point $x\in A$ can be used to generate a starshaped set which is upper bounded by the convex hull of $A$.
Property \ref{p:st_hull}\ref{p:hull_union} simplifies finding the starshaped hull for complex regions which can be described as combinations of simpler subsets, e.g. as the union of several polygons and/or convex sets, as the hull can be computed separately for each subset.

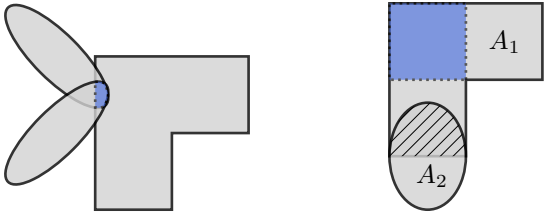
\begin{figure}
    \centering
    \begin{subfigure}[t]{0.49\linewidth}
        \begin{center}  
            \input{figures/hull_kernel_intersecting.pgf}
        \end{center}
         \caption{The kernel intersection, $K_{\cap}$, of the three sets, in blue, is nonempty and the union is starshaped w.r.t. any point in $K_{\cap}$ according to Proposition \ref{prop:intersecting_kernel}.}
         \label{fig:hull_kernel_intersecting}
    \end{subfigure}
    \hfill
    \begin{subfigure}[t]{0.49\linewidth}
    \begin{center}  
        \input{figures/hull_kernel_non_intersecting.pgf}
    \end{center}
     \caption{Two intersecting starshaped set, $A_1$ and $A_2$, where the intersection, $A_1\cap A_2$, is the hatched area. The union, $A_{\cup}$, is starshaped even though the kernel intersection, $K_{\cap}$, is empty since $\textnormal{ker}(A_1)$, in blue, is disjoint from $\textnormal{ker}(A_2)=A_2$.}
     \label{fig:hull_kernel_non_intersecting}
    \end{subfigure}
    \caption{Examples of where the union, $A_{\cup}$, of several starshaped sets is starshaped, with $\textnormal{ker}(A_{\cup})$ shown in blue.}
    \label{fig:th_ker_intersect}
\end{figure}

An algorithm to find the starshaped hull of a polygon has been presented in \cite{arkin_etal_98}.
For a convex set $A_{\textrm{conv}}$ the starshaped hull w.r.t. $x$ is given as
\begin{equation}
\label{eq:local_hull_convex}
    SH_x(A_{\textrm{conv}}) = A_{\textrm{conv}} \cup CH\left(\mathcal{T}_{A_{\textrm{conv}}}(x) \cup x\right).
\end{equation}
In words, the starshaped hull w.r.t. any exterior point, $x$, is the union of the set itself and the cone with apex in $x$ and base given as the convex hull of all tangent points of $A$ through $x$. As a consequence, the starshaped hull of a convex set is also convex.
In the case of $A\subset \mathbb{R}^2$ we have that $CH\left(\mathcal{T}_A(x) \cup x\right)$ is a triangle with vertices in $x$ and in two tangent points of $A$ through $x$, see Fig. \ref{fig:ellipse_hull}. When $A\subset \mathbb{R}^3$ we have that $CH\left(\mathcal{T}_A(x) \cup x\right)$ is the solid cone with apex $x$ and base as the planar intersection of $A$ containing three, and thus all, tangent points of $A$ through $x$, see Fig. \ref{fig:ellipsoid_hull}.

\begin{figure}
    \centering
    \begin{subfigure}[t]{0.49\linewidth}
        \begin{center}
            \input{figures/hull_ellipse.pgf}
        \end{center}
         \caption{$A\in \mathbb{R}^2$ is an ellipse.}
         \label{fig:ellipse_hull}
    \end{subfigure}
    \hfill
    \begin{subfigure}[t]{0.49\linewidth}
     \centering
     \includegraphics[width=\linewidth]{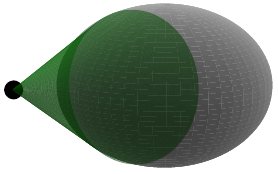}
     \caption{$A\in \mathbb{R}^3$ is an ellipsoid.}
     \label{fig:ellipsoid_hull}
    \end{subfigure}
    \caption{Starshaped hull, $SH_x(A)$, w.r.t. a point $x$ for a convex set $A$ in $\mathbb{R}^2$ and $\mathbb{R}^3$. The sets $A$ are shown in gray, the point $x$ as a black dot and the hull extended cone, $CH\left(\mathcal{T}_A(x) \cup x\right)$, in green.}
    \label{fig:st_hull_convex}
\end{figure}

In the design of the proposed method, we will additionally make use of the following two characteristics for starshaped sets.
\begin{proposition}
\label{prop:intersecting_kernel}
Let $\mathcal{A}$ be a collection of starshaped sets $A\subset \mathbb{R}^n,\ \forall A\in \mathcal{A}$ with union $A_{\cup}=\bigcup_{A\in\mathcal{A}} A$ and kernel intersection $K_{\cap}=\bigcap_{A\in\mathcal{A}} \textnormal{ker}(A)$. 
If $K_{\cap} \neq \emptyset$, the union of all sets, $A_{\cup}$, is starshaped and $K_{\cap} \subset \textnormal{ker}\left(A_{\cup}\right)$.
\end{proposition}
\begin{proof}
See Appendix \ref{a:proof_prop_intersecting_kernel}.
\end{proof}

\begin{proposition}
\label{prop:strictly_star}
Let $A\subset \mathbb{R}^n$ be a starshaped set. $A$ is strictly starshaped if the kernel of $A$ has a nonempty interior. That is 
\begin{equation}
    \textnormal{int ker} (A) \neq \emptyset \Rightarrow A \textnormal{ is strictly starshaped}
\end{equation}
\end{proposition}
\begin{proof}
See Appendix \ref{a:proof_l_boundary_mapping}.
\end{proof}

The implication of Proposition \ref{prop:intersecting_kernel} is illustrated in Fig. \ref{fig:hull_kernel_intersecting} where three starshaped sets are intersecting. Note that the condition $K_{\cap} \neq \emptyset$ in Proposition \ref{prop:intersecting_kernel} is sufficient but not necessary for $A_{\cup}$ to be starshaped. Consider for example the starshaped polygon $A_1$ and ellipse $A_2$ in Fig. \ref{fig:hull_kernel_non_intersecting}. The kernels of the starshaped sets are disjoint, i.e. $K_{\cap} = \emptyset$, but the union $A_{\cup}$ is starshaped with $\textnormal{ker}(A_{\cup}) = \textnormal{ker}(A_1)$.

\begin{remark}
From Proposition \ref{prop:intersecting_kernel} we can conclude that the union of any combination of convex sets is starshaped if their intersection is nonempty since $\textnormal{ker}(A)=A$ for any convex set $A$.
\end{remark}

\subsection{Excluding points from the starshaped hull}
\label{sec:admissible_kernel}
In some scenarios it is desired to find a starshaped set containing a set, $A$, while ensuring that some points of interest, $\bar{X}$, are not included in the resulting starshaped set. For instance, in Problem \ref{problem}, we need all obstacles to be starshaped but the robot and goal positions should remain outside the extended starshaped obstacles. The starshaped hull of $A$ provides a starshaped enclosing of $A$, but it does not inherently provide any way to exclude specific points. The shape of the set depends on the point selected for generating the hull and as a consequence, we have that the starshaped hull w.r.t. some points is disjoint from $\bar{X}$, while it is not w.r.t. other points. To enable the exclusion of $\bar{X}$ from the starshaped hull, we introduce the admissible kernel defined as follows:
\begin{definition}
\label{d:admissible_kernel}
Let $A \subset \mathbb{R}^n$ and $\Bar{X} \subset \mathbb{R}^n$. The \textit{admissible kernel} for $A$ excluding $\Bar{X}$, denoted as $\textnormal{ad ker}(A,\bar{X})$, is the set such that the starshaped hull of $A$ at any $x\in \textnormal{ad ker}(A,\bar{X})$ does not contain any point in $\Bar{X}$. That is,
\begin{align}
    \textnormal{ad ker}(A,\bar{X}) = \{x\in \mathbb{R}^n : SH_x(A) \cap \Bar{X} = \emptyset\}.
\end{align}
\end{definition}
Given the admissible kernel for the sets $A$ and $\bar{X}$, any point $x\in \textnormal{ad ker}(A,\bar{X})$ can be used for the starshaped hull to generate a starshaped set which contains $A$ and excludes all points in $\Bar{X}$. For computing the admissible kernel, the following two properties are useful:

\begin{property}
\label{p:adm_kernel_union}
Let $\mathcal{A}$ be a collection of sets $A\subset \mathbb{R}^n,\ \forall A\in \mathcal{A}$ with union $A_{\cup}=\bigcup_{A\in \mathcal{A}}A$. The admissible kernel for $A_{\cup}$ excluding a point set $\bar{X}$ is the intersection of the admissible kernel excluding $\bar{X}$ for all subsets $A\in\mathcal{A}$. That is,
$$\textnormal{ad ker}\left(A_{\cup}, \bar{X}\right) = \bigcap_{A\in\mathcal{A}} \textnormal{ad ker}(A,\bar{X})$$
\end{property}
\begin{proof}
See Appendix \ref{a:proof_p_adm_ker_union}.
\end{proof}

\begin{property}
\label{p:adm_ker_intersection}
The admissible kernel for the starshaped hull of a set $A$ excluding a point set $\bar{X}$ is the intersection of all admissible kernels given each individual point $\bar{x}\in\bar{X}$. That is,
\begin{align*}
    \textnormal{ad ker}(A,\bar{X}) = \bigcap_{\bar{x}\in\bar{X}} \textnormal{ad ker}(A,\{\bar{x}\})
\end{align*}
\end{property}
\begin{proof}
See Appendix \ref{a:proof_p_adm_ker_intersection}.
\end{proof}

As a consequence of Property \ref{p:adm_kernel_union}, the admissible kernel for a set given as the union of several "simpler" subsets, i.e. for intersecting obstacles as in Problem \ref{problem}, can be found by intersecting the admissible kernel for each subset. In Fig. \ref{fig:admissible_kernel_intersecting} an example is shown of how it can be used to find the admissible kernel for the union of two sets.
From Property \ref{p:adm_ker_intersection} we have that when more than one point should be excluded from the starshaped hull, the admissible kernel can be computed separately for all points and then be combined through intersection. In Fig. \ref{fig:admissible_kernel_two_points} an example is shown of how it can be used to find the admissible kernel excluding two points.

\begin{figure}
    \centering
    \begin{subfigure}[t]{0.49\linewidth}
        \centering
        \includegraphics[width=\linewidth]{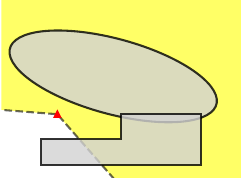}
         \caption{Admissible kernel for the ellipse, $\textnormal{ad ker}(A_1,\{\bar{x}\})$, shown in yellow.}
         \label{fig:admissible_kernel_intersecting1}
    \end{subfigure}
    \hfill
    \begin{subfigure}[t]{0.49\linewidth}
        \centering
        \includegraphics[width=\linewidth]{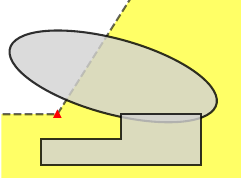}
     \caption{Admissible kernel for the polygon, $\textnormal{ad ker}(A_2,\{\bar{x}\})$, shown in yellow.}
     \label{fig:admissible_kernel_intersecting2}
    \end{subfigure}
    \begin{subfigure}[t]{0.8\linewidth}
        \centering
        \includegraphics[width=\linewidth]{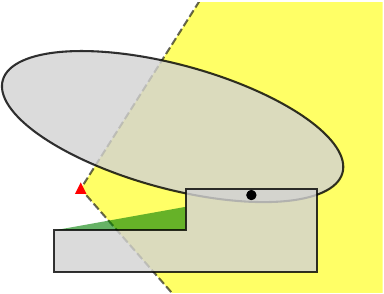}
    \caption{The admissible kernel for the set, $\textnormal{ad ker}(A,\{\bar{x}\})$, shown in yellow, is found as the intersection of $\textnormal{ad ker}(A_1,\{\bar{x}\})$ and $\textnormal{ad ker}(A_2,\{\bar{x}\})$. The hull extended region, $SH_x(A) \setminus A$, is shown in green for a point $x\in \textnormal{ad ker}(A,\bar{X})$, shown as black dot. Clearly, $\bar{x} \not\in SH_x(A)$ as desired.}
    \label{fig:admissible_kernel_intersecting3}
    \end{subfigure}
    \caption{A set given as the union, $A=A_1\cup A_2$, of an ellipse, $A_1$, and a polygon, $A_2$. The set is shown in gray and a point to exclude, $\bar{x}$, is shown as a red triangle.}
    \label{fig:admissible_kernel_intersecting}
\end{figure}

\begin{figure}
    \centering
    \begin{subfigure}[t]{0.49\linewidth}
        \centering
        \includegraphics[width=\linewidth]{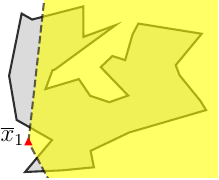}
         \caption{Admissible kernel for the polygon excluding point $x_1$, $K_{\bar{x}_1}(A)$, shown in yellow.}
         \label{fig:admissible_kernel_two_points1}
    \end{subfigure}
    \hfill
    \begin{subfigure}[t]{0.49\linewidth}
        \centering
        \includegraphics[width=\linewidth]{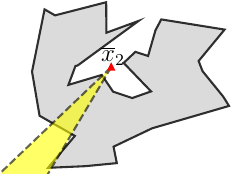}
     \caption{Admissible kernel for the polygon excluding point $x_2$, $K_{\bar{x}_2}(A)$, shown in yellow.}
     \label{fig:admissible_kernel_two_points2}
    \end{subfigure}
    \begin{subfigure}[t]{0.8\linewidth}
        \centering
        \includegraphics[width=\linewidth]{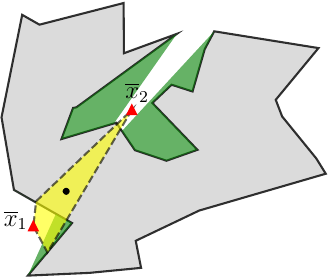}
    \caption{The admissible kernel, $\textnormal{ad ker}(A,\bar{X})$, in yellow, can be found by intersection of the two cones $K_{\bar{x}_1}(A)$ and $K_{\bar{x}_2}(A)$. The hull extended region, $SH_x(A) \setminus A$, is shown in green for a point $x\in \textnormal{ad ker}(A,\bar{X})$, shown as black dot. Clearly, $\bar{x}_1 \not\in SH_x(A)$ and $\bar{x}_2 \not\in SH_x(A)$ as desired.}
    \label{fig:admissible_kernel_two_points3}
    \end{subfigure}
    \caption{Non-starshaped polygon, $A$, in gray, and two points to exclude $\Bar{X}=\{\Bar{x}_1, \Bar{x}_2\}$, shown as a red triangles.}
    \label{fig:admissible_kernel_two_points}
\end{figure}

Combining Property \ref{p:adm_kernel_union} and \ref{p:adm_ker_intersection}, we have
\begin{equation}
\label{eq:adm_ker_combination}
    \textnormal{ad ker}(A_{\cup},\bar{X}) = \bigcap_{A\in\mathcal{A}} \bigcap_{\bar{x}\in\bar{X}} \textnormal{ad ker}(A,\{\bar{x}\}).
\end{equation}
Eq. \ref{eq:adm_ker_combination} is instrumental for simplifying the problem of finding the admissible kernel for a complex set excluding several points by decomposing it into subproblems of finding the admissible kernel for a simple set excluding a single point.

A point $\bar{x}\in\bar{X}$ can be classified into three distinct types w.r.t. $A$. It may be a point in the set, a bounded exterior point or a free exterior point. The difference between the two latter, illustrated as $\bar{z}$ and $\bar{x}$ in Fig. \ref{fig:admissible_kernel_polygon}, is that there exists a ray in some direction emanating from a free exterior point which does not intersect $A$, while a bounded exterior point is fully surrounded by $A$. Obviously, any exterior point to a convex set is a free exterior point.
\begin{property}
\label{p:ad_ker_nonempty}
The admissible kernel for the starshaped hull of $A$ excluding the singleton set $\{\bar{x}\}$ is nonempty if and only if $\bar{x}$ is a free exterior point of $A$.
\end{property}
\begin{proof}
See appendix \ref{a:proof_p_ad_ker_nonempty}.
\end{proof}

The admissible kernel for any 2-dimensional set, $A\subset \mathbb{R}^2$, given any free exterior point, $\bar{x}$, is found as the cone
\begin{equation}
\label{eq:adm_ker_cone}
    \textnormal{ad ker}(A,\{\bar{x}\}) = \textnormal{int} C_{\angle}\left(r(\Bar{x},\overrightarrow{t_1(\Bar{x})\Bar{x}}),r(\Bar{x},\overrightarrow{t_2(\Bar{x})\Bar{x}})\right)
\end{equation}
where $t_1(\bar{x}), t_2(\bar{x}) \in \mathcal{T}_A(\bar{x})$.
For a convex set specified by the boundary function $b(x)=0 \textnormal{ iff } x\in\partial A$, the tangent points can be found by solving for $t$ such that $\frac{db}{dx}(t) \cdot (t-\bar{x}) = 0$ and $b(t) = 0$ and the tangent points are set such that the order $xt_1t_2$ is CW. For a polygon, the tangent points can be found as proposed in \cite{freeman_loutrel_67} by expressing the vertices in polar coordinates w.r.t. $\bar{x}$, and $t_1$ and $t_2$ are chosen as the vertices with maximum and minimum polar angles, respectively. 
Note that the admissible kernel is given by the interior, and does not include the boundary, of the generated cone in \eqref{eq:adm_ker_cone}. Consider for example Figs. \ref{fig:admissible_kernel_ellipse} and \ref{fig:admissible_kernel_polygon} where the admissible kernel for two sets $A\subset \mathbb{R}^2$ are shown. Since the starshaped hull w.r.t. any boundary point, $x \in \partial \textnormal{ad ker}(A,\{\bar{x}\})$, would contain the line segments $l[x,t_1]$ and $l[x,t_2]$, it would also contain $\bar{x}$ as this is part of both lines.

The admissible kernel for any convex set, $A_{\textrm{conv}} \subset \mathbb{R}^n,\ n\geq 2$, is found as
\begin{equation}
\label{eq:adm_kernel_convex}
    \textnormal{ad ker}A_{\textrm{conv}}(\Bar{x}) = \mathbb{R}^n \setminus \left\{r(\Bar{x},\overrightarrow{y\Bar{x}}) : y \in CH(\mathcal{T}_{A_\textnormal{conv}}(x)) \right\}
\end{equation}
assuming $\bar{x}$ is an exterior point of $A_{\textnormal{conv}}$. Since only rays emanating from $\bar{x}$ in directions from points in $A_{\textnormal{conv}}$ are excluded in \eqref{eq:adm_kernel_convex}, the admissible kernel for a convex set fully contains the set. In Fig. \ref{fig:admissible_kernel_ellipsoid} an example of the admissible kernel excluding a single point for an ellipsoid is illustrated.

\begin{figure}
    \centering
    \begin{subfigure}[t]{0.49\linewidth}
        \centering
        \includegraphics[width=\linewidth]{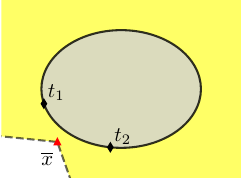}
         \caption{$A\in \mathbb{R}^2$ is an ellipse. $\textnormal{ad ker}(A,\{\bar{x}\})$ is the cone shown in yellow. The tangent points, $t_1$ and $t_2$, of $A$ through $\Bar{x}$ are also depicted.}
         \label{fig:admissible_kernel_ellipse}
    \end{subfigure}
    \hfill
    \begin{subfigure}[t]{0.49\linewidth}
        \centering
        \includegraphics[width=\linewidth]{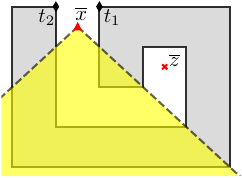}
     \caption{$A\in \mathbb{R}^2$ is a polygon. The admissible kernel given the free exterior point $\bar{x}$, $\textnormal{ad ker}(A,\{\bar{x}\})$, is shown in yellow. The admissible kernel for the bounded exterior point, $\bar{z}$, is $\textnormal{ad ker}(A,\{\bar{z}\})=\emptyset$.}
     \label{fig:admissible_kernel_polygon}
    \end{subfigure}
    \begin{subfigure}[t]{\linewidth}
     \centering
     \includegraphics[width=0.5\linewidth]{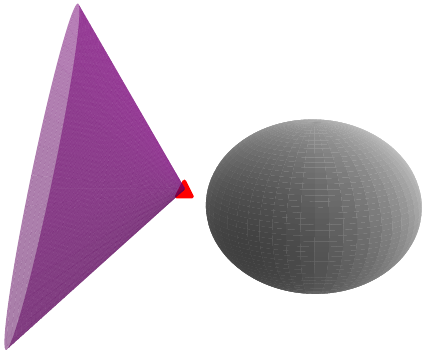}
     \caption{$A\in \mathbb{R}^3$ is an ellipsoid. The cone $\mathbb{R}^3 \setminus \textnormal{ad ker}(A,\bar{X})$ is shown in purple.}
     \label{fig:admissible_kernel_ellipsoid}
    \end{subfigure}
    \caption{Admissible kernels, $\textnormal{ad ker}(A,\{\bar{x}\})$, for some sets $A$, shown in gray, given the excluding point $\Bar{x}$, shown as a red triangle.}
    \label{fig:admissible_kernel_single}
\end{figure}

\subsection{Starshaped hull for strictly starshaped sets}
Depending on the set $A$ and the point $x$ used for generating $SH_x(A)$, the kernel of the starshaped hull can be the singleton $\textnormal{ker}(SH_x(A)) = \{x\}$. Furthermore, from the manner that the starshaped hull is constructed there may exist more than one boundary point along some direction from all (or the only) kernel point(s). In other words, it is not strictly starshaped.
In Fig. \ref{fig:singleton_kernel} an example of this is illustrated where the starshaped hull is generated w.r.t. a point $x$, depicted as a black circle, and several boundary points of the hull are located along the four directions shown as red dashed lines from the singleton kernel $x$.
For this reason, the starshaped hull in its original definition is not appropriate to apply when strictly starshapes are needed. To treat this issue, we introduce the starshaped hull with specified kernel. 
\begin{definition}
\label{d:st_hull_kernel}
Let $A \subset \mathbb{R}^n$ and $K \subset \mathbb{R}^n$. The \textit{starshaped hull of $A$ with specified kernel $K$}, denoted as $SH_{\textnormal{ker}K}(A)$, is defined as the smallest starshaped set such that $A\subset SH_{\textnormal{ker}K}(A)$ and $K \subset \textnormal{ker}\left(SH_{\textnormal{ker}K}(A)\right)$.
\end{definition}

\begin{figure}
    \centering
    \begin{subfigure}[t]{0.49\linewidth}
        \begin{center}  
            \input{figures/singleton_kernel.pgf}
        \end{center}
         \caption{There exists several boundary points of $SH_x(A)$ along the red dashed lines in four directions from $x$.}
         \label{fig:singleton_kernel}
    \end{subfigure}
    \hfill
    \begin{subfigure}[t]{0.49\linewidth}
        \begin{center}  
            \input{figures/triangle_kernel.pgf}
        \end{center}
     \caption{There exists only one boundary point of $SH_{\textnormal{ker}K}(A)$ along each direction from $x\in \textnormal{int ker}SH_{\textnormal{ker}K}(A)$ when $K$ contains three affinely independent points. $CH(K)$ is shown in blue.}
     \label{fig:triangle_kernel}
    \end{subfigure}
    \caption{A non-starshaped polygon, shown in gray, with hull extended region, shown in green.}
    \label{fig:polygon_desired_kernel}
\end{figure}
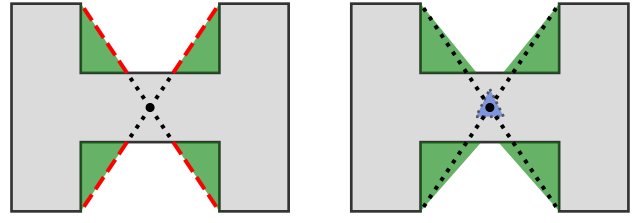

Property \ref{p:st_hull} for the starshaped hull w.r.t. a single point cannot be directly applied for the starshaped hull with specified kernel. Instead, we have the following properties.
\begin{property}
\label{p:st_hull_ker}
Let $A\subset \mathbb{R}^n$ and let $\mathcal{B}$ be a collection of sets $B\subset \mathbb{R}^n,\ \forall B\in \mathcal{B}$.
\begin{enumerate}[label=\alph*.]
    \item $SH_{\textnormal{ker}K}(A) = SH_{\textnormal{ker}CH(K)}(A)  = \displaystyle\bigcup_{k \in CH(K)}SH_k(A)$ \label{p:st_hull_ker_ch_ker}
    \item $SH_{\textnormal{ker}K}(A) = A \Leftrightarrow A$ is starshaped and $K \subset \textnormal{ker}(A)$ \label{p:st_hull_ker_A}
    \item $SH_{\textnormal{ker}K}(A) \subset CH(A),\ \forall K \subset CH(A)$ \label{p:st_hull_ker_ch_bound}
    \item $\displaystyle SH_{\textnormal{ker}K}\left(\bigcup_{B\in\mathcal{B}} B\right) = \bigcup_{B\in\mathcal{B}} SH_{\textnormal{ker}K}(B)$ \label{p:st_hull_ker_union}
\end{enumerate}
\end{property}
\begin{proof}
See Appendix \ref{a:proof_p_star_hull_ker}.
\end{proof}
While it is sufficient to have $K\subset \textnormal{ker}\left(SH_{\textnormal{ker}K}(A)\right)$ by Definition \ref{d:st_hull_kernel}, Property \ref{p:st_hull_ker}\ref{p:st_hull_ker_ch_ker} states that $CH(K)$ is also contained in $\textnormal{ker}\left(SH_{\textnormal{ker}K}(A)\right)$ in all cases, since $SH_{\textnormal{ker}K}(A) = SH_{\textnormal{ker}CH(K)}(A)$. This will prove to be instrumental for generating sets which are guaranteed to be strictly starshaped. Additionally, Property \ref{p:st_hull_ker}\ref{p:st_hull_ker_ch_ker} provide a direct relation between the starshaped hull with specified kernel and the starshaped hull w.r.t. a point.
Properties \ref{p:st_hull_ker}\ref{p:st_hull_ker_A}-\ref{p:st_hull_ker}\ref{p:st_hull_ker_union} directly relate to Properties \ref{p:st_hull}\ref{p:hull_star_kernel}-\ref{p:st_hull}\ref{p:hull_union}
Property \ref{p:st_hull_ker}\ref{p:st_hull_ker_A} provides a guarantee that a starshaped set, $A$, is not expanded by the operation $SH_{\textnormal{ker}K}(A)$ if the specified kernel points are selected within the kernel of $A$ and Property \ref{p:st_hull_ker}\ref{p:st_hull_ker_ch_bound} provide an upper bound if the specified kernel points are selected within the convex hull of $A$. Property \ref{p:st_hull_ker}\ref{p:st_hull_ker_union} simplifies finding the starshaped hull with specified kernel for complex regions which can be described as combinations of simpler subsets, e.g. as the union of several polygons and/or convex sets, since the hull can be computed separately for each subset.

Using Proposition \ref{prop:strictly_star} in combination with Property \ref{p:st_hull_ker}\ref{p:st_hull_ker_ch_ker}, a sufficient condition on the specified kernel can be derived for $SH_{\textnormal{ker}K}(A)$ to be strictly starshaped as stated in the following property.

\begin{property}
\label{p:st_hull_ker_strict}
Let $A\subset \mathbb{R}^n$ and let $K\subset \mathbb{R}^n$. The starshaped hull of $A$ with specified kernel $K$, $SH_{\textnormal{ker}K}(A)$, is strictly starshaped if $K$ contains $n+1$ affinely independent points. 
\end{property}
\begin{proof}
See Appendix \ref{a:proof_p_st_hull_ker_strict}.
\end{proof}

From Property \ref{p:st_hull_ker_strict} we have that any set $A\in \mathbb{R}^n$, can be enclosed by a strictly starshaped set with $SH_{\textnormal{ker}K}(A)$ given that $K$ is chosen as $n+1$ affinely independent points. Specifically, in $\mathbb{R}^2$, it is sufficient to select $K$ as three points which are not collinear. 

As stated in the previous section, the admissible kernel provides a useful instrument when choosing the kernel point $x$ for generating the starshaped hull, $SH_x(A)$, such that some specified points, $\bar{X}$, are excluded. However, the admissible kernel does not provide such a guarantee when the starshaped hull with specified kernel, $SH_{\textnormal{ker}K}(A)$, is used. This is evident from Fig. \ref{fig:desired_kernel_ellipse3} where $\bar{x}$ is contained by $SH_{\textnormal{ker}K}(A)$ even though the kernel points are selected withing the admissible kernel, $K\subset \textnormal{ad ker}(A,\{\bar{x}\})$. To extend the applicability of the admissible kernel to the starshaped hull with specified kernel, consider the following property.

\begin{property}
\label{p:st_hull_ker_adm_ker}
Let $A\subset \mathbb{R}^n$, $K\subset \mathbb{R}^n$, $\bar{X} \subset \mathbb{R}^n$ and $\textnormal{ad ker}(A,\bar{X})$ be the admissible kernel for $A$ excluding $\bar{X}$. If $CH(K)$ is contained by $\textnormal{ad ker}(A,\bar{X})$, no point $\bar{x}\in\bar{X}$ is included in the starshaped hull of $A$ with specified kernel $K$. That is,
\begin{equation}
    CH(K) \subset \textnormal{ad ker}(A,\bar{X}) \Rightarrow SH_{\textnormal{ker}K}(A) \cap \bar{X} = \emptyset.
\end{equation}
\end{property}
\begin{proof}
See Appendix \ref{proof:p_st_hull_ker_adm_ker}.
\end{proof}
From Properties \ref{p:st_hull_ker_strict} and \ref{p:st_hull_ker_adm_ker} we can now conclude that given the admissible kernel for $A\subset \mathbb{R}^n$ excluding $\bar{X}\subset \mathbb{R}^n$, $SH_{\textnormal{ker}K}(A)$ is guaranteed to be a strictly starshaped set which does not contain any $\bar{x}\in\bar{X}$ if $K$ is chosen as $n+1$ affinely independent points such that $CH(K)\subset \textnormal{ad ker}(A,\bar{X})$.

Before deriving the expressions for the starshaped hull with specified kernel, note that the naive approach to separately generate the starshaped hull w.r.t. each $k\in K$ and combine them does not provide the desired result, i.e. $SH_{\textnormal{ker}K}(A) \neq \bigcup_{k\in K} SH_k(A)$ in general. This is evident from Fig. \ref{fig:desired_kernel_ellipse2} where, given an ellipse, $A$, and a set of two points, $K=\{k_1,k_2\}$, the union $B=SH_{k_1}(A)\cup SH_{k_2}(A)$ is shown. Since $l[k_1,k_2] \not\in B$ neither $k_1$ nor $k_2$ belongs to the kernel of $B$. 
Instead, the starshaped hull should be generated  w.r.t. each $k\in CH(K)$ according to Property \ref{p:st_hull_ker}\ref{p:st_hull_ker_ch_ker} However, as $CH(K)$ is a continuous set for any non-singleton set $K$, it becomes intractable in practice to generate the starshaped hull w.r.t. each $k\in CH(K)$ and an explicit method to find $SH_{\textnormal{ker}K}(A)$ is needed.
For a convex set, $A_{\textrm{conv}}$, the extended region of the starshaped hull with specified kernel $K$ is based on the convex hull of the kernel points and their corresponding points of tangency. That is,
\begin{equation}
\label{eq:star_hull_specified_convex}
    SH_{\textnormal{ker}K}(A_{\textrm{conv}}) = A_{\textrm{conv}} \cup CH\left(\bigcup_{k\in K} \mathcal{T}_k(A_{\textrm{conv}}) \cup k \right).
\end{equation} 
Naturally, this is equivalent to \eqref{eq:local_hull_convex} for a single specified kernel point. In Fig. \ref{fig:desired_kernel_ellipse1} the case with an ellipse and two specified kernel points is again considered and $SH_{\textnormal{ker}K}(A)$ is depicted. As the resulting shape is convex, it follows that $K \subset \textnormal{ker}\left(SH_{\textnormal{ker}K}(A)\right)$.

\begin{figure}
    \centering
    \begin{subfigure}[t]{0.49\linewidth}
        \begin{center}  
            \input{figures/desired_kernel_ellipse2.pgf}
        \end{center}
    \caption{Union of the two starshaped hulls generated at the kernel points separately, i.e. $\bigcup_{k\in K}SH_k(A)$.}
    \label{fig:desired_kernel_ellipse2}
    \end{subfigure}
    \hfill
    \begin{subfigure}[t]{0.49\linewidth}
        \begin{center}  
            \input{figures/desired_kernel_ellipse1.pgf}
        \end{center}
    \caption{The starshaped hull of $A$ with specified kernel $K$, $SH_{\textnormal{ker}K}(A)$.}
    \label{fig:desired_kernel_ellipse1}
    \end{subfigure}
    \begin{subfigure}[t]{0.49\linewidth}
        \centering
        \includegraphics[width=\linewidth]{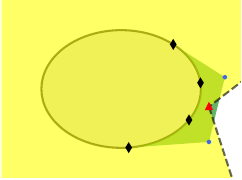}
    \caption{The excluding point $\bar{x}$, shown as red triangle, is contained by $SH_{\textnormal{ker}K}(A)$ even though $K$ is selected within the admissible kernel, $\textnormal{ad ker}(A,\{\bar{x}\})$, shown in yellow.}
    \label{fig:desired_kernel_ellipse3}
    \end{subfigure}
    \caption{An ellipse, $A$, shown in gray, with hull extended region shown in green given two kernel points, $K$, shown as blue dots.}
    \label{fig:desired_kernel_ellipse}
\end{figure}

Inspired by the approach in \cite{arkin_etal_98}, an algorithm to find the starshaped hull with a specified kernel for a polygon has been developed and is given in Algorithm \ref{alg:kernel_hull_polygon}. In Fig. \ref{fig:triangle_kernel} the starshaped hull with three affinely independent points, i.e. a triangle, is depicted for a polygon, $P$. 
In contrast to the starshaped hull w.r.t. $x$, shown in Fig. \ref{fig:singleton_kernel}, $SH_{\textnormal{ker}K}(P)$ is strictly starshaped in accordance with Property \ref{p:st_hull_ker_strict} and therefore only one boundary point exists in each direction from any interior point of $CH(K)$.

\begin{remark}
For a single specified kernel point, i.e. $K$ is a singleton set, Algorithm \ref{alg:kernel_hull_polygon} simplifies to the algorithm for generating the starshaped hull of a polygon w.r.t. a point presented in \cite{arkin_etal_98} with the instrumental distinction that all vertices, and not only convex vertices, are considered in the iteration.
\end{remark}

\begin{algorithm}
\caption{Finding starshaped hull with desired kernel for a polygon}
\label{alg:kernel_hull_polygon}
\SetKwInOut{Input}{Input}
\SetKwInOut{Output}{Output}
\SetKwProg{Init}{init}{}{}
\Input{A polygon $P$ and a finite point set $K$}
\Output{The minimum starshaped polygon $SH_{\textnormal{ker}K}(P)$ s.t. $P\subset SH_{\textnormal{ker}K}(P)$ and $K\subset \textnormal{ker}\left(SH_{\textnormal{ker}K}(P)\right)$}
Initialize $P^{\star}$ as empty list\;
$k^*, \Bar{v}, e_1, e_2 \gets$ centroid of $K$\; 
Order $P=v_1,v_2,...,v_N$ s.t. $v_1$ is the vertex with largest $x$-value and the order $v_Nv_1v_2$ is CCW\;
\ForEach{vertex $v \in P$}{%
  \If{$r(v, \overrightarrow{kv})$ does not intersect interior of $P,\ \forall k\in K$}{%
    Append $v$ to $P^{\star}$\;
    $v' \gets$ vertex preceding $v$ in $P$\;
    \eIf{$\exists k \in K \textnormal{ s.t. } l[k,v]$ intersects $l[e_1,e_2]$}{%
        $e_1 \gets$ closest intersection to $e_2$ of $l[k,v]$ and $l[e_1,e_2],\ \forall k\in K$\;
    }{%
        \ForEach{$k \in K$}{%
            \eIf{$l[k,v]$ intersects interior of $P$}{%
                $u \gets$ last intersection of $l[k,v]$ and $P$\;
                \If{$r(u,\overrightarrow{k'v})$ does not intersect interior of $P,\ \forall k' \in K, k'\neq k$}{%
                    \If{$\exists k' \in K, k'\neq k \textnormal{ s.t. } l[k',\Bar{v}]$ intersects $l[u,v]$}{%
                        $u \gets$ intersection of $l[k',\Bar{v}]$ and $l[u,v]$;
                    }
                    Append $u$ to $P^{\star}$\;
                    $e_1 \gets u$\;
                    $e_2 \gets v$\;
                    \If{$uvv'$ is CCW}{%
                        Swap last two elements of $P^{\star}$\;
                    }
                }
            }{%
                \If{$r(k,\overrightarrow{k'v})$ does not intersect interior of $P,\ \forall k' \in K, k'\neq k$}{%
                    Append $k$ to $P^{\star}$\;
                    $e_1 \gets k$\;
                    $e_2 \gets v$\;
                    \If{$kvv'$ is CCW}{%
                        Swap last two elements of $P^{\star}$\;
                    }
                }
            }
        }
    }
  $\Bar{v} \gets $ last element of $P^{\star}$\;
  }
}
\ForEach{consecutive vertices $v, v' \in P^{\star}$}{%
    \If{$\exists k \in K \textnormal{ s.t. } kvv'$ is CW}{%
        Insert $k$ in $P^{\star}$ between $v$ and $v'$\;
    }
}
\KwRet{$P^{\star}$}\;
\end{algorithm}

\subsection{Global starshaped hull (*)}
For the sake of completeness, we here introduce the concept of global starshaped hull. Readers that are interested in the application of the previously introduced concepts can skip this section as the global starshaped hull is not used in the proposed algorithm.
The starshaped hull of $A$ w.r.t. $x$ as defined in Definition \ref{d:st_hull} does not only consider the set to enclose but also depends on w.r.t. which point the hull is generated. Thus, in contrast to the convex hull, it is not unique for a set. Moreover, it is in general not the smallest starshaped set which contains $A$. The global starshaped hull, defined in the following, is in these aspects more closely related to the convex hull.
\begin{definition}
Let $A\subset \mathbb{R}^n$. The \textit{global starshaped hull of $A$}, denoted $SH(A)$, is defined as the smallest starshaped set such that $A\subset SH(A)$.
\label{d:st_hull_global}
\end{definition}
Obviously, there is a close relation to the starshaped hull as defined in Definition \ref{d:st_hull}, and equivalently to Definition \ref{d:st_hull_global} the global starshaped hull can be given as 
\begin{equation}
\label{eq:st_hull_min}
    SH(A) = \min_{x\in \mathbb{R}^n} \lambda(SH_x(A)).
\end{equation}
where $\lambda(\cdot)$ denotes the Lesbegue measure.
The definition for the global starshaped hull coincides with the minimum-area starshaped hull treated in \cite{arkin_etal_98} for the special case of 2-dimensional sets. 

When some points should be excluded from the hull, the admissible kernel can be used to restrict the search space for the optimization problem in \eqref{eq:st_hull_min}. That is, the global starshaped hull of $A$ excluding $\Bar{X}$, denoted $SH_{\neg\bar{X}}(A)$, can be defined as the smallest starshaped set which contains $A$ such that $SH_{\neg \Bar{X}}(A) \cap \Bar{X} = \emptyset$ and we have
\begin{equation}
    SH_{\neg \Bar{X}}(A) = \displaystyle\min_{x\in \textnormal{ad ker}(A,\bar{X})} \lambda(SH_x(A)).
\end{equation}
With a slight modification of the method to find the global starshaped hull for a polygon, $P$, presented in \cite{arkin_etal_98}, it may be applied to find $SH_{\neg\Bar{X}}(P)$. In particular, only cells from the cell decomposition that lie inside the admissible kernel should be considered in the minimization step.

\section{Forming disjoint star worlds}
\label{sec:forming_star_obstacles} 
Since any convex set with nonempty interior is strictly starshaped, we have from the formulation of Problem \ref{problem} that $\mathcal{F} = \mathcal{W} \setminus \bigcup_{\mathcal{O}_i\in\mathcal{O}}\mathcal{O}_i$ is a star world satisfying \eqref{eq:obstacle_inclusion}-\eqref{eq:obstacle_goal_exclusion} if each polygon obstacle is strictly starshaped. 
A simple solution in case no disjoint star world exists satisfying all conditions is hence to construct $\mathcal{O}^{\star}$ as
\begin{equation}
\label{eq:convex_decomposition}
    \mathcal{O}^{\star} = \{CD(\mathcal{O}_i) : \mathcal{O}_i \in \mathcal{O} \}
\end{equation}
where $CD(\cdot)$ is a convex decomposition of the considered set. For any convex obstacle we trivially have $CD(\mathcal{O}_i)=\mathcal{O}_i$ while for any concave polygon a convex decomposition can for instance be found using Hertel Mehlhorn algorithm \cite{hertel_mehlhorn_83}. 
In an attempt to find a disjoint star world solving Problem \ref{problem} we present Algorithm \ref{alg:starify_obstacles} which is discussed in the following section.

\subsection{Algorithm}
The fundamental idea of Algorithm \ref{alg:starify_obstacles} is to create clusters of obstacles by combining intersecting obstacles followed by a generation of starshaped obstacles which fully contain each cluster in an iterative manner. In Algorithm \ref{alg:starify_obstacles} the notation $Cl$ is used for the set of all clusters, $cl\subset\mathcal{O}$, containing one or several obstacles.
Each iteration in the algorithm can be divided into three main steps: computation of admissible kernels for each cluster, generation of starshaped obstacles containing each cluster, and clustering of intersecting starshaped obstacles. The number of clusters does in this way never increase and the algorithm converges whenever the number of clusters remains the same after an iteration. For the initial iteration, each single obstacle is considered as a separate cluster. 

The steps are illustrated in Fig. \ref{fig:starify_steps}. First, the admissible kernel is found for each cluster given the excluding points defined as the robot and goal position. \added{For space restriction, this is only} shown for one ellipse and the polygon in Fig. \ref{fig:step1_admker_ell} and \ref{fig:step1_admker_pol}\added{, respectively}. Next, in Fig. \ref{fig:step1_star}, new strictly starshaped obstacles are generated using the starshaped hull with three ($n=2$) specified kernel points. Since all starshaped obstacles intersect in Fig. \ref{fig:step1_star}, the obstacles are combined into one cluster in Fig. \ref{fig:step1_cluster}. As the number of clusters has been reduced, the process is iterated once again as illustrated in Figs. \ref{fig:step2_admker}-\ref{fig:step2_star}. Obviously, no change is made in the clustering stage when only one cluster is considered and the algorithm terminates.

\LinesNumbered
\let\oldnl\nl
\newcommand{\nonl}{\renewcommand{\nl}{\let\nl\oldnl}}

\begin{algorithm}
\SetKwRepeat{Do}{do}{while}
\caption{Forming disjoint star worlds}
\label{alg:starify_obstacles}
\SetKwInOut{Input}{Input}
\SetKwInOut{Output}{Output}
\SetKwProg{Init}{init}{}{}
\Input{A set of obstacles, $\mathcal{O}$, as in Problem \ref{problem}, the robot position, $x$, and goal position, $x_g$.}
\Output{A set of (disjoint) strictly starshaped obstacles, $\mathcal{O}^{\star}$.}
$Cl \gets \mathcal{O}$\;
\Do{$N_{Cl} \neq$ $|Cl|$}{
    $N_{Cl} \gets$ $|Cl|$\;
    $\mathcal{O}^{\star} \gets \emptyset$\;
    \ForEach{$cl \in Cl$}{
        \If{$\textnormal{ad ker}(cl, \{x, x_g\}) = \emptyset$ \tikzmark{kernel_top}}{ \label{l:empty_kernel}
                \KwRet{$\mathcal{O}^{\star}$ \textnormal{as in} \eqref{eq:convex_decomposition}}\label{l:return_intersect_stars}\; \label{l:premature_return}
        }
        \tikzmark{star_top}$K \gets n+1$ affinely independent\nonl\\\ points s.t. $CH(K) \subset \textnormal{ad ker}(cl, \{x, x_g\})$\; \label{l:kernel_points}
        Add $SH_{\textnormal{ker}K}(cl)$ to $\mathcal{O}^{\star}$\tikzmark{star_bottom}\; \label{l:star_hull}
    }
    $Cl^{\star} \gets$ clusters of $\mathcal{O}^{\star}$ s.t. no region in \ \ \tikzmark{right} \nonl\\\ 
    one cluster intersects with a region in \nonl\\\ 
    another\; \label{l:cluster_star}
    $Cl \gets \emptyset$\;
    \ForEach{$cl^{\star} \in Cl^{\star}$}{
        Add $cl=\{\mathcal{O}_i \in \mathcal{O} : \mathcal{O}_i \subset cl^{\star}\}$ to $Cl$ \tikzmark{cluster_bottom}\; \label{l:cluster_obs}
    }
}
\KwRet{$\mathcal{O}^{\star}$}\;
\end{algorithm}

\begin{figure}
    \begin{minipage}{0.49\linewidth}
        \begin{subfigure}[t]{\linewidth}
            \input{figures/step1_admker0.pgf}
             \caption{Admissible kernel for the leftmost ellipse.}
             \label{fig:step1_admker_ell}
        \end{subfigure}
        \hfill
        \begin{subfigure}[t]{\linewidth}
            \input{figures/step1_admker3.pgf}
             \caption{Admissible kernel for the polygon.}
             \label{fig:step1_admker_pol}
        \end{subfigure}
        \hfill
        \begin{subfigure}[t]{\linewidth}
            \input{figures/step1_star.pgf}
             \caption{Starshaped obstacles after one iteration. The convex hull of the specified kernel used for the starshaped hull, $CH(K)$, is shown in blue.}
             \label{fig:step1_star}
        \end{subfigure}
    \end{minipage}
    \begin{minipage}{0.49\linewidth}
        \begin{subfigure}[t]{\linewidth}
            \input{figures/step1_cluster.pgf}
             \caption{Clustering of original obstacles contained in intersecting starshaped obstacles into one single cluster.}
             \label{fig:step1_cluster}
        \end{subfigure}
        \hfill
        \begin{subfigure}[t]{\linewidth}
            \input{figures/step2_admker0.pgf}
             \caption{Admissible kernel for the new cluster.}
             \label{fig:step2_admker}
        \end{subfigure}
        \hfill
        \begin{subfigure}[t]{\linewidth}
            \input{figures/step2_star.pgf}
             \caption{Starshaped hull of the new cluster. The convex hull of the specified kernel is shown in blue}
             \label{fig:step2_star}
        \end{subfigure}
    \end{minipage}
    \caption{Steps in Algorithm \ref{alg:starify_obstacles} when four obstacles are combined into one starshaped obstacle in two iterations. Each cluster is identified with a separate color.}
    \label{fig:starify_steps}
\end{figure}
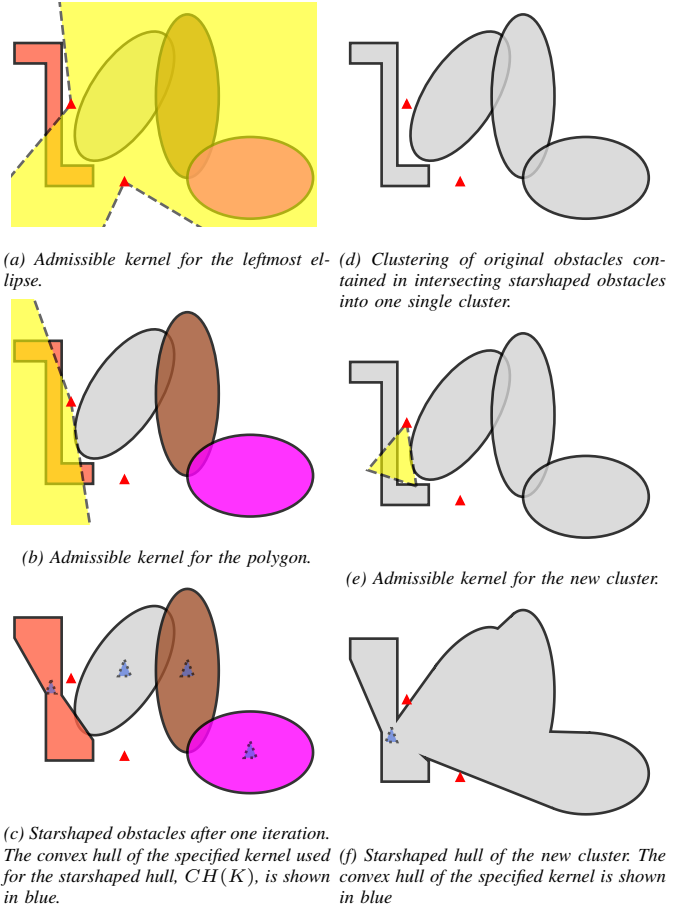
\subsubsection{Admissible kernel}
According to \eqref{eq:adm_ker_combination} the admissible kernel for a cluster can be found as the intersection of the admissible kernels for the corresponding cluster obstacles, derived using \eqref{eq:adm_ker_cone} and \eqref{eq:adm_kernel_convex} for the $\mathbb{R}^2$ and $\mathbb{R}^3$ case, respectively. Consequently, for computational efficiency, the admissible kernel for each obstacle can be computed once outside the loop and the admissible kernel for all clusters in each iteration can be found by pure intersections.
As discussed in Sec. \ref{sec:admissible_kernel}, the admissible kernel may be the empty set if $x$ and/or $x_g$ are bounded exterior points for the evaluating set, i.e. there exists no starshaped set which contains the cluster obstacle(s) and at the same time does not contain $x$ and $x_g$. In such case, the algorithm is terminated by providing a solution based on convex decomposition (line \ref{l:premature_return}) resulting in an intersecting star world.

\subsubsection{Starshaped hull}
To generate a starshaped set containing each cluster, the starshaped hull with specified kernel is applied (line \ref{l:star_hull}). In particular, $n+1$ affinely independent points with corresponding convex hull contained by the admissible kernel are selected as the specified kernel points (line \ref{l:kernel_points}), forming a triangle or polyhedron, depending on the dimensionality. This ensures that the resulting set is strictly starshaped and does not contain $x$ nor $x_g$ according to Properties \ref{p:st_hull_ker_strict} and \ref{p:st_hull_ker_adm_ker}.
The starshaped hull is found using \eqref{eq:star_hull_specified_convex} for convex sets and Algorithm \ref{alg:kernel_hull_polygon} for polygons. For clusters with more than one obstacle, Property \ref{p:st_hull_ker}\ref{p:st_hull_ker_union} can be leveraged to compute the starshaped hull of the union of all cluster obstacles.

\subsubsection{Clustering}
The clustering is applied on the original obstacle set, $\mathcal{O}$, (line \ref{l:cluster_obs}) but is determined by the intersection of the current starshaped obstacles, $\mathcal{O}^{\star}$, (line \ref{l:cluster_star}). For instance, even though the original polygon is disjoint from the ellipses, all four original obstacles are combined into one cluster in Fig. \ref{fig:step1_cluster} since the starshaped obstacles in $\mathcal{O}^{\star}$ containing them are intersecting in Fig. \ref{fig:step1_star}.

Unless Algorithm \ref{alg:starify_obstacles} terminates prematurely in line \ref{l:premature_return}, $\mathcal{O}^{\star}$ consists of mutually disjoint strictly starshaped obstacles which all contain the original obstacles but not the robot nor goal position, i.e. $\mathcal{F}^{\star} = \mathcal{W} \setminus \bigcup_{\mathcal{O}_i^{\star}\in \mathcal{O}^{\star}}\mathcal{O}_i^{\star}$ is a disjoint star world such that $x\in \mathcal{F}^{\star}$ and $x_g\in \mathcal{F}^{\star}$, and provides a solution to Problem \ref{problem}. 
A premature termination occurs either if a single polygon obstacle surrounds the robot and/or goal, as for $\bar{z}$ in Fig. \ref{fig:admissible_kernel_polygon}, or if the combination of all obstacles in a cluster surrounds the robot and/or goal. When a single polygon or when intersecting obstacles in combination surround $x$ and/or $x_g$, there exists no solution with a disjoint star world to Problem \ref{problem}. However, disjoint obstacles may as well be combined into one cluster leading to a termination with an intersecting star world if the combination of the clustered obstacles surround $x$ and/or $x_g$. This is exemplified in Fig. \ref{fig:condition_not_satisfied} with two polygons which in combination surround $x$ and where the generated hulls of the two polygons, when following the kernel point selection in Algorithm \ref{alg:kernel_point_selection} as discussed below, intersect such that they are combined into one cluster. As the clustered obstacles now surround $x$, the algorithm terminates with an intersecting star world as depicted in Fig. \ref{fig:intersecting_star_world_polygons}, jeopardizing convergence to the goal as will be exemplified in Sec. \ref{sec:examples}. As seen in Fig. \ref{fig:disjoint_star_world_polygons} there exists a disjoint star world solving Problem \ref{problem} in this scenario.

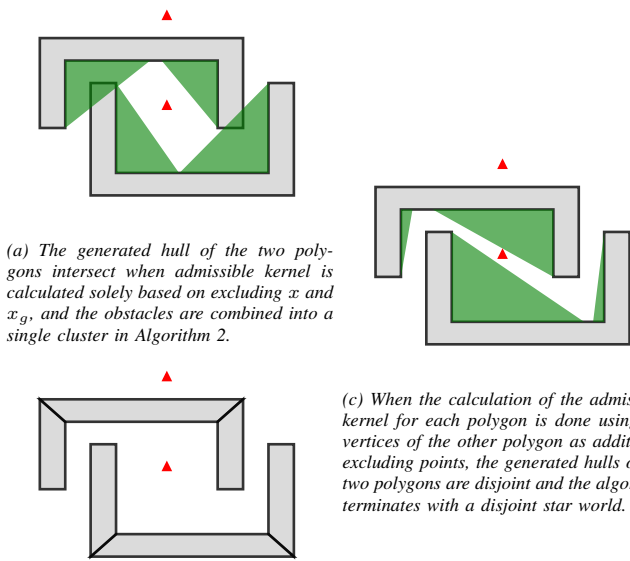
\begin{figure}
    \begin{minipage}{0.49\linewidth}
        \begin{subfigure}[t]{\linewidth}
            \input{figures/hull_intersecting_star_world_polygons.pgf}
             \caption{The generated hull of the two polygons intersect when admissible kernel is calculated solely based on excluding $x$ and $x_g$, and the obstacles are combined into a single cluster in Algorithm \ref{alg:starify_obstacles}.}
             \label{fig:hull_intersecting_star_world_polygons}
        \end{subfigure}
        \hfill
        \begin{subfigure}[t]{\linewidth}
            \input{figures/intersecting_star_world_polygons.pgf}
             \caption{The resulting set $\mathcal{O}^{\star}$ is generated according to \eqref{eq:convex_decomposition} resulting in six obstacles forming an intersecting star world.}
             \label{fig:intersecting_star_world_polygons}
        \end{subfigure}
    \end{minipage}
    \begin{minipage}{0.49\linewidth}
        \begin{subfigure}[t]{\linewidth}
            \input{figures/disjoint_star_world_polygons.pgf}
             \caption{When the calculation of the admissible kernel for each polygon is done using the vertices of the other polygon as additional excluding points, the generated hulls of the two polygons are disjoint and the algorithm terminates with a disjoint star world.}
             \label{fig:disjoint_star_world_polygons}
        \end{subfigure}
    \end{minipage}
    \caption{Example of when the sufficient condition for obtaining a disjoint star world with Algorithm \ref{alg:starify_obstacles} is not met since $\bigcup_{\mathcal{O}_i\in\mathcal{O}}\mathcal{O}_i$ surrounds $x$.}
    \label{fig:condition_not_satisfied}
\end{figure}


\subsection{Excluding obstacle points}
In Fig. \ref{fig:starify_steps} all obstacles are combined into one although the polygon is disjoint from the ellipses. If it is desired to maintain disjoint obstacles when possible, additional excluding points can be introduced in the computation of the admissible kernel for a cluster. Specifically, the additional excluding points should be representative for all obstacles which are not in the cluster. The points can for instance be selected as the vertices for a polygon and the two extreme points in each axis for an ellipse. When adopting this approach for the example considered in Fig. \ref{fig:starify_steps}, the algorithm terminates with two starshaped obstacles and not one as the polygon remains disjoint from the ellipses as seen in Fig. \ref{fig:starify_steps_ex_obs}. The key to obtain disjoint obstacles is here the reduced admissible kernel for the polygon (compare Figs. \ref{fig:step1_admker_pol} and \ref{fig:step1_admker_pol_ex_obs}) which restricts the kernel points to be selected such that the starshaped hull of the polygon does not intersect the ellipse. Adopting this approach can also prevent unnecessary premature termination of Algorithm \ref{alg:starify_obstacles} since disjoint obstacles can to larger extent remain in separate clusters, see Fig \ref{fig:disjoint_star_world_polygons}.
Of course, the introduction of additional excluding points in this manner drastically increases the computational complexity as the admissible kernel for each obstacle needs to be computed for every added excluding point and this procedure is in many cases not suitable for online operation. 

\begin{figure}
    \begin{minipage}{0.49\linewidth}
        \begin{subfigure}[t]{\linewidth}
            \input{figures/step1_admker3_ex_obs.pgf}
             \caption{Admissible kernel for the polygon excluding $x$, $x_g$ and the four extreme points of each ellipse.}
             \label{fig:step1_admker_pol_ex_obs}
        \end{subfigure}
        \hfill
        \begin{subfigure}[t]{\linewidth}
            \input{figures/step1_star_ex_obs.pgf}
             \caption{Starshaped hull of the obstacles.}
             \label{fig:step1_star_ex_obs}
        \end{subfigure}
    \end{minipage}
    \begin{minipage}{0.49\linewidth}
        \begin{subfigure}[t]{\linewidth}
            \input{figures/step2_admker0_ex_obs.pgf}
             \caption{Admissible kernel for the cluster containing the three ellipses excluding $x$, $x_g$ and the vertices of the polygon.}
             \label{fig:step2_admker_ex_obs}
        \end{subfigure}
        \hfill
        \begin{subfigure}[t]{\linewidth}
            \input{figures/step2_star_ex_obs.pgf}
             \caption{Starshaped hull of the new clusters.}
             \label{fig:step2_star_ex_obs}
        \end{subfigure}
    \end{minipage}
    \caption{Steps in Algorithm \ref{alg:starify_obstacles} when four obstacles are combined into two starshaped obstacle.}
    \label{fig:starify_steps_ex_obs}
\end{figure}
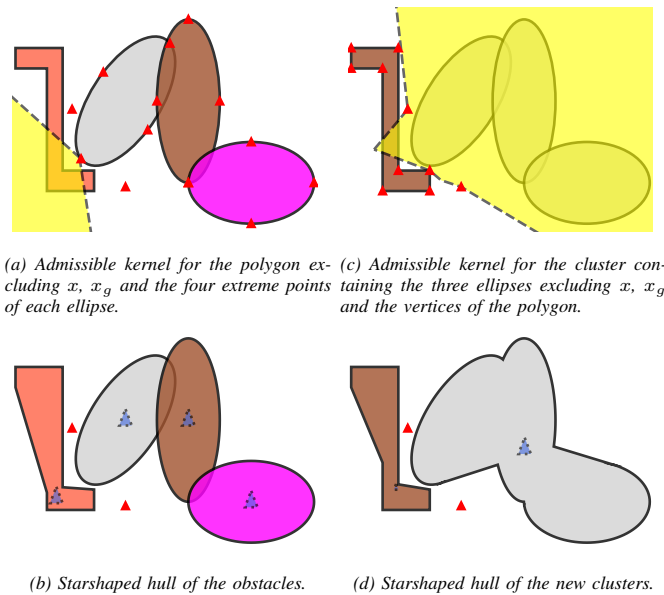

\subsection{Kernel point selection}
\label{sec:kernel_point_selection}
Algorithm \ref{alg:starify_obstacles} does not offer a constructive way to select the kernel points, but rather a condition which they must fulfill (line \ref{l:kernel_points}). However, the choice of specified kernel points may have a big impact on the resulting hull shape as already seen and they should not be selected randomly in the admissible kernel. In particular, in cases where the admissible kernel is unbounded (as for both clusters in Fig. \ref{fig:step1_admker_ell} and \ref{fig:step1_admker_pol}) random selection is directly unsuitable. A reasonable selection can be found in the set $\textnormal{ad ker}(cl,\{x, x_g\}) \cap CH(cl)$ resulting in a starshaped hull which is upper bounded by the convex hull according to property \ref{p:st_hull_ker}\ref{p:st_hull_ker_ch_bound} For computational efficiency the convex hull can be omitted and a selection from $\textnormal{ad ker}(cl,\{x, x_g\}) \cap cl$ can be made instead, assuming it is non-empty. 
The authors' recommendation as a general strategy is to select the kernel points as a small equilateral triangle (regular tetrahedron in $\mathbb{R}^3$) 
contained by this set. 

\begin{algorithm}
\caption{Kernel point selection}
\label{alg:kernel_point_selection}
\SetKwInOut{Input}{Input}
\SetKwInOut{Output}{Output}
\SetKwInOut{Parameter}{Parameter}
\SetKwProg{Init}{init}{}{}
\Input{A cluster of obstacles, $cl$, points to exclude, $\bar{X}$, and centroid selection for $cl$ at previous time iteration, $k_c^{prev}$.}
\Output{Kernel point selection for Algorithm \ref{alg:starify_obstacles}, $K$.}
\Parameter{Desired distance between specified kernel points, $l$.}
$S \gets \textnormal{ad ker}(cl, \bar{X})$\;
\If{$S \cap cl \neq \emptyset$}{
$S \gets S \cap cl$\;
}
Split $S$ by the line $l(x,x_g)$ into $S_1$ and $S_2$ s.t. $xx_gs_1$ is CW $\forall s_1\in S_1$\label{l:S_split}\;
\If{$k_c^{prev}$ exists}{
\uIf{$k_c^{prev}\in S_1\cup S_2$}{
$S_1 \gets k_c^{prev}$\label{l:S1_sel}\;
}
\ElseIf{$xx_gk_c^{prev}$ is CCW\label{l:S2_sel_check}}{
$S_1\gets S_2$ \label{l:S2_sel}\;
}
}
$k_c \gets$ point in $S_1$ closest to centroid of $S_1$\label{l:kcfromS1}\;
$K \gets$ largest equilateral triangle (regular tetrahedron in $\mathbb{R}^3$) in $\textnormal{ad ker}(cl, \bar{X})$ with maximum side length, $l$, and centroid in $k_c$\;
\KwRet{$K$}\;
\end{algorithm}

When Algorithm \ref{alg:starify_obstacles} is used in an online fashion to treat a dynamic workspace with moving obstacles it may be favorable to keep track of clusters and corresponding kernel points from previous time step in order to maintain as similar obstacle shapes as possible in sequential time steps. That is, for a cluster consisting of the same obstacles as a cluster in the previous time step, the kernel points can be reused, assuming they still are in the admissible kernel. This enables the generation of a smoother trajectory in the motion planning stage as compared to cases where obstacle shapes drastically change from one time step to another. A procedure to select the kernel points is shown in Algorithm \ref{alg:kernel_point_selection}. \added{Here, if the cluster centroid, $k_c^{prev}$, exists (i.e. the same obstacles were combined at previous time step), the kernel selection is based on this point to keep the shape as similar as possible (lines \ref{l:S1_sel} and \ref{l:kcfromS1}), see e.g. the right cluster in Fig. \ref{fig:soads_moving_1}-\ref{fig:soads_moving_2} and Fig. \ref{fig:soads_moving_3}-\ref{fig:soads_moving_4}. If the obstacles have moved such that $k_c^{prev}$ is not a valid kernel point, the kernel centroid is selected to remain on the same side of the line $l(x,x_g)$ as previous time instance (lines \ref{l:S2_sel_check}-\ref{l:S2_sel} and \ref{l:kcfromS1}). This is done to enable a more consistent motion generated by the motion planner as the direction for obstacle circumvention highly depends on a selected \textit{center point} of the obstacle, $x_{c,i}\in \text{ker}(\mathcal{O}_i)$ (see Sec. \ref{sec:obstacle_representation}). An example of the kernel point selection is illustrated in Fig. \ref{fig:kernel_update}.

\begin{figure}
    \centering
    \begin{subfigure}[t]{0.53\linewidth}
        \begin{center}  
            \input{figures/kernel_update_1.pgf}
        \end{center}
         \caption{The robot circumvents the obstacle on the opposite side of $l(x,x_g)$ from $k_c$.}
         \label{fig:kernel_update_1}
    \end{subfigure}
    \begin{subfigure}[t]{0.48\linewidth}
        \begin{center}              \input{figures/kernel_update_2.pgf}
        \end{center}
         \caption{The cluster of obstacles has moved such that the previous kernel point, $k_c^{prev}$, (black square) is not in the cluster region. The selection set, $S_1$, (hatched area) is then restricted to the region on the same side of $l(x,x_g)$ as $k_c^{prev}$ and the direction for obstacle circumvention is same as at previous time step.}
         \label{fig:kernel_update_2}
    \end{subfigure}
    \hfill
    \begin{subfigure}[t]{0.48\linewidth}
        \begin{center}  
            \input{figures/kernel_update_3.pgf}
        \end{center}
         \caption{If $k_c$ would be selected on the other side of $l(x,x_g)$, i.e., not in $S_1$, the resulting direction for obstacle circumvention is opposite from previous time step.}
         \label{fig:kernel_update_3}
    \end{subfigure}
    \caption{Kernel selection procedure for moving obstacles. The obstacles are shown in grey, the robot position, $x$, as black circle, the goal position, $x_g$, as black star, the centroid for kernel point selection, $k_c$, as green diamond and the line segment $l[x,x_g]$ as blue dashed line.}
    \label{fig:kernel_update}
\end{figure}
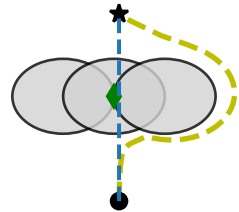
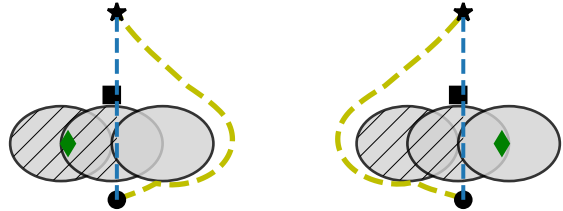
}

Note that Algorithm \ref{alg:starify_obstacles} is designed to find a solution to Problem \ref{problem} in an efficient manner and is not optimal neither in number of obstacles in $\mathcal{O}^{\star}$ nor in the total space covered by $\mathcal{O}^{\star}$. For instance, in Fig. \ref{fig:starify_steps} the kernel points are selected according to the authors' recommendation as described above. A selection of kernel points for the polygon close to the bottom, as in Fig. \ref{fig:step1_star_ex_obs}, would in this specific case lead to a termination of the algorithm with two separate obstacles (compared to one single obstacle in Fig. \ref{fig:starify_steps}) with a smaller total area coverage of the obstacles.

\added{
\subsection{Implementation}
\label{sec:implementation}
In this section, application aspects of the proposed algorithm\footnote{Code for the implementation of the algorithm can be found at https://github.com/albindgit/starworlds.} are investigated. For simplicity, the obstacles, $\mathcal{O}$, are here considered to be ellipses or convex polygons. However, note that any concave obstacle can still be modelled as the combination of several convex obstacles since intersections are allowed. 
Let $N$ be a large number such that an $N$-gon approximates an ellipse closely\footnote{The polygon approximations of the ellipses are used in the clustering stage to identify intersections of obstacles.} and such that any polygon in $\mathcal{O}$ has less vertices than $N$. Then the computational complexity for Algorithm \ref{alg:starify_obstacles} can be upper bounded by
\begin{equation}
\label{eq:compute_complex}
    O(\underbrace{|\mathcal{O}|N + M|\mathcal{O}|}_{\textnormal{Admissible kernel}} + \underbrace{M|\mathcal{O}|N}_{\textnormal{Starshaped hull}} + \underbrace{M|\mathcal{O}|^2N}_{\textnormal{Clustering}})
\end{equation}
where $M$ is the number of iterations of Algorithm \ref{alg:starify_obstacles} before termination. See Appendix \ref{a:compute_time} for derivation.
In the extreme case, all obstacles except two are disjoint, but yet are combined into a single cluster one by one. This means that the number of iterations, $M$, are bounded by $|\mathcal{O}|$ and a conservative estimate of the upper bound for Algorithm \ref{alg:starify_obstacles} given in \eqref{eq:compute_complex} is $O(|\mathcal{O}|^3N)$. 
However, this is a rather unlikely scenario and since all intersecting obstacles in $\mathcal{O}$ are combined in the first iteration, the algorithm terminates after two iterations in most scenarios.

Since the cubic complexity appear to be conservative, we introduce the following quantitative statistical study. We evaluate the computation time for $1000$ random scenarios with $5$-$50$ obstacles randomly placed in a squared scene. The size of the scene width is dynamically scaled such that the area covered by obstacles remains approximately $25 \%$ of the total scene area for all cases. This is to have a similar densely populated scene in all cases, independent of the number of obstacles.
Half of the obstacles are ellipses with each axis length $a\sim \mathcal{N}(1, 0.2^2)$ and $N=30$ vertices for the polygon approximations. The other obstacles are polygons with 10 vertices randomly sampled inside a $2\times 2$ box based on the proof in \cite{valtr_95}\footnote{The steps to generate the polygons are given in https://cglab.ca/~sander/misc/ConvexGeneration/convex.html}. The robot and goal positions are uniformly randomly placed in the obstacle free area. An example is shown in Fig. \ref{fig:compute_time_scene}. The test was performed on a computer with an Intel Core i7-8650U CPU @ 1.90GHz processor. 
From Fig. \ref{fig:compute_times} it is clear that the computation time rather tends to have linear than cubic growth with the number of obstacles. The empirical study confirms that the cubic complexity obtained based on the assumption $M=|\mathcal{O}|$ is conservative. In fact, of the 1000 cases in the study the algorithm terminated after two iterations 972 times and after three iterations 28 times. No case exceeded three iterations and one can in practice assume $M\leq 5$ independently of $|\mathcal{O}|$. 
Moreover, the actual computation time for clustering remains a relatively small fraction of the full computation ($20\%$ at worst for the cases with 50 obstacles in the study). Hence, the computational burden is dominated by the admissible kernel evaluation and generation of starshaped hull. Since these both grow linearly with $|\mathcal{O}|$, see \eqref{eq:compute_complex}, also the computation time for Algorithm \ref{alg:starify_obstacles} tend to grow linearly with number of obstacles for practical scenarios (in scenes with reasonable low number of obstacles).
}

\begin{figure}
    \centering
    \begin{subfigure}[t]{0.49\linewidth}
        \includegraphics[width=\linewidth]{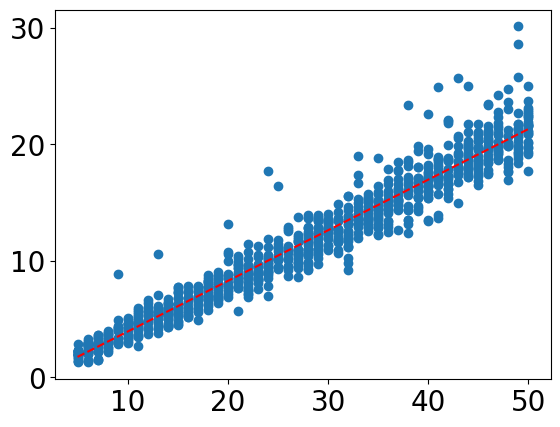}
        \caption{Computation time [ms] of Algorithm \ref{alg:starify_obstacles} over 1000 random scenarios ranging from 5-50 obstacles.}
        \label{fig:compute_times}
    \end{subfigure}
    \hfill
    \begin{subfigure}[t]{0.49\linewidth}
        \includegraphics[width=\linewidth]{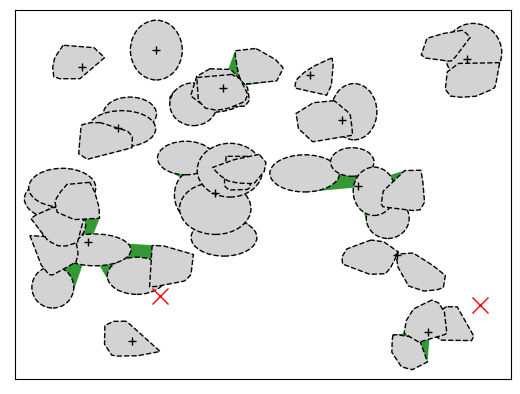}
        \caption{Random sample scene from the evaluation set with 44 obstacles clustered to 13 starshaped obstacles.}
        \label{fig:compute_time_scene}
    \end{subfigure}
    \label{fig:compute_time}
    \caption{Evaluation of the computation time for Algortihm \ref{alg:starify_obstacles}.}
\end{figure}

\section{Obstacle avoidance in a starshaped world}
\label{sec:examples}
To illustrate the utility provided by Algorithm \ref{alg:starify_obstacles}, we here present some results when it is used as obstacle modelling for motion planning based on the setting illustrated in Fig. \ref{fig:block_soads}. That is, a set of (disjoint) starshaped obstacles, $\mathcal{O}^{\star}$, is generated with Algorithm \ref{alg:starify_obstacles} given a workspace with obstacles, $\mathcal{O}$, the robot position, $x$, and the goal position, $x_g$. $\mathcal{O}^{\star}$ is used when deciding on the movement for the robot, $\dot{x}$, with the motion planner in \cite{huber_etal_22}. 

\subsection{Obstacle representation} 
\label{sec:obstacle_representation}
\added{The motion planner requires a center point, $x_{c,i}\in \textnormal{int ker} (\mathcal{O}^{\star}_i)$, to be defined for each obstacle, $\mathcal{O}^{\star}_i$, to enable the formulation of an obstacle function.}
Thus, the kernel of the obstacle, or at least some subset of it, needs to be known. While for a convex set the starshaped kernel is the set itself and the calculation of the kernel for a polygon with $N$ vertices can be done in $O(N)$ time \cite{lee_preparata_79}, it may be computationally demanding to find the kernel for a general starshaped set. 
However, since Algorithm \ref{alg:starify_obstacles} returns a collection of obstacles generated using the starshaped hull with specified kernel, it directly provides a subset of the kernel of each obstacle.
In particular, according to Property \ref{p:st_hull_ker}\ref{p:st_hull_ker_ch_ker} we have that $CH(K_i)\subset \textnormal{ker}(\mathcal{O}^{\star}_i)$, where $K_i$ is the specified kernel set corresponding to $\mathcal{O}^{\star}_i$, and a center point can be chosen from the known set $\textnormal{int} CH(K_i)$.
Convergence to the goal is guaranteed given a set of disjoint obstacles under the condition that no obstacle center point lies in the line containing the robot and goal position, $l(x,x_g)$\added{ \cite{huber_etal_19}}. In other words, convergence is guaranteed if the center point of each obstacle lies in its \textit{convergence center point set} defined as $\mathcal{X}^c_i = \text{int ker}(\mathcal{O}^{\star}_i) \setminus l(x,x_g)$. Since $\textnormal{int}CH(K_i)$ is an open nonempty set of dimension $n$ according to Property \ref{p:st_hull_ker_strict}, the set $\bar{\mathcal{X}}^c_i = \textnormal{int}CH(K_i) \setminus l(x,x_g) \subset \mathcal{X}^c_i$ is also nonempty and any $x_{c,i} \in \bar{\mathcal{X}}^c_i$ is a valid center point selection which guarantees convergence to the goal. 
In the following examples $x_{c,i}$ is chosen as the centroid of $K_i$, \added{i.e. $k_c$ from Algorithm \ref{alg:kernel_point_selection}.}

\subsection{Examples}
\subsubsection{Intersecting obstacles}
Consider the scenario in Fig. \ref{fig:soads_intersecting_ellipses} with three intersecting ellipses. The vector field for the motion planner is depicted in Fig. \ref{fig:soads_intersecting_ellipses1} for the case when Algorithm \ref{alg:starify_obstacles} is not used, i.e. with $\mathcal{O}^{\star}=\mathcal{O}$. Clearly, the obstacles are not disjoint, which is a condition for guaranteed convergence to the goal. As a result, there exists two attractors, apart from the goal position, at points of obstacle intersection.  When Algorithm \ref{alg:starify_obstacles} is applied on the original obstacles as in Fig. \ref{fig:block_soads}, the resulting set is disjoint (a single obstacle for the case in Fig. \ref{fig:soads_intersecting_ellipses2}) and the robot converges to the goal.

\begin{figure}
    \centering
    \begin{subfigure}[t]{0.49\linewidth}
        \centering
        \includegraphics[width=\linewidth]{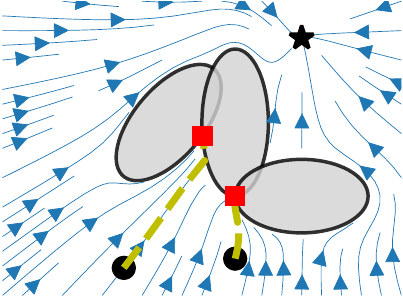}
         \caption{When the intersecting obstacles are treated as separate shapes there exist two local minima at the intersection points, shown as red squares, and convergence to the goal is not achieved.}
         \label{fig:soads_intersecting_ellipses1}
    \end{subfigure}
    \hfill
    \begin{subfigure}[t]{0.49\linewidth}
        \centering
        \includegraphics[width=\linewidth]{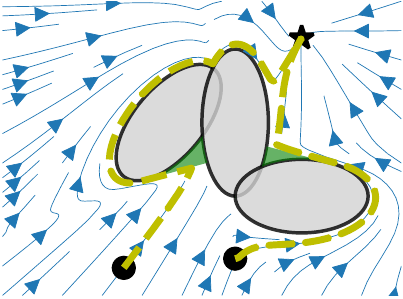}
     \caption{When the obstacles are combined into one using the starshaped hull following Algorithm \ref{alg:starify_obstacles} the convergence guarantee is preserved.}
     \label{fig:soads_intersecting_ellipses2}
    \end{subfigure}
    \caption{Vector field of motion planner \cite{huber_etal_19} for three intersecting ellipses with goal position shown as black star. The resulting path is also depicted, in yellow dashed line, from two initial positions, shown as black dots.}
    \label{fig:soads_intersecting_ellipses}
\end{figure}

\subsubsection{Moving obstacles}
\added{
In Fig. \ref{fig:soads_moving} a scenario with three humans walking around in an area containing two walls is illustrated. The humans are modelled as circle obstacles and the walls as polygon obstacles (one wall as a single convex polygon and one as the combination of three convex polygons).}
Since the motion planner treats the robot as a point mass, all obstacles are inflated by the robot radius. As a result, the obstacles may intersect when a human is close to a wall or another human.
\added{Throughout the time period when the same obstacles are clustered into a single starshaped obstacle, i.e. for the rightmost cluster in both Figs. \ref{fig:soads_moving_1}-\ref{fig:soads_moving_2} and Figs. \ref{fig:soads_moving_3}-\ref{fig:soads_moving_4}, the specified kernel points as well as the center point are kept static such that a similar shape of the starshaped obstacle is obtained to enable a smooth trajectory for the robot, as discussed in Sec. \ref{sec:kernel_point_selection}. 
Note that both in Fig. \ref{fig:soads_moving_3} and \ref{fig:soads_moving_6} more than two obstacles share an intersecting region such that Forest of Stars \cite{rimon_koditschek_92} is not applicable.}
\added{The average computation time for the algorithm during simulation was 4.9 ms with a standard deviation of 0.6 ms and a maximum time of 7.5 ms (using same computer as in Sec. \ref{sec:implementation}).}

\begin{figure}
    \centering
    \begin{subfigure}[t]{0.49\linewidth}
        \centering
        \includegraphics[width=\linewidth]{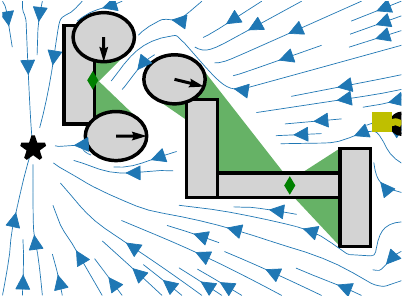}
         \caption{}
         \label{fig:soads_moving_1}
    \end{subfigure}
    \hfill
    \begin{subfigure}[t]{0.49\linewidth}
        \centering
        \includegraphics[width=\linewidth]{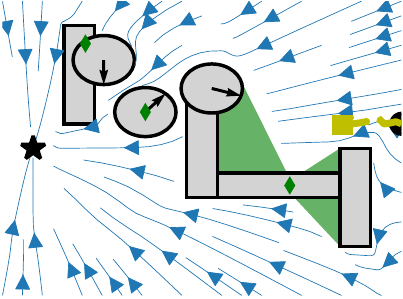}
         \caption{}
         \label{fig:soads_moving_2}
    \end{subfigure}
    \begin{subfigure}[t]{0.49\linewidth}
        \centering
        \includegraphics[width=\linewidth]{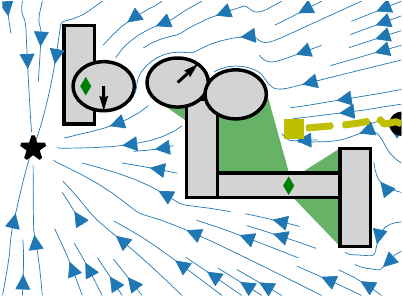}
         \caption{}
         \label{fig:soads_moving_3}
    \end{subfigure}
    \begin{subfigure}[t]{0.49\linewidth}
        \centering
        \includegraphics[width=\linewidth]{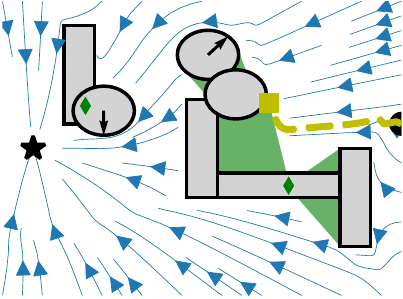}
         \caption{}
         \label{fig:soads_moving_4}
    \end{subfigure}
    \begin{subfigure}[t]{0.49\linewidth}
        \centering
        \includegraphics[width=\linewidth]{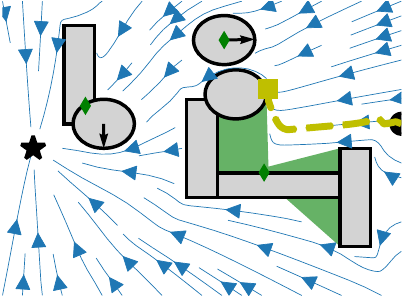}
         \caption{}
         \label{fig:soads_moving_5}
    \end{subfigure}
    \begin{subfigure}[t]{0.49\linewidth}
        \centering
        \includegraphics[width=\linewidth]{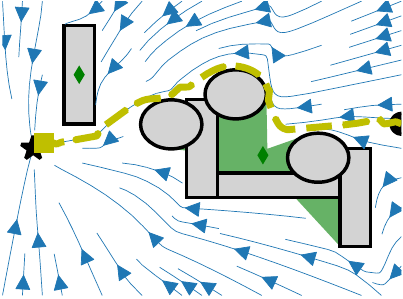}
         \caption{}
         \label{fig:soads_moving_6}
    \end{subfigure}
    \caption{Robot path, in yellow, and momentary vector field of the motion planner, shown as blue arrows, given the goal position, shown as black star, and the starshaped obstacles with center point shown as green diamond. The starshaped obstacles consist of the original obstacles, in grey, and (possibly) the hull extended region, in green. The momentary velocity of the moving obstacles is illustrated with a black arrow.}
    \label{fig:soads_moving}
\end{figure}


\section{Conclusion}
\label{sec:conclusions}
In this paper we have considered the problem of reshaping a set of intersecting obstacles into a set of disjoint strictly starshaped obstacles such that motion planning methods operating in star worlds can be applied while maintaining convergence properties. To this end, we have elaborated on the concept of starshaped hull and its properties. The admissible kernel for a set has been introduced to enable excluding points of interest from the starshaped hull and the starshaped hull with specified kernel has been introduced to ensure that the resulting set is strictly starshaped. Using the concepts of admissible kernel and starshaped hull with specified kernel, an algorithm has been designed to modifying a workspace of intersecting obstacles to a workspace of disjoint starshaped obstacles. 
The utilization of the proposed method has been illustrated with examples of a robot operating in a 2D workspace using a motion planner in combination with the developed algorithm.

In this work we laid the foundations for generating star worlds, a process that can support the design of real-time control algorithms for a variety of robots operating in dynamical environments. The robot workspace is assumed to be the full Euclidean space. In many scenarios this is not the case, e.g. for a robot operating in a closed room, and an extension of the method to also consider workspace boundaries would be beneficial. Moreover, the algorithm is not complete in the sense that it may in some scenarios provide an intersecting star world even though a disjoint star world exists. Further investigation and modification of the method is needed to address complex scenarios where the robot and goal are surrounded by obstacles.
\appendix

\subsection{Proof of Property \ref{p:st_hull}}
\label{a:proof_p_st_hull}
\begin{enumerate}[label=\alph*.]
    \item ($\Rightarrow$): $SH_x(A)$ is a starshaped set w.r.t. $x$ by definition, i.e. $SH_x(A)$ is starshaped with $x\in \textnormal{ker}\left(SH_x(A)\right)$. Since $A=SH_x(A)$ also $A$ is starshaped with $x\in \textnormal{ker}(A)$.\\
    ($\Leftarrow$): Since $x\in \textnormal{ker}(A)$, $A$ is starshaped w.r.t. $x$ from the definition of the starshaped kernel. Obviously, $A$ is the minimum starshaped set which contains $A$ and we have $SH_x(A) = A$.
    \item 
    \begin{align*}
    SH_x(A)\ &= \bigcup_{y\in A} l[x,y] \subset \bigcup_{y\in CH(A)} l[x,y]\\
    & \subset \bigcup_{x\in CH(A)}\bigcup_{y\in CH(A)} l[x,y] = CH(A)
    \end{align*}
    \item 
    \begin{align*}
    SH_x\left(\bigcup_{B\in\mathcal{B}} B\right) &= \bigcup_{y\in \bigcup_{B\in\mathcal{B}} B} l[x,y] = \bigcup_{B\in\mathcal{B}}\bigcup_{y\in B} l[x,y] \\
    &= \bigcup_{B\in\mathcal{B}}SH_x(B)
    \end{align*}
\end{enumerate}

\subsection{Proof of Proposition \ref{prop:intersecting_kernel}}
\label{a:proof_prop_intersecting_kernel}
Let $x\in K_{\cap}$. Since $x \in \textnormal{ker}(A),\ \forall A\in \mathcal{A}$ we have that $A = SH_x(A),\ \forall A\in \mathcal{A}$ according to Property \ref{p:st_hull}\ref{p:hull_star_kernel} Thus we have $A_{\cup} = \bigcup_{A\in\mathcal{A}} SH_x(A) = SH_x\left(\bigcup_{A\in\mathcal{A}} A\right) = SH_x(A_{\cup})$ from Property \ref{p:st_hull}\ref{p:hull_union} which is starshaped by definition. From Property \ref{p:st_hull}\ref{p:hull_star_kernel} we have $x \in \textnormal{ker}\left(A_{\cup}\right)$.
This holds for any selection $x\in K_{\cap}$ and we can conclude $K_{\cap} \subset \textnormal{ker}\left(A_{\cup}\right)$.

\subsection{Proof of Proposition \ref{prop:strictly_star}}
\label{a:proof_l_boundary_mapping}
Let $x_c \in \textnormal{int ker} (A)$ and assume $A$ is not strictly starshaped. Since $A$ is not strictly starshaped there exists more than one boundary point in some direction, i.e. there exist two boundary points $x_{b1}, x_{b2} \in \partial A$ such that $\frac{x_{b1}-x_c}{\lVert x_{b1}-x_c \rVert}=\frac{x_{b2}-x_c}{\lVert x_{b2}-x_c \rVert}$. Let $z_1$ be the point, $x_{b1}$ or $x_{b2}$, which is furthest away from $x_c$ and $z_2$ be the closest one, such that $z_2 \in l^o[x_c,z_1]$ where $l^o[x_c,z_1]$ is the open line segment from $x_c$ to $z_1$.
As $\textnormal{int ker} (A)$ is nonempty, there exists an $n$-ball, $B_{\epsilon}(x_c)$, with radius $\epsilon$ around $x_c$, such that $B_{\epsilon}(x_c) \subset \textnormal{ker}(A)$. From the starshapedness of $A$ we have $l[x,z_1] \subset A,\ \forall x \in \textnormal{ker}(A)$, and the cone $\mathcal{C}_{z_1}=\bigcup_{x \in B_{\epsilon}(x_c)}$ $l[x, z_1]$ is hence contained by $A$. The cone $\mathcal{C}_{z_1}$ is centered around the axis aligned with $l[x_c,z_1]$, and $l^o[x_c,z_1] \subset \textnormal{int}\mathcal{C}_{z_1} \subset \textnormal{int} A$. Thus $z_2 \in \textnormal{int} A$. That is, the closest point to $x_c$ of $x_{b1}$ and $x_{b2}$ is an interior of point of $A$. This is a contradiction to $x_{b1}, x_{b2} \in \partial A$ and we can conclude that $A$ is strictly starshaped.

\subsection{Proof of Property \ref{p:adm_kernel_union}}
\label{a:proof_p_adm_ker_union}
We have $x\in \textnormal{ad ker}(A_{\cup},\bar{X}) \Leftrightarrow SH_x(A) \cap \bar{X} = \emptyset, \forall A\in \mathcal{A}$ since
\begin{align*}
    \textnormal{ad ker}(A_{\cup}, \bar{X}) =& \{x\in \mathbb{R}^n : \Bar{X} \cap SH_x(\bigcup_{A\in\mathcal{A}}A) = \emptyset\} \\ 
    \stackrel{\text{Pr. \ref{p:st_hull}\ref{p:hull_union}}}{=}& \{x\in \mathbb{R}^n : \Bar{X} \cap \bigcup_{A\in\mathcal{A}} SH_x(A) = \emptyset\} \\
    =& \{x\in \mathbb{R}^n : \bigcup_{A\in\mathcal{A}} \Bar{X} \cap SH_x(A) = \emptyset\}.
\end{align*}
Moreover, $x\in \bigcap_{A\in\mathcal{A}}\textnormal{ad ker}(A,\bar{X}) \Leftrightarrow SH_x(A) \cap \bar{X} = \emptyset, \forall A\in \mathcal{A}$ since
\begin{align*}
    \bigcap_{A\in\mathcal{A}}\textnormal{ad ker}(A,\bar{X}) =& \bigcap_{A\in\mathcal{A}}\{x\in \mathbb{R}^n : \Bar{X} \cap SH_x(A) = \emptyset\}.
\end{align*}
Thus, $x\in \textnormal{ad ker}(A_{\cup},\bar{X}) \Leftrightarrow x\in \bigcap_{A\in\mathcal{A}}\textnormal{ad ker}(A,\bar{X})$, or equivalently $\textnormal{ad ker}(A_{\cup},\bar{X}) = \bigcap_{A\in\mathcal{A}}\textnormal{ad ker}(A,\bar{X})$.

\subsection{Proof of Property \ref{p:adm_ker_intersection}}
\label{a:proof_p_adm_ker_intersection}
We have $x\not\in \textnormal{ad ker}(A,\bar{X}) \Leftrightarrow \exists \bar{x}\in \bar{X}, SH_x(A) \cap \{\bar{x}\} \neq \emptyset$ since
\begin{align*}
    \textnormal{ad ker}(A, \bar{X}) =& \{x\in \mathbb{R}^n : SH_x(A) \cap \Bar{X} = \emptyset\} \\ 
    =& \{x\in \mathbb{R}^n : \bigcup_{\bar{x}\in \Bar{X}} SH_x(A) \cap \{\bar{x}\} = \emptyset\}.
\end{align*}
Moreover, $x\not\in \bigcap_{\bar{x}\in \bar{X}}\textnormal{ad ker}(A,\{\bar{x}\}) \Leftrightarrow  \exists \bar{x}\in \bar{X}, SH_x(A) \cap \{\bar{x}\} \neq \emptyset$ since
\begin{align*}
   \textnormal{ad ker}(A, \bar{X}) =& \bigcap_{\bar{x}\in \Bar{X}}\{x\in \mathbb{R}^n : SH_x(A) \cap \{\bar{x}\} = \emptyset\}.
\end{align*}
Thus, $x\not\in \textnormal{ad ker}(A,\bar{X}) \Leftrightarrow x\not\in \bigcap_{\bar{x}\in \bar{X}}\textnormal{ad ker}(A,\{\bar{x}\})$, or equivalently $\textnormal{ad ker}(A,\bar{X}) = \bigcap_{\bar{x}\in \bar{X}}\textnormal{ad ker}(A,\{\bar{x}\})$.

\subsection{Proof of Property \ref{p:ad_ker_nonempty}}
\label{a:proof_p_ad_ker_nonempty}
First, let $\bar{x}\in A$. Since $A\subset SH_x(A)$ for any $x\in \mathbb{R}^n$ by definition, we have $SH_x(A)\cap\{\bar{x}\} = \{\bar{x}\}, \forall x\in \mathbb{R}^n$, and thus $\textnormal{ad ker}(A,\{\bar{x}\}) = \emptyset$.\\
Now, let $\bar{x}$ be a bounded exterior point of $A$. Since $A$ fully surrounds $\bar{x}$, for any point $x\in \mathbb{R}^n$, there exists another point $x_a\in A$ such that $\bar{x} \in l[x,x_a]$. Thus, $SH_x(A)=\bigcup_{y\in A}l[x,y]$ contains $\bar{x}$ for any $x\in \mathbb{R}^n$ and we have $\textnormal{ad ker}(A,\{\bar{x}\}) = \emptyset$.\\
Finally, let $\bar{x}$ be a free exterior point of $A$ where the ray $r(\bar{x},v)$ does not intersect $A$ and let $x = \bar{x} - \epsilon v$, with $\epsilon$ chosen small enough such that $l[x,\bar{x}] \cap A = \emptyset$. Obviously $x$ is also a free exterior point since for $r(x,v) = l[x,\bar{x}] \cup r(\bar{x},v)$ we have $r(x,v)\cap A = \emptyset$. Assume that $\textnormal{ad ker}(A, \{\bar{x}\})=\emptyset$. Then $\exists y\in A \textnormal{ s.t. } \bar{x} \in l[x,y]$ since $SH_x(A) = \bigcup_{y\in A}l[x,y]$. For this $y$, we have $l[x,y]\subset r(x,v)$ and, hence, $l[x,y]\cap A = \emptyset$. Thus, $y \not\in A$ which is a contradiction and we have $\textnormal{ad ker}(A, \{\bar{x}\}) \neq \emptyset$.

\subsection{Proof of Property \ref{p:st_hull_ker}}
\label{a:proof_p_star_hull_ker}
\begin{enumerate}[label=\alph*.] 
    \item For convenience, let $S_1 = SH_{\textnormal{ker}K}(A)$, $S_2 = SH_{\textnormal{ker}CH(K)}(A)$ and $S_{\cup} = \bigcup_{k \in CH(K)}SH_k(A)$.\\
    The kernel of any starshaped set is convex and the minimum convex set containing $K$ is $CH(K)$. Hence, $CH(K) \subset \textnormal{ker}(S_1)$ and, obviously, $CH(K)\subset \textnormal{ker}(S_2)$. That is, $l[k,y]\subset S_i, \forall k\in CH(K), \forall y \in S_i$, for $i=1,2$. Since $A\subset S_i, i=1,2$ it follows that $l[k,y]\subset S_i, \forall k\in CH(K), \forall y \in A$ for $i=1,2$. We then have $\bigcup_{k\in CH(K)}\bigcup_{y\in A}l[k,y] \subset S_i, i=1,2$, or equivalently $S_{\cup} \subset S_i, i=1,2$.
    By noting that $S_{\cup} = \bigcup_{k \in CH(K)}\bigcup_{y\in A}l[k,y]=\bigcup_{y \in A}SH_y(CH(K))$, we can write it as $S_{\cup} = \bigcup_{S^{\star}\in\mathcal{S}}S^{\star}$, where $\mathcal{S}$ is the collection of starshaped sets $\mathcal{S} = \{SH_y(CH(K)) : y\in A\}$.
    As stated in Sec. \ref{sec:starshaped_hull}, each subset $S^{\star}\in \mathcal{S}$ is convex as a consequence of \eqref{eq:local_hull_convex}. Since $CH(K)\subset S^{\star}, \forall S^{\star}$ and the kernel for a convex set is the set itself we have $CH(K)\subset \textnormal{ker}(S^{\star}), \forall S^{\star} \in \mathcal{S}$ such that $CH(K)\subset K_{\cap}$ where $K_{\cap} = \bigcap_{S^{\star}\in\mathcal{S}}\textnormal{ker}(S^{\star})$. 
    We can now conclude from Proposition \ref{prop:intersecting_kernel} that $S_{\cup}$ is a starshaped set with $CH(K) \subset \textnormal{ker}(S_{\cup})$. Additionally, it is clear that $A\subset S_{\cup}$ since $A\subset SH_k(A), \forall k$. Thus, $S_{\cup}$ is a starshaped set containing $A$ with $CH(K)$ contained by its kernel. Since $S_{\cup}\subset S_2$ it is also the smallest set satisfying these conditions and we have $S_2 = S_{\cup}$. Similarly, $S_{\cup}$ is a starshaped set containing $A$ with $K$ contained by its kernel. Since $S_{\cup}\subset S_1$ it is also the smallest set satisfying these conditions and we have $S_1 = S_{\cup}$.
    \item ($\Rightarrow$): $SH_{\textnormal{ker}K}(A)$ is a starshaped set w.r.t. $\forall k \in K$ by definition, i.e. $SH_{\textnormal{ker}K}(A)$ is starshaped with $k\in \textnormal{ker}\left(SH_{\textnormal{ker}K}(A)\right)$. 
    Thus, if
    $A=SH_{\textnormal{ker}K}(A)$, $A$ is also starshaped with $K\subset \textnormal{ker}(A)$.\\
    ($\Leftarrow$): Since $K \subset \textnormal{ker}(A)$, $A$ is starshaped w.r.t. $\forall k\in K$. Obviously, $A$ is the minimum starshaped set which contains $A$ and we have $SH_{\textnormal{ker}K}(A) = A$.
    \item A convex set is a starshaped set with kernel given as the set itself. Thus, $CH(A)$ is a starshaped set with $A\subset CH(A)$ and for any $K \subset CH(A)$ we have $K \subset \textnormal{ker}\left(CH(A)\right)$. Thus, under assumption that $K\subset CH(A)$, $CH(A)$ is a starshaped set containing $A$ with $K$ contained by its kernel. If it is the smallest such set, we have $SH_{\textnormal{ker}K}(A) = CH(A)$, otherwise $SH_{\textnormal{ker}K}(A) \subset CH(A)$. Thus, $SH_{\textnormal{ker}K}(A)\subset CH(A), \forall K\subset CH(A)$.
    \item We have
    \begin{align*}
        &SH_{\textnormal{ker}K}\left(\bigcup_{B\in\mathcal{B}}B\right) \stackrel{\text{Pr. \ref{p:st_hull_ker}\ref{p:st_hull_ker_ch_ker}}}{=} \bigcup_{k\in CH(K)}SH_k\left(\bigcup_{B\in\mathcal{B}}B\right)\\
        \stackrel{\text{Pr. \ref{p:st_hull}\ref{p:hull_union}}}{=} &\bigcup_{k\in CH(K)}\bigcup_{B\in\mathcal{B}}SH_k(B) = \bigcup_{B\in\mathcal{B}}\bigcup_{k\in CH(K)}SH_k(B)\\
        \stackrel{\text{Pr. \ref{p:st_hull_ker}\ref{p:st_hull_ker_ch_ker}}}{=} &\bigcup_{B\in\mathcal{B}}SH_{\textnormal{ker}K}(B).
    \end{align*}
\end{enumerate}

\subsection{Proof of Property \ref{p:st_hull_ker_strict}}
\label{a:proof_p_st_hull_ker_strict}
The convex hull of $n+1$ affinely independent points is of dimension $n$ and thus $CH(K)$ is guaranteed to have a nonempty interior. Knowing from Property \ref{p:st_hull_ker}\ref{p:st_hull_ker_ch_ker} that $CH(K)\subset \textnormal{ker}(SH_{\textnormal{ker}K}(A))$, we therefore have that $\textnormal{int ker}\left(SH_{\textnormal{ker}K}(A)\right) \neq \emptyset$. Using Proposition \ref{prop:strictly_star}, $SH_{\textnormal{ker}K}(A)$ can be concluded to be a strictly starshaped set.

\subsection{Proof of Property \ref{p:st_hull_ker_adm_ker}}
\label{proof:p_st_hull_ker_adm_ker}
From the definition of the admissible kernel we have $CH(K)\subset \textnormal{ad ker}(A,\bar{X}) \Rightarrow SH_k(A) \cap \bar{X} = \emptyset, \forall k \in CH(K)$, and it follows that $\bigcup_{k\in CH(K)}SH_k(A) \cap \bar{X} = \emptyset$. Using Property \ref{p:st_hull_ker}\ref{p:st_hull_ker_ch_ker} this can be written as $SH_{\textnormal{ker}K}(A) \cap \bar{X} = \emptyset$.

\added{
\subsection{Computational complexity}
\label{a:compute_time}
First, note that the time to compute two tangent points of an obstacle through an exterior point is done in $O(1)$ time for an ellipse and $O(N)$ time for a polygon with $N$ vertices \cite{preparata_shamos_12}. Further note that we have $\sum_{cl\in Cl}|cl| = |\mathcal{O}|$ at any iteration.


\subsubsection{Admissible kernel}
As stated in Sec. \ref{sec:forming_star_obstacles}, the admissible kernel for each obstacle is computed once before starting the iterations in Algorithm \ref{alg:starify_obstacles}.
From \eqref{eq:adm_ker_cone}, we see that to find the admissible kernel for a 2D set given any free exterior point it suffices to compute the two tangent points of the set through the excluding point. The initial computation time for finding $\textnormal{ad ker}(\mathcal{O}_i,x_g)$ for all $\mathcal{O}_i\in\mathcal{O}$ is hence upper bounded by $O(|\mathcal{O}|N)$. 

Finding $\textnormal{ad ker}(cl,\bar{X})$ for one cluster, $cl$, can be divided into computing $\textnormal{ad ker}(cl,x)$ and $\textnormal{ad ker}(cl,x_g)$ followed by an intersection of the two sets according to Property \ref{p:adm_ker_intersection}.
$\textnormal{ad ker}(cl,x)$ is found by the intersection of all cones $\textnormal{ad ker}(\mathcal{O}_i,x), \mathcal{O}_i\in cl$ according to Property \ref{p:adm_kernel_union}. Since the cones share apex, the intersection of two cones can be found in $O(1)$ time. Hence, both $\textnormal{ad ker}(cl,x)$ and $\textnormal{ad ker}(cl,x_g)$ are found in $O(|cl|)$ time. 
The cones $\textnormal{ad ker}(cl,x)$ and $\textnormal{ad ker}(cl,x_g)$ can be approximated as polygons with a maximum of 7 vertices by applying a bounding box of arbitrarily large height and width.
The time to find the intersection of $\textnormal{ad ker}(cl,x)$ and $\textnormal{ad ker}(cl,x_g)$ is thus problem independent. Hence, the total time to find $\textnormal{ad ker}(cl,\bar{X})$ is $O(|cl|)$ and the total computation time for all clusters is $\sum_{cl\in Cl} O(|cl|) = O(|\mathcal{O}|)$ in one iteration.

\subsubsection{Starshaped hull}
In Algorithm \ref{alg:starify_obstacles}, a starshaped obstacle corresponding to a cluster of $\lvert cl\rvert$ original obstacles can be represented as a set of $2\lvert cl\rvert$ overlapping convex shapes. Specifically, according to \eqref{eq:star_hull_specified_convex} and Property \ref{p:st_hull_ker}\ref{p:st_hull_ker_union}, each original obstacle, in addition to the obstacle shape, contributes with the convex hull of the 3 kernel points and their respective tangent points.
Finding $SH_{\textnormal{ker}K}(\mathcal{O}_i)$ for an obstacle, $\mathcal{O}_i$, hence amounts to computing the 6 tangent points, which is upper bounded by $O(3N)$ time, followed by the convex hull of 9 points, which is independent of the problem variables (i.e. has complexity $O(1)$). The complexity for finding $SH_{\textnormal{ker}K}(\mathcal{O}_i)$ is thus upper bounded by $O(N)$ and for one cluster of obstacles $O(|cl|N)$. The upper bound in one iteration for all clusters is then $\sum_{cl\in Cl} O(|cl|N) = O(|\mathcal{O}|N)$.  


\subsubsection{Clustering}
Identifying intersections of two obstacles are made using polygons (polygon approximations for ellipses) which is upper bounded to be done in $O(N)$ time \cite{gosh_84}. As the number of subshapes is $2|\mathcal{O}|$, comparing all subshapes to identify intersections is done in $O(|\mathcal{O}|^2N)$ time for each iteration.

}

\newpage
\bibliographystyle{ieeetr}
\bibliography{references}

\end{document}

%% file: figures/obstacle_representation.pgf
\begingroup%
\makeatletter%
\begin{pgfpicture}%
\pgfpathrectangle{\pgfpointorigin}{\pgfqpoint{1.674053in}{1.255540in}}%
\pgfusepath{use as bounding box, clip}%
\begin{pgfscope}%
\pgfsetbuttcap%
\pgfsetmiterjoin%
\definecolor{currentfill}{rgb}{1.000000,1.000000,1.000000}%
\pgfsetfillcolor{currentfill}%
\pgfsetlinewidth{0.000000pt}%
\definecolor{currentstroke}{rgb}{1.000000,1.000000,1.000000}%
\pgfsetstrokecolor{currentstroke}%
\pgfsetdash{}{0pt}%
\pgfpathmoveto{\pgfqpoint{0.000000in}{0.000000in}}%
\pgfpathlineto{\pgfqpoint{1.674053in}{0.000000in}}%
\pgfpathlineto{\pgfqpoint{1.674053in}{1.255540in}}%
\pgfpathlineto{\pgfqpoint{0.000000in}{1.255540in}}%
\pgfpathclose%
\pgfusepath{fill}%
\end{pgfscope}%
\begin{pgfscope}%
\pgfpathrectangle{\pgfqpoint{0.039236in}{0.039236in}}{\pgfqpoint{1.595582in}{1.177068in}}%
\pgfusepath{clip}%
\pgfsetbuttcap%
\pgfsetmiterjoin%
\definecolor{currentfill}{rgb}{0.827451,0.827451,0.827451}%
\pgfsetfillcolor{currentfill}%
\pgfsetfillopacity{0.800000}%
\pgfsetlinewidth{1.003750pt}%
\definecolor{currentstroke}{rgb}{0.000000,0.000000,0.000000}%
\pgfsetstrokecolor{currentstroke}%
\pgfsetstrokeopacity{0.800000}%
\pgfsetdash{}{0pt}%
\pgfpathmoveto{\pgfqpoint{1.025380in}{0.340583in}}%
\pgfpathlineto{\pgfqpoint{1.022984in}{0.323683in}}%
\pgfpathlineto{\pgfqpoint{1.019001in}{0.306916in}}%
\pgfpathlineto{\pgfqpoint{1.013446in}{0.290349in}}%
\pgfpathlineto{\pgfqpoint{1.006342in}{0.274048in}}%
\pgfpathlineto{\pgfqpoint{0.997716in}{0.258076in}}%
\pgfpathlineto{\pgfqpoint{0.987604in}{0.242496in}}%
\pgfpathlineto{\pgfqpoint{0.976044in}{0.227371in}}%
\pgfpathlineto{\pgfqpoint{0.963082in}{0.212760in}}%
\pgfpathlineto{\pgfqpoint{0.948769in}{0.198719in}}%
\pgfpathlineto{\pgfqpoint{0.933162in}{0.185306in}}%
\pgfpathlineto{\pgfqpoint{0.916323in}{0.172572in}}%
\pgfpathlineto{\pgfqpoint{0.898318in}{0.160569in}}%
\pgfpathlineto{\pgfqpoint{0.879217in}{0.149343in}}%
\pgfpathlineto{\pgfqpoint{0.859097in}{0.138938in}}%
\pgfpathlineto{\pgfqpoint{0.838037in}{0.129397in}}%
\pgfpathlineto{\pgfqpoint{0.816120in}{0.120755in}}%
\pgfpathlineto{\pgfqpoint{0.793432in}{0.113048in}}%
\pgfpathlineto{\pgfqpoint{0.770062in}{0.106307in}}%
\pgfpathlineto{\pgfqpoint{0.746104in}{0.100556in}}%
\pgfpathlineto{\pgfqpoint{0.721652in}{0.095820in}}%
\pgfpathlineto{\pgfqpoint{0.696802in}{0.092117in}}%
\pgfpathlineto{\pgfqpoint{0.671652in}{0.089462in}}%
\pgfpathlineto{\pgfqpoint{0.646302in}{0.087864in}}%
\pgfpathlineto{\pgfqpoint{0.620851in}{0.087331in}}%
\pgfpathlineto{\pgfqpoint{0.595400in}{0.087864in}}%
\pgfpathlineto{\pgfqpoint{0.570050in}{0.089462in}}%
\pgfpathlineto{\pgfqpoint{0.544900in}{0.092117in}}%
\pgfpathlineto{\pgfqpoint{0.520050in}{0.095820in}}%
\pgfpathlineto{\pgfqpoint{0.495597in}{0.100556in}}%
\pgfpathlineto{\pgfqpoint{0.471639in}{0.106307in}}%
\pgfpathlineto{\pgfqpoint{0.448270in}{0.113048in}}%
\pgfpathlineto{\pgfqpoint{0.425582in}{0.120755in}}%
\pgfpathlineto{\pgfqpoint{0.403665in}{0.129397in}}%
\pgfpathlineto{\pgfqpoint{0.382604in}{0.138938in}}%
\pgfpathlineto{\pgfqpoint{0.362484in}{0.149343in}}%
\pgfpathlineto{\pgfqpoint{0.343384in}{0.160569in}}%
\pgfpathlineto{\pgfqpoint{0.325379in}{0.172572in}}%
\pgfpathlineto{\pgfqpoint{0.308539in}{0.185306in}}%
\pgfpathlineto{\pgfqpoint{0.292933in}{0.198719in}}%
\pgfpathlineto{\pgfqpoint{0.278620in}{0.212760in}}%
\pgfpathlineto{\pgfqpoint{0.265658in}{0.227371in}}%
\pgfpathlineto{\pgfqpoint{0.254098in}{0.242496in}}%
\pgfpathlineto{\pgfqpoint{0.243985in}{0.258076in}}%
\pgfpathlineto{\pgfqpoint{0.235360in}{0.274048in}}%
\pgfpathlineto{\pgfqpoint{0.228256in}{0.290349in}}%
\pgfpathlineto{\pgfqpoint{0.222701in}{0.306916in}}%
\pgfpathlineto{\pgfqpoint{0.218718in}{0.323683in}}%
\pgfpathlineto{\pgfqpoint{0.216321in}{0.340583in}}%
\pgfpathlineto{\pgfqpoint{0.215522in}{0.357550in}}%
\pgfpathlineto{\pgfqpoint{0.216321in}{0.374518in}}%
\pgfpathlineto{\pgfqpoint{0.218718in}{0.391418in}}%
\pgfpathlineto{\pgfqpoint{0.222701in}{0.408184in}}%
\pgfpathlineto{\pgfqpoint{0.228256in}{0.424751in}}%
\pgfpathlineto{\pgfqpoint{0.235360in}{0.441053in}}%
\pgfpathlineto{\pgfqpoint{0.243985in}{0.457025in}}%
\pgfpathlineto{\pgfqpoint{0.254098in}{0.472604in}}%
\pgfpathlineto{\pgfqpoint{0.265658in}{0.487730in}}%
\pgfpathlineto{\pgfqpoint{0.278620in}{0.502341in}}%
\pgfpathlineto{\pgfqpoint{0.285704in}{0.509290in}}%
\pgfpathlineto{\pgfqpoint{0.285542in}{0.509372in}}%
\pgfpathlineto{\pgfqpoint{0.278236in}{0.513693in}}%
\pgfpathlineto{\pgfqpoint{0.271216in}{0.518464in}}%
\pgfpathlineto{\pgfqpoint{0.264509in}{0.523666in}}%
\pgfpathlineto{\pgfqpoint{0.258142in}{0.529279in}}%
\pgfpathlineto{\pgfqpoint{0.252141in}{0.535281in}}%
\pgfpathlineto{\pgfqpoint{0.246528in}{0.541648in}}%
\pgfpathlineto{\pgfqpoint{0.241325in}{0.548354in}}%
\pgfpathlineto{\pgfqpoint{0.236554in}{0.555374in}}%
\pgfpathlineto{\pgfqpoint{0.232234in}{0.562680in}}%
\pgfpathlineto{\pgfqpoint{0.228380in}{0.570243in}}%
\pgfpathlineto{\pgfqpoint{0.225010in}{0.578033in}}%
\pgfpathlineto{\pgfqpoint{0.222134in}{0.586019in}}%
\pgfpathlineto{\pgfqpoint{0.219766in}{0.594169in}}%
\pgfpathlineto{\pgfqpoint{0.217915in}{0.602453in}}%
\pgfpathlineto{\pgfqpoint{0.216587in}{0.610836in}}%
\pgfpathlineto{\pgfqpoint{0.215788in}{0.619286in}}%
\pgfpathlineto{\pgfqpoint{0.215522in}{0.627770in}}%
\pgfpathlineto{\pgfqpoint{0.215788in}{0.636253in}}%
\pgfpathlineto{\pgfqpoint{0.216587in}{0.644704in}}%
\pgfpathlineto{\pgfqpoint{0.217915in}{0.653087in}}%
\pgfpathlineto{\pgfqpoint{0.219766in}{0.661370in}}%
\pgfpathlineto{\pgfqpoint{0.222134in}{0.669521in}}%
\pgfpathlineto{\pgfqpoint{0.225010in}{0.677507in}}%
\pgfpathlineto{\pgfqpoint{0.228380in}{0.685297in}}%
\pgfpathlineto{\pgfqpoint{0.232234in}{0.692859in}}%
\pgfpathlineto{\pgfqpoint{0.236554in}{0.700165in}}%
\pgfpathlineto{\pgfqpoint{0.241325in}{0.707185in}}%
\pgfpathlineto{\pgfqpoint{0.246528in}{0.713892in}}%
\pgfpathlineto{\pgfqpoint{0.252141in}{0.720259in}}%
\pgfpathlineto{\pgfqpoint{0.258142in}{0.726261in}}%
\pgfpathlineto{\pgfqpoint{0.264509in}{0.731874in}}%
\pgfpathlineto{\pgfqpoint{0.271216in}{0.737076in}}%
\pgfpathlineto{\pgfqpoint{0.278236in}{0.741847in}}%
\pgfpathlineto{\pgfqpoint{0.285542in}{0.746167in}}%
\pgfpathlineto{\pgfqpoint{0.293104in}{0.750021in}}%
\pgfpathlineto{\pgfqpoint{0.300894in}{0.753392in}}%
\pgfpathlineto{\pgfqpoint{0.308880in}{0.756267in}}%
\pgfpathlineto{\pgfqpoint{0.317031in}{0.758635in}}%
\pgfpathlineto{\pgfqpoint{0.325314in}{0.760486in}}%
\pgfpathlineto{\pgfqpoint{0.333698in}{0.761814in}}%
\pgfpathlineto{\pgfqpoint{0.342148in}{0.762613in}}%
\pgfpathlineto{\pgfqpoint{0.350631in}{0.762880in}}%
\pgfpathlineto{\pgfqpoint{0.359115in}{0.762613in}}%
\pgfpathlineto{\pgfqpoint{0.367565in}{0.761814in}}%
\pgfpathlineto{\pgfqpoint{0.375948in}{0.760486in}}%
\pgfpathlineto{\pgfqpoint{0.377653in}{0.760105in}}%
\pgfpathlineto{\pgfqpoint{0.377653in}{1.168209in}}%
\pgfpathlineto{\pgfqpoint{0.918092in}{0.897989in}}%
\pgfpathlineto{\pgfqpoint{1.458531in}{1.168209in}}%
\pgfpathlineto{\pgfqpoint{1.458531in}{0.357550in}}%
\pgfpathlineto{\pgfqpoint{1.026180in}{0.357550in}}%
\pgfpathclose%
\pgfusepath{stroke,fill}%
\end{pgfscope}%
\begin{pgfscope}%
\pgfpathrectangle{\pgfqpoint{0.039236in}{0.039236in}}{\pgfqpoint{1.595582in}{1.177068in}}%
\pgfusepath{clip}%
\pgfsetbuttcap%
\pgfsetmiterjoin%
\definecolor{currentfill}{rgb}{0.254902,0.411765,0.882353}%
\pgfsetfillcolor{currentfill}%
\pgfsetfillopacity{0.600000}%
\pgfsetlinewidth{1.003750pt}%
\definecolor{currentstroke}{rgb}{0.000000,0.000000,0.000000}%
\pgfsetstrokecolor{currentstroke}%
\pgfsetstrokeopacity{0.600000}%
\pgfsetdash{{1.000000pt}{1.650000pt}}{0.000000pt}%
\pgfpathmoveto{\pgfqpoint{1.038196in}{0.612456in}}%
\pgfpathlineto{\pgfqpoint{1.038196in}{0.612456in}}%
\pgfpathlineto{\pgfqpoint{0.560557in}{0.719222in}}%
\pgfpathlineto{\pgfqpoint{0.436451in}{0.657169in}}%
\pgfpathlineto{\pgfqpoint{0.377653in}{0.599490in}}%
\pgfpathlineto{\pgfqpoint{0.377653in}{0.462439in}}%
\pgfpathlineto{\pgfqpoint{0.583509in}{0.357550in}}%
\pgfpathlineto{\pgfqpoint{1.026180in}{0.357550in}}%
\pgfpathclose%
\pgfusepath{stroke,fill}%
\end{pgfscope}%
\end{pgfpicture}%
\makeatother%
\endgroup%

%% file: figures/hull_kernel_intersecting.pgf
\begingroup%
\makeatletter%
\begin{pgfpicture}%
\pgfpathrectangle{\pgfpointorigin}{\pgfqpoint{1.674053in}{1.255540in}}%
\pgfusepath{use as bounding box, clip}%
\begin{pgfscope}%
\pgfsetbuttcap%
\pgfsetmiterjoin%
\definecolor{currentfill}{rgb}{1.000000,1.000000,1.000000}%
\pgfsetfillcolor{currentfill}%
\pgfsetlinewidth{0.000000pt}%
\definecolor{currentstroke}{rgb}{1.000000,1.000000,1.000000}%
\pgfsetstrokecolor{currentstroke}%
\pgfsetdash{}{0pt}%
\pgfpathmoveto{\pgfqpoint{0.000000in}{0.000000in}}%
\pgfpathlineto{\pgfqpoint{1.674053in}{0.000000in}}%
\pgfpathlineto{\pgfqpoint{1.674053in}{1.255540in}}%
\pgfpathlineto{\pgfqpoint{0.000000in}{1.255540in}}%
\pgfpathclose%
\pgfusepath{fill}%
\end{pgfscope}%
\begin{pgfscope}%
\pgfpathrectangle{\pgfqpoint{0.039236in}{0.039236in}}{\pgfqpoint{1.595582in}{1.177068in}}%
\pgfusepath{clip}%
\pgfsetbuttcap%
\pgfsetmiterjoin%
\definecolor{currentfill}{rgb}{0.827451,0.827451,0.827451}%
\pgfsetfillcolor{currentfill}%
\pgfsetfillopacity{0.800000}%
\pgfsetlinewidth{1.003750pt}%
\definecolor{currentstroke}{rgb}{0.000000,0.000000,0.000000}%
\pgfsetstrokecolor{currentstroke}%
\pgfsetstrokeopacity{0.800000}%
\pgfsetdash{}{0pt}%
\pgfpathmoveto{\pgfqpoint{0.674985in}{0.087331in}}%
\pgfpathlineto{\pgfqpoint{1.076402in}{0.087331in}}%
\pgfpathlineto{\pgfqpoint{1.076402in}{0.488748in}}%
\pgfpathlineto{\pgfqpoint{1.477820in}{0.488748in}}%
\pgfpathlineto{\pgfqpoint{1.477820in}{0.890165in}}%
\pgfpathlineto{\pgfqpoint{0.674985in}{0.890165in}}%
\pgfpathclose%
\pgfusepath{stroke,fill}%
\end{pgfscope}%
\begin{pgfscope}%
\pgfpathrectangle{\pgfqpoint{0.039236in}{0.039236in}}{\pgfqpoint{1.595582in}{1.177068in}}%
\pgfusepath{clip}%
\pgfsetbuttcap%
\pgfsetmiterjoin%
\definecolor{currentfill}{rgb}{0.827451,0.827451,0.827451}%
\pgfsetfillcolor{currentfill}%
\pgfsetfillopacity{0.800000}%
\pgfsetlinewidth{1.003750pt}%
\definecolor{currentstroke}{rgb}{0.000000,0.000000,0.000000}%
\pgfsetstrokecolor{currentstroke}%
\pgfsetstrokeopacity{0.800000}%
\pgfsetdash{}{0pt}%
\pgfpathmoveto{\pgfqpoint{0.729737in}{0.634705in}}%
\pgfpathcurveto{\pgfqpoint{0.752320in}{0.657288in}}{\pgfqpoint{0.747064in}{0.705866in}}{\pgfqpoint{0.715127in}{0.769740in}}%
\pgfpathcurveto{\pgfqpoint{0.683190in}{0.833614in}}{\pgfqpoint{0.627179in}{0.907570in}}{\pgfqpoint{0.559430in}{0.975319in}}%
\pgfpathcurveto{\pgfqpoint{0.491681in}{1.043068in}}{\pgfqpoint{0.417726in}{1.099079in}}{\pgfqpoint{0.353851in}{1.131016in}}%
\pgfpathcurveto{\pgfqpoint{0.289977in}{1.162953in}}{\pgfqpoint{0.241399in}{1.168209in}}{\pgfqpoint{0.218816in}{1.145626in}}%
\pgfpathcurveto{\pgfqpoint{0.196233in}{1.123043in}}{\pgfqpoint{0.201489in}{1.074465in}}{\pgfqpoint{0.233426in}{1.010591in}}%
\pgfpathcurveto{\pgfqpoint{0.265363in}{0.946716in}}{\pgfqpoint{0.321374in}{0.872761in}}{\pgfqpoint{0.389123in}{0.805012in}}%
\pgfpathcurveto{\pgfqpoint{0.456872in}{0.737263in}}{\pgfqpoint{0.530827in}{0.681252in}}{\pgfqpoint{0.594702in}{0.649315in}}%
\pgfpathcurveto{\pgfqpoint{0.658576in}{0.617378in}}{\pgfqpoint{0.707154in}{0.612122in}}{\pgfqpoint{0.729737in}{0.634705in}}%
\pgfpathclose%
\pgfusepath{stroke,fill}%
\end{pgfscope}%
\begin{pgfscope}%
\pgfpathrectangle{\pgfqpoint{0.039236in}{0.039236in}}{\pgfqpoint{1.595582in}{1.177068in}}%
\pgfusepath{clip}%
\pgfsetbuttcap%
\pgfsetmiterjoin%
\definecolor{currentfill}{rgb}{0.827451,0.827451,0.827451}%
\pgfsetfillcolor{currentfill}%
\pgfsetfillopacity{0.800000}%
\pgfsetlinewidth{1.003750pt}%
\definecolor{currentstroke}{rgb}{0.000000,0.000000,0.000000}%
\pgfsetstrokecolor{currentstroke}%
\pgfsetstrokeopacity{0.800000}%
\pgfsetdash{}{0pt}%
\pgfpathmoveto{\pgfqpoint{0.218816in}{0.233288in}}%
\pgfpathcurveto{\pgfqpoint{0.241399in}{0.210705in}}{\pgfqpoint{0.289977in}{0.215961in}}{\pgfqpoint{0.353851in}{0.247898in}}%
\pgfpathcurveto{\pgfqpoint{0.417726in}{0.279835in}}{\pgfqpoint{0.491681in}{0.335846in}}{\pgfqpoint{0.559430in}{0.403595in}}%
\pgfpathcurveto{\pgfqpoint{0.627179in}{0.471344in}}{\pgfqpoint{0.683190in}{0.545299in}}{\pgfqpoint{0.715127in}{0.609173in}}%
\pgfpathcurveto{\pgfqpoint{0.747064in}{0.673048in}}{\pgfqpoint{0.752320in}{0.721626in}}{\pgfqpoint{0.729737in}{0.744209in}}%
\pgfpathcurveto{\pgfqpoint{0.707154in}{0.766791in}}{\pgfqpoint{0.658576in}{0.761536in}}{\pgfqpoint{0.594702in}{0.729598in}}%
\pgfpathcurveto{\pgfqpoint{0.530827in}{0.697661in}}{\pgfqpoint{0.456872in}{0.641650in}}{\pgfqpoint{0.389123in}{0.573902in}}%
\pgfpathcurveto{\pgfqpoint{0.321374in}{0.506153in}}{\pgfqpoint{0.265363in}{0.432197in}}{\pgfqpoint{0.233426in}{0.368323in}}%
\pgfpathcurveto{\pgfqpoint{0.201489in}{0.304449in}}{\pgfqpoint{0.196233in}{0.255871in}}{\pgfqpoint{0.218816in}{0.233288in}}%
\pgfpathclose%
\pgfusepath{stroke,fill}%
\end{pgfscope}%
\begin{pgfscope}%
\pgfpathrectangle{\pgfqpoint{0.039236in}{0.039236in}}{\pgfqpoint{1.595582in}{1.177068in}}%
\pgfusepath{clip}%
\pgfsetbuttcap%
\pgfsetmiterjoin%
\definecolor{currentfill}{rgb}{0.254902,0.411765,0.882353}%
\pgfsetfillcolor{currentfill}%
\pgfsetfillopacity{0.600000}%
\pgfsetlinewidth{1.003750pt}%
\definecolor{currentstroke}{rgb}{0.000000,0.000000,0.000000}%
\pgfsetstrokecolor{currentstroke}%
\pgfsetstrokeopacity{0.600000}%
\pgfsetdash{{1.000000pt}{1.650000pt}}{0.000000pt}%
\pgfpathmoveto{\pgfqpoint{0.674985in}{0.756907in}}%
\pgfpathlineto{\pgfqpoint{0.680450in}{0.757616in}}%
\pgfpathlineto{\pgfqpoint{0.690920in}{0.758019in}}%
\pgfpathlineto{\pgfqpoint{0.700534in}{0.757360in}}%
\pgfpathlineto{\pgfqpoint{0.709256in}{0.755640in}}%
\pgfpathlineto{\pgfqpoint{0.717050in}{0.752867in}}%
\pgfpathlineto{\pgfqpoint{0.723886in}{0.749051in}}%
\pgfpathlineto{\pgfqpoint{0.725249in}{0.747923in}}%
\pgfpathlineto{\pgfqpoint{0.725391in}{0.747609in}}%
\pgfpathlineto{\pgfqpoint{0.731000in}{0.733546in}}%
\pgfpathlineto{\pgfqpoint{0.735596in}{0.720101in}}%
\pgfpathlineto{\pgfqpoint{0.739161in}{0.707327in}}%
\pgfpathlineto{\pgfqpoint{0.741681in}{0.695275in}}%
\pgfpathlineto{\pgfqpoint{0.742435in}{0.689457in}}%
\pgfpathlineto{\pgfqpoint{0.741681in}{0.683639in}}%
\pgfpathlineto{\pgfqpoint{0.739161in}{0.671587in}}%
\pgfpathlineto{\pgfqpoint{0.735596in}{0.658813in}}%
\pgfpathlineto{\pgfqpoint{0.731000in}{0.645368in}}%
\pgfpathlineto{\pgfqpoint{0.725391in}{0.631305in}}%
\pgfpathlineto{\pgfqpoint{0.725249in}{0.630991in}}%
\pgfpathlineto{\pgfqpoint{0.723886in}{0.629862in}}%
\pgfpathlineto{\pgfqpoint{0.717050in}{0.626047in}}%
\pgfpathlineto{\pgfqpoint{0.709256in}{0.623274in}}%
\pgfpathlineto{\pgfqpoint{0.700534in}{0.621554in}}%
\pgfpathlineto{\pgfqpoint{0.690920in}{0.620894in}}%
\pgfpathlineto{\pgfqpoint{0.680450in}{0.621297in}}%
\pgfpathlineto{\pgfqpoint{0.674985in}{0.622006in}}%
\pgfpathclose%
\pgfusepath{stroke,fill}%
\end{pgfscope}%
\end{pgfpicture}%
\makeatother%
\endgroup%

%% file: figures/hull_kernel_non_intersecting.pgf
\begingroup%
\makeatletter%
\begin{pgfpicture}%
\pgfpathrectangle{\pgfpointorigin}{\pgfqpoint{1.674053in}{1.255540in}}%
\pgfusepath{use as bounding box, clip}%
\begin{pgfscope}%
\pgfsetbuttcap%
\pgfsetmiterjoin%
\definecolor{currentfill}{rgb}{1.000000,1.000000,1.000000}%
\pgfsetfillcolor{currentfill}%
\pgfsetlinewidth{0.000000pt}%
\definecolor{currentstroke}{rgb}{1.000000,1.000000,1.000000}%
\pgfsetstrokecolor{currentstroke}%
\pgfsetdash{}{0pt}%
\pgfpathmoveto{\pgfqpoint{0.000000in}{0.000000in}}%
\pgfpathlineto{\pgfqpoint{1.674053in}{0.000000in}}%
\pgfpathlineto{\pgfqpoint{1.674053in}{1.255540in}}%
\pgfpathlineto{\pgfqpoint{0.000000in}{1.255540in}}%
\pgfpathclose%
\pgfusepath{fill}%
\end{pgfscope}%
\begin{pgfscope}%
\pgfpathrectangle{\pgfqpoint{0.039236in}{0.039236in}}{\pgfqpoint{1.595582in}{1.177068in}}%
\pgfusepath{clip}%
\pgfsetbuttcap%
\pgfsetmiterjoin%
\definecolor{currentfill}{rgb}{0.827451,0.827451,0.827451}%
\pgfsetfillcolor{currentfill}%
\pgfsetfillopacity{0.800000}%
\pgfsetlinewidth{1.003750pt}%
\definecolor{currentstroke}{rgb}{0.000000,0.000000,0.000000}%
\pgfsetstrokecolor{currentstroke}%
\pgfsetstrokeopacity{0.800000}%
\pgfsetdash{}{0pt}%
\pgfpathmoveto{\pgfqpoint{0.436701in}{0.367558in}}%
\pgfpathlineto{\pgfqpoint{0.837026in}{0.367558in}}%
\pgfpathlineto{\pgfqpoint{0.837026in}{0.767884in}}%
\pgfpathlineto{\pgfqpoint{1.237352in}{0.767884in}}%
\pgfpathlineto{\pgfqpoint{1.237352in}{1.168209in}}%
\pgfpathlineto{\pgfqpoint{0.436701in}{1.168209in}}%
\pgfpathclose%
\pgfusepath{stroke,fill}%
\end{pgfscope}%
\begin{pgfscope}%
\pgfpathrectangle{\pgfqpoint{0.039236in}{0.039236in}}{\pgfqpoint{1.595582in}{1.177068in}}%
\pgfusepath{clip}%
\pgfsetbuttcap%
\pgfsetmiterjoin%
\definecolor{currentfill}{rgb}{0.827451,0.827451,0.827451}%
\pgfsetfillcolor{currentfill}%
\pgfsetfillopacity{0.800000}%
\pgfsetlinewidth{1.003750pt}%
\definecolor{currentstroke}{rgb}{0.000000,0.000000,0.000000}%
\pgfsetstrokecolor{currentstroke}%
\pgfsetstrokeopacity{0.800000}%
\pgfsetdash{}{0pt}%
\pgfpathmoveto{\pgfqpoint{0.636864in}{0.087331in}}%
\pgfpathcurveto{\pgfqpoint{0.689948in}{0.087331in}}{\pgfqpoint{0.740864in}{0.116857in}}{\pgfqpoint{0.778400in}{0.169408in}}%
\pgfpathcurveto{\pgfqpoint{0.815936in}{0.221958in}}{\pgfqpoint{0.837026in}{0.293241in}}{\pgfqpoint{0.837026in}{0.367558in}}%
\pgfpathcurveto{\pgfqpoint{0.837026in}{0.441876in}}{\pgfqpoint{0.815936in}{0.513159in}}{\pgfqpoint{0.778400in}{0.565709in}}%
\pgfpathcurveto{\pgfqpoint{0.740864in}{0.618260in}}{\pgfqpoint{0.689948in}{0.647786in}}{\pgfqpoint{0.636864in}{0.647786in}}%
\pgfpathcurveto{\pgfqpoint{0.583780in}{0.647786in}}{\pgfqpoint{0.532863in}{0.618260in}}{\pgfqpoint{0.495328in}{0.565709in}}%
\pgfpathcurveto{\pgfqpoint{0.457792in}{0.513159in}}{\pgfqpoint{0.436701in}{0.441876in}}{\pgfqpoint{0.436701in}{0.367558in}}%
\pgfpathcurveto{\pgfqpoint{0.436701in}{0.293241in}}{\pgfqpoint{0.457792in}{0.221958in}}{\pgfqpoint{0.495328in}{0.169408in}}%
\pgfpathcurveto{\pgfqpoint{0.532863in}{0.116857in}}{\pgfqpoint{0.583780in}{0.087331in}}{\pgfqpoint{0.636864in}{0.087331in}}%
\pgfpathclose%
\pgfusepath{stroke,fill}%
\end{pgfscope}%
\begin{pgfscope}%
\pgfpathrectangle{\pgfqpoint{0.039236in}{0.039236in}}{\pgfqpoint{1.595582in}{1.177068in}}%
\pgfusepath{clip}%
\pgfsetbuttcap%
\pgfsetmiterjoin%
\definecolor{currentfill}{rgb}{0.254902,0.411765,0.882353}%
\pgfsetfillcolor{currentfill}%
\pgfsetfillopacity{0.600000}%
\pgfsetlinewidth{1.003750pt}%
\definecolor{currentstroke}{rgb}{0.000000,0.000000,0.000000}%
\pgfsetstrokecolor{currentstroke}%
\pgfsetstrokeopacity{0.600000}%
\pgfsetdash{{1.000000pt}{1.650000pt}}{0.000000pt}%
\pgfpathmoveto{\pgfqpoint{0.837026in}{0.767884in}}%
\pgfpathlineto{\pgfqpoint{0.837026in}{1.168209in}}%
\pgfpathlineto{\pgfqpoint{0.436701in}{1.168209in}}%
\pgfpathlineto{\pgfqpoint{0.436701in}{0.767884in}}%
\pgfpathclose%
\pgfusepath{stroke,fill}%
\end{pgfscope}%
\begin{pgfscope}%
\pgfpathrectangle{\pgfqpoint{0.039236in}{0.039236in}}{\pgfqpoint{1.595582in}{1.177068in}}%
\pgfusepath{clip}%
\pgfsetbuttcap%
\pgfsetmiterjoin%
\pgfsetlinewidth{0.000000pt}%
\definecolor{currentstroke}{rgb}{0.000000,0.000000,0.000000}%
\pgfsetstrokecolor{currentstroke}%
\pgfsetdash{}{0pt}%
\pgfpathmoveto{\pgfqpoint{0.436701in}{0.367558in}}%
\pgfpathlineto{\pgfqpoint{0.437096in}{0.385154in}}%
\pgfpathlineto{\pgfqpoint{0.438280in}{0.402680in}}%
\pgfpathlineto{\pgfqpoint{0.440247in}{0.420068in}}%
\pgfpathlineto{\pgfqpoint{0.442990in}{0.437248in}}%
\pgfpathlineto{\pgfqpoint{0.446498in}{0.454154in}}%
\pgfpathlineto{\pgfqpoint{0.450757in}{0.470717in}}%
\pgfpathlineto{\pgfqpoint{0.455751in}{0.486874in}}%
\pgfpathlineto{\pgfqpoint{0.461460in}{0.502559in}}%
\pgfpathlineto{\pgfqpoint{0.467861in}{0.517712in}}%
\pgfpathlineto{\pgfqpoint{0.474929in}{0.532272in}}%
\pgfpathlineto{\pgfqpoint{0.482636in}{0.546182in}}%
\pgfpathlineto{\pgfqpoint{0.490952in}{0.559387in}}%
\pgfpathlineto{\pgfqpoint{0.499843in}{0.571836in}}%
\pgfpathlineto{\pgfqpoint{0.509275in}{0.583478in}}%
\pgfpathlineto{\pgfqpoint{0.519211in}{0.594267in}}%
\pgfpathlineto{\pgfqpoint{0.529611in}{0.604162in}}%
\pgfpathlineto{\pgfqpoint{0.540435in}{0.613124in}}%
\pgfpathlineto{\pgfqpoint{0.551639in}{0.621116in}}%
\pgfpathlineto{\pgfqpoint{0.563179in}{0.628107in}}%
\pgfpathlineto{\pgfqpoint{0.575010in}{0.634071in}}%
\pgfpathlineto{\pgfqpoint{0.587085in}{0.638982in}}%
\pgfpathlineto{\pgfqpoint{0.599357in}{0.642822in}}%
\pgfpathlineto{\pgfqpoint{0.611777in}{0.645576in}}%
\pgfpathlineto{\pgfqpoint{0.624296in}{0.647233in}}%
\pgfpathlineto{\pgfqpoint{0.636864in}{0.647786in}}%
\pgfpathlineto{\pgfqpoint{0.649432in}{0.647233in}}%
\pgfpathlineto{\pgfqpoint{0.661951in}{0.645576in}}%
\pgfpathlineto{\pgfqpoint{0.674371in}{0.642822in}}%
\pgfpathlineto{\pgfqpoint{0.686642in}{0.638982in}}%
\pgfpathlineto{\pgfqpoint{0.698717in}{0.634071in}}%
\pgfpathlineto{\pgfqpoint{0.710549in}{0.628107in}}%
\pgfpathlineto{\pgfqpoint{0.722089in}{0.621116in}}%
\pgfpathlineto{\pgfqpoint{0.733293in}{0.613124in}}%
\pgfpathlineto{\pgfqpoint{0.744116in}{0.604162in}}%
\pgfpathlineto{\pgfqpoint{0.754516in}{0.594267in}}%
\pgfpathlineto{\pgfqpoint{0.764452in}{0.583478in}}%
\pgfpathlineto{\pgfqpoint{0.773885in}{0.571836in}}%
\pgfpathlineto{\pgfqpoint{0.782776in}{0.559387in}}%
\pgfpathlineto{\pgfqpoint{0.791092in}{0.546182in}}%
\pgfpathlineto{\pgfqpoint{0.798799in}{0.532272in}}%
\pgfpathlineto{\pgfqpoint{0.805867in}{0.517712in}}%
\pgfpathlineto{\pgfqpoint{0.812268in}{0.502559in}}%
\pgfpathlineto{\pgfqpoint{0.817976in}{0.486874in}}%
\pgfpathlineto{\pgfqpoint{0.822970in}{0.470717in}}%
\pgfpathlineto{\pgfqpoint{0.827230in}{0.454154in}}%
\pgfpathlineto{\pgfqpoint{0.830738in}{0.437248in}}%
\pgfpathlineto{\pgfqpoint{0.833481in}{0.420068in}}%
\pgfpathlineto{\pgfqpoint{0.835448in}{0.402680in}}%
\pgfpathlineto{\pgfqpoint{0.836631in}{0.385154in}}%
\pgfpathlineto{\pgfqpoint{0.837026in}{0.367558in}}%
\pgfpathclose%
\pgfusepath{}%
\end{pgfscope}%
\begin{pgfscope}%
\pgfsetbuttcap%
\pgfsetmiterjoin%
\pgfsetlinewidth{0.000000pt}%
\definecolor{currentstroke}{rgb}{0.000000,0.000000,0.000000}%
\pgfsetstrokecolor{currentstroke}%
\pgfsetdash{}{0pt}%
\pgfpathrectangle{\pgfqpoint{0.039236in}{0.039236in}}{\pgfqpoint{1.595582in}{1.177068in}}%
\pgfusepath{clip}%
\pgfpathmoveto{\pgfqpoint{0.436701in}{0.367558in}}%
\pgfpathlineto{\pgfqpoint{0.437096in}{0.385154in}}%
\pgfpathlineto{\pgfqpoint{0.438280in}{0.402680in}}%
\pgfpathlineto{\pgfqpoint{0.440247in}{0.420068in}}%
\pgfpathlineto{\pgfqpoint{0.442990in}{0.437248in}}%
\pgfpathlineto{\pgfqpoint{0.446498in}{0.454154in}}%
\pgfpathlineto{\pgfqpoint{0.450757in}{0.470717in}}%
\pgfpathlineto{\pgfqpoint{0.455751in}{0.486874in}}%
\pgfpathlineto{\pgfqpoint{0.461460in}{0.502559in}}%
\pgfpathlineto{\pgfqpoint{0.467861in}{0.517712in}}%
\pgfpathlineto{\pgfqpoint{0.474929in}{0.532272in}}%
\pgfpathlineto{\pgfqpoint{0.482636in}{0.546182in}}%
\pgfpathlineto{\pgfqpoint{0.490952in}{0.559387in}}%
\pgfpathlineto{\pgfqpoint{0.499843in}{0.571836in}}%
\pgfpathlineto{\pgfqpoint{0.509275in}{0.583478in}}%
\pgfpathlineto{\pgfqpoint{0.519211in}{0.594267in}}%
\pgfpathlineto{\pgfqpoint{0.529611in}{0.604162in}}%
\pgfpathlineto{\pgfqpoint{0.540435in}{0.613124in}}%
\pgfpathlineto{\pgfqpoint{0.551639in}{0.621116in}}%
\pgfpathlineto{\pgfqpoint{0.563179in}{0.628107in}}%
\pgfpathlineto{\pgfqpoint{0.575010in}{0.634071in}}%
\pgfpathlineto{\pgfqpoint{0.587085in}{0.638982in}}%
\pgfpathlineto{\pgfqpoint{0.599357in}{0.642822in}}%
\pgfpathlineto{\pgfqpoint{0.611777in}{0.645576in}}%
\pgfpathlineto{\pgfqpoint{0.624296in}{0.647233in}}%
\pgfpathlineto{\pgfqpoint{0.636864in}{0.647786in}}%
\pgfpathlineto{\pgfqpoint{0.649432in}{0.647233in}}%
\pgfpathlineto{\pgfqpoint{0.661951in}{0.645576in}}%
\pgfpathlineto{\pgfqpoint{0.674371in}{0.642822in}}%
\pgfpathlineto{\pgfqpoint{0.686642in}{0.638982in}}%
\pgfpathlineto{\pgfqpoint{0.698717in}{0.634071in}}%
\pgfpathlineto{\pgfqpoint{0.710549in}{0.628107in}}%
\pgfpathlineto{\pgfqpoint{0.722089in}{0.621116in}}%
\pgfpathlineto{\pgfqpoint{0.733293in}{0.613124in}}%
\pgfpathlineto{\pgfqpoint{0.744116in}{0.604162in}}%
\pgfpathlineto{\pgfqpoint{0.754516in}{0.594267in}}%
\pgfpathlineto{\pgfqpoint{0.764452in}{0.583478in}}%
\pgfpathlineto{\pgfqpoint{0.773885in}{0.571836in}}%
\pgfpathlineto{\pgfqpoint{0.782776in}{0.559387in}}%
\pgfpathlineto{\pgfqpoint{0.791092in}{0.546182in}}%
\pgfpathlineto{\pgfqpoint{0.798799in}{0.532272in}}%
\pgfpathlineto{\pgfqpoint{0.805867in}{0.517712in}}%
\pgfpathlineto{\pgfqpoint{0.812268in}{0.502559in}}%
\pgfpathlineto{\pgfqpoint{0.817976in}{0.486874in}}%
\pgfpathlineto{\pgfqpoint{0.822970in}{0.470717in}}%
\pgfpathlineto{\pgfqpoint{0.827230in}{0.454154in}}%
\pgfpathlineto{\pgfqpoint{0.830738in}{0.437248in}}%
\pgfpathlineto{\pgfqpoint{0.833481in}{0.420068in}}%
\pgfpathlineto{\pgfqpoint{0.835448in}{0.402680in}}%
\pgfpathlineto{\pgfqpoint{0.836631in}{0.385154in}}%
\pgfpathlineto{\pgfqpoint{0.837026in}{0.367558in}}%
\pgfpathclose%
\pgfusepath{clip}%
\pgfsys@defobject{currentpattern}{\pgfqpoint{0in}{0in}}{\pgfqpoint{1in}{1in}}{%
\begin{pgfscope}%
\pgfpathrectangle{\pgfqpoint{0in}{0in}}{\pgfqpoint{1in}{1in}}%
\pgfusepath{clip}%
\pgfpathmoveto{\pgfqpoint{-0.500000in}{0.500000in}}%
\pgfpathlineto{\pgfqpoint{0.500000in}{1.500000in}}%
\pgfpathmoveto{\pgfqpoint{-0.466667in}{0.466667in}}%
\pgfpathlineto{\pgfqpoint{0.533333in}{1.466667in}}%
\pgfpathmoveto{\pgfqpoint{-0.433333in}{0.433333in}}%
\pgfpathlineto{\pgfqpoint{0.566667in}{1.433333in}}%
\pgfpathmoveto{\pgfqpoint{-0.400000in}{0.400000in}}%
\pgfpathlineto{\pgfqpoint{0.600000in}{1.400000in}}%
\pgfpathmoveto{\pgfqpoint{-0.366667in}{0.366667in}}%
\pgfpathlineto{\pgfqpoint{0.633333in}{1.366667in}}%
\pgfpathmoveto{\pgfqpoint{-0.333333in}{0.333333in}}%
\pgfpathlineto{\pgfqpoint{0.666667in}{1.333333in}}%
\pgfpathmoveto{\pgfqpoint{-0.300000in}{0.300000in}}%
\pgfpathlineto{\pgfqpoint{0.700000in}{1.300000in}}%
\pgfpathmoveto{\pgfqpoint{-0.266667in}{0.266667in}}%
\pgfpathlineto{\pgfqpoint{0.733333in}{1.266667in}}%
\pgfpathmoveto{\pgfqpoint{-0.233333in}{0.233333in}}%
\pgfpathlineto{\pgfqpoint{0.766667in}{1.233333in}}%
\pgfpathmoveto{\pgfqpoint{-0.200000in}{0.200000in}}%
\pgfpathlineto{\pgfqpoint{0.800000in}{1.200000in}}%
\pgfpathmoveto{\pgfqpoint{-0.166667in}{0.166667in}}%
\pgfpathlineto{\pgfqpoint{0.833333in}{1.166667in}}%
\pgfpathmoveto{\pgfqpoint{-0.133333in}{0.133333in}}%
\pgfpathlineto{\pgfqpoint{0.866667in}{1.133333in}}%
\pgfpathmoveto{\pgfqpoint{-0.100000in}{0.100000in}}%
\pgfpathlineto{\pgfqpoint{0.900000in}{1.100000in}}%
\pgfpathmoveto{\pgfqpoint{-0.066667in}{0.066667in}}%
\pgfpathlineto{\pgfqpoint{0.933333in}{1.066667in}}%
\pgfpathmoveto{\pgfqpoint{-0.033333in}{0.033333in}}%
\pgfpathlineto{\pgfqpoint{0.966667in}{1.033333in}}%
\pgfpathmoveto{\pgfqpoint{0.000000in}{0.000000in}}%
\pgfpathlineto{\pgfqpoint{1.000000in}{1.000000in}}%
\pgfpathmoveto{\pgfqpoint{0.033333in}{-0.033333in}}%
\pgfpathlineto{\pgfqpoint{1.033333in}{0.966667in}}%
\pgfpathmoveto{\pgfqpoint{0.066667in}{-0.066667in}}%
\pgfpathlineto{\pgfqpoint{1.066667in}{0.933333in}}%
\pgfpathmoveto{\pgfqpoint{0.100000in}{-0.100000in}}%
\pgfpathlineto{\pgfqpoint{1.100000in}{0.900000in}}%
\pgfpathmoveto{\pgfqpoint{0.133333in}{-0.133333in}}%
\pgfpathlineto{\pgfqpoint{1.133333in}{0.866667in}}%
\pgfpathmoveto{\pgfqpoint{0.166667in}{-0.166667in}}%
\pgfpathlineto{\pgfqpoint{1.166667in}{0.833333in}}%
\pgfpathmoveto{\pgfqpoint{0.200000in}{-0.200000in}}%
\pgfpathlineto{\pgfqpoint{1.200000in}{0.800000in}}%
\pgfpathmoveto{\pgfqpoint{0.233333in}{-0.233333in}}%
\pgfpathlineto{\pgfqpoint{1.233333in}{0.766667in}}%
\pgfpathmoveto{\pgfqpoint{0.266667in}{-0.266667in}}%
\pgfpathlineto{\pgfqpoint{1.266667in}{0.733333in}}%
\pgfpathmoveto{\pgfqpoint{0.300000in}{-0.300000in}}%
\pgfpathlineto{\pgfqpoint{1.300000in}{0.700000in}}%
\pgfpathmoveto{\pgfqpoint{0.333333in}{-0.333333in}}%
\pgfpathlineto{\pgfqpoint{1.333333in}{0.666667in}}%
\pgfpathmoveto{\pgfqpoint{0.366667in}{-0.366667in}}%
\pgfpathlineto{\pgfqpoint{1.366667in}{0.633333in}}%
\pgfpathmoveto{\pgfqpoint{0.400000in}{-0.400000in}}%
\pgfpathlineto{\pgfqpoint{1.400000in}{0.600000in}}%
\pgfpathmoveto{\pgfqpoint{0.433333in}{-0.433333in}}%
\pgfpathlineto{\pgfqpoint{1.433333in}{0.566667in}}%
\pgfpathmoveto{\pgfqpoint{0.466667in}{-0.466667in}}%
\pgfpathlineto{\pgfqpoint{1.466667in}{0.533333in}}%
\pgfpathmoveto{\pgfqpoint{0.500000in}{-0.500000in}}%
\pgfpathlineto{\pgfqpoint{1.500000in}{0.500000in}}%
\pgfusepath{stroke}%
\end{pgfscope}%
}%
\pgfsys@transformshift{0.436701in}{0.367558in}%
\pgfsys@useobject{currentpattern}{}%
\pgfsys@transformshift{1in}{0in}%
\pgfsys@transformshift{-1in}{0in}%
\pgfsys@transformshift{0in}{1in}%
\end{pgfscope}%
\begin{pgfscope}%
\definecolor{textcolor}{rgb}{0.000000,0.000000,0.000000}%
\pgfsetstrokecolor{textcolor}%
\pgfsetfillcolor{textcolor}%
\pgftext[x=0.957124in,y=0.928014in,left,base]{\color{textcolor}\rmfamily\fontsize{10.000000}{12.000000}\selectfont \(\displaystyle A_1\)}%
\end{pgfscope}%
\begin{pgfscope}%
\definecolor{textcolor}{rgb}{0.000000,0.000000,0.000000}%
\pgfsetstrokecolor{textcolor}%
\pgfsetfillcolor{textcolor}%
\pgftext[x=0.576815in,y=0.227445in,left,base]{\color{textcolor}\rmfamily\fontsize{10.000000}{12.000000}\selectfont \(\displaystyle A_2\)}%
\end{pgfscope}%
\end{pgfpicture}%
\makeatother%
\endgroup%

%% file: figures/hull_ellipse.pgf
\begingroup%
\makeatletter%
\begin{pgfpicture}%
\pgfpathrectangle{\pgfpointorigin}{\pgfqpoint{1.674053in}{1.255540in}}%
\pgfusepath{use as bounding box, clip}%
\begin{pgfscope}%
\pgfsetbuttcap%
\pgfsetmiterjoin%
\definecolor{currentfill}{rgb}{1.000000,1.000000,1.000000}%
\pgfsetfillcolor{currentfill}%
\pgfsetlinewidth{0.000000pt}%
\definecolor{currentstroke}{rgb}{1.000000,1.000000,1.000000}%
\pgfsetstrokecolor{currentstroke}%
\pgfsetdash{}{0pt}%
\pgfpathmoveto{\pgfqpoint{0.000000in}{0.000000in}}%
\pgfpathlineto{\pgfqpoint{1.674053in}{0.000000in}}%
\pgfpathlineto{\pgfqpoint{1.674053in}{1.255540in}}%
\pgfpathlineto{\pgfqpoint{0.000000in}{1.255540in}}%
\pgfpathclose%
\pgfusepath{fill}%
\end{pgfscope}%
\begin{pgfscope}%
\pgfpathrectangle{\pgfqpoint{0.039236in}{0.039236in}}{\pgfqpoint{1.595582in}{1.177068in}}%
\pgfusepath{clip}%
\pgfsetbuttcap%
\pgfsetmiterjoin%
\definecolor{currentfill}{rgb}{0.827451,0.827451,0.827451}%
\pgfsetfillcolor{currentfill}%
\pgfsetfillopacity{0.800000}%
\pgfsetlinewidth{1.003750pt}%
\definecolor{currentstroke}{rgb}{0.000000,0.000000,0.000000}%
\pgfsetstrokecolor{currentstroke}%
\pgfsetstrokeopacity{0.800000}%
\pgfsetdash{}{0pt}%
\pgfpathmoveto{\pgfqpoint{0.691968in}{0.337657in}}%
\pgfpathcurveto{\pgfqpoint{0.845840in}{0.337657in}}{\pgfqpoint{0.993431in}{0.368224in}}{\pgfqpoint{1.102235in}{0.422626in}}%
\pgfpathcurveto{\pgfqpoint{1.211039in}{0.477029in}}{\pgfqpoint{1.272173in}{0.550824in}}{\pgfqpoint{1.272173in}{0.627760in}}%
\pgfpathcurveto{\pgfqpoint{1.272173in}{0.704696in}}{\pgfqpoint{1.211039in}{0.778492in}}{\pgfqpoint{1.102235in}{0.832894in}}%
\pgfpathcurveto{\pgfqpoint{0.993431in}{0.887296in}}{\pgfqpoint{0.845840in}{0.917863in}}{\pgfqpoint{0.691968in}{0.917863in}}%
\pgfpathcurveto{\pgfqpoint{0.538095in}{0.917863in}}{\pgfqpoint{0.390505in}{0.887296in}}{\pgfqpoint{0.281700in}{0.832894in}}%
\pgfpathcurveto{\pgfqpoint{0.172896in}{0.778492in}}{\pgfqpoint{0.111762in}{0.704696in}}{\pgfqpoint{0.111762in}{0.627760in}}%
\pgfpathcurveto{\pgfqpoint{0.111762in}{0.550824in}}{\pgfqpoint{0.172896in}{0.477029in}}{\pgfqpoint{0.281700in}{0.422626in}}%
\pgfpathcurveto{\pgfqpoint{0.390505in}{0.368224in}}{\pgfqpoint{0.538095in}{0.337657in}}{\pgfqpoint{0.691968in}{0.337657in}}%
\pgfpathclose%
\pgfusepath{stroke,fill}%
\end{pgfscope}%
\begin{pgfscope}%
\pgfpathrectangle{\pgfqpoint{0.039236in}{0.039236in}}{\pgfqpoint{1.595582in}{1.177068in}}%
\pgfusepath{clip}%
\pgfsetbuttcap%
\pgfsetmiterjoin%
\definecolor{currentfill}{rgb}{0.000000,0.501961,0.000000}%
\pgfsetfillcolor{currentfill}%
\pgfsetfillopacity{0.600000}%
\pgfsetlinewidth{0.000000pt}%
\definecolor{currentstroke}{rgb}{0.000000,0.000000,0.000000}%
\pgfsetstrokecolor{currentstroke}%
\pgfsetstrokeopacity{0.600000}%
\pgfsetdash{}{0pt}%
\pgfpathmoveto{\pgfqpoint{0.691942in}{0.917863in}}%
\pgfpathlineto{\pgfqpoint{1.227537in}{0.516176in}}%
\pgfpathlineto{\pgfqpoint{1.562291in}{0.917853in}}%
\pgfpathlineto{\pgfqpoint{1.562276in}{0.917882in}}%
\pgfpathclose%
\pgfusepath{fill}%
\end{pgfscope}%
\begin{pgfscope}%
\pgfpathrectangle{\pgfqpoint{0.039236in}{0.039236in}}{\pgfqpoint{1.595582in}{1.177068in}}%
\pgfusepath{clip}%
\pgfsetrectcap%
\pgfsetroundjoin%
\pgfsetlinewidth{1.505625pt}%
\definecolor{currentstroke}{rgb}{0.121569,0.466667,0.705882}%
\pgfsetstrokecolor{currentstroke}%
\pgfsetdash{}{0pt}%
\pgfpathmoveto{\pgfqpoint{0.691968in}{0.627760in}}%
\pgfusepath{stroke}%
\end{pgfscope}%
\begin{pgfscope}%
\pgfpathrectangle{\pgfqpoint{0.039236in}{0.039236in}}{\pgfqpoint{1.595582in}{1.177068in}}%
\pgfusepath{clip}%
\pgfsetrectcap%
\pgfsetroundjoin%
\pgfsetlinewidth{1.505625pt}%
\definecolor{currentstroke}{rgb}{1.000000,0.498039,0.054902}%
\pgfsetstrokecolor{currentstroke}%
\pgfsetdash{}{0pt}%
\pgfpathmoveto{\pgfqpoint{1.160598in}{0.783974in}}%
\pgfusepath{stroke}%
\end{pgfscope}%
\begin{pgfscope}%
\pgfpathrectangle{\pgfqpoint{0.039236in}{0.039236in}}{\pgfqpoint{1.595582in}{1.177068in}}%
\pgfusepath{clip}%
\pgfsetrectcap%
\pgfsetroundjoin%
\pgfsetlinewidth{1.505625pt}%
\definecolor{currentstroke}{rgb}{0.000000,0.000000,0.000000}%
\pgfsetstrokecolor{currentstroke}%
\pgfsetdash{}{0pt}%
\pgfpathmoveto{\pgfqpoint{1.562276in}{0.917863in}}%
\pgfusepath{stroke}%
\end{pgfscope}%
\begin{pgfscope}%
\pgfpathrectangle{\pgfqpoint{0.039236in}{0.039236in}}{\pgfqpoint{1.595582in}{1.177068in}}%
\pgfusepath{clip}%
\pgfsetbuttcap%
\pgfsetroundjoin%
\definecolor{currentfill}{rgb}{0.000000,0.000000,0.000000}%
\pgfsetfillcolor{currentfill}%
\pgfsetlinewidth{1.003750pt}%
\definecolor{currentstroke}{rgb}{0.000000,0.000000,0.000000}%
\pgfsetstrokecolor{currentstroke}%
\pgfsetdash{}{0pt}%
\pgfsys@defobject{currentmarker}{\pgfqpoint{-0.016667in}{-0.016667in}}{\pgfqpoint{0.016667in}{0.016667in}}{%
\pgfpathmoveto{\pgfqpoint{0.000000in}{-0.016667in}}%
\pgfpathcurveto{\pgfqpoint{0.004420in}{-0.016667in}}{\pgfqpoint{0.008660in}{-0.014911in}}{\pgfqpoint{0.011785in}{-0.011785in}}%
\pgfpathcurveto{\pgfqpoint{0.014911in}{-0.008660in}}{\pgfqpoint{0.016667in}{-0.004420in}}{\pgfqpoint{0.016667in}{0.000000in}}%
\pgfpathcurveto{\pgfqpoint{0.016667in}{0.004420in}}{\pgfqpoint{0.014911in}{0.008660in}}{\pgfqpoint{0.011785in}{0.011785in}}%
\pgfpathcurveto{\pgfqpoint{0.008660in}{0.014911in}}{\pgfqpoint{0.004420in}{0.016667in}}{\pgfqpoint{0.000000in}{0.016667in}}%
\pgfpathcurveto{\pgfqpoint{-0.004420in}{0.016667in}}{\pgfqpoint{-0.008660in}{0.014911in}}{\pgfqpoint{-0.011785in}{0.011785in}}%
\pgfpathcurveto{\pgfqpoint{-0.014911in}{0.008660in}}{\pgfqpoint{-0.016667in}{0.004420in}}{\pgfqpoint{-0.016667in}{0.000000in}}%
\pgfpathcurveto{\pgfqpoint{-0.016667in}{-0.004420in}}{\pgfqpoint{-0.014911in}{-0.008660in}}{\pgfqpoint{-0.011785in}{-0.011785in}}%
\pgfpathcurveto{\pgfqpoint{-0.008660in}{-0.014911in}}{\pgfqpoint{-0.004420in}{-0.016667in}}{\pgfqpoint{0.000000in}{-0.016667in}}%
\pgfpathclose%
\pgfusepath{stroke,fill}%
}%
\begin{pgfscope}%
\pgfsys@transformshift{1.562276in}{0.917863in}%
\pgfsys@useobject{currentmarker}{}%
\end{pgfscope}%
\end{pgfscope}%
\end{pgfpicture}%
\makeatother%
\endgroup%

%% file: figures/singleton_kernel.pgf
\begingroup%
\makeatletter%
\begin{pgfpicture}%
\pgfpathrectangle{\pgfpointorigin}{\pgfqpoint{1.674053in}{1.255540in}}%
\pgfusepath{use as bounding box, clip}%
\begin{pgfscope}%
\pgfsetbuttcap%
\pgfsetmiterjoin%
\definecolor{currentfill}{rgb}{1.000000,1.000000,1.000000}%
\pgfsetfillcolor{currentfill}%
\pgfsetlinewidth{0.000000pt}%
\definecolor{currentstroke}{rgb}{1.000000,1.000000,1.000000}%
\pgfsetstrokecolor{currentstroke}%
\pgfsetdash{}{0pt}%
\pgfpathmoveto{\pgfqpoint{0.000000in}{0.000000in}}%
\pgfpathlineto{\pgfqpoint{1.674053in}{0.000000in}}%
\pgfpathlineto{\pgfqpoint{1.674053in}{1.255540in}}%
\pgfpathlineto{\pgfqpoint{0.000000in}{1.255540in}}%
\pgfpathclose%
\pgfusepath{fill}%
\end{pgfscope}%
\begin{pgfscope}%
\pgfpathrectangle{\pgfqpoint{0.039236in}{0.039236in}}{\pgfqpoint{1.595582in}{1.177068in}}%
\pgfusepath{clip}%
\pgfsetbuttcap%
\pgfsetmiterjoin%
\definecolor{currentfill}{rgb}{0.827451,0.827451,0.827451}%
\pgfsetfillcolor{currentfill}%
\pgfsetfillopacity{0.800000}%
\pgfsetlinewidth{1.003750pt}%
\definecolor{currentstroke}{rgb}{0.000000,0.000000,0.000000}%
\pgfsetstrokecolor{currentstroke}%
\pgfsetstrokeopacity{0.800000}%
\pgfsetdash{}{0pt}%
\pgfpathmoveto{\pgfqpoint{0.111762in}{0.083822in}}%
\pgfpathlineto{\pgfqpoint{0.474394in}{0.083822in}}%
\pgfpathlineto{\pgfqpoint{0.474394in}{0.446454in}}%
\pgfpathlineto{\pgfqpoint{1.199659in}{0.446454in}}%
\pgfpathlineto{\pgfqpoint{1.199659in}{0.083822in}}%
\pgfpathlineto{\pgfqpoint{1.562291in}{0.083822in}}%
\pgfpathlineto{\pgfqpoint{1.562291in}{1.171718in}}%
\pgfpathlineto{\pgfqpoint{1.199659in}{1.171718in}}%
\pgfpathlineto{\pgfqpoint{1.199659in}{0.809086in}}%
\pgfpathlineto{\pgfqpoint{0.474394in}{0.809086in}}%
\pgfpathlineto{\pgfqpoint{0.474394in}{1.171718in}}%
\pgfpathlineto{\pgfqpoint{0.111762in}{1.171718in}}%
\pgfpathclose%
\pgfusepath{stroke,fill}%
\end{pgfscope}%
\begin{pgfscope}%
\pgfpathrectangle{\pgfqpoint{0.039236in}{0.039236in}}{\pgfqpoint{1.595582in}{1.177068in}}%
\pgfusepath{clip}%
\pgfsetbuttcap%
\pgfsetmiterjoin%
\definecolor{currentfill}{rgb}{0.000000,0.501961,0.000000}%
\pgfsetfillcolor{currentfill}%
\pgfsetfillopacity{0.600000}%
\pgfsetlinewidth{0.000000pt}%
\definecolor{currentstroke}{rgb}{0.000000,0.000000,0.000000}%
\pgfsetstrokecolor{currentstroke}%
\pgfsetstrokeopacity{0.600000}%
\pgfsetdash{}{0pt}%
\pgfpathmoveto{\pgfqpoint{1.199659in}{1.171718in}}%
\pgfpathlineto{\pgfqpoint{1.199659in}{0.809086in}}%
\pgfpathlineto{\pgfqpoint{0.957904in}{0.809086in}}%
\pgfpathclose%
\pgfusepath{fill}%
\end{pgfscope}%
\begin{pgfscope}%
\pgfpathrectangle{\pgfqpoint{0.039236in}{0.039236in}}{\pgfqpoint{1.595582in}{1.177068in}}%
\pgfusepath{clip}%
\pgfsetbuttcap%
\pgfsetmiterjoin%
\definecolor{currentfill}{rgb}{0.000000,0.501961,0.000000}%
\pgfsetfillcolor{currentfill}%
\pgfsetfillopacity{0.600000}%
\pgfsetlinewidth{0.000000pt}%
\definecolor{currentstroke}{rgb}{0.000000,0.000000,0.000000}%
\pgfsetstrokecolor{currentstroke}%
\pgfsetstrokeopacity{0.600000}%
\pgfsetdash{}{0pt}%
\pgfpathmoveto{\pgfqpoint{0.716149in}{0.809086in}}%
\pgfpathlineto{\pgfqpoint{0.474394in}{0.809086in}}%
\pgfpathlineto{\pgfqpoint{0.474394in}{1.171718in}}%
\pgfpathclose%
\pgfusepath{fill}%
\end{pgfscope}%
\begin{pgfscope}%
\pgfpathrectangle{\pgfqpoint{0.039236in}{0.039236in}}{\pgfqpoint{1.595582in}{1.177068in}}%
\pgfusepath{clip}%
\pgfsetbuttcap%
\pgfsetmiterjoin%
\definecolor{currentfill}{rgb}{0.000000,0.501961,0.000000}%
\pgfsetfillcolor{currentfill}%
\pgfsetfillopacity{0.600000}%
\pgfsetlinewidth{0.000000pt}%
\definecolor{currentstroke}{rgb}{0.000000,0.000000,0.000000}%
\pgfsetstrokecolor{currentstroke}%
\pgfsetstrokeopacity{0.600000}%
\pgfsetdash{}{0pt}%
\pgfpathmoveto{\pgfqpoint{0.474394in}{0.083822in}}%
\pgfpathlineto{\pgfqpoint{0.474394in}{0.446454in}}%
\pgfpathlineto{\pgfqpoint{0.716149in}{0.446454in}}%
\pgfpathclose%
\pgfusepath{fill}%
\end{pgfscope}%
\begin{pgfscope}%
\pgfpathrectangle{\pgfqpoint{0.039236in}{0.039236in}}{\pgfqpoint{1.595582in}{1.177068in}}%
\pgfusepath{clip}%
\pgfsetbuttcap%
\pgfsetmiterjoin%
\definecolor{currentfill}{rgb}{0.000000,0.501961,0.000000}%
\pgfsetfillcolor{currentfill}%
\pgfsetfillopacity{0.600000}%
\pgfsetlinewidth{0.000000pt}%
\definecolor{currentstroke}{rgb}{0.000000,0.000000,0.000000}%
\pgfsetstrokecolor{currentstroke}%
\pgfsetstrokeopacity{0.600000}%
\pgfsetdash{}{0pt}%
\pgfpathmoveto{\pgfqpoint{0.957904in}{0.446454in}}%
\pgfpathlineto{\pgfqpoint{1.199659in}{0.446454in}}%
\pgfpathlineto{\pgfqpoint{1.199659in}{0.083822in}}%
\pgfpathclose%
\pgfusepath{fill}%
\end{pgfscope}%
\begin{pgfscope}%
\pgfpathrectangle{\pgfqpoint{0.039236in}{0.039236in}}{\pgfqpoint{1.595582in}{1.177068in}}%
\pgfusepath{clip}%
\pgfsetrectcap%
\pgfsetroundjoin%
\pgfsetlinewidth{1.505625pt}%
\definecolor{currentstroke}{rgb}{0.121569,0.466667,0.705882}%
\pgfsetstrokecolor{currentstroke}%
\pgfsetdash{}{0pt}%
\pgfpathmoveto{\pgfqpoint{0.837026in}{0.627770in}}%
\pgfusepath{stroke}%
\end{pgfscope}%
\begin{pgfscope}%
\pgfpathrectangle{\pgfqpoint{0.039236in}{0.039236in}}{\pgfqpoint{1.595582in}{1.177068in}}%
\pgfusepath{clip}%
\pgfsetbuttcap%
\pgfsetroundjoin%
\pgfsetlinewidth{1.505625pt}%
\definecolor{currentstroke}{rgb}{0.000000,0.000000,0.000000}%
\pgfsetstrokecolor{currentstroke}%
\pgfsetdash{{1.500000pt}{2.475000pt}}{0.000000pt}%
\pgfpathmoveto{\pgfqpoint{0.837026in}{0.627770in}}%
\pgfpathlineto{\pgfqpoint{0.716149in}{0.446454in}}%
\pgfusepath{stroke}%
\end{pgfscope}%
\begin{pgfscope}%
\pgfpathrectangle{\pgfqpoint{0.039236in}{0.039236in}}{\pgfqpoint{1.595582in}{1.177068in}}%
\pgfusepath{clip}%
\pgfsetbuttcap%
\pgfsetroundjoin%
\pgfsetlinewidth{1.505625pt}%
\definecolor{currentstroke}{rgb}{1.000000,0.000000,0.000000}%
\pgfsetstrokecolor{currentstroke}%
\pgfsetdash{{5.550000pt}{2.400000pt}}{0.000000pt}%
\pgfpathmoveto{\pgfqpoint{0.716149in}{0.446454in}}%
\pgfpathlineto{\pgfqpoint{0.474394in}{0.083822in}}%
\pgfusepath{stroke}%
\end{pgfscope}%
\begin{pgfscope}%
\pgfpathrectangle{\pgfqpoint{0.039236in}{0.039236in}}{\pgfqpoint{1.595582in}{1.177068in}}%
\pgfusepath{clip}%
\pgfsetbuttcap%
\pgfsetroundjoin%
\pgfsetlinewidth{1.505625pt}%
\definecolor{currentstroke}{rgb}{0.000000,0.000000,0.000000}%
\pgfsetstrokecolor{currentstroke}%
\pgfsetdash{{1.500000pt}{2.475000pt}}{0.000000pt}%
\pgfpathmoveto{\pgfqpoint{0.837026in}{0.627770in}}%
\pgfpathlineto{\pgfqpoint{0.957904in}{0.446454in}}%
\pgfusepath{stroke}%
\end{pgfscope}%
\begin{pgfscope}%
\pgfpathrectangle{\pgfqpoint{0.039236in}{0.039236in}}{\pgfqpoint{1.595582in}{1.177068in}}%
\pgfusepath{clip}%
\pgfsetbuttcap%
\pgfsetroundjoin%
\pgfsetlinewidth{1.505625pt}%
\definecolor{currentstroke}{rgb}{1.000000,0.000000,0.000000}%
\pgfsetstrokecolor{currentstroke}%
\pgfsetdash{{5.550000pt}{2.400000pt}}{0.000000pt}%
\pgfpathmoveto{\pgfqpoint{0.957904in}{0.446454in}}%
\pgfpathlineto{\pgfqpoint{1.199659in}{0.083822in}}%
\pgfusepath{stroke}%
\end{pgfscope}%
\begin{pgfscope}%
\pgfpathrectangle{\pgfqpoint{0.039236in}{0.039236in}}{\pgfqpoint{1.595582in}{1.177068in}}%
\pgfusepath{clip}%
\pgfsetbuttcap%
\pgfsetroundjoin%
\pgfsetlinewidth{1.505625pt}%
\definecolor{currentstroke}{rgb}{0.000000,0.000000,0.000000}%
\pgfsetstrokecolor{currentstroke}%
\pgfsetdash{{1.500000pt}{2.475000pt}}{0.000000pt}%
\pgfpathmoveto{\pgfqpoint{0.837026in}{0.627770in}}%
\pgfpathlineto{\pgfqpoint{0.957904in}{0.809086in}}%
\pgfusepath{stroke}%
\end{pgfscope}%
\begin{pgfscope}%
\pgfpathrectangle{\pgfqpoint{0.039236in}{0.039236in}}{\pgfqpoint{1.595582in}{1.177068in}}%
\pgfusepath{clip}%
\pgfsetbuttcap%
\pgfsetroundjoin%
\pgfsetlinewidth{1.505625pt}%
\definecolor{currentstroke}{rgb}{1.000000,0.000000,0.000000}%
\pgfsetstrokecolor{currentstroke}%
\pgfsetdash{{5.550000pt}{2.400000pt}}{0.000000pt}%
\pgfpathmoveto{\pgfqpoint{0.957904in}{0.809086in}}%
\pgfpathlineto{\pgfqpoint{1.199659in}{1.171718in}}%
\pgfusepath{stroke}%
\end{pgfscope}%
\begin{pgfscope}%
\pgfpathrectangle{\pgfqpoint{0.039236in}{0.039236in}}{\pgfqpoint{1.595582in}{1.177068in}}%
\pgfusepath{clip}%
\pgfsetbuttcap%
\pgfsetroundjoin%
\pgfsetlinewidth{1.505625pt}%
\definecolor{currentstroke}{rgb}{0.000000,0.000000,0.000000}%
\pgfsetstrokecolor{currentstroke}%
\pgfsetdash{{1.500000pt}{2.475000pt}}{0.000000pt}%
\pgfpathmoveto{\pgfqpoint{0.837026in}{0.627770in}}%
\pgfpathlineto{\pgfqpoint{0.716149in}{0.809086in}}%
\pgfusepath{stroke}%
\end{pgfscope}%
\begin{pgfscope}%
\pgfpathrectangle{\pgfqpoint{0.039236in}{0.039236in}}{\pgfqpoint{1.595582in}{1.177068in}}%
\pgfusepath{clip}%
\pgfsetbuttcap%
\pgfsetroundjoin%
\pgfsetlinewidth{1.505625pt}%
\definecolor{currentstroke}{rgb}{1.000000,0.000000,0.000000}%
\pgfsetstrokecolor{currentstroke}%
\pgfsetdash{{5.550000pt}{2.400000pt}}{0.000000pt}%
\pgfpathmoveto{\pgfqpoint{0.716149in}{0.809086in}}%
\pgfpathlineto{\pgfqpoint{0.474394in}{1.171718in}}%
\pgfusepath{stroke}%
\end{pgfscope}%
\begin{pgfscope}%
\pgfpathrectangle{\pgfqpoint{0.039236in}{0.039236in}}{\pgfqpoint{1.595582in}{1.177068in}}%
\pgfusepath{clip}%
\pgfsetrectcap%
\pgfsetroundjoin%
\pgfsetlinewidth{1.505625pt}%
\definecolor{currentstroke}{rgb}{0.000000,0.000000,0.000000}%
\pgfsetstrokecolor{currentstroke}%
\pgfsetdash{}{0pt}%
\pgfpathmoveto{\pgfqpoint{0.837026in}{0.627770in}}%
\pgfusepath{stroke}%
\end{pgfscope}%
\begin{pgfscope}%
\pgfpathrectangle{\pgfqpoint{0.039236in}{0.039236in}}{\pgfqpoint{1.595582in}{1.177068in}}%
\pgfusepath{clip}%
\pgfsetbuttcap%
\pgfsetroundjoin%
\definecolor{currentfill}{rgb}{0.000000,0.000000,0.000000}%
\pgfsetfillcolor{currentfill}%
\pgfsetlinewidth{1.003750pt}%
\definecolor{currentstroke}{rgb}{0.000000,0.000000,0.000000}%
\pgfsetstrokecolor{currentstroke}%
\pgfsetdash{}{0pt}%
\pgfsys@defobject{currentmarker}{\pgfqpoint{-0.016667in}{-0.016667in}}{\pgfqpoint{0.016667in}{0.016667in}}{%
\pgfpathmoveto{\pgfqpoint{0.000000in}{-0.016667in}}%
\pgfpathcurveto{\pgfqpoint{0.004420in}{-0.016667in}}{\pgfqpoint{0.008660in}{-0.014911in}}{\pgfqpoint{0.011785in}{-0.011785in}}%
\pgfpathcurveto{\pgfqpoint{0.014911in}{-0.008660in}}{\pgfqpoint{0.016667in}{-0.004420in}}{\pgfqpoint{0.016667in}{0.000000in}}%
\pgfpathcurveto{\pgfqpoint{0.016667in}{0.004420in}}{\pgfqpoint{0.014911in}{0.008660in}}{\pgfqpoint{0.011785in}{0.011785in}}%
\pgfpathcurveto{\pgfqpoint{0.008660in}{0.014911in}}{\pgfqpoint{0.004420in}{0.016667in}}{\pgfqpoint{0.000000in}{0.016667in}}%
\pgfpathcurveto{\pgfqpoint{-0.004420in}{0.016667in}}{\pgfqpoint{-0.008660in}{0.014911in}}{\pgfqpoint{-0.011785in}{0.011785in}}%
\pgfpathcurveto{\pgfqpoint{-0.014911in}{0.008660in}}{\pgfqpoint{-0.016667in}{0.004420in}}{\pgfqpoint{-0.016667in}{0.000000in}}%
\pgfpathcurveto{\pgfqpoint{-0.016667in}{-0.004420in}}{\pgfqpoint{-0.014911in}{-0.008660in}}{\pgfqpoint{-0.011785in}{-0.011785in}}%
\pgfpathcurveto{\pgfqpoint{-0.008660in}{-0.014911in}}{\pgfqpoint{-0.004420in}{-0.016667in}}{\pgfqpoint{0.000000in}{-0.016667in}}%
\pgfpathclose%
\pgfusepath{stroke,fill}%
}%
\begin{pgfscope}%
\pgfsys@transformshift{0.837026in}{0.627770in}%
\pgfsys@useobject{currentmarker}{}%
\end{pgfscope}%
\end{pgfscope}%
\end{pgfpicture}%
\makeatother%
\endgroup%

%% file: figures/triangle_kernel.pgf
\begingroup%
\makeatletter%
\begin{pgfpicture}%
\pgfpathrectangle{\pgfpointorigin}{\pgfqpoint{1.674053in}{1.255540in}}%
\pgfusepath{use as bounding box, clip}%
\begin{pgfscope}%
\pgfsetbuttcap%
\pgfsetmiterjoin%
\definecolor{currentfill}{rgb}{1.000000,1.000000,1.000000}%
\pgfsetfillcolor{currentfill}%
\pgfsetlinewidth{0.000000pt}%
\definecolor{currentstroke}{rgb}{1.000000,1.000000,1.000000}%
\pgfsetstrokecolor{currentstroke}%
\pgfsetdash{}{0pt}%
\pgfpathmoveto{\pgfqpoint{0.000000in}{0.000000in}}%
\pgfpathlineto{\pgfqpoint{1.674053in}{0.000000in}}%
\pgfpathlineto{\pgfqpoint{1.674053in}{1.255540in}}%
\pgfpathlineto{\pgfqpoint{0.000000in}{1.255540in}}%
\pgfpathclose%
\pgfusepath{fill}%
\end{pgfscope}%
\begin{pgfscope}%
\pgfpathrectangle{\pgfqpoint{0.039236in}{0.039236in}}{\pgfqpoint{1.595582in}{1.177068in}}%
\pgfusepath{clip}%
\pgfsetbuttcap%
\pgfsetmiterjoin%
\definecolor{currentfill}{rgb}{0.827451,0.827451,0.827451}%
\pgfsetfillcolor{currentfill}%
\pgfsetfillopacity{0.800000}%
\pgfsetlinewidth{1.003750pt}%
\definecolor{currentstroke}{rgb}{0.000000,0.000000,0.000000}%
\pgfsetstrokecolor{currentstroke}%
\pgfsetstrokeopacity{0.800000}%
\pgfsetdash{}{0pt}%
\pgfpathmoveto{\pgfqpoint{0.111762in}{0.083822in}}%
\pgfpathlineto{\pgfqpoint{0.474394in}{0.083822in}}%
\pgfpathlineto{\pgfqpoint{0.474394in}{0.446454in}}%
\pgfpathlineto{\pgfqpoint{1.199659in}{0.446454in}}%
\pgfpathlineto{\pgfqpoint{1.199659in}{0.083822in}}%
\pgfpathlineto{\pgfqpoint{1.562291in}{0.083822in}}%
\pgfpathlineto{\pgfqpoint{1.562291in}{1.171718in}}%
\pgfpathlineto{\pgfqpoint{1.199659in}{1.171718in}}%
\pgfpathlineto{\pgfqpoint{1.199659in}{0.809086in}}%
\pgfpathlineto{\pgfqpoint{0.474394in}{0.809086in}}%
\pgfpathlineto{\pgfqpoint{0.474394in}{1.171718in}}%
\pgfpathlineto{\pgfqpoint{0.111762in}{1.171718in}}%
\pgfpathclose%
\pgfusepath{stroke,fill}%
\end{pgfscope}%
\begin{pgfscope}%
\pgfpathrectangle{\pgfqpoint{0.039236in}{0.039236in}}{\pgfqpoint{1.595582in}{1.177068in}}%
\pgfusepath{clip}%
\pgfsetbuttcap%
\pgfsetmiterjoin%
\definecolor{currentfill}{rgb}{0.000000,0.501961,0.000000}%
\pgfsetfillcolor{currentfill}%
\pgfsetfillopacity{0.600000}%
\pgfsetlinewidth{0.000000pt}%
\definecolor{currentstroke}{rgb}{0.000000,0.000000,0.000000}%
\pgfsetstrokecolor{currentstroke}%
\pgfsetstrokeopacity{0.600000}%
\pgfsetdash{}{0pt}%
\pgfpathmoveto{\pgfqpoint{1.199659in}{1.171718in}}%
\pgfpathlineto{\pgfqpoint{1.199659in}{0.809086in}}%
\pgfpathlineto{\pgfqpoint{0.905632in}{0.809086in}}%
\pgfpathclose%
\pgfusepath{fill}%
\end{pgfscope}%
\begin{pgfscope}%
\pgfpathrectangle{\pgfqpoint{0.039236in}{0.039236in}}{\pgfqpoint{1.595582in}{1.177068in}}%
\pgfusepath{clip}%
\pgfsetbuttcap%
\pgfsetmiterjoin%
\definecolor{currentfill}{rgb}{0.000000,0.501961,0.000000}%
\pgfsetfillcolor{currentfill}%
\pgfsetfillopacity{0.600000}%
\pgfsetlinewidth{0.000000pt}%
\definecolor{currentstroke}{rgb}{0.000000,0.000000,0.000000}%
\pgfsetstrokecolor{currentstroke}%
\pgfsetstrokeopacity{0.600000}%
\pgfsetdash{}{0pt}%
\pgfpathmoveto{\pgfqpoint{0.768420in}{0.809086in}}%
\pgfpathlineto{\pgfqpoint{0.474394in}{0.809086in}}%
\pgfpathlineto{\pgfqpoint{0.474394in}{1.171718in}}%
\pgfpathclose%
\pgfusepath{fill}%
\end{pgfscope}%
\begin{pgfscope}%
\pgfpathrectangle{\pgfqpoint{0.039236in}{0.039236in}}{\pgfqpoint{1.595582in}{1.177068in}}%
\pgfusepath{clip}%
\pgfsetbuttcap%
\pgfsetmiterjoin%
\definecolor{currentfill}{rgb}{0.000000,0.501961,0.000000}%
\pgfsetfillcolor{currentfill}%
\pgfsetfillopacity{0.600000}%
\pgfsetlinewidth{0.000000pt}%
\definecolor{currentstroke}{rgb}{0.000000,0.000000,0.000000}%
\pgfsetstrokecolor{currentstroke}%
\pgfsetstrokeopacity{0.600000}%
\pgfsetdash{}{0pt}%
\pgfpathmoveto{\pgfqpoint{0.474394in}{0.083822in}}%
\pgfpathlineto{\pgfqpoint{0.474394in}{0.446454in}}%
\pgfpathlineto{\pgfqpoint{0.792803in}{0.446454in}}%
\pgfpathclose%
\pgfusepath{fill}%
\end{pgfscope}%
\begin{pgfscope}%
\pgfpathrectangle{\pgfqpoint{0.039236in}{0.039236in}}{\pgfqpoint{1.595582in}{1.177068in}}%
\pgfusepath{clip}%
\pgfsetbuttcap%
\pgfsetmiterjoin%
\definecolor{currentfill}{rgb}{0.000000,0.501961,0.000000}%
\pgfsetfillcolor{currentfill}%
\pgfsetfillopacity{0.600000}%
\pgfsetlinewidth{0.000000pt}%
\definecolor{currentstroke}{rgb}{0.000000,0.000000,0.000000}%
\pgfsetstrokecolor{currentstroke}%
\pgfsetstrokeopacity{0.600000}%
\pgfsetdash{}{0pt}%
\pgfpathmoveto{\pgfqpoint{0.881250in}{0.446454in}}%
\pgfpathlineto{\pgfqpoint{1.199659in}{0.446454in}}%
\pgfpathlineto{\pgfqpoint{1.199659in}{0.083822in}}%
\pgfpathclose%
\pgfusepath{fill}%
\end{pgfscope}%
\begin{pgfscope}%
\pgfpathrectangle{\pgfqpoint{0.039236in}{0.039236in}}{\pgfqpoint{1.595582in}{1.177068in}}%
\pgfusepath{clip}%
\pgfsetbuttcap%
\pgfsetmiterjoin%
\definecolor{currentfill}{rgb}{0.254902,0.411765,0.882353}%
\pgfsetfillcolor{currentfill}%
\pgfsetfillopacity{0.600000}%
\pgfsetlinewidth{1.003750pt}%
\definecolor{currentstroke}{rgb}{0.000000,0.000000,0.000000}%
\pgfsetstrokecolor{currentstroke}%
\pgfsetstrokeopacity{0.600000}%
\pgfsetdash{{1.000000pt}{1.650000pt}}{0.000000pt}%
\pgfpathmoveto{\pgfqpoint{0.837026in}{0.724472in}}%
\pgfpathlineto{\pgfqpoint{0.909553in}{0.579419in}}%
\pgfpathlineto{\pgfqpoint{0.764500in}{0.579419in}}%
\pgfpathclose%
\pgfusepath{stroke,fill}%
\end{pgfscope}%
\begin{pgfscope}%
\pgfpathrectangle{\pgfqpoint{0.039236in}{0.039236in}}{\pgfqpoint{1.595582in}{1.177068in}}%
\pgfusepath{clip}%
\pgfsetrectcap%
\pgfsetroundjoin%
\pgfsetlinewidth{1.505625pt}%
\definecolor{currentstroke}{rgb}{0.121569,0.466667,0.705882}%
\pgfsetstrokecolor{currentstroke}%
\pgfsetdash{}{0pt}%
\pgfpathmoveto{\pgfqpoint{0.837026in}{0.627770in}}%
\pgfusepath{stroke}%
\end{pgfscope}%
\begin{pgfscope}%
\pgfpathrectangle{\pgfqpoint{0.039236in}{0.039236in}}{\pgfqpoint{1.595582in}{1.177068in}}%
\pgfusepath{clip}%
\pgfsetrectcap%
\pgfsetroundjoin%
\pgfsetlinewidth{1.505625pt}%
\definecolor{currentstroke}{rgb}{0.000000,0.000000,0.000000}%
\pgfsetstrokecolor{currentstroke}%
\pgfsetdash{}{0pt}%
\pgfpathmoveto{\pgfqpoint{0.837026in}{0.627770in}}%
\pgfusepath{stroke}%
\end{pgfscope}%
\begin{pgfscope}%
\pgfpathrectangle{\pgfqpoint{0.039236in}{0.039236in}}{\pgfqpoint{1.595582in}{1.177068in}}%
\pgfusepath{clip}%
\pgfsetbuttcap%
\pgfsetroundjoin%
\definecolor{currentfill}{rgb}{0.000000,0.000000,0.000000}%
\pgfsetfillcolor{currentfill}%
\pgfsetlinewidth{1.003750pt}%
\definecolor{currentstroke}{rgb}{0.000000,0.000000,0.000000}%
\pgfsetstrokecolor{currentstroke}%
\pgfsetdash{}{0pt}%
\pgfsys@defobject{currentmarker}{\pgfqpoint{-0.016667in}{-0.016667in}}{\pgfqpoint{0.016667in}{0.016667in}}{%
\pgfpathmoveto{\pgfqpoint{0.000000in}{-0.016667in}}%
\pgfpathcurveto{\pgfqpoint{0.004420in}{-0.016667in}}{\pgfqpoint{0.008660in}{-0.014911in}}{\pgfqpoint{0.011785in}{-0.011785in}}%
\pgfpathcurveto{\pgfqpoint{0.014911in}{-0.008660in}}{\pgfqpoint{0.016667in}{-0.004420in}}{\pgfqpoint{0.016667in}{0.000000in}}%
\pgfpathcurveto{\pgfqpoint{0.016667in}{0.004420in}}{\pgfqpoint{0.014911in}{0.008660in}}{\pgfqpoint{0.011785in}{0.011785in}}%
\pgfpathcurveto{\pgfqpoint{0.008660in}{0.014911in}}{\pgfqpoint{0.004420in}{0.016667in}}{\pgfqpoint{0.000000in}{0.016667in}}%
\pgfpathcurveto{\pgfqpoint{-0.004420in}{0.016667in}}{\pgfqpoint{-0.008660in}{0.014911in}}{\pgfqpoint{-0.011785in}{0.011785in}}%
\pgfpathcurveto{\pgfqpoint{-0.014911in}{0.008660in}}{\pgfqpoint{-0.016667in}{0.004420in}}{\pgfqpoint{-0.016667in}{0.000000in}}%
\pgfpathcurveto{\pgfqpoint{-0.016667in}{-0.004420in}}{\pgfqpoint{-0.014911in}{-0.008660in}}{\pgfqpoint{-0.011785in}{-0.011785in}}%
\pgfpathcurveto{\pgfqpoint{-0.008660in}{-0.014911in}}{\pgfqpoint{-0.004420in}{-0.016667in}}{\pgfqpoint{0.000000in}{-0.016667in}}%
\pgfpathclose%
\pgfusepath{stroke,fill}%
}%
\begin{pgfscope}%
\pgfsys@transformshift{0.837026in}{0.627770in}%
\pgfsys@useobject{currentmarker}{}%
\end{pgfscope}%
\end{pgfscope}%
\begin{pgfscope}%
\pgfpathrectangle{\pgfqpoint{0.039236in}{0.039236in}}{\pgfqpoint{1.595582in}{1.177068in}}%
\pgfusepath{clip}%
\pgfsetrectcap%
\pgfsetroundjoin%
\pgfsetlinewidth{1.505625pt}%
\definecolor{currentstroke}{rgb}{1.000000,0.498039,0.054902}%
\pgfsetstrokecolor{currentstroke}%
\pgfsetdash{}{0pt}%
\pgfpathmoveto{\pgfqpoint{0.837026in}{0.627770in}}%
\pgfusepath{stroke}%
\end{pgfscope}%
\begin{pgfscope}%
\pgfpathrectangle{\pgfqpoint{0.039236in}{0.039236in}}{\pgfqpoint{1.595582in}{1.177068in}}%
\pgfusepath{clip}%
\pgfsetbuttcap%
\pgfsetroundjoin%
\pgfsetlinewidth{1.505625pt}%
\definecolor{currentstroke}{rgb}{0.000000,0.000000,0.000000}%
\pgfsetstrokecolor{currentstroke}%
\pgfsetdash{{1.500000pt}{2.475000pt}}{0.000000pt}%
\pgfpathmoveto{\pgfqpoint{0.837026in}{0.627770in}}%
\pgfpathlineto{\pgfqpoint{0.474394in}{0.083822in}}%
\pgfusepath{stroke}%
\end{pgfscope}%
\begin{pgfscope}%
\pgfpathrectangle{\pgfqpoint{0.039236in}{0.039236in}}{\pgfqpoint{1.595582in}{1.177068in}}%
\pgfusepath{clip}%
\pgfsetbuttcap%
\pgfsetroundjoin%
\pgfsetlinewidth{1.505625pt}%
\definecolor{currentstroke}{rgb}{0.000000,0.000000,0.000000}%
\pgfsetstrokecolor{currentstroke}%
\pgfsetdash{{1.500000pt}{2.475000pt}}{0.000000pt}%
\pgfpathmoveto{\pgfqpoint{0.837026in}{0.627770in}}%
\pgfpathlineto{\pgfqpoint{1.199659in}{0.083822in}}%
\pgfusepath{stroke}%
\end{pgfscope}%
\begin{pgfscope}%
\pgfpathrectangle{\pgfqpoint{0.039236in}{0.039236in}}{\pgfqpoint{1.595582in}{1.177068in}}%
\pgfusepath{clip}%
\pgfsetbuttcap%
\pgfsetroundjoin%
\pgfsetlinewidth{1.505625pt}%
\definecolor{currentstroke}{rgb}{0.000000,0.000000,0.000000}%
\pgfsetstrokecolor{currentstroke}%
\pgfsetdash{{1.500000pt}{2.475000pt}}{0.000000pt}%
\pgfpathmoveto{\pgfqpoint{0.837026in}{0.627770in}}%
\pgfpathlineto{\pgfqpoint{1.199659in}{1.171718in}}%
\pgfusepath{stroke}%
\end{pgfscope}%
\begin{pgfscope}%
\pgfpathrectangle{\pgfqpoint{0.039236in}{0.039236in}}{\pgfqpoint{1.595582in}{1.177068in}}%
\pgfusepath{clip}%
\pgfsetbuttcap%
\pgfsetroundjoin%
\pgfsetlinewidth{1.505625pt}%
\definecolor{currentstroke}{rgb}{0.000000,0.000000,0.000000}%
\pgfsetstrokecolor{currentstroke}%
\pgfsetdash{{1.500000pt}{2.475000pt}}{0.000000pt}%
\pgfpathmoveto{\pgfqpoint{0.837026in}{0.627770in}}%
\pgfpathlineto{\pgfqpoint{0.474394in}{1.171718in}}%
\pgfusepath{stroke}%
\end{pgfscope}%
\end{pgfpicture}%
\makeatother%
\endgroup%

%% file: figures/desired_kernel_ellipse2.pgf
\begingroup%
\makeatletter%
\begin{pgfpicture}%
\pgfpathrectangle{\pgfpointorigin}{\pgfqpoint{1.674053in}{1.255540in}}%
\pgfusepath{use as bounding box, clip}%
\begin{pgfscope}%
\pgfsetbuttcap%
\pgfsetmiterjoin%
\definecolor{currentfill}{rgb}{1.000000,1.000000,1.000000}%
\pgfsetfillcolor{currentfill}%
\pgfsetlinewidth{0.000000pt}%
\definecolor{currentstroke}{rgb}{1.000000,1.000000,1.000000}%
\pgfsetstrokecolor{currentstroke}%
\pgfsetdash{}{0pt}%
\pgfpathmoveto{\pgfqpoint{0.000000in}{0.000000in}}%
\pgfpathlineto{\pgfqpoint{1.674053in}{0.000000in}}%
\pgfpathlineto{\pgfqpoint{1.674053in}{1.255540in}}%
\pgfpathlineto{\pgfqpoint{0.000000in}{1.255540in}}%
\pgfpathclose%
\pgfusepath{fill}%
\end{pgfscope}%
\begin{pgfscope}%
\pgfpathrectangle{\pgfqpoint{0.039236in}{0.039236in}}{\pgfqpoint{1.595582in}{1.177068in}}%
\pgfusepath{clip}%
\pgfsetbuttcap%
\pgfsetmiterjoin%
\definecolor{currentfill}{rgb}{0.827451,0.827451,0.827451}%
\pgfsetfillcolor{currentfill}%
\pgfsetfillopacity{0.800000}%
\pgfsetlinewidth{1.003750pt}%
\definecolor{currentstroke}{rgb}{0.000000,0.000000,0.000000}%
\pgfsetstrokecolor{currentstroke}%
\pgfsetstrokeopacity{0.800000}%
\pgfsetdash{}{0pt}%
\pgfpathmoveto{\pgfqpoint{0.837026in}{0.235414in}}%
\pgfpathcurveto{\pgfqpoint{0.978077in}{0.235414in}}{\pgfqpoint{1.113370in}{0.276755in}}{\pgfqpoint{1.213109in}{0.350332in}}%
\pgfpathcurveto{\pgfqpoint{1.312847in}{0.423909in}}{\pgfqpoint{1.368887in}{0.523716in}}{\pgfqpoint{1.368887in}{0.627770in}}%
\pgfpathcurveto{\pgfqpoint{1.368887in}{0.731824in}}{\pgfqpoint{1.312847in}{0.831630in}}{\pgfqpoint{1.213109in}{0.905207in}}%
\pgfpathcurveto{\pgfqpoint{1.113370in}{0.978785in}}{\pgfqpoint{0.978077in}{1.020126in}}{\pgfqpoint{0.837026in}{1.020126in}}%
\pgfpathcurveto{\pgfqpoint{0.695975in}{1.020126in}}{\pgfqpoint{0.560682in}{0.978785in}}{\pgfqpoint{0.460944in}{0.905207in}}%
\pgfpathcurveto{\pgfqpoint{0.361206in}{0.831630in}}{\pgfqpoint{0.305166in}{0.731824in}}{\pgfqpoint{0.305166in}{0.627770in}}%
\pgfpathcurveto{\pgfqpoint{0.305166in}{0.523716in}}{\pgfqpoint{0.361206in}{0.423909in}}{\pgfqpoint{0.460944in}{0.350332in}}%
\pgfpathcurveto{\pgfqpoint{0.560682in}{0.276755in}}{\pgfqpoint{0.695975in}{0.235414in}}{\pgfqpoint{0.837026in}{0.235414in}}%
\pgfpathclose%
\pgfusepath{stroke,fill}%
\end{pgfscope}%
\begin{pgfscope}%
\pgfpathrectangle{\pgfqpoint{0.039236in}{0.039236in}}{\pgfqpoint{1.595582in}{1.177068in}}%
\pgfusepath{clip}%
\pgfsetbuttcap%
\pgfsetmiterjoin%
\definecolor{currentfill}{rgb}{0.000000,0.501961,0.000000}%
\pgfsetfillcolor{currentfill}%
\pgfsetfillopacity{0.600000}%
\pgfsetlinewidth{0.000000pt}%
\definecolor{currentstroke}{rgb}{0.000000,0.000000,0.000000}%
\pgfsetstrokecolor{currentstroke}%
\pgfsetstrokeopacity{0.600000}%
\pgfsetdash{}{0pt}%
\pgfpathmoveto{\pgfqpoint{1.422073in}{0.274649in}}%
\pgfpathlineto{\pgfqpoint{0.887328in}{0.237172in}}%
\pgfpathlineto{\pgfqpoint{0.887562in}{0.237383in}}%
\pgfpathlineto{\pgfqpoint{0.903686in}{0.238508in}}%
\pgfpathlineto{\pgfqpoint{0.936687in}{0.242363in}}%
\pgfpathlineto{\pgfqpoint{0.969295in}{0.247740in}}%
\pgfpathlineto{\pgfqpoint{1.001380in}{0.254617in}}%
\pgfpathlineto{\pgfqpoint{1.032817in}{0.262966in}}%
\pgfpathlineto{\pgfqpoint{1.063482in}{0.272755in}}%
\pgfpathlineto{\pgfqpoint{1.093252in}{0.283946in}}%
\pgfpathlineto{\pgfqpoint{1.122012in}{0.296493in}}%
\pgfpathlineto{\pgfqpoint{1.149646in}{0.310347in}}%
\pgfpathlineto{\pgfqpoint{1.176047in}{0.325454in}}%
\pgfpathlineto{\pgfqpoint{1.201110in}{0.341754in}}%
\pgfpathlineto{\pgfqpoint{1.224736in}{0.359184in}}%
\pgfpathlineto{\pgfqpoint{1.246832in}{0.377673in}}%
\pgfpathlineto{\pgfqpoint{1.267311in}{0.397149in}}%
\pgfpathlineto{\pgfqpoint{1.286091in}{0.417535in}}%
\pgfpathlineto{\pgfqpoint{1.303099in}{0.438751in}}%
\pgfpathlineto{\pgfqpoint{1.318268in}{0.460713in}}%
\pgfpathlineto{\pgfqpoint{1.331538in}{0.483334in}}%
\pgfpathlineto{\pgfqpoint{1.342856in}{0.506525in}}%
\pgfpathlineto{\pgfqpoint{1.352178in}{0.530195in}}%
\pgfpathlineto{\pgfqpoint{1.359466in}{0.554250in}}%
\pgfpathlineto{\pgfqpoint{1.364693in}{0.578595in}}%
\pgfpathlineto{\pgfqpoint{1.367837in}{0.603134in}}%
\pgfpathlineto{\pgfqpoint{1.368887in}{0.627770in}}%
\pgfpathlineto{\pgfqpoint{1.367837in}{0.652406in}}%
\pgfpathlineto{\pgfqpoint{1.365768in}{0.668553in}}%
\pgfpathlineto{\pgfqpoint{1.365979in}{0.668743in}}%
\pgfpathclose%
\pgfusepath{fill}%
\end{pgfscope}%
\begin{pgfscope}%
\pgfpathrectangle{\pgfqpoint{0.039236in}{0.039236in}}{\pgfqpoint{1.595582in}{1.177068in}}%
\pgfusepath{clip}%
\pgfsetbuttcap%
\pgfsetmiterjoin%
\definecolor{currentfill}{rgb}{0.000000,0.501961,0.000000}%
\pgfsetfillcolor{currentfill}%
\pgfsetfillopacity{0.600000}%
\pgfsetlinewidth{0.000000pt}%
\definecolor{currentstroke}{rgb}{0.000000,0.000000,0.000000}%
\pgfsetstrokecolor{currentstroke}%
\pgfsetstrokeopacity{0.600000}%
\pgfsetdash{}{0pt}%
\pgfpathmoveto{\pgfqpoint{1.184156in}{0.925035in}}%
\pgfpathlineto{\pgfqpoint{1.528445in}{0.706241in}}%
\pgfpathlineto{\pgfqpoint{1.289225in}{0.421223in}}%
\pgfpathlineto{\pgfqpoint{1.289188in}{0.421398in}}%
\pgfpathlineto{\pgfqpoint{1.286091in}{0.417535in}}%
\pgfpathlineto{\pgfqpoint{1.303099in}{0.438751in}}%
\pgfpathlineto{\pgfqpoint{1.318268in}{0.460713in}}%
\pgfpathlineto{\pgfqpoint{1.331538in}{0.483334in}}%
\pgfpathlineto{\pgfqpoint{1.342856in}{0.506525in}}%
\pgfpathlineto{\pgfqpoint{1.352178in}{0.530195in}}%
\pgfpathlineto{\pgfqpoint{1.359466in}{0.554250in}}%
\pgfpathlineto{\pgfqpoint{1.364693in}{0.578595in}}%
\pgfpathlineto{\pgfqpoint{1.367837in}{0.603134in}}%
\pgfpathlineto{\pgfqpoint{1.368887in}{0.627770in}}%
\pgfpathlineto{\pgfqpoint{1.367837in}{0.652406in}}%
\pgfpathlineto{\pgfqpoint{1.364693in}{0.676945in}}%
\pgfpathlineto{\pgfqpoint{1.359466in}{0.701290in}}%
\pgfpathlineto{\pgfqpoint{1.352178in}{0.725345in}}%
\pgfpathlineto{\pgfqpoint{1.342856in}{0.749015in}}%
\pgfpathlineto{\pgfqpoint{1.331538in}{0.772206in}}%
\pgfpathlineto{\pgfqpoint{1.318268in}{0.794827in}}%
\pgfpathlineto{\pgfqpoint{1.303099in}{0.816789in}}%
\pgfpathlineto{\pgfqpoint{1.286091in}{0.838005in}}%
\pgfpathlineto{\pgfqpoint{1.267311in}{0.858391in}}%
\pgfpathlineto{\pgfqpoint{1.246832in}{0.877867in}}%
\pgfpathlineto{\pgfqpoint{1.224736in}{0.896356in}}%
\pgfpathlineto{\pgfqpoint{1.201110in}{0.913785in}}%
\pgfpathlineto{\pgfqpoint{1.184210in}{0.924776in}}%
\pgfpathclose%
\pgfusepath{fill}%
\end{pgfscope}%
\begin{pgfscope}%
\pgfpathrectangle{\pgfqpoint{0.039236in}{0.039236in}}{\pgfqpoint{1.595582in}{1.177068in}}%
\pgfusepath{clip}%
\pgfsetbuttcap%
\pgfsetroundjoin%
\definecolor{currentfill}{rgb}{0.254902,0.411765,0.882353}%
\pgfsetfillcolor{currentfill}%
\pgfsetlinewidth{1.003750pt}%
\definecolor{currentstroke}{rgb}{0.254902,0.411765,0.882353}%
\pgfsetstrokecolor{currentstroke}%
\pgfsetdash{}{0pt}%
\pgfsys@defobject{currentmarker}{\pgfqpoint{-0.008333in}{-0.008333in}}{\pgfqpoint{0.008333in}{0.008333in}}{%
\pgfpathmoveto{\pgfqpoint{0.000000in}{-0.008333in}}%
\pgfpathcurveto{\pgfqpoint{0.002210in}{-0.008333in}}{\pgfqpoint{0.004330in}{-0.007455in}}{\pgfqpoint{0.005893in}{-0.005893in}}%
\pgfpathcurveto{\pgfqpoint{0.007455in}{-0.004330in}}{\pgfqpoint{0.008333in}{-0.002210in}}{\pgfqpoint{0.008333in}{0.000000in}}%
\pgfpathcurveto{\pgfqpoint{0.008333in}{0.002210in}}{\pgfqpoint{0.007455in}{0.004330in}}{\pgfqpoint{0.005893in}{0.005893in}}%
\pgfpathcurveto{\pgfqpoint{0.004330in}{0.007455in}}{\pgfqpoint{0.002210in}{0.008333in}}{\pgfqpoint{0.000000in}{0.008333in}}%
\pgfpathcurveto{\pgfqpoint{-0.002210in}{0.008333in}}{\pgfqpoint{-0.004330in}{0.007455in}}{\pgfqpoint{-0.005893in}{0.005893in}}%
\pgfpathcurveto{\pgfqpoint{-0.007455in}{0.004330in}}{\pgfqpoint{-0.008333in}{0.002210in}}{\pgfqpoint{-0.008333in}{0.000000in}}%
\pgfpathcurveto{\pgfqpoint{-0.008333in}{-0.002210in}}{\pgfqpoint{-0.007455in}{-0.004330in}}{\pgfqpoint{-0.005893in}{-0.005893in}}%
\pgfpathcurveto{\pgfqpoint{-0.004330in}{-0.007455in}}{\pgfqpoint{-0.002210in}{-0.008333in}}{\pgfqpoint{0.000000in}{-0.008333in}}%
\pgfpathclose%
\pgfusepath{stroke,fill}%
}%
\begin{pgfscope}%
\pgfsys@transformshift{1.422073in}{0.274649in}%
\pgfsys@useobject{currentmarker}{}%
\end{pgfscope}%
\begin{pgfscope}%
\pgfsys@transformshift{1.528445in}{0.706241in}%
\pgfsys@useobject{currentmarker}{}%
\end{pgfscope}%
\end{pgfscope}%
\begin{pgfscope}%
\pgfpathrectangle{\pgfqpoint{0.039236in}{0.039236in}}{\pgfqpoint{1.595582in}{1.177068in}}%
\pgfusepath{clip}%
\pgfsetbuttcap%
\pgfsetmiterjoin%
\definecolor{currentfill}{rgb}{0.000000,0.000000,0.000000}%
\pgfsetfillcolor{currentfill}%
\pgfsetlinewidth{1.003750pt}%
\definecolor{currentstroke}{rgb}{0.000000,0.000000,0.000000}%
\pgfsetstrokecolor{currentstroke}%
\pgfsetdash{}{0pt}%
\pgfsys@defobject{currentmarker}{\pgfqpoint{-0.014142in}{-0.023570in}}{\pgfqpoint{0.014142in}{0.023570in}}{%
\pgfpathmoveto{\pgfqpoint{-0.000000in}{-0.023570in}}%
\pgfpathlineto{\pgfqpoint{0.014142in}{0.000000in}}%
\pgfpathlineto{\pgfqpoint{0.000000in}{0.023570in}}%
\pgfpathlineto{\pgfqpoint{-0.014142in}{0.000000in}}%
\pgfpathclose%
\pgfusepath{stroke,fill}%
}%
\begin{pgfscope}%
\pgfsys@transformshift{1.365979in}{0.668743in}%
\pgfsys@useobject{currentmarker}{}%
\end{pgfscope}%
\end{pgfscope}%
\begin{pgfscope}%
\pgfpathrectangle{\pgfqpoint{0.039236in}{0.039236in}}{\pgfqpoint{1.595582in}{1.177068in}}%
\pgfusepath{clip}%
\pgfsetbuttcap%
\pgfsetmiterjoin%
\definecolor{currentfill}{rgb}{0.000000,0.000000,0.000000}%
\pgfsetfillcolor{currentfill}%
\pgfsetlinewidth{1.003750pt}%
\definecolor{currentstroke}{rgb}{0.000000,0.000000,0.000000}%
\pgfsetstrokecolor{currentstroke}%
\pgfsetdash{}{0pt}%
\pgfsys@defobject{currentmarker}{\pgfqpoint{-0.014142in}{-0.023570in}}{\pgfqpoint{0.014142in}{0.023570in}}{%
\pgfpathmoveto{\pgfqpoint{-0.000000in}{-0.023570in}}%
\pgfpathlineto{\pgfqpoint{0.014142in}{0.000000in}}%
\pgfpathlineto{\pgfqpoint{0.000000in}{0.023570in}}%
\pgfpathlineto{\pgfqpoint{-0.014142in}{0.000000in}}%
\pgfpathclose%
\pgfusepath{stroke,fill}%
}%
\begin{pgfscope}%
\pgfsys@transformshift{0.887328in}{0.237172in}%
\pgfsys@useobject{currentmarker}{}%
\end{pgfscope}%
\end{pgfscope}%
\begin{pgfscope}%
\pgfpathrectangle{\pgfqpoint{0.039236in}{0.039236in}}{\pgfqpoint{1.595582in}{1.177068in}}%
\pgfusepath{clip}%
\pgfsetbuttcap%
\pgfsetmiterjoin%
\definecolor{currentfill}{rgb}{0.000000,0.000000,0.000000}%
\pgfsetfillcolor{currentfill}%
\pgfsetlinewidth{1.003750pt}%
\definecolor{currentstroke}{rgb}{0.000000,0.000000,0.000000}%
\pgfsetstrokecolor{currentstroke}%
\pgfsetdash{}{0pt}%
\pgfsys@defobject{currentmarker}{\pgfqpoint{-0.014142in}{-0.023570in}}{\pgfqpoint{0.014142in}{0.023570in}}{%
\pgfpathmoveto{\pgfqpoint{-0.000000in}{-0.023570in}}%
\pgfpathlineto{\pgfqpoint{0.014142in}{0.000000in}}%
\pgfpathlineto{\pgfqpoint{0.000000in}{0.023570in}}%
\pgfpathlineto{\pgfqpoint{-0.014142in}{0.000000in}}%
\pgfpathclose%
\pgfusepath{stroke,fill}%
}%
\begin{pgfscope}%
\pgfsys@transformshift{1.184156in}{0.925035in}%
\pgfsys@useobject{currentmarker}{}%
\end{pgfscope}%
\end{pgfscope}%
\begin{pgfscope}%
\pgfpathrectangle{\pgfqpoint{0.039236in}{0.039236in}}{\pgfqpoint{1.595582in}{1.177068in}}%
\pgfusepath{clip}%
\pgfsetbuttcap%
\pgfsetmiterjoin%
\definecolor{currentfill}{rgb}{0.000000,0.000000,0.000000}%
\pgfsetfillcolor{currentfill}%
\pgfsetlinewidth{1.003750pt}%
\definecolor{currentstroke}{rgb}{0.000000,0.000000,0.000000}%
\pgfsetstrokecolor{currentstroke}%
\pgfsetdash{}{0pt}%
\pgfsys@defobject{currentmarker}{\pgfqpoint{-0.014142in}{-0.023570in}}{\pgfqpoint{0.014142in}{0.023570in}}{%
\pgfpathmoveto{\pgfqpoint{-0.000000in}{-0.023570in}}%
\pgfpathlineto{\pgfqpoint{0.014142in}{0.000000in}}%
\pgfpathlineto{\pgfqpoint{0.000000in}{0.023570in}}%
\pgfpathlineto{\pgfqpoint{-0.014142in}{0.000000in}}%
\pgfpathclose%
\pgfusepath{stroke,fill}%
}%
\begin{pgfscope}%
\pgfsys@transformshift{1.289225in}{0.421223in}%
\pgfsys@useobject{currentmarker}{}%
\end{pgfscope}%
\end{pgfscope}%
\end{pgfpicture}%
\makeatother%
\endgroup%

%% file: figures/desired_kernel_ellipse1.pgf
\begingroup%
\makeatletter%
\begin{pgfpicture}%
\pgfpathrectangle{\pgfpointorigin}{\pgfqpoint{1.674053in}{1.255540in}}%
\pgfusepath{use as bounding box, clip}%
\begin{pgfscope}%
\pgfsetbuttcap%
\pgfsetmiterjoin%
\definecolor{currentfill}{rgb}{1.000000,1.000000,1.000000}%
\pgfsetfillcolor{currentfill}%
\pgfsetlinewidth{0.000000pt}%
\definecolor{currentstroke}{rgb}{1.000000,1.000000,1.000000}%
\pgfsetstrokecolor{currentstroke}%
\pgfsetdash{}{0pt}%
\pgfpathmoveto{\pgfqpoint{0.000000in}{0.000000in}}%
\pgfpathlineto{\pgfqpoint{1.674053in}{0.000000in}}%
\pgfpathlineto{\pgfqpoint{1.674053in}{1.255540in}}%
\pgfpathlineto{\pgfqpoint{0.000000in}{1.255540in}}%
\pgfpathclose%
\pgfusepath{fill}%
\end{pgfscope}%
\begin{pgfscope}%
\pgfpathrectangle{\pgfqpoint{0.039236in}{0.039236in}}{\pgfqpoint{1.595582in}{1.177068in}}%
\pgfusepath{clip}%
\pgfsetbuttcap%
\pgfsetmiterjoin%
\definecolor{currentfill}{rgb}{0.827451,0.827451,0.827451}%
\pgfsetfillcolor{currentfill}%
\pgfsetfillopacity{0.800000}%
\pgfsetlinewidth{1.003750pt}%
\definecolor{currentstroke}{rgb}{0.000000,0.000000,0.000000}%
\pgfsetstrokecolor{currentstroke}%
\pgfsetstrokeopacity{0.800000}%
\pgfsetdash{}{0pt}%
\pgfpathmoveto{\pgfqpoint{0.837026in}{0.235414in}}%
\pgfpathcurveto{\pgfqpoint{0.978077in}{0.235414in}}{\pgfqpoint{1.113370in}{0.276755in}}{\pgfqpoint{1.213109in}{0.350332in}}%
\pgfpathcurveto{\pgfqpoint{1.312847in}{0.423909in}}{\pgfqpoint{1.368887in}{0.523716in}}{\pgfqpoint{1.368887in}{0.627770in}}%
\pgfpathcurveto{\pgfqpoint{1.368887in}{0.731824in}}{\pgfqpoint{1.312847in}{0.831630in}}{\pgfqpoint{1.213109in}{0.905207in}}%
\pgfpathcurveto{\pgfqpoint{1.113370in}{0.978785in}}{\pgfqpoint{0.978077in}{1.020126in}}{\pgfqpoint{0.837026in}{1.020126in}}%
\pgfpathcurveto{\pgfqpoint{0.695975in}{1.020126in}}{\pgfqpoint{0.560682in}{0.978785in}}{\pgfqpoint{0.460944in}{0.905207in}}%
\pgfpathcurveto{\pgfqpoint{0.361206in}{0.831630in}}{\pgfqpoint{0.305166in}{0.731824in}}{\pgfqpoint{0.305166in}{0.627770in}}%
\pgfpathcurveto{\pgfqpoint{0.305166in}{0.523716in}}{\pgfqpoint{0.361206in}{0.423909in}}{\pgfqpoint{0.460944in}{0.350332in}}%
\pgfpathcurveto{\pgfqpoint{0.560682in}{0.276755in}}{\pgfqpoint{0.695975in}{0.235414in}}{\pgfqpoint{0.837026in}{0.235414in}}%
\pgfpathclose%
\pgfusepath{stroke,fill}%
\end{pgfscope}%
\begin{pgfscope}%
\pgfpathrectangle{\pgfqpoint{0.039236in}{0.039236in}}{\pgfqpoint{1.595582in}{1.177068in}}%
\pgfusepath{clip}%
\pgfsetbuttcap%
\pgfsetmiterjoin%
\definecolor{currentfill}{rgb}{0.000000,0.501961,0.000000}%
\pgfsetfillcolor{currentfill}%
\pgfsetfillopacity{0.600000}%
\pgfsetlinewidth{0.000000pt}%
\definecolor{currentstroke}{rgb}{0.000000,0.000000,0.000000}%
\pgfsetstrokecolor{currentstroke}%
\pgfsetstrokeopacity{0.600000}%
\pgfsetdash{}{0pt}%
\pgfpathmoveto{\pgfqpoint{1.528445in}{0.706241in}}%
\pgfpathlineto{\pgfqpoint{1.422073in}{0.274649in}}%
\pgfpathlineto{\pgfqpoint{0.887328in}{0.237172in}}%
\pgfpathlineto{\pgfqpoint{0.887414in}{0.237373in}}%
\pgfpathlineto{\pgfqpoint{0.870422in}{0.236188in}}%
\pgfpathlineto{\pgfqpoint{0.903686in}{0.238508in}}%
\pgfpathlineto{\pgfqpoint{0.936687in}{0.242363in}}%
\pgfpathlineto{\pgfqpoint{0.969295in}{0.247740in}}%
\pgfpathlineto{\pgfqpoint{1.001380in}{0.254617in}}%
\pgfpathlineto{\pgfqpoint{1.032817in}{0.262966in}}%
\pgfpathlineto{\pgfqpoint{1.063482in}{0.272755in}}%
\pgfpathlineto{\pgfqpoint{1.093252in}{0.283946in}}%
\pgfpathlineto{\pgfqpoint{1.122012in}{0.296493in}}%
\pgfpathlineto{\pgfqpoint{1.149646in}{0.310347in}}%
\pgfpathlineto{\pgfqpoint{1.176047in}{0.325454in}}%
\pgfpathlineto{\pgfqpoint{1.201110in}{0.341754in}}%
\pgfpathlineto{\pgfqpoint{1.224736in}{0.359184in}}%
\pgfpathlineto{\pgfqpoint{1.246832in}{0.377673in}}%
\pgfpathlineto{\pgfqpoint{1.267311in}{0.397149in}}%
\pgfpathlineto{\pgfqpoint{1.286091in}{0.417535in}}%
\pgfpathlineto{\pgfqpoint{1.303099in}{0.438751in}}%
\pgfpathlineto{\pgfqpoint{1.318268in}{0.460713in}}%
\pgfpathlineto{\pgfqpoint{1.331538in}{0.483334in}}%
\pgfpathlineto{\pgfqpoint{1.342856in}{0.506525in}}%
\pgfpathlineto{\pgfqpoint{1.352178in}{0.530195in}}%
\pgfpathlineto{\pgfqpoint{1.359466in}{0.554250in}}%
\pgfpathlineto{\pgfqpoint{1.364693in}{0.578595in}}%
\pgfpathlineto{\pgfqpoint{1.367837in}{0.603134in}}%
\pgfpathlineto{\pgfqpoint{1.368887in}{0.627770in}}%
\pgfpathlineto{\pgfqpoint{1.367837in}{0.652406in}}%
\pgfpathlineto{\pgfqpoint{1.364693in}{0.676945in}}%
\pgfpathlineto{\pgfqpoint{1.359466in}{0.701290in}}%
\pgfpathlineto{\pgfqpoint{1.352178in}{0.725345in}}%
\pgfpathlineto{\pgfqpoint{1.342856in}{0.749015in}}%
\pgfpathlineto{\pgfqpoint{1.331538in}{0.772206in}}%
\pgfpathlineto{\pgfqpoint{1.318268in}{0.794827in}}%
\pgfpathlineto{\pgfqpoint{1.303099in}{0.816789in}}%
\pgfpathlineto{\pgfqpoint{1.286091in}{0.838005in}}%
\pgfpathlineto{\pgfqpoint{1.267311in}{0.858391in}}%
\pgfpathlineto{\pgfqpoint{1.246832in}{0.877867in}}%
\pgfpathlineto{\pgfqpoint{1.224736in}{0.896356in}}%
\pgfpathlineto{\pgfqpoint{1.201110in}{0.913785in}}%
\pgfpathlineto{\pgfqpoint{1.176047in}{0.930085in}}%
\pgfpathlineto{\pgfqpoint{1.184081in}{0.924861in}}%
\pgfpathlineto{\pgfqpoint{1.184156in}{0.925035in}}%
\pgfpathclose%
\pgfusepath{fill}%
\end{pgfscope}%
\begin{pgfscope}%
\pgfpathrectangle{\pgfqpoint{0.039236in}{0.039236in}}{\pgfqpoint{1.595582in}{1.177068in}}%
\pgfusepath{clip}%
\pgfsetbuttcap%
\pgfsetroundjoin%
\definecolor{currentfill}{rgb}{0.254902,0.411765,0.882353}%
\pgfsetfillcolor{currentfill}%
\pgfsetlinewidth{1.003750pt}%
\definecolor{currentstroke}{rgb}{0.254902,0.411765,0.882353}%
\pgfsetstrokecolor{currentstroke}%
\pgfsetdash{}{0pt}%
\pgfsys@defobject{currentmarker}{\pgfqpoint{-0.008333in}{-0.008333in}}{\pgfqpoint{0.008333in}{0.008333in}}{%
\pgfpathmoveto{\pgfqpoint{0.000000in}{-0.008333in}}%
\pgfpathcurveto{\pgfqpoint{0.002210in}{-0.008333in}}{\pgfqpoint{0.004330in}{-0.007455in}}{\pgfqpoint{0.005893in}{-0.005893in}}%
\pgfpathcurveto{\pgfqpoint{0.007455in}{-0.004330in}}{\pgfqpoint{0.008333in}{-0.002210in}}{\pgfqpoint{0.008333in}{0.000000in}}%
\pgfpathcurveto{\pgfqpoint{0.008333in}{0.002210in}}{\pgfqpoint{0.007455in}{0.004330in}}{\pgfqpoint{0.005893in}{0.005893in}}%
\pgfpathcurveto{\pgfqpoint{0.004330in}{0.007455in}}{\pgfqpoint{0.002210in}{0.008333in}}{\pgfqpoint{0.000000in}{0.008333in}}%
\pgfpathcurveto{\pgfqpoint{-0.002210in}{0.008333in}}{\pgfqpoint{-0.004330in}{0.007455in}}{\pgfqpoint{-0.005893in}{0.005893in}}%
\pgfpathcurveto{\pgfqpoint{-0.007455in}{0.004330in}}{\pgfqpoint{-0.008333in}{0.002210in}}{\pgfqpoint{-0.008333in}{0.000000in}}%
\pgfpathcurveto{\pgfqpoint{-0.008333in}{-0.002210in}}{\pgfqpoint{-0.007455in}{-0.004330in}}{\pgfqpoint{-0.005893in}{-0.005893in}}%
\pgfpathcurveto{\pgfqpoint{-0.004330in}{-0.007455in}}{\pgfqpoint{-0.002210in}{-0.008333in}}{\pgfqpoint{0.000000in}{-0.008333in}}%
\pgfpathclose%
\pgfusepath{stroke,fill}%
}%
\begin{pgfscope}%
\pgfsys@transformshift{1.422073in}{0.274649in}%
\pgfsys@useobject{currentmarker}{}%
\end{pgfscope}%
\begin{pgfscope}%
\pgfsys@transformshift{1.528445in}{0.706241in}%
\pgfsys@useobject{currentmarker}{}%
\end{pgfscope}%
\end{pgfscope}%
\begin{pgfscope}%
\pgfpathrectangle{\pgfqpoint{0.039236in}{0.039236in}}{\pgfqpoint{1.595582in}{1.177068in}}%
\pgfusepath{clip}%
\pgfsetbuttcap%
\pgfsetmiterjoin%
\definecolor{currentfill}{rgb}{0.000000,0.000000,0.000000}%
\pgfsetfillcolor{currentfill}%
\pgfsetlinewidth{1.003750pt}%
\definecolor{currentstroke}{rgb}{0.000000,0.000000,0.000000}%
\pgfsetstrokecolor{currentstroke}%
\pgfsetdash{}{0pt}%
\pgfsys@defobject{currentmarker}{\pgfqpoint{-0.014142in}{-0.023570in}}{\pgfqpoint{0.014142in}{0.023570in}}{%
\pgfpathmoveto{\pgfqpoint{-0.000000in}{-0.023570in}}%
\pgfpathlineto{\pgfqpoint{0.014142in}{0.000000in}}%
\pgfpathlineto{\pgfqpoint{0.000000in}{0.023570in}}%
\pgfpathlineto{\pgfqpoint{-0.014142in}{0.000000in}}%
\pgfpathclose%
\pgfusepath{stroke,fill}%
}%
\begin{pgfscope}%
\pgfsys@transformshift{1.365979in}{0.668743in}%
\pgfsys@useobject{currentmarker}{}%
\end{pgfscope}%
\end{pgfscope}%
\begin{pgfscope}%
\pgfpathrectangle{\pgfqpoint{0.039236in}{0.039236in}}{\pgfqpoint{1.595582in}{1.177068in}}%
\pgfusepath{clip}%
\pgfsetbuttcap%
\pgfsetmiterjoin%
\definecolor{currentfill}{rgb}{0.000000,0.000000,0.000000}%
\pgfsetfillcolor{currentfill}%
\pgfsetlinewidth{1.003750pt}%
\definecolor{currentstroke}{rgb}{0.000000,0.000000,0.000000}%
\pgfsetstrokecolor{currentstroke}%
\pgfsetdash{}{0pt}%
\pgfsys@defobject{currentmarker}{\pgfqpoint{-0.014142in}{-0.023570in}}{\pgfqpoint{0.014142in}{0.023570in}}{%
\pgfpathmoveto{\pgfqpoint{-0.000000in}{-0.023570in}}%
\pgfpathlineto{\pgfqpoint{0.014142in}{0.000000in}}%
\pgfpathlineto{\pgfqpoint{0.000000in}{0.023570in}}%
\pgfpathlineto{\pgfqpoint{-0.014142in}{0.000000in}}%
\pgfpathclose%
\pgfusepath{stroke,fill}%
}%
\begin{pgfscope}%
\pgfsys@transformshift{0.887328in}{0.237172in}%
\pgfsys@useobject{currentmarker}{}%
\end{pgfscope}%
\end{pgfscope}%
\begin{pgfscope}%
\pgfpathrectangle{\pgfqpoint{0.039236in}{0.039236in}}{\pgfqpoint{1.595582in}{1.177068in}}%
\pgfusepath{clip}%
\pgfsetbuttcap%
\pgfsetmiterjoin%
\definecolor{currentfill}{rgb}{0.000000,0.000000,0.000000}%
\pgfsetfillcolor{currentfill}%
\pgfsetlinewidth{1.003750pt}%
\definecolor{currentstroke}{rgb}{0.000000,0.000000,0.000000}%
\pgfsetstrokecolor{currentstroke}%
\pgfsetdash{}{0pt}%
\pgfsys@defobject{currentmarker}{\pgfqpoint{-0.014142in}{-0.023570in}}{\pgfqpoint{0.014142in}{0.023570in}}{%
\pgfpathmoveto{\pgfqpoint{-0.000000in}{-0.023570in}}%
\pgfpathlineto{\pgfqpoint{0.014142in}{0.000000in}}%
\pgfpathlineto{\pgfqpoint{0.000000in}{0.023570in}}%
\pgfpathlineto{\pgfqpoint{-0.014142in}{0.000000in}}%
\pgfpathclose%
\pgfusepath{stroke,fill}%
}%
\begin{pgfscope}%
\pgfsys@transformshift{1.184156in}{0.925035in}%
\pgfsys@useobject{currentmarker}{}%
\end{pgfscope}%
\end{pgfscope}%
\begin{pgfscope}%
\pgfpathrectangle{\pgfqpoint{0.039236in}{0.039236in}}{\pgfqpoint{1.595582in}{1.177068in}}%
\pgfusepath{clip}%
\pgfsetbuttcap%
\pgfsetmiterjoin%
\definecolor{currentfill}{rgb}{0.000000,0.000000,0.000000}%
\pgfsetfillcolor{currentfill}%
\pgfsetlinewidth{1.003750pt}%
\definecolor{currentstroke}{rgb}{0.000000,0.000000,0.000000}%
\pgfsetstrokecolor{currentstroke}%
\pgfsetdash{}{0pt}%
\pgfsys@defobject{currentmarker}{\pgfqpoint{-0.014142in}{-0.023570in}}{\pgfqpoint{0.014142in}{0.023570in}}{%
\pgfpathmoveto{\pgfqpoint{-0.000000in}{-0.023570in}}%
\pgfpathlineto{\pgfqpoint{0.014142in}{0.000000in}}%
\pgfpathlineto{\pgfqpoint{0.000000in}{0.023570in}}%
\pgfpathlineto{\pgfqpoint{-0.014142in}{0.000000in}}%
\pgfpathclose%
\pgfusepath{stroke,fill}%
}%
\begin{pgfscope}%
\pgfsys@transformshift{1.289225in}{0.421223in}%
\pgfsys@useobject{currentmarker}{}%
\end{pgfscope}%
\end{pgfscope}%
\end{pgfpicture}%
\makeatother%
\endgroup%

%% file: figures/step1_admker0.pgf
\begingroup%
\makeatletter%
\begin{pgfpicture}%
\pgfpathrectangle{\pgfpointorigin}{\pgfqpoint{1.674053in}{1.255540in}}%
\pgfusepath{use as bounding box, clip}%
\begin{pgfscope}%
\pgfsetbuttcap%
\pgfsetmiterjoin%
\definecolor{currentfill}{rgb}{1.000000,1.000000,1.000000}%
\pgfsetfillcolor{currentfill}%
\pgfsetlinewidth{0.000000pt}%
\definecolor{currentstroke}{rgb}{1.000000,1.000000,1.000000}%
\pgfsetstrokecolor{currentstroke}%
\pgfsetdash{}{0pt}%
\pgfpathmoveto{\pgfqpoint{0.000000in}{0.000000in}}%
\pgfpathlineto{\pgfqpoint{1.674053in}{0.000000in}}%
\pgfpathlineto{\pgfqpoint{1.674053in}{1.255540in}}%
\pgfpathlineto{\pgfqpoint{0.000000in}{1.255540in}}%
\pgfpathclose%
\pgfusepath{fill}%
\end{pgfscope}%
\begin{pgfscope}%
\pgfpathrectangle{\pgfqpoint{0.039236in}{0.039236in}}{\pgfqpoint{1.595582in}{1.177068in}}%
\pgfusepath{clip}%
\pgfsetbuttcap%
\pgfsetmiterjoin%
\definecolor{currentfill}{rgb}{0.827451,0.827451,0.827451}%
\pgfsetfillcolor{currentfill}%
\pgfsetfillopacity{0.800000}%
\pgfsetlinewidth{1.003750pt}%
\definecolor{currentstroke}{rgb}{0.000000,0.000000,0.000000}%
\pgfsetstrokecolor{currentstroke}%
\pgfsetstrokeopacity{0.800000}%
\pgfsetdash{}{0pt}%
\pgfpathmoveto{\pgfqpoint{0.747724in}{0.572746in}}%
\pgfpathcurveto{\pgfqpoint{0.809418in}{0.653012in}}{\pgfqpoint{0.856338in}{0.745946in}}{\pgfqpoint{0.878150in}{0.831082in}}%
\pgfpathcurveto{\pgfqpoint{0.899962in}{0.916217in}}{\pgfqpoint{0.894885in}{0.986602in}}{\pgfqpoint{0.864038in}{1.026735in}}%
\pgfpathcurveto{\pgfqpoint{0.833192in}{1.066868in}}{\pgfqpoint{0.779093in}{1.073472in}}{\pgfqpoint{0.713657in}{1.045094in}}%
\pgfpathcurveto{\pgfqpoint{0.648221in}{1.016716in}}{\pgfqpoint{0.576790in}{0.955671in}}{\pgfqpoint{0.515096in}{0.875405in}}%
\pgfpathcurveto{\pgfqpoint{0.453402in}{0.795139in}}{\pgfqpoint{0.406483in}{0.702204in}}{\pgfqpoint{0.384671in}{0.617069in}}%
\pgfpathcurveto{\pgfqpoint{0.362859in}{0.531934in}}{\pgfqpoint{0.367935in}{0.461549in}}{\pgfqpoint{0.398782in}{0.421416in}}%
\pgfpathcurveto{\pgfqpoint{0.429629in}{0.381283in}}{\pgfqpoint{0.483728in}{0.374678in}}{\pgfqpoint{0.549164in}{0.403057in}}%
\pgfpathcurveto{\pgfqpoint{0.614600in}{0.431435in}}{\pgfqpoint{0.686031in}{0.492480in}}{\pgfqpoint{0.747724in}{0.572746in}}%
\pgfpathclose%
\pgfusepath{stroke,fill}%
\end{pgfscope}%
\begin{pgfscope}%
\pgfpathrectangle{\pgfqpoint{0.039236in}{0.039236in}}{\pgfqpoint{1.595582in}{1.177068in}}%
\pgfusepath{clip}%
\pgfsetbuttcap%
\pgfsetmiterjoin%
\definecolor{currentfill}{rgb}{0.627451,0.321569,0.176471}%
\pgfsetfillcolor{currentfill}%
\pgfsetfillopacity{0.800000}%
\pgfsetlinewidth{1.003750pt}%
\definecolor{currentstroke}{rgb}{0.000000,0.000000,0.000000}%
\pgfsetstrokecolor{currentstroke}%
\pgfsetstrokeopacity{0.800000}%
\pgfsetdash{}{0pt}%
\pgfpathmoveto{\pgfqpoint{1.124889in}{0.724075in}}%
\pgfpathcurveto{\pgfqpoint{1.124889in}{0.837589in}}{\pgfqpoint{1.107557in}{0.946468in}}{\pgfqpoint{1.076710in}{1.026735in}}%
\pgfpathcurveto{\pgfqpoint{1.045863in}{1.107001in}}{\pgfqpoint{1.004020in}{1.152100in}}{\pgfqpoint{0.960396in}{1.152100in}}%
\pgfpathcurveto{\pgfqpoint{0.916772in}{1.152100in}}{\pgfqpoint{0.874929in}{1.107001in}}{\pgfqpoint{0.844082in}{1.026735in}}%
\pgfpathcurveto{\pgfqpoint{0.813235in}{0.946468in}}{\pgfqpoint{0.795903in}{0.837589in}}{\pgfqpoint{0.795903in}{0.724075in}}%
\pgfpathcurveto{\pgfqpoint{0.795903in}{0.610562in}}{\pgfqpoint{0.813235in}{0.501682in}}{\pgfqpoint{0.844082in}{0.421416in}}%
\pgfpathcurveto{\pgfqpoint{0.874929in}{0.341150in}}{\pgfqpoint{0.916772in}{0.296051in}}{\pgfqpoint{0.960396in}{0.296051in}}%
\pgfpathcurveto{\pgfqpoint{1.004020in}{0.296051in}}{\pgfqpoint{1.045863in}{0.341150in}}{\pgfqpoint{1.076710in}{0.421416in}}%
\pgfpathcurveto{\pgfqpoint{1.107557in}{0.501682in}}{\pgfqpoint{1.124889in}{0.610562in}}{\pgfqpoint{1.124889in}{0.724075in}}%
\pgfpathclose%
\pgfusepath{stroke,fill}%
\end{pgfscope}%
\begin{pgfscope}%
\pgfpathrectangle{\pgfqpoint{0.039236in}{0.039236in}}{\pgfqpoint{1.595582in}{1.177068in}}%
\pgfusepath{clip}%
\pgfsetbuttcap%
\pgfsetmiterjoin%
\definecolor{currentfill}{rgb}{1.000000,0.000000,1.000000}%
\pgfsetfillcolor{currentfill}%
\pgfsetfillopacity{0.800000}%
\pgfsetlinewidth{1.003750pt}%
\definecolor{currentstroke}{rgb}{0.000000,0.000000,0.000000}%
\pgfsetstrokecolor{currentstroke}%
\pgfsetstrokeopacity{0.800000}%
\pgfsetdash{}{0pt}%
\pgfpathmoveto{\pgfqpoint{1.289382in}{0.082038in}}%
\pgfpathcurveto{\pgfqpoint{1.376630in}{0.082038in}}{\pgfqpoint{1.460316in}{0.104588in}}{\pgfqpoint{1.522010in}{0.144721in}}%
\pgfpathcurveto{\pgfqpoint{1.583704in}{0.184854in}}{\pgfqpoint{1.618368in}{0.239294in}}{\pgfqpoint{1.618368in}{0.296051in}}%
\pgfpathcurveto{\pgfqpoint{1.618368in}{0.352807in}}{\pgfqpoint{1.583704in}{0.407247in}}{\pgfqpoint{1.522010in}{0.447380in}}%
\pgfpathcurveto{\pgfqpoint{1.460316in}{0.487513in}}{\pgfqpoint{1.376630in}{0.510063in}}{\pgfqpoint{1.289382in}{0.510063in}}%
\pgfpathcurveto{\pgfqpoint{1.202134in}{0.510063in}}{\pgfqpoint{1.118448in}{0.487513in}}{\pgfqpoint{1.056754in}{0.447380in}}%
\pgfpathcurveto{\pgfqpoint{0.995060in}{0.407247in}}{\pgfqpoint{0.960396in}{0.352807in}}{\pgfqpoint{0.960396in}{0.296051in}}%
\pgfpathcurveto{\pgfqpoint{0.960396in}{0.239294in}}{\pgfqpoint{0.995060in}{0.184854in}}{\pgfqpoint{1.056754in}{0.144721in}}%
\pgfpathcurveto{\pgfqpoint{1.118448in}{0.104588in}}{\pgfqpoint{1.202134in}{0.082038in}}{\pgfqpoint{1.289382in}{0.082038in}}%
\pgfpathclose%
\pgfusepath{stroke,fill}%
\end{pgfscope}%
\begin{pgfscope}%
\pgfpathrectangle{\pgfqpoint{0.039236in}{0.039236in}}{\pgfqpoint{1.595582in}{1.177068in}}%
\pgfusepath{clip}%
\pgfsetbuttcap%
\pgfsetmiterjoin%
\definecolor{currentfill}{rgb}{1.000000,0.388235,0.278431}%
\pgfsetfillcolor{currentfill}%
\pgfsetfillopacity{0.800000}%
\pgfsetlinewidth{1.003750pt}%
\definecolor{currentstroke}{rgb}{0.000000,0.000000,0.000000}%
\pgfsetstrokecolor{currentstroke}%
\pgfsetstrokeopacity{0.800000}%
\pgfsetdash{}{0pt}%
\pgfpathmoveto{\pgfqpoint{0.055685in}{1.002292in}}%
\pgfpathlineto{\pgfqpoint{0.302424in}{1.002292in}}%
\pgfpathlineto{\pgfqpoint{0.302424in}{0.360254in}}%
\pgfpathlineto{\pgfqpoint{0.466917in}{0.360254in}}%
\pgfpathlineto{\pgfqpoint{0.466917in}{0.253248in}}%
\pgfpathlineto{\pgfqpoint{0.220178in}{0.253248in}}%
\pgfpathlineto{\pgfqpoint{0.220178in}{0.895285in}}%
\pgfpathlineto{\pgfqpoint{0.055685in}{0.895285in}}%
\pgfpathclose%
\pgfusepath{stroke,fill}%
\end{pgfscope}%
\begin{pgfscope}%
\pgfpathrectangle{\pgfqpoint{0.039236in}{0.039236in}}{\pgfqpoint{1.595582in}{1.177068in}}%
\pgfusepath{clip}%
\pgfsetbuttcap%
\pgfsetmiterjoin%
\definecolor{currentfill}{rgb}{1.000000,1.000000,0.000000}%
\pgfsetfillcolor{currentfill}%
\pgfsetfillopacity{0.600000}%
\pgfsetlinewidth{1.003750pt}%
\definecolor{currentstroke}{rgb}{0.000000,0.000000,0.000000}%
\pgfsetstrokecolor{currentstroke}%
\pgfsetstrokeopacity{0.600000}%
\pgfsetdash{{3.700000pt}{1.600000pt}}{0.000000pt}%
\pgfpathmoveto{\pgfqpoint{0.351772in}{0.681273in}}%
\pgfpathlineto{\pgfqpoint{0.267936in}{1.430316in}}%
\pgfpathlineto{\pgfqpoint{1.799310in}{1.430316in}}%
\pgfpathlineto{\pgfqpoint{1.799310in}{-0.174777in}}%
\pgfpathlineto{\pgfqpoint{1.382048in}{-0.174777in}}%
\pgfpathlineto{\pgfqpoint{0.631410in}{0.274649in}}%
\pgfpathlineto{\pgfqpoint{0.423340in}{-0.174777in}}%
\pgfpathlineto{\pgfqpoint{-0.125257in}{-0.174777in}}%
\pgfpathlineto{\pgfqpoint{-0.125257in}{0.111207in}}%
\pgfpathclose%
\pgfusepath{stroke,fill}%
\end{pgfscope}%
\begin{pgfscope}%
\pgfpathrectangle{\pgfqpoint{0.039236in}{0.039236in}}{\pgfqpoint{1.595582in}{1.177068in}}%
\pgfusepath{clip}%
\pgfsetbuttcap%
\pgfsetmiterjoin%
\definecolor{currentfill}{rgb}{1.000000,0.000000,0.000000}%
\pgfsetfillcolor{currentfill}%
\pgfsetlinewidth{1.003750pt}%
\definecolor{currentstroke}{rgb}{1.000000,0.000000,0.000000}%
\pgfsetstrokecolor{currentstroke}%
\pgfsetdash{}{0pt}%
\pgfsys@defobject{currentmarker}{\pgfqpoint{-0.016667in}{-0.016667in}}{\pgfqpoint{0.016667in}{0.016667in}}{%
\pgfpathmoveto{\pgfqpoint{0.000000in}{0.016667in}}%
\pgfpathlineto{\pgfqpoint{-0.016667in}{-0.016667in}}%
\pgfpathlineto{\pgfqpoint{0.016667in}{-0.016667in}}%
\pgfpathclose%
\pgfusepath{stroke,fill}%
}%
\begin{pgfscope}%
\pgfsys@transformshift{0.351772in}{0.681273in}%
\pgfsys@useobject{currentmarker}{}%
\end{pgfscope}%
\begin{pgfscope}%
\pgfsys@transformshift{0.631410in}{0.274649in}%
\pgfsys@useobject{currentmarker}{}%
\end{pgfscope}%
\end{pgfscope}%
\end{pgfpicture}%
\makeatother%
\endgroup%

%% file: figures/step1_admker3.pgf
\begingroup%
\makeatletter%
\begin{pgfpicture}%
\pgfpathrectangle{\pgfpointorigin}{\pgfqpoint{1.674053in}{1.255540in}}%
\pgfusepath{use as bounding box, clip}%
\begin{pgfscope}%
\pgfsetbuttcap%
\pgfsetmiterjoin%
\definecolor{currentfill}{rgb}{1.000000,1.000000,1.000000}%
\pgfsetfillcolor{currentfill}%
\pgfsetlinewidth{0.000000pt}%
\definecolor{currentstroke}{rgb}{1.000000,1.000000,1.000000}%
\pgfsetstrokecolor{currentstroke}%
\pgfsetdash{}{0pt}%
\pgfpathmoveto{\pgfqpoint{0.000000in}{0.000000in}}%
\pgfpathlineto{\pgfqpoint{1.674053in}{0.000000in}}%
\pgfpathlineto{\pgfqpoint{1.674053in}{1.255540in}}%
\pgfpathlineto{\pgfqpoint{0.000000in}{1.255540in}}%
\pgfpathclose%
\pgfusepath{fill}%
\end{pgfscope}%
\begin{pgfscope}%
\pgfpathrectangle{\pgfqpoint{0.039236in}{0.039236in}}{\pgfqpoint{1.595582in}{1.177068in}}%
\pgfusepath{clip}%
\pgfsetbuttcap%
\pgfsetmiterjoin%
\definecolor{currentfill}{rgb}{0.827451,0.827451,0.827451}%
\pgfsetfillcolor{currentfill}%
\pgfsetfillopacity{0.800000}%
\pgfsetlinewidth{1.003750pt}%
\definecolor{currentstroke}{rgb}{0.000000,0.000000,0.000000}%
\pgfsetstrokecolor{currentstroke}%
\pgfsetstrokeopacity{0.800000}%
\pgfsetdash{}{0pt}%
\pgfpathmoveto{\pgfqpoint{0.747724in}{0.572746in}}%
\pgfpathcurveto{\pgfqpoint{0.809418in}{0.653012in}}{\pgfqpoint{0.856338in}{0.745946in}}{\pgfqpoint{0.878150in}{0.831082in}}%
\pgfpathcurveto{\pgfqpoint{0.899962in}{0.916217in}}{\pgfqpoint{0.894885in}{0.986602in}}{\pgfqpoint{0.864038in}{1.026735in}}%
\pgfpathcurveto{\pgfqpoint{0.833192in}{1.066868in}}{\pgfqpoint{0.779093in}{1.073472in}}{\pgfqpoint{0.713657in}{1.045094in}}%
\pgfpathcurveto{\pgfqpoint{0.648221in}{1.016716in}}{\pgfqpoint{0.576790in}{0.955671in}}{\pgfqpoint{0.515096in}{0.875405in}}%
\pgfpathcurveto{\pgfqpoint{0.453402in}{0.795139in}}{\pgfqpoint{0.406483in}{0.702204in}}{\pgfqpoint{0.384671in}{0.617069in}}%
\pgfpathcurveto{\pgfqpoint{0.362859in}{0.531934in}}{\pgfqpoint{0.367935in}{0.461549in}}{\pgfqpoint{0.398782in}{0.421416in}}%
\pgfpathcurveto{\pgfqpoint{0.429629in}{0.381283in}}{\pgfqpoint{0.483728in}{0.374678in}}{\pgfqpoint{0.549164in}{0.403057in}}%
\pgfpathcurveto{\pgfqpoint{0.614600in}{0.431435in}}{\pgfqpoint{0.686031in}{0.492480in}}{\pgfqpoint{0.747724in}{0.572746in}}%
\pgfpathclose%
\pgfusepath{stroke,fill}%
\end{pgfscope}%
\begin{pgfscope}%
\pgfpathrectangle{\pgfqpoint{0.039236in}{0.039236in}}{\pgfqpoint{1.595582in}{1.177068in}}%
\pgfusepath{clip}%
\pgfsetbuttcap%
\pgfsetmiterjoin%
\definecolor{currentfill}{rgb}{0.627451,0.321569,0.176471}%
\pgfsetfillcolor{currentfill}%
\pgfsetfillopacity{0.800000}%
\pgfsetlinewidth{1.003750pt}%
\definecolor{currentstroke}{rgb}{0.000000,0.000000,0.000000}%
\pgfsetstrokecolor{currentstroke}%
\pgfsetstrokeopacity{0.800000}%
\pgfsetdash{}{0pt}%
\pgfpathmoveto{\pgfqpoint{1.124889in}{0.724075in}}%
\pgfpathcurveto{\pgfqpoint{1.124889in}{0.837589in}}{\pgfqpoint{1.107557in}{0.946468in}}{\pgfqpoint{1.076710in}{1.026735in}}%
\pgfpathcurveto{\pgfqpoint{1.045863in}{1.107001in}}{\pgfqpoint{1.004020in}{1.152100in}}{\pgfqpoint{0.960396in}{1.152100in}}%
\pgfpathcurveto{\pgfqpoint{0.916772in}{1.152100in}}{\pgfqpoint{0.874929in}{1.107001in}}{\pgfqpoint{0.844082in}{1.026735in}}%
\pgfpathcurveto{\pgfqpoint{0.813235in}{0.946468in}}{\pgfqpoint{0.795903in}{0.837589in}}{\pgfqpoint{0.795903in}{0.724075in}}%
\pgfpathcurveto{\pgfqpoint{0.795903in}{0.610562in}}{\pgfqpoint{0.813235in}{0.501682in}}{\pgfqpoint{0.844082in}{0.421416in}}%
\pgfpathcurveto{\pgfqpoint{0.874929in}{0.341150in}}{\pgfqpoint{0.916772in}{0.296051in}}{\pgfqpoint{0.960396in}{0.296051in}}%
\pgfpathcurveto{\pgfqpoint{1.004020in}{0.296051in}}{\pgfqpoint{1.045863in}{0.341150in}}{\pgfqpoint{1.076710in}{0.421416in}}%
\pgfpathcurveto{\pgfqpoint{1.107557in}{0.501682in}}{\pgfqpoint{1.124889in}{0.610562in}}{\pgfqpoint{1.124889in}{0.724075in}}%
\pgfpathclose%
\pgfusepath{stroke,fill}%
\end{pgfscope}%
\begin{pgfscope}%
\pgfpathrectangle{\pgfqpoint{0.039236in}{0.039236in}}{\pgfqpoint{1.595582in}{1.177068in}}%
\pgfusepath{clip}%
\pgfsetbuttcap%
\pgfsetmiterjoin%
\definecolor{currentfill}{rgb}{1.000000,0.000000,1.000000}%
\pgfsetfillcolor{currentfill}%
\pgfsetfillopacity{0.800000}%
\pgfsetlinewidth{1.003750pt}%
\definecolor{currentstroke}{rgb}{0.000000,0.000000,0.000000}%
\pgfsetstrokecolor{currentstroke}%
\pgfsetstrokeopacity{0.800000}%
\pgfsetdash{}{0pt}%
\pgfpathmoveto{\pgfqpoint{1.289382in}{0.082038in}}%
\pgfpathcurveto{\pgfqpoint{1.376630in}{0.082038in}}{\pgfqpoint{1.460316in}{0.104588in}}{\pgfqpoint{1.522010in}{0.144721in}}%
\pgfpathcurveto{\pgfqpoint{1.583704in}{0.184854in}}{\pgfqpoint{1.618368in}{0.239294in}}{\pgfqpoint{1.618368in}{0.296051in}}%
\pgfpathcurveto{\pgfqpoint{1.618368in}{0.352807in}}{\pgfqpoint{1.583704in}{0.407247in}}{\pgfqpoint{1.522010in}{0.447380in}}%
\pgfpathcurveto{\pgfqpoint{1.460316in}{0.487513in}}{\pgfqpoint{1.376630in}{0.510063in}}{\pgfqpoint{1.289382in}{0.510063in}}%
\pgfpathcurveto{\pgfqpoint{1.202134in}{0.510063in}}{\pgfqpoint{1.118448in}{0.487513in}}{\pgfqpoint{1.056754in}{0.447380in}}%
\pgfpathcurveto{\pgfqpoint{0.995060in}{0.407247in}}{\pgfqpoint{0.960396in}{0.352807in}}{\pgfqpoint{0.960396in}{0.296051in}}%
\pgfpathcurveto{\pgfqpoint{0.960396in}{0.239294in}}{\pgfqpoint{0.995060in}{0.184854in}}{\pgfqpoint{1.056754in}{0.144721in}}%
\pgfpathcurveto{\pgfqpoint{1.118448in}{0.104588in}}{\pgfqpoint{1.202134in}{0.082038in}}{\pgfqpoint{1.289382in}{0.082038in}}%
\pgfpathclose%
\pgfusepath{stroke,fill}%
\end{pgfscope}%
\begin{pgfscope}%
\pgfpathrectangle{\pgfqpoint{0.039236in}{0.039236in}}{\pgfqpoint{1.595582in}{1.177068in}}%
\pgfusepath{clip}%
\pgfsetbuttcap%
\pgfsetmiterjoin%
\definecolor{currentfill}{rgb}{1.000000,0.388235,0.278431}%
\pgfsetfillcolor{currentfill}%
\pgfsetfillopacity{0.800000}%
\pgfsetlinewidth{1.003750pt}%
\definecolor{currentstroke}{rgb}{0.000000,0.000000,0.000000}%
\pgfsetstrokecolor{currentstroke}%
\pgfsetstrokeopacity{0.800000}%
\pgfsetdash{}{0pt}%
\pgfpathmoveto{\pgfqpoint{0.055685in}{1.002292in}}%
\pgfpathlineto{\pgfqpoint{0.302424in}{1.002292in}}%
\pgfpathlineto{\pgfqpoint{0.302424in}{0.360254in}}%
\pgfpathlineto{\pgfqpoint{0.466917in}{0.360254in}}%
\pgfpathlineto{\pgfqpoint{0.466917in}{0.253248in}}%
\pgfpathlineto{\pgfqpoint{0.220178in}{0.253248in}}%
\pgfpathlineto{\pgfqpoint{0.220178in}{0.895285in}}%
\pgfpathlineto{\pgfqpoint{0.055685in}{0.895285in}}%
\pgfpathclose%
\pgfusepath{stroke,fill}%
\end{pgfscope}%
\begin{pgfscope}%
\pgfpathrectangle{\pgfqpoint{0.039236in}{0.039236in}}{\pgfqpoint{1.595582in}{1.177068in}}%
\pgfusepath{clip}%
\pgfsetbuttcap%
\pgfsetmiterjoin%
\definecolor{currentfill}{rgb}{1.000000,1.000000,0.000000}%
\pgfsetfillcolor{currentfill}%
\pgfsetfillopacity{0.600000}%
\pgfsetlinewidth{1.003750pt}%
\definecolor{currentstroke}{rgb}{0.000000,0.000000,0.000000}%
\pgfsetstrokecolor{currentstroke}%
\pgfsetstrokeopacity{0.600000}%
\pgfsetdash{{3.700000pt}{1.600000pt}}{0.000000pt}%
\pgfpathmoveto{\pgfqpoint{0.351772in}{0.681273in}}%
\pgfpathlineto{\pgfqpoint{0.483367in}{-0.174777in}}%
\pgfpathlineto{\pgfqpoint{-0.125257in}{-0.174777in}}%
\pgfpathlineto{\pgfqpoint{-0.125257in}{1.430316in}}%
\pgfpathlineto{\pgfqpoint{0.083100in}{1.430316in}}%
\pgfpathclose%
\pgfusepath{stroke,fill}%
\end{pgfscope}%
\begin{pgfscope}%
\pgfpathrectangle{\pgfqpoint{0.039236in}{0.039236in}}{\pgfqpoint{1.595582in}{1.177068in}}%
\pgfusepath{clip}%
\pgfsetbuttcap%
\pgfsetmiterjoin%
\definecolor{currentfill}{rgb}{1.000000,0.000000,0.000000}%
\pgfsetfillcolor{currentfill}%
\pgfsetlinewidth{1.003750pt}%
\definecolor{currentstroke}{rgb}{1.000000,0.000000,0.000000}%
\pgfsetstrokecolor{currentstroke}%
\pgfsetdash{}{0pt}%
\pgfsys@defobject{currentmarker}{\pgfqpoint{-0.016667in}{-0.016667in}}{\pgfqpoint{0.016667in}{0.016667in}}{%
\pgfpathmoveto{\pgfqpoint{0.000000in}{0.016667in}}%
\pgfpathlineto{\pgfqpoint{-0.016667in}{-0.016667in}}%
\pgfpathlineto{\pgfqpoint{0.016667in}{-0.016667in}}%
\pgfpathclose%
\pgfusepath{stroke,fill}%
}%
\begin{pgfscope}%
\pgfsys@transformshift{0.351772in}{0.681273in}%
\pgfsys@useobject{currentmarker}{}%
\end{pgfscope}%
\begin{pgfscope}%
\pgfsys@transformshift{0.631410in}{0.274649in}%
\pgfsys@useobject{currentmarker}{}%
\end{pgfscope}%
\end{pgfscope}%
\end{pgfpicture}%
\makeatother%
\endgroup%

%% file: figures/step1_star.pgf
\begingroup%
\makeatletter%
\begin{pgfpicture}%
\pgfpathrectangle{\pgfpointorigin}{\pgfqpoint{1.674053in}{1.255540in}}%
\pgfusepath{use as bounding box, clip}%
\begin{pgfscope}%
\pgfsetbuttcap%
\pgfsetmiterjoin%
\definecolor{currentfill}{rgb}{1.000000,1.000000,1.000000}%
\pgfsetfillcolor{currentfill}%
\pgfsetlinewidth{0.000000pt}%
\definecolor{currentstroke}{rgb}{1.000000,1.000000,1.000000}%
\pgfsetstrokecolor{currentstroke}%
\pgfsetdash{}{0pt}%
\pgfpathmoveto{\pgfqpoint{0.000000in}{0.000000in}}%
\pgfpathlineto{\pgfqpoint{1.674053in}{0.000000in}}%
\pgfpathlineto{\pgfqpoint{1.674053in}{1.255540in}}%
\pgfpathlineto{\pgfqpoint{0.000000in}{1.255540in}}%
\pgfpathclose%
\pgfusepath{fill}%
\end{pgfscope}%
\begin{pgfscope}%
\pgfpathrectangle{\pgfqpoint{0.039236in}{0.039236in}}{\pgfqpoint{1.595582in}{1.177068in}}%
\pgfusepath{clip}%
\pgfsetbuttcap%
\pgfsetmiterjoin%
\definecolor{currentfill}{rgb}{0.827451,0.827451,0.827451}%
\pgfsetfillcolor{currentfill}%
\pgfsetfillopacity{0.800000}%
\pgfsetlinewidth{1.003750pt}%
\definecolor{currentstroke}{rgb}{0.000000,0.000000,0.000000}%
\pgfsetstrokecolor{currentstroke}%
\pgfsetstrokeopacity{0.800000}%
\pgfsetdash{}{0pt}%
\pgfpathmoveto{\pgfqpoint{0.864038in}{1.026735in}}%
\pgfpathlineto{\pgfqpoint{0.856276in}{1.035640in}}%
\pgfpathlineto{\pgfqpoint{0.847626in}{1.043315in}}%
\pgfpathlineto{\pgfqpoint{0.838123in}{1.049730in}}%
\pgfpathlineto{\pgfqpoint{0.827804in}{1.054860in}}%
\pgfpathlineto{\pgfqpoint{0.816710in}{1.058685in}}%
\pgfpathlineto{\pgfqpoint{0.804884in}{1.061189in}}%
\pgfpathlineto{\pgfqpoint{0.792374in}{1.062363in}}%
\pgfpathlineto{\pgfqpoint{0.779229in}{1.062201in}}%
\pgfpathlineto{\pgfqpoint{0.765500in}{1.060706in}}%
\pgfpathlineto{\pgfqpoint{0.751243in}{1.057881in}}%
\pgfpathlineto{\pgfqpoint{0.736512in}{1.053740in}}%
\pgfpathlineto{\pgfqpoint{0.721366in}{1.048297in}}%
\pgfpathlineto{\pgfqpoint{0.705866in}{1.041574in}}%
\pgfpathlineto{\pgfqpoint{0.690071in}{1.033599in}}%
\pgfpathlineto{\pgfqpoint{0.674046in}{1.024402in}}%
\pgfpathlineto{\pgfqpoint{0.657851in}{1.014020in}}%
\pgfpathlineto{\pgfqpoint{0.641553in}{1.002494in}}%
\pgfpathlineto{\pgfqpoint{0.625214in}{0.989869in}}%
\pgfpathlineto{\pgfqpoint{0.608900in}{0.976194in}}%
\pgfpathlineto{\pgfqpoint{0.592675in}{0.961525in}}%
\pgfpathlineto{\pgfqpoint{0.576603in}{0.945919in}}%
\pgfpathlineto{\pgfqpoint{0.560747in}{0.929437in}}%
\pgfpathlineto{\pgfqpoint{0.545169in}{0.912145in}}%
\pgfpathlineto{\pgfqpoint{0.529933in}{0.894111in}}%
\pgfpathlineto{\pgfqpoint{0.515096in}{0.875405in}}%
\pgfpathlineto{\pgfqpoint{0.500719in}{0.856102in}}%
\pgfpathlineto{\pgfqpoint{0.486857in}{0.836278in}}%
\pgfpathlineto{\pgfqpoint{0.473566in}{0.816012in}}%
\pgfpathlineto{\pgfqpoint{0.460898in}{0.795382in}}%
\pgfpathlineto{\pgfqpoint{0.448903in}{0.774472in}}%
\pgfpathlineto{\pgfqpoint{0.437628in}{0.753362in}}%
\pgfpathlineto{\pgfqpoint{0.427118in}{0.732136in}}%
\pgfpathlineto{\pgfqpoint{0.417414in}{0.710879in}}%
\pgfpathlineto{\pgfqpoint{0.408555in}{0.689674in}}%
\pgfpathlineto{\pgfqpoint{0.400575in}{0.668605in}}%
\pgfpathlineto{\pgfqpoint{0.393506in}{0.647755in}}%
\pgfpathlineto{\pgfqpoint{0.387376in}{0.627205in}}%
\pgfpathlineto{\pgfqpoint{0.382209in}{0.607039in}}%
\pgfpathlineto{\pgfqpoint{0.378026in}{0.587334in}}%
\pgfpathlineto{\pgfqpoint{0.374842in}{0.568168in}}%
\pgfpathlineto{\pgfqpoint{0.372672in}{0.549618in}}%
\pgfpathlineto{\pgfqpoint{0.371522in}{0.531757in}}%
\pgfpathlineto{\pgfqpoint{0.371398in}{0.514654in}}%
\pgfpathlineto{\pgfqpoint{0.372300in}{0.498378in}}%
\pgfpathlineto{\pgfqpoint{0.374225in}{0.482993in}}%
\pgfpathlineto{\pgfqpoint{0.377164in}{0.468559in}}%
\pgfpathlineto{\pgfqpoint{0.381107in}{0.455133in}}%
\pgfpathlineto{\pgfqpoint{0.386038in}{0.442769in}}%
\pgfpathlineto{\pgfqpoint{0.391938in}{0.431515in}}%
\pgfpathlineto{\pgfqpoint{0.398782in}{0.421416in}}%
\pgfpathlineto{\pgfqpoint{0.406545in}{0.412511in}}%
\pgfpathlineto{\pgfqpoint{0.415194in}{0.404836in}}%
\pgfpathlineto{\pgfqpoint{0.424698in}{0.398421in}}%
\pgfpathlineto{\pgfqpoint{0.435017in}{0.393291in}}%
\pgfpathlineto{\pgfqpoint{0.446111in}{0.389466in}}%
\pgfpathlineto{\pgfqpoint{0.457936in}{0.386962in}}%
\pgfpathlineto{\pgfqpoint{0.470446in}{0.385788in}}%
\pgfpathlineto{\pgfqpoint{0.483591in}{0.385949in}}%
\pgfpathlineto{\pgfqpoint{0.497320in}{0.387445in}}%
\pgfpathlineto{\pgfqpoint{0.511578in}{0.390270in}}%
\pgfpathlineto{\pgfqpoint{0.526309in}{0.394411in}}%
\pgfpathlineto{\pgfqpoint{0.541454in}{0.399854in}}%
\pgfpathlineto{\pgfqpoint{0.556955in}{0.406576in}}%
\pgfpathlineto{\pgfqpoint{0.572749in}{0.414552in}}%
\pgfpathlineto{\pgfqpoint{0.588775in}{0.423748in}}%
\pgfpathlineto{\pgfqpoint{0.604969in}{0.434131in}}%
\pgfpathlineto{\pgfqpoint{0.621268in}{0.445657in}}%
\pgfpathlineto{\pgfqpoint{0.637606in}{0.458282in}}%
\pgfpathlineto{\pgfqpoint{0.653920in}{0.471956in}}%
\pgfpathlineto{\pgfqpoint{0.670145in}{0.486625in}}%
\pgfpathlineto{\pgfqpoint{0.686218in}{0.502232in}}%
\pgfpathlineto{\pgfqpoint{0.702074in}{0.518714in}}%
\pgfpathlineto{\pgfqpoint{0.717651in}{0.536006in}}%
\pgfpathlineto{\pgfqpoint{0.732888in}{0.554040in}}%
\pgfpathlineto{\pgfqpoint{0.747724in}{0.572746in}}%
\pgfpathlineto{\pgfqpoint{0.762102in}{0.592049in}}%
\pgfpathlineto{\pgfqpoint{0.775963in}{0.611872in}}%
\pgfpathlineto{\pgfqpoint{0.789254in}{0.632139in}}%
\pgfpathlineto{\pgfqpoint{0.801922in}{0.652768in}}%
\pgfpathlineto{\pgfqpoint{0.813918in}{0.673679in}}%
\pgfpathlineto{\pgfqpoint{0.825192in}{0.694789in}}%
\pgfpathlineto{\pgfqpoint{0.835703in}{0.716014in}}%
\pgfpathlineto{\pgfqpoint{0.845406in}{0.737271in}}%
\pgfpathlineto{\pgfqpoint{0.854266in}{0.758477in}}%
\pgfpathlineto{\pgfqpoint{0.862246in}{0.779546in}}%
\pgfpathlineto{\pgfqpoint{0.869315in}{0.800396in}}%
\pgfpathlineto{\pgfqpoint{0.875444in}{0.820945in}}%
\pgfpathlineto{\pgfqpoint{0.880611in}{0.841112in}}%
\pgfpathlineto{\pgfqpoint{0.884795in}{0.860817in}}%
\pgfpathlineto{\pgfqpoint{0.887978in}{0.879983in}}%
\pgfpathlineto{\pgfqpoint{0.890149in}{0.898533in}}%
\pgfpathlineto{\pgfqpoint{0.891299in}{0.916394in}}%
\pgfpathlineto{\pgfqpoint{0.891423in}{0.933497in}}%
\pgfpathlineto{\pgfqpoint{0.890520in}{0.949773in}}%
\pgfpathlineto{\pgfqpoint{0.888596in}{0.965158in}}%
\pgfpathlineto{\pgfqpoint{0.885656in}{0.979592in}}%
\pgfpathlineto{\pgfqpoint{0.881713in}{0.993017in}}%
\pgfpathlineto{\pgfqpoint{0.876782in}{1.005381in}}%
\pgfpathlineto{\pgfqpoint{0.870883in}{1.016635in}}%
\pgfpathclose%
\pgfusepath{stroke,fill}%
\end{pgfscope}%
\begin{pgfscope}%
\pgfpathrectangle{\pgfqpoint{0.039236in}{0.039236in}}{\pgfqpoint{1.595582in}{1.177068in}}%
\pgfusepath{clip}%
\pgfsetbuttcap%
\pgfsetmiterjoin%
\definecolor{currentfill}{rgb}{0.254902,0.411765,0.882353}%
\pgfsetfillcolor{currentfill}%
\pgfsetfillopacity{0.600000}%
\pgfsetlinewidth{1.003750pt}%
\definecolor{currentstroke}{rgb}{0.000000,0.000000,0.000000}%
\pgfsetstrokecolor{currentstroke}%
\pgfsetstrokeopacity{0.600000}%
\pgfsetdash{{1.000000pt}{1.650000pt}}{0.000000pt}%
\pgfpathmoveto{\pgfqpoint{0.631410in}{0.781145in}}%
\pgfpathlineto{\pgfqpoint{0.664309in}{0.695540in}}%
\pgfpathlineto{\pgfqpoint{0.598512in}{0.695540in}}%
\pgfpathclose%
\pgfusepath{stroke,fill}%
\end{pgfscope}%
\begin{pgfscope}%
\pgfpathrectangle{\pgfqpoint{0.039236in}{0.039236in}}{\pgfqpoint{1.595582in}{1.177068in}}%
\pgfusepath{clip}%
\pgfsetbuttcap%
\pgfsetmiterjoin%
\definecolor{currentfill}{rgb}{0.627451,0.321569,0.176471}%
\pgfsetfillcolor{currentfill}%
\pgfsetfillopacity{0.800000}%
\pgfsetlinewidth{1.003750pt}%
\definecolor{currentstroke}{rgb}{0.000000,0.000000,0.000000}%
\pgfsetstrokecolor{currentstroke}%
\pgfsetstrokeopacity{0.800000}%
\pgfsetdash{}{0pt}%
\pgfpathmoveto{\pgfqpoint{0.960396in}{1.152100in}}%
\pgfpathlineto{\pgfqpoint{0.950068in}{1.151256in}}%
\pgfpathlineto{\pgfqpoint{0.939780in}{1.148725in}}%
\pgfpathlineto{\pgfqpoint{0.929573in}{1.144519in}}%
\pgfpathlineto{\pgfqpoint{0.919488in}{1.138653in}}%
\pgfpathlineto{\pgfqpoint{0.909565in}{1.131151in}}%
\pgfpathlineto{\pgfqpoint{0.899842in}{1.122043in}}%
\pgfpathlineto{\pgfqpoint{0.890358in}{1.111364in}}%
\pgfpathlineto{\pgfqpoint{0.881151in}{1.099156in}}%
\pgfpathlineto{\pgfqpoint{0.872256in}{1.085469in}}%
\pgfpathlineto{\pgfqpoint{0.863710in}{1.070355in}}%
\pgfpathlineto{\pgfqpoint{0.855544in}{1.053874in}}%
\pgfpathlineto{\pgfqpoint{0.847793in}{1.036092in}}%
\pgfpathlineto{\pgfqpoint{0.840486in}{1.017079in}}%
\pgfpathlineto{\pgfqpoint{0.833652in}{0.996909in}}%
\pgfpathlineto{\pgfqpoint{0.827319in}{0.975662in}}%
\pgfpathlineto{\pgfqpoint{0.821510in}{0.953423in}}%
\pgfpathlineto{\pgfqpoint{0.816250in}{0.930278in}}%
\pgfpathlineto{\pgfqpoint{0.811558in}{0.906320in}}%
\pgfpathlineto{\pgfqpoint{0.807454in}{0.881642in}}%
\pgfpathlineto{\pgfqpoint{0.803954in}{0.856342in}}%
\pgfpathlineto{\pgfqpoint{0.801071in}{0.830521in}}%
\pgfpathlineto{\pgfqpoint{0.798817in}{0.804279in}}%
\pgfpathlineto{\pgfqpoint{0.797200in}{0.777721in}}%
\pgfpathlineto{\pgfqpoint{0.796228in}{0.750951in}}%
\pgfpathlineto{\pgfqpoint{0.795903in}{0.724075in}}%
\pgfpathlineto{\pgfqpoint{0.796228in}{0.697199in}}%
\pgfpathlineto{\pgfqpoint{0.797200in}{0.670430in}}%
\pgfpathlineto{\pgfqpoint{0.798817in}{0.643872in}}%
\pgfpathlineto{\pgfqpoint{0.801071in}{0.617630in}}%
\pgfpathlineto{\pgfqpoint{0.803954in}{0.591808in}}%
\pgfpathlineto{\pgfqpoint{0.807454in}{0.566509in}}%
\pgfpathlineto{\pgfqpoint{0.811558in}{0.541831in}}%
\pgfpathlineto{\pgfqpoint{0.816250in}{0.517873in}}%
\pgfpathlineto{\pgfqpoint{0.821510in}{0.494728in}}%
\pgfpathlineto{\pgfqpoint{0.827319in}{0.472489in}}%
\pgfpathlineto{\pgfqpoint{0.833652in}{0.451242in}}%
\pgfpathlineto{\pgfqpoint{0.840486in}{0.431072in}}%
\pgfpathlineto{\pgfqpoint{0.847793in}{0.412059in}}%
\pgfpathlineto{\pgfqpoint{0.855544in}{0.394277in}}%
\pgfpathlineto{\pgfqpoint{0.863710in}{0.377796in}}%
\pgfpathlineto{\pgfqpoint{0.872256in}{0.362682in}}%
\pgfpathlineto{\pgfqpoint{0.881151in}{0.348994in}}%
\pgfpathlineto{\pgfqpoint{0.890358in}{0.336787in}}%
\pgfpathlineto{\pgfqpoint{0.899842in}{0.326108in}}%
\pgfpathlineto{\pgfqpoint{0.909565in}{0.317000in}}%
\pgfpathlineto{\pgfqpoint{0.919488in}{0.309498in}}%
\pgfpathlineto{\pgfqpoint{0.929573in}{0.303632in}}%
\pgfpathlineto{\pgfqpoint{0.939780in}{0.299426in}}%
\pgfpathlineto{\pgfqpoint{0.950068in}{0.296895in}}%
\pgfpathlineto{\pgfqpoint{0.960396in}{0.296051in}}%
\pgfpathlineto{\pgfqpoint{0.970725in}{0.296895in}}%
\pgfpathlineto{\pgfqpoint{0.981013in}{0.299426in}}%
\pgfpathlineto{\pgfqpoint{0.991219in}{0.303632in}}%
\pgfpathlineto{\pgfqpoint{1.001304in}{0.309498in}}%
\pgfpathlineto{\pgfqpoint{1.011227in}{0.317000in}}%
\pgfpathlineto{\pgfqpoint{1.020950in}{0.326108in}}%
\pgfpathlineto{\pgfqpoint{1.030434in}{0.336787in}}%
\pgfpathlineto{\pgfqpoint{1.039641in}{0.348994in}}%
\pgfpathlineto{\pgfqpoint{1.048536in}{0.362682in}}%
\pgfpathlineto{\pgfqpoint{1.057083in}{0.377796in}}%
\pgfpathlineto{\pgfqpoint{1.065248in}{0.394277in}}%
\pgfpathlineto{\pgfqpoint{1.072999in}{0.412059in}}%
\pgfpathlineto{\pgfqpoint{1.080306in}{0.431072in}}%
\pgfpathlineto{\pgfqpoint{1.087140in}{0.451242in}}%
\pgfpathlineto{\pgfqpoint{1.093474in}{0.472489in}}%
\pgfpathlineto{\pgfqpoint{1.099282in}{0.494728in}}%
\pgfpathlineto{\pgfqpoint{1.104542in}{0.517873in}}%
\pgfpathlineto{\pgfqpoint{1.109234in}{0.541831in}}%
\pgfpathlineto{\pgfqpoint{1.113338in}{0.566509in}}%
\pgfpathlineto{\pgfqpoint{1.116838in}{0.591808in}}%
\pgfpathlineto{\pgfqpoint{1.119721in}{0.617630in}}%
\pgfpathlineto{\pgfqpoint{1.121975in}{0.643872in}}%
\pgfpathlineto{\pgfqpoint{1.123592in}{0.670430in}}%
\pgfpathlineto{\pgfqpoint{1.124564in}{0.697199in}}%
\pgfpathlineto{\pgfqpoint{1.124889in}{0.724075in}}%
\pgfpathlineto{\pgfqpoint{1.124564in}{0.750951in}}%
\pgfpathlineto{\pgfqpoint{1.123592in}{0.777721in}}%
\pgfpathlineto{\pgfqpoint{1.121975in}{0.804279in}}%
\pgfpathlineto{\pgfqpoint{1.119721in}{0.830521in}}%
\pgfpathlineto{\pgfqpoint{1.116838in}{0.856342in}}%
\pgfpathlineto{\pgfqpoint{1.113338in}{0.881642in}}%
\pgfpathlineto{\pgfqpoint{1.109234in}{0.906320in}}%
\pgfpathlineto{\pgfqpoint{1.104542in}{0.930278in}}%
\pgfpathlineto{\pgfqpoint{1.099282in}{0.953423in}}%
\pgfpathlineto{\pgfqpoint{1.093474in}{0.975662in}}%
\pgfpathlineto{\pgfqpoint{1.087140in}{0.996909in}}%
\pgfpathlineto{\pgfqpoint{1.080306in}{1.017079in}}%
\pgfpathlineto{\pgfqpoint{1.072999in}{1.036092in}}%
\pgfpathlineto{\pgfqpoint{1.065248in}{1.053874in}}%
\pgfpathlineto{\pgfqpoint{1.057083in}{1.070355in}}%
\pgfpathlineto{\pgfqpoint{1.048536in}{1.085469in}}%
\pgfpathlineto{\pgfqpoint{1.039641in}{1.099156in}}%
\pgfpathlineto{\pgfqpoint{1.030434in}{1.111364in}}%
\pgfpathlineto{\pgfqpoint{1.020950in}{1.122043in}}%
\pgfpathlineto{\pgfqpoint{1.011227in}{1.131151in}}%
\pgfpathlineto{\pgfqpoint{1.001304in}{1.138653in}}%
\pgfpathlineto{\pgfqpoint{0.991219in}{1.144519in}}%
\pgfpathlineto{\pgfqpoint{0.981013in}{1.148725in}}%
\pgfpathlineto{\pgfqpoint{0.970725in}{1.151256in}}%
\pgfpathclose%
\pgfusepath{stroke,fill}%
\end{pgfscope}%
\begin{pgfscope}%
\pgfpathrectangle{\pgfqpoint{0.039236in}{0.039236in}}{\pgfqpoint{1.595582in}{1.177068in}}%
\pgfusepath{clip}%
\pgfsetbuttcap%
\pgfsetmiterjoin%
\definecolor{currentfill}{rgb}{0.254902,0.411765,0.882353}%
\pgfsetfillcolor{currentfill}%
\pgfsetfillopacity{0.600000}%
\pgfsetlinewidth{1.003750pt}%
\definecolor{currentstroke}{rgb}{0.000000,0.000000,0.000000}%
\pgfsetstrokecolor{currentstroke}%
\pgfsetstrokeopacity{0.600000}%
\pgfsetdash{{1.000000pt}{1.650000pt}}{0.000000pt}%
\pgfpathmoveto{\pgfqpoint{0.960396in}{0.781145in}}%
\pgfpathlineto{\pgfqpoint{0.993295in}{0.695540in}}%
\pgfpathlineto{\pgfqpoint{0.927498in}{0.695540in}}%
\pgfpathclose%
\pgfusepath{stroke,fill}%
\end{pgfscope}%
\begin{pgfscope}%
\pgfpathrectangle{\pgfqpoint{0.039236in}{0.039236in}}{\pgfqpoint{1.595582in}{1.177068in}}%
\pgfusepath{clip}%
\pgfsetbuttcap%
\pgfsetmiterjoin%
\definecolor{currentfill}{rgb}{1.000000,0.000000,1.000000}%
\pgfsetfillcolor{currentfill}%
\pgfsetfillopacity{0.800000}%
\pgfsetlinewidth{1.003750pt}%
\definecolor{currentstroke}{rgb}{0.000000,0.000000,0.000000}%
\pgfsetstrokecolor{currentstroke}%
\pgfsetstrokeopacity{0.800000}%
\pgfsetdash{}{0pt}%
\pgfpathmoveto{\pgfqpoint{1.618368in}{0.296051in}}%
\pgfpathlineto{\pgfqpoint{1.617719in}{0.309488in}}%
\pgfpathlineto{\pgfqpoint{1.615774in}{0.322873in}}%
\pgfpathlineto{\pgfqpoint{1.612541in}{0.336152in}}%
\pgfpathlineto{\pgfqpoint{1.608032in}{0.349273in}}%
\pgfpathlineto{\pgfqpoint{1.602266in}{0.362184in}}%
\pgfpathlineto{\pgfqpoint{1.595265in}{0.374834in}}%
\pgfpathlineto{\pgfqpoint{1.587057in}{0.387173in}}%
\pgfpathlineto{\pgfqpoint{1.577675in}{0.399152in}}%
\pgfpathlineto{\pgfqpoint{1.567154in}{0.410724in}}%
\pgfpathlineto{\pgfqpoint{1.555537in}{0.421844in}}%
\pgfpathlineto{\pgfqpoint{1.542870in}{0.432467in}}%
\pgfpathlineto{\pgfqpoint{1.529202in}{0.442552in}}%
\pgfpathlineto{\pgfqpoint{1.514588in}{0.452059in}}%
\pgfpathlineto{\pgfqpoint{1.499086in}{0.460950in}}%
\pgfpathlineto{\pgfqpoint{1.482755in}{0.469190in}}%
\pgfpathlineto{\pgfqpoint{1.465661in}{0.476747in}}%
\pgfpathlineto{\pgfqpoint{1.447872in}{0.483591in}}%
\pgfpathlineto{\pgfqpoint{1.429457in}{0.489695in}}%
\pgfpathlineto{\pgfqpoint{1.410490in}{0.495034in}}%
\pgfpathlineto{\pgfqpoint{1.391044in}{0.499588in}}%
\pgfpathlineto{\pgfqpoint{1.371197in}{0.503339in}}%
\pgfpathlineto{\pgfqpoint{1.351028in}{0.506272in}}%
\pgfpathlineto{\pgfqpoint{1.330615in}{0.508375in}}%
\pgfpathlineto{\pgfqpoint{1.310039in}{0.509641in}}%
\pgfpathlineto{\pgfqpoint{1.289382in}{0.510063in}}%
\pgfpathlineto{\pgfqpoint{1.268725in}{0.509641in}}%
\pgfpathlineto{\pgfqpoint{1.248149in}{0.508375in}}%
\pgfpathlineto{\pgfqpoint{1.227736in}{0.506272in}}%
\pgfpathlineto{\pgfqpoint{1.207567in}{0.503339in}}%
\pgfpathlineto{\pgfqpoint{1.187720in}{0.499588in}}%
\pgfpathlineto{\pgfqpoint{1.168274in}{0.495034in}}%
\pgfpathlineto{\pgfqpoint{1.149307in}{0.489695in}}%
\pgfpathlineto{\pgfqpoint{1.130892in}{0.483591in}}%
\pgfpathlineto{\pgfqpoint{1.113103in}{0.476747in}}%
\pgfpathlineto{\pgfqpoint{1.096009in}{0.469190in}}%
\pgfpathlineto{\pgfqpoint{1.079679in}{0.460950in}}%
\pgfpathlineto{\pgfqpoint{1.064176in}{0.452059in}}%
\pgfpathlineto{\pgfqpoint{1.049562in}{0.442552in}}%
\pgfpathlineto{\pgfqpoint{1.035894in}{0.432467in}}%
\pgfpathlineto{\pgfqpoint{1.023227in}{0.421844in}}%
\pgfpathlineto{\pgfqpoint{1.011610in}{0.410724in}}%
\pgfpathlineto{\pgfqpoint{1.001089in}{0.399152in}}%
\pgfpathlineto{\pgfqpoint{0.991707in}{0.387173in}}%
\pgfpathlineto{\pgfqpoint{0.983499in}{0.374834in}}%
\pgfpathlineto{\pgfqpoint{0.976498in}{0.362184in}}%
\pgfpathlineto{\pgfqpoint{0.970732in}{0.349273in}}%
\pgfpathlineto{\pgfqpoint{0.966223in}{0.336152in}}%
\pgfpathlineto{\pgfqpoint{0.962990in}{0.322873in}}%
\pgfpathlineto{\pgfqpoint{0.961045in}{0.309488in}}%
\pgfpathlineto{\pgfqpoint{0.960396in}{0.296051in}}%
\pgfpathlineto{\pgfqpoint{0.961045in}{0.282613in}}%
\pgfpathlineto{\pgfqpoint{0.962990in}{0.269228in}}%
\pgfpathlineto{\pgfqpoint{0.966223in}{0.255949in}}%
\pgfpathlineto{\pgfqpoint{0.970732in}{0.242828in}}%
\pgfpathlineto{\pgfqpoint{0.976498in}{0.229917in}}%
\pgfpathlineto{\pgfqpoint{0.983499in}{0.217267in}}%
\pgfpathlineto{\pgfqpoint{0.991707in}{0.204928in}}%
\pgfpathlineto{\pgfqpoint{1.001089in}{0.192949in}}%
\pgfpathlineto{\pgfqpoint{1.011610in}{0.181377in}}%
\pgfpathlineto{\pgfqpoint{1.023227in}{0.170257in}}%
\pgfpathlineto{\pgfqpoint{1.035894in}{0.159634in}}%
\pgfpathlineto{\pgfqpoint{1.049562in}{0.149549in}}%
\pgfpathlineto{\pgfqpoint{1.064176in}{0.140042in}}%
\pgfpathlineto{\pgfqpoint{1.079679in}{0.131151in}}%
\pgfpathlineto{\pgfqpoint{1.096009in}{0.122911in}}%
\pgfpathlineto{\pgfqpoint{1.113103in}{0.115354in}}%
\pgfpathlineto{\pgfqpoint{1.130892in}{0.108510in}}%
\pgfpathlineto{\pgfqpoint{1.149307in}{0.102406in}}%
\pgfpathlineto{\pgfqpoint{1.168274in}{0.097067in}}%
\pgfpathlineto{\pgfqpoint{1.187720in}{0.092513in}}%
\pgfpathlineto{\pgfqpoint{1.207567in}{0.088762in}}%
\pgfpathlineto{\pgfqpoint{1.227736in}{0.085829in}}%
\pgfpathlineto{\pgfqpoint{1.248149in}{0.083726in}}%
\pgfpathlineto{\pgfqpoint{1.268725in}{0.082460in}}%
\pgfpathlineto{\pgfqpoint{1.289382in}{0.082038in}}%
\pgfpathlineto{\pgfqpoint{1.310039in}{0.082460in}}%
\pgfpathlineto{\pgfqpoint{1.330615in}{0.083726in}}%
\pgfpathlineto{\pgfqpoint{1.351028in}{0.085829in}}%
\pgfpathlineto{\pgfqpoint{1.371197in}{0.088762in}}%
\pgfpathlineto{\pgfqpoint{1.391044in}{0.092513in}}%
\pgfpathlineto{\pgfqpoint{1.410490in}{0.097067in}}%
\pgfpathlineto{\pgfqpoint{1.429457in}{0.102406in}}%
\pgfpathlineto{\pgfqpoint{1.447872in}{0.108510in}}%
\pgfpathlineto{\pgfqpoint{1.465661in}{0.115354in}}%
\pgfpathlineto{\pgfqpoint{1.482755in}{0.122911in}}%
\pgfpathlineto{\pgfqpoint{1.499086in}{0.131151in}}%
\pgfpathlineto{\pgfqpoint{1.514588in}{0.140042in}}%
\pgfpathlineto{\pgfqpoint{1.529202in}{0.149549in}}%
\pgfpathlineto{\pgfqpoint{1.542870in}{0.159634in}}%
\pgfpathlineto{\pgfqpoint{1.555537in}{0.170257in}}%
\pgfpathlineto{\pgfqpoint{1.567154in}{0.181377in}}%
\pgfpathlineto{\pgfqpoint{1.577675in}{0.192949in}}%
\pgfpathlineto{\pgfqpoint{1.587057in}{0.204928in}}%
\pgfpathlineto{\pgfqpoint{1.595265in}{0.217267in}}%
\pgfpathlineto{\pgfqpoint{1.602266in}{0.229917in}}%
\pgfpathlineto{\pgfqpoint{1.608032in}{0.242828in}}%
\pgfpathlineto{\pgfqpoint{1.612541in}{0.255949in}}%
\pgfpathlineto{\pgfqpoint{1.615774in}{0.269228in}}%
\pgfpathlineto{\pgfqpoint{1.617719in}{0.282613in}}%
\pgfpathclose%
\pgfusepath{stroke,fill}%
\end{pgfscope}%
\begin{pgfscope}%
\pgfpathrectangle{\pgfqpoint{0.039236in}{0.039236in}}{\pgfqpoint{1.595582in}{1.177068in}}%
\pgfusepath{clip}%
\pgfsetbuttcap%
\pgfsetmiterjoin%
\definecolor{currentfill}{rgb}{0.254902,0.411765,0.882353}%
\pgfsetfillcolor{currentfill}%
\pgfsetfillopacity{0.600000}%
\pgfsetlinewidth{1.003750pt}%
\definecolor{currentstroke}{rgb}{0.000000,0.000000,0.000000}%
\pgfsetstrokecolor{currentstroke}%
\pgfsetstrokeopacity{0.600000}%
\pgfsetdash{{1.000000pt}{1.650000pt}}{0.000000pt}%
\pgfpathmoveto{\pgfqpoint{1.289382in}{0.353121in}}%
\pgfpathlineto{\pgfqpoint{1.322281in}{0.267516in}}%
\pgfpathlineto{\pgfqpoint{1.256483in}{0.267516in}}%
\pgfpathclose%
\pgfusepath{stroke,fill}%
\end{pgfscope}%
\begin{pgfscope}%
\pgfpathrectangle{\pgfqpoint{0.039236in}{0.039236in}}{\pgfqpoint{1.595582in}{1.177068in}}%
\pgfusepath{clip}%
\pgfsetbuttcap%
\pgfsetmiterjoin%
\definecolor{currentfill}{rgb}{1.000000,0.388235,0.278431}%
\pgfsetfillcolor{currentfill}%
\pgfsetfillopacity{0.800000}%
\pgfsetlinewidth{1.003750pt}%
\definecolor{currentstroke}{rgb}{0.000000,0.000000,0.000000}%
\pgfsetstrokecolor{currentstroke}%
\pgfsetstrokeopacity{0.800000}%
\pgfsetdash{}{0pt}%
\pgfpathmoveto{\pgfqpoint{0.466917in}{0.360254in}}%
\pgfpathlineto{\pgfqpoint{0.302424in}{0.595629in}}%
\pgfpathlineto{\pgfqpoint{0.302424in}{1.002292in}}%
\pgfpathlineto{\pgfqpoint{0.055685in}{1.002292in}}%
\pgfpathlineto{\pgfqpoint{0.055685in}{0.895285in}}%
\pgfpathlineto{\pgfqpoint{0.220178in}{0.605926in}}%
\pgfpathlineto{\pgfqpoint{0.220178in}{0.253248in}}%
\pgfpathlineto{\pgfqpoint{0.466917in}{0.253248in}}%
\pgfpathclose%
\pgfusepath{stroke,fill}%
\end{pgfscope}%
\begin{pgfscope}%
\pgfpathrectangle{\pgfqpoint{0.039236in}{0.039236in}}{\pgfqpoint{1.595582in}{1.177068in}}%
\pgfusepath{clip}%
\pgfsetbuttcap%
\pgfsetmiterjoin%
\definecolor{currentfill}{rgb}{0.254902,0.411765,0.882353}%
\pgfsetfillcolor{currentfill}%
\pgfsetfillopacity{0.600000}%
\pgfsetlinewidth{1.003750pt}%
\definecolor{currentstroke}{rgb}{0.000000,0.000000,0.000000}%
\pgfsetstrokecolor{currentstroke}%
\pgfsetstrokeopacity{0.600000}%
\pgfsetdash{{1.000000pt}{1.650000pt}}{0.000000pt}%
\pgfpathmoveto{\pgfqpoint{0.247072in}{0.674833in}}%
\pgfpathlineto{\pgfqpoint{0.273720in}{0.605493in}}%
\pgfpathlineto{\pgfqpoint{0.220424in}{0.605493in}}%
\pgfpathclose%
\pgfusepath{stroke,fill}%
\end{pgfscope}%
\begin{pgfscope}%
\pgfpathrectangle{\pgfqpoint{0.039236in}{0.039236in}}{\pgfqpoint{1.595582in}{1.177068in}}%
\pgfusepath{clip}%
\pgfsetbuttcap%
\pgfsetmiterjoin%
\definecolor{currentfill}{rgb}{1.000000,0.000000,0.000000}%
\pgfsetfillcolor{currentfill}%
\pgfsetlinewidth{1.003750pt}%
\definecolor{currentstroke}{rgb}{1.000000,0.000000,0.000000}%
\pgfsetstrokecolor{currentstroke}%
\pgfsetdash{}{0pt}%
\pgfsys@defobject{currentmarker}{\pgfqpoint{-0.016667in}{-0.016667in}}{\pgfqpoint{0.016667in}{0.016667in}}{%
\pgfpathmoveto{\pgfqpoint{0.000000in}{0.016667in}}%
\pgfpathlineto{\pgfqpoint{-0.016667in}{-0.016667in}}%
\pgfpathlineto{\pgfqpoint{0.016667in}{-0.016667in}}%
\pgfpathclose%
\pgfusepath{stroke,fill}%
}%
\begin{pgfscope}%
\pgfsys@transformshift{0.351772in}{0.681273in}%
\pgfsys@useobject{currentmarker}{}%
\end{pgfscope}%
\begin{pgfscope}%
\pgfsys@transformshift{0.631410in}{0.274649in}%
\pgfsys@useobject{currentmarker}{}%
\end{pgfscope}%
\end{pgfscope}%
\end{pgfpicture}%
\makeatother%
\endgroup%

%% file: figures/step1_cluster.pgf
\begingroup%
\makeatletter%
\begin{pgfpicture}%
\pgfpathrectangle{\pgfpointorigin}{\pgfqpoint{1.674053in}{1.255540in}}%
\pgfusepath{use as bounding box, clip}%
\begin{pgfscope}%
\pgfsetbuttcap%
\pgfsetmiterjoin%
\definecolor{currentfill}{rgb}{1.000000,1.000000,1.000000}%
\pgfsetfillcolor{currentfill}%
\pgfsetlinewidth{0.000000pt}%
\definecolor{currentstroke}{rgb}{1.000000,1.000000,1.000000}%
\pgfsetstrokecolor{currentstroke}%
\pgfsetdash{}{0pt}%
\pgfpathmoveto{\pgfqpoint{0.000000in}{0.000000in}}%
\pgfpathlineto{\pgfqpoint{1.674053in}{0.000000in}}%
\pgfpathlineto{\pgfqpoint{1.674053in}{1.255540in}}%
\pgfpathlineto{\pgfqpoint{0.000000in}{1.255540in}}%
\pgfpathclose%
\pgfusepath{fill}%
\end{pgfscope}%
\begin{pgfscope}%
\pgfpathrectangle{\pgfqpoint{0.039236in}{0.039236in}}{\pgfqpoint{1.595582in}{1.177068in}}%
\pgfusepath{clip}%
\pgfsetbuttcap%
\pgfsetmiterjoin%
\definecolor{currentfill}{rgb}{0.827451,0.827451,0.827451}%
\pgfsetfillcolor{currentfill}%
\pgfsetfillopacity{0.800000}%
\pgfsetlinewidth{1.003750pt}%
\definecolor{currentstroke}{rgb}{0.000000,0.000000,0.000000}%
\pgfsetstrokecolor{currentstroke}%
\pgfsetstrokeopacity{0.800000}%
\pgfsetdash{}{0pt}%
\pgfpathmoveto{\pgfqpoint{0.747724in}{0.572746in}}%
\pgfpathcurveto{\pgfqpoint{0.809418in}{0.653012in}}{\pgfqpoint{0.856338in}{0.745946in}}{\pgfqpoint{0.878150in}{0.831082in}}%
\pgfpathcurveto{\pgfqpoint{0.899962in}{0.916217in}}{\pgfqpoint{0.894885in}{0.986602in}}{\pgfqpoint{0.864038in}{1.026735in}}%
\pgfpathcurveto{\pgfqpoint{0.833192in}{1.066868in}}{\pgfqpoint{0.779093in}{1.073472in}}{\pgfqpoint{0.713657in}{1.045094in}}%
\pgfpathcurveto{\pgfqpoint{0.648221in}{1.016716in}}{\pgfqpoint{0.576790in}{0.955671in}}{\pgfqpoint{0.515096in}{0.875405in}}%
\pgfpathcurveto{\pgfqpoint{0.453402in}{0.795139in}}{\pgfqpoint{0.406483in}{0.702204in}}{\pgfqpoint{0.384671in}{0.617069in}}%
\pgfpathcurveto{\pgfqpoint{0.362859in}{0.531934in}}{\pgfqpoint{0.367935in}{0.461549in}}{\pgfqpoint{0.398782in}{0.421416in}}%
\pgfpathcurveto{\pgfqpoint{0.429629in}{0.381283in}}{\pgfqpoint{0.483728in}{0.374678in}}{\pgfqpoint{0.549164in}{0.403057in}}%
\pgfpathcurveto{\pgfqpoint{0.614600in}{0.431435in}}{\pgfqpoint{0.686031in}{0.492480in}}{\pgfqpoint{0.747724in}{0.572746in}}%
\pgfpathclose%
\pgfusepath{stroke,fill}%
\end{pgfscope}%
\begin{pgfscope}%
\pgfpathrectangle{\pgfqpoint{0.039236in}{0.039236in}}{\pgfqpoint{1.595582in}{1.177068in}}%
\pgfusepath{clip}%
\pgfsetbuttcap%
\pgfsetmiterjoin%
\definecolor{currentfill}{rgb}{0.827451,0.827451,0.827451}%
\pgfsetfillcolor{currentfill}%
\pgfsetfillopacity{0.800000}%
\pgfsetlinewidth{1.003750pt}%
\definecolor{currentstroke}{rgb}{0.000000,0.000000,0.000000}%
\pgfsetstrokecolor{currentstroke}%
\pgfsetstrokeopacity{0.800000}%
\pgfsetdash{}{0pt}%
\pgfpathmoveto{\pgfqpoint{1.124889in}{0.724075in}}%
\pgfpathcurveto{\pgfqpoint{1.124889in}{0.837589in}}{\pgfqpoint{1.107557in}{0.946468in}}{\pgfqpoint{1.076710in}{1.026735in}}%
\pgfpathcurveto{\pgfqpoint{1.045863in}{1.107001in}}{\pgfqpoint{1.004020in}{1.152100in}}{\pgfqpoint{0.960396in}{1.152100in}}%
\pgfpathcurveto{\pgfqpoint{0.916772in}{1.152100in}}{\pgfqpoint{0.874929in}{1.107001in}}{\pgfqpoint{0.844082in}{1.026735in}}%
\pgfpathcurveto{\pgfqpoint{0.813235in}{0.946468in}}{\pgfqpoint{0.795903in}{0.837589in}}{\pgfqpoint{0.795903in}{0.724075in}}%
\pgfpathcurveto{\pgfqpoint{0.795903in}{0.610562in}}{\pgfqpoint{0.813235in}{0.501682in}}{\pgfqpoint{0.844082in}{0.421416in}}%
\pgfpathcurveto{\pgfqpoint{0.874929in}{0.341150in}}{\pgfqpoint{0.916772in}{0.296051in}}{\pgfqpoint{0.960396in}{0.296051in}}%
\pgfpathcurveto{\pgfqpoint{1.004020in}{0.296051in}}{\pgfqpoint{1.045863in}{0.341150in}}{\pgfqpoint{1.076710in}{0.421416in}}%
\pgfpathcurveto{\pgfqpoint{1.107557in}{0.501682in}}{\pgfqpoint{1.124889in}{0.610562in}}{\pgfqpoint{1.124889in}{0.724075in}}%
\pgfpathclose%
\pgfusepath{stroke,fill}%
\end{pgfscope}%
\begin{pgfscope}%
\pgfpathrectangle{\pgfqpoint{0.039236in}{0.039236in}}{\pgfqpoint{1.595582in}{1.177068in}}%
\pgfusepath{clip}%
\pgfsetbuttcap%
\pgfsetmiterjoin%
\definecolor{currentfill}{rgb}{0.827451,0.827451,0.827451}%
\pgfsetfillcolor{currentfill}%
\pgfsetfillopacity{0.800000}%
\pgfsetlinewidth{1.003750pt}%
\definecolor{currentstroke}{rgb}{0.000000,0.000000,0.000000}%
\pgfsetstrokecolor{currentstroke}%
\pgfsetstrokeopacity{0.800000}%
\pgfsetdash{}{0pt}%
\pgfpathmoveto{\pgfqpoint{1.289382in}{0.082038in}}%
\pgfpathcurveto{\pgfqpoint{1.376630in}{0.082038in}}{\pgfqpoint{1.460316in}{0.104588in}}{\pgfqpoint{1.522010in}{0.144721in}}%
\pgfpathcurveto{\pgfqpoint{1.583704in}{0.184854in}}{\pgfqpoint{1.618368in}{0.239294in}}{\pgfqpoint{1.618368in}{0.296051in}}%
\pgfpathcurveto{\pgfqpoint{1.618368in}{0.352807in}}{\pgfqpoint{1.583704in}{0.407247in}}{\pgfqpoint{1.522010in}{0.447380in}}%
\pgfpathcurveto{\pgfqpoint{1.460316in}{0.487513in}}{\pgfqpoint{1.376630in}{0.510063in}}{\pgfqpoint{1.289382in}{0.510063in}}%
\pgfpathcurveto{\pgfqpoint{1.202134in}{0.510063in}}{\pgfqpoint{1.118448in}{0.487513in}}{\pgfqpoint{1.056754in}{0.447380in}}%
\pgfpathcurveto{\pgfqpoint{0.995060in}{0.407247in}}{\pgfqpoint{0.960396in}{0.352807in}}{\pgfqpoint{0.960396in}{0.296051in}}%
\pgfpathcurveto{\pgfqpoint{0.960396in}{0.239294in}}{\pgfqpoint{0.995060in}{0.184854in}}{\pgfqpoint{1.056754in}{0.144721in}}%
\pgfpathcurveto{\pgfqpoint{1.118448in}{0.104588in}}{\pgfqpoint{1.202134in}{0.082038in}}{\pgfqpoint{1.289382in}{0.082038in}}%
\pgfpathclose%
\pgfusepath{stroke,fill}%
\end{pgfscope}%
\begin{pgfscope}%
\pgfpathrectangle{\pgfqpoint{0.039236in}{0.039236in}}{\pgfqpoint{1.595582in}{1.177068in}}%
\pgfusepath{clip}%
\pgfsetbuttcap%
\pgfsetmiterjoin%
\definecolor{currentfill}{rgb}{0.827451,0.827451,0.827451}%
\pgfsetfillcolor{currentfill}%
\pgfsetfillopacity{0.800000}%
\pgfsetlinewidth{1.003750pt}%
\definecolor{currentstroke}{rgb}{0.000000,0.000000,0.000000}%
\pgfsetstrokecolor{currentstroke}%
\pgfsetstrokeopacity{0.800000}%
\pgfsetdash{}{0pt}%
\pgfpathmoveto{\pgfqpoint{0.055685in}{1.002292in}}%
\pgfpathlineto{\pgfqpoint{0.302424in}{1.002292in}}%
\pgfpathlineto{\pgfqpoint{0.302424in}{0.360254in}}%
\pgfpathlineto{\pgfqpoint{0.466917in}{0.360254in}}%
\pgfpathlineto{\pgfqpoint{0.466917in}{0.253248in}}%
\pgfpathlineto{\pgfqpoint{0.220178in}{0.253248in}}%
\pgfpathlineto{\pgfqpoint{0.220178in}{0.895285in}}%
\pgfpathlineto{\pgfqpoint{0.055685in}{0.895285in}}%
\pgfpathclose%
\pgfusepath{stroke,fill}%
\end{pgfscope}%
\begin{pgfscope}%
\pgfpathrectangle{\pgfqpoint{0.039236in}{0.039236in}}{\pgfqpoint{1.595582in}{1.177068in}}%
\pgfusepath{clip}%
\pgfsetbuttcap%
\pgfsetmiterjoin%
\definecolor{currentfill}{rgb}{1.000000,0.000000,0.000000}%
\pgfsetfillcolor{currentfill}%
\pgfsetlinewidth{1.003750pt}%
\definecolor{currentstroke}{rgb}{1.000000,0.000000,0.000000}%
\pgfsetstrokecolor{currentstroke}%
\pgfsetdash{}{0pt}%
\pgfsys@defobject{currentmarker}{\pgfqpoint{-0.016667in}{-0.016667in}}{\pgfqpoint{0.016667in}{0.016667in}}{%
\pgfpathmoveto{\pgfqpoint{0.000000in}{0.016667in}}%
\pgfpathlineto{\pgfqpoint{-0.016667in}{-0.016667in}}%
\pgfpathlineto{\pgfqpoint{0.016667in}{-0.016667in}}%
\pgfpathclose%
\pgfusepath{stroke,fill}%
}%
\begin{pgfscope}%
\pgfsys@transformshift{0.351772in}{0.681273in}%
\pgfsys@useobject{currentmarker}{}%
\end{pgfscope}%
\begin{pgfscope}%
\pgfsys@transformshift{0.631410in}{0.274649in}%
\pgfsys@useobject{currentmarker}{}%
\end{pgfscope}%
\end{pgfscope}%
\end{pgfpicture}%
\makeatother%
\endgroup%

%% file: figures/step2_admker0.pgf
\begingroup%
\makeatletter%
\begin{pgfpicture}%
\pgfpathrectangle{\pgfpointorigin}{\pgfqpoint{1.674053in}{1.255540in}}%
\pgfusepath{use as bounding box, clip}%
\begin{pgfscope}%
\pgfsetbuttcap%
\pgfsetmiterjoin%
\definecolor{currentfill}{rgb}{1.000000,1.000000,1.000000}%
\pgfsetfillcolor{currentfill}%
\pgfsetlinewidth{0.000000pt}%
\definecolor{currentstroke}{rgb}{1.000000,1.000000,1.000000}%
\pgfsetstrokecolor{currentstroke}%
\pgfsetdash{}{0pt}%
\pgfpathmoveto{\pgfqpoint{0.000000in}{0.000000in}}%
\pgfpathlineto{\pgfqpoint{1.674053in}{0.000000in}}%
\pgfpathlineto{\pgfqpoint{1.674053in}{1.255540in}}%
\pgfpathlineto{\pgfqpoint{0.000000in}{1.255540in}}%
\pgfpathclose%
\pgfusepath{fill}%
\end{pgfscope}%
\begin{pgfscope}%
\pgfpathrectangle{\pgfqpoint{0.039236in}{0.039236in}}{\pgfqpoint{1.595582in}{1.177068in}}%
\pgfusepath{clip}%
\pgfsetbuttcap%
\pgfsetmiterjoin%
\definecolor{currentfill}{rgb}{0.827451,0.827451,0.827451}%
\pgfsetfillcolor{currentfill}%
\pgfsetfillopacity{0.800000}%
\pgfsetlinewidth{1.003750pt}%
\definecolor{currentstroke}{rgb}{0.000000,0.000000,0.000000}%
\pgfsetstrokecolor{currentstroke}%
\pgfsetstrokeopacity{0.800000}%
\pgfsetdash{}{0pt}%
\pgfpathmoveto{\pgfqpoint{0.747724in}{0.572746in}}%
\pgfpathcurveto{\pgfqpoint{0.809418in}{0.653012in}}{\pgfqpoint{0.856338in}{0.745946in}}{\pgfqpoint{0.878150in}{0.831082in}}%
\pgfpathcurveto{\pgfqpoint{0.899962in}{0.916217in}}{\pgfqpoint{0.894885in}{0.986602in}}{\pgfqpoint{0.864038in}{1.026735in}}%
\pgfpathcurveto{\pgfqpoint{0.833192in}{1.066868in}}{\pgfqpoint{0.779093in}{1.073472in}}{\pgfqpoint{0.713657in}{1.045094in}}%
\pgfpathcurveto{\pgfqpoint{0.648221in}{1.016716in}}{\pgfqpoint{0.576790in}{0.955671in}}{\pgfqpoint{0.515096in}{0.875405in}}%
\pgfpathcurveto{\pgfqpoint{0.453402in}{0.795139in}}{\pgfqpoint{0.406483in}{0.702204in}}{\pgfqpoint{0.384671in}{0.617069in}}%
\pgfpathcurveto{\pgfqpoint{0.362859in}{0.531934in}}{\pgfqpoint{0.367935in}{0.461549in}}{\pgfqpoint{0.398782in}{0.421416in}}%
\pgfpathcurveto{\pgfqpoint{0.429629in}{0.381283in}}{\pgfqpoint{0.483728in}{0.374678in}}{\pgfqpoint{0.549164in}{0.403057in}}%
\pgfpathcurveto{\pgfqpoint{0.614600in}{0.431435in}}{\pgfqpoint{0.686031in}{0.492480in}}{\pgfqpoint{0.747724in}{0.572746in}}%
\pgfpathclose%
\pgfusepath{stroke,fill}%
\end{pgfscope}%
\begin{pgfscope}%
\pgfpathrectangle{\pgfqpoint{0.039236in}{0.039236in}}{\pgfqpoint{1.595582in}{1.177068in}}%
\pgfusepath{clip}%
\pgfsetbuttcap%
\pgfsetmiterjoin%
\definecolor{currentfill}{rgb}{0.827451,0.827451,0.827451}%
\pgfsetfillcolor{currentfill}%
\pgfsetfillopacity{0.800000}%
\pgfsetlinewidth{1.003750pt}%
\definecolor{currentstroke}{rgb}{0.000000,0.000000,0.000000}%
\pgfsetstrokecolor{currentstroke}%
\pgfsetstrokeopacity{0.800000}%
\pgfsetdash{}{0pt}%
\pgfpathmoveto{\pgfqpoint{1.124889in}{0.724075in}}%
\pgfpathcurveto{\pgfqpoint{1.124889in}{0.837589in}}{\pgfqpoint{1.107557in}{0.946468in}}{\pgfqpoint{1.076710in}{1.026735in}}%
\pgfpathcurveto{\pgfqpoint{1.045863in}{1.107001in}}{\pgfqpoint{1.004020in}{1.152100in}}{\pgfqpoint{0.960396in}{1.152100in}}%
\pgfpathcurveto{\pgfqpoint{0.916772in}{1.152100in}}{\pgfqpoint{0.874929in}{1.107001in}}{\pgfqpoint{0.844082in}{1.026735in}}%
\pgfpathcurveto{\pgfqpoint{0.813235in}{0.946468in}}{\pgfqpoint{0.795903in}{0.837589in}}{\pgfqpoint{0.795903in}{0.724075in}}%
\pgfpathcurveto{\pgfqpoint{0.795903in}{0.610562in}}{\pgfqpoint{0.813235in}{0.501682in}}{\pgfqpoint{0.844082in}{0.421416in}}%
\pgfpathcurveto{\pgfqpoint{0.874929in}{0.341150in}}{\pgfqpoint{0.916772in}{0.296051in}}{\pgfqpoint{0.960396in}{0.296051in}}%
\pgfpathcurveto{\pgfqpoint{1.004020in}{0.296051in}}{\pgfqpoint{1.045863in}{0.341150in}}{\pgfqpoint{1.076710in}{0.421416in}}%
\pgfpathcurveto{\pgfqpoint{1.107557in}{0.501682in}}{\pgfqpoint{1.124889in}{0.610562in}}{\pgfqpoint{1.124889in}{0.724075in}}%
\pgfpathclose%
\pgfusepath{stroke,fill}%
\end{pgfscope}%
\begin{pgfscope}%
\pgfpathrectangle{\pgfqpoint{0.039236in}{0.039236in}}{\pgfqpoint{1.595582in}{1.177068in}}%
\pgfusepath{clip}%
\pgfsetbuttcap%
\pgfsetmiterjoin%
\definecolor{currentfill}{rgb}{0.827451,0.827451,0.827451}%
\pgfsetfillcolor{currentfill}%
\pgfsetfillopacity{0.800000}%
\pgfsetlinewidth{1.003750pt}%
\definecolor{currentstroke}{rgb}{0.000000,0.000000,0.000000}%
\pgfsetstrokecolor{currentstroke}%
\pgfsetstrokeopacity{0.800000}%
\pgfsetdash{}{0pt}%
\pgfpathmoveto{\pgfqpoint{1.289382in}{0.082038in}}%
\pgfpathcurveto{\pgfqpoint{1.376630in}{0.082038in}}{\pgfqpoint{1.460316in}{0.104588in}}{\pgfqpoint{1.522010in}{0.144721in}}%
\pgfpathcurveto{\pgfqpoint{1.583704in}{0.184854in}}{\pgfqpoint{1.618368in}{0.239294in}}{\pgfqpoint{1.618368in}{0.296051in}}%
\pgfpathcurveto{\pgfqpoint{1.618368in}{0.352807in}}{\pgfqpoint{1.583704in}{0.407247in}}{\pgfqpoint{1.522010in}{0.447380in}}%
\pgfpathcurveto{\pgfqpoint{1.460316in}{0.487513in}}{\pgfqpoint{1.376630in}{0.510063in}}{\pgfqpoint{1.289382in}{0.510063in}}%
\pgfpathcurveto{\pgfqpoint{1.202134in}{0.510063in}}{\pgfqpoint{1.118448in}{0.487513in}}{\pgfqpoint{1.056754in}{0.447380in}}%
\pgfpathcurveto{\pgfqpoint{0.995060in}{0.407247in}}{\pgfqpoint{0.960396in}{0.352807in}}{\pgfqpoint{0.960396in}{0.296051in}}%
\pgfpathcurveto{\pgfqpoint{0.960396in}{0.239294in}}{\pgfqpoint{0.995060in}{0.184854in}}{\pgfqpoint{1.056754in}{0.144721in}}%
\pgfpathcurveto{\pgfqpoint{1.118448in}{0.104588in}}{\pgfqpoint{1.202134in}{0.082038in}}{\pgfqpoint{1.289382in}{0.082038in}}%
\pgfpathclose%
\pgfusepath{stroke,fill}%
\end{pgfscope}%
\begin{pgfscope}%
\pgfpathrectangle{\pgfqpoint{0.039236in}{0.039236in}}{\pgfqpoint{1.595582in}{1.177068in}}%
\pgfusepath{clip}%
\pgfsetbuttcap%
\pgfsetmiterjoin%
\definecolor{currentfill}{rgb}{0.827451,0.827451,0.827451}%
\pgfsetfillcolor{currentfill}%
\pgfsetfillopacity{0.800000}%
\pgfsetlinewidth{1.003750pt}%
\definecolor{currentstroke}{rgb}{0.000000,0.000000,0.000000}%
\pgfsetstrokecolor{currentstroke}%
\pgfsetstrokeopacity{0.800000}%
\pgfsetdash{}{0pt}%
\pgfpathmoveto{\pgfqpoint{0.055685in}{1.002292in}}%
\pgfpathlineto{\pgfqpoint{0.302424in}{1.002292in}}%
\pgfpathlineto{\pgfqpoint{0.302424in}{0.360254in}}%
\pgfpathlineto{\pgfqpoint{0.466917in}{0.360254in}}%
\pgfpathlineto{\pgfqpoint{0.466917in}{0.253248in}}%
\pgfpathlineto{\pgfqpoint{0.220178in}{0.253248in}}%
\pgfpathlineto{\pgfqpoint{0.220178in}{0.895285in}}%
\pgfpathlineto{\pgfqpoint{0.055685in}{0.895285in}}%
\pgfpathclose%
\pgfusepath{stroke,fill}%
\end{pgfscope}%
\begin{pgfscope}%
\pgfpathrectangle{\pgfqpoint{0.039236in}{0.039236in}}{\pgfqpoint{1.595582in}{1.177068in}}%
\pgfusepath{clip}%
\pgfsetbuttcap%
\pgfsetmiterjoin%
\definecolor{currentfill}{rgb}{1.000000,1.000000,0.000000}%
\pgfsetfillcolor{currentfill}%
\pgfsetfillopacity{0.600000}%
\pgfsetlinewidth{1.003750pt}%
\definecolor{currentstroke}{rgb}{0.000000,0.000000,0.000000}%
\pgfsetstrokecolor{currentstroke}%
\pgfsetstrokeopacity{0.600000}%
\pgfsetdash{{3.700000pt}{1.600000pt}}{0.000000pt}%
\pgfpathmoveto{\pgfqpoint{0.402571in}{0.350815in}}%
\pgfpathlineto{\pgfqpoint{0.146552in}{0.436028in}}%
\pgfpathlineto{\pgfqpoint{0.351772in}{0.681273in}}%
\pgfpathclose%
\pgfusepath{stroke,fill}%
\end{pgfscope}%
\begin{pgfscope}%
\pgfpathrectangle{\pgfqpoint{0.039236in}{0.039236in}}{\pgfqpoint{1.595582in}{1.177068in}}%
\pgfusepath{clip}%
\pgfsetbuttcap%
\pgfsetmiterjoin%
\definecolor{currentfill}{rgb}{1.000000,0.000000,0.000000}%
\pgfsetfillcolor{currentfill}%
\pgfsetlinewidth{1.003750pt}%
\definecolor{currentstroke}{rgb}{1.000000,0.000000,0.000000}%
\pgfsetstrokecolor{currentstroke}%
\pgfsetdash{}{0pt}%
\pgfsys@defobject{currentmarker}{\pgfqpoint{-0.016667in}{-0.016667in}}{\pgfqpoint{0.016667in}{0.016667in}}{%
\pgfpathmoveto{\pgfqpoint{0.000000in}{0.016667in}}%
\pgfpathlineto{\pgfqpoint{-0.016667in}{-0.016667in}}%
\pgfpathlineto{\pgfqpoint{0.016667in}{-0.016667in}}%
\pgfpathclose%
\pgfusepath{stroke,fill}%
}%
\begin{pgfscope}%
\pgfsys@transformshift{0.351772in}{0.681273in}%
\pgfsys@useobject{currentmarker}{}%
\end{pgfscope}%
\begin{pgfscope}%
\pgfsys@transformshift{0.631410in}{0.274649in}%
\pgfsys@useobject{currentmarker}{}%
\end{pgfscope}%
\end{pgfscope}%
\end{pgfpicture}%
\makeatother%
\endgroup%

%% file: figures/step2_star.pgf
\begingroup%
\makeatletter%
\begin{pgfpicture}%
\pgfpathrectangle{\pgfpointorigin}{\pgfqpoint{1.674053in}{1.255540in}}%
\pgfusepath{use as bounding box, clip}%
\begin{pgfscope}%
\pgfsetbuttcap%
\pgfsetmiterjoin%
\definecolor{currentfill}{rgb}{1.000000,1.000000,1.000000}%
\pgfsetfillcolor{currentfill}%
\pgfsetlinewidth{0.000000pt}%
\definecolor{currentstroke}{rgb}{1.000000,1.000000,1.000000}%
\pgfsetstrokecolor{currentstroke}%
\pgfsetdash{}{0pt}%
\pgfpathmoveto{\pgfqpoint{0.000000in}{0.000000in}}%
\pgfpathlineto{\pgfqpoint{1.674053in}{0.000000in}}%
\pgfpathlineto{\pgfqpoint{1.674053in}{1.255540in}}%
\pgfpathlineto{\pgfqpoint{0.000000in}{1.255540in}}%
\pgfpathclose%
\pgfusepath{fill}%
\end{pgfscope}%
\begin{pgfscope}%
\pgfpathrectangle{\pgfqpoint{0.039236in}{0.039236in}}{\pgfqpoint{1.595582in}{1.177068in}}%
\pgfusepath{clip}%
\pgfsetbuttcap%
\pgfsetmiterjoin%
\definecolor{currentfill}{rgb}{0.827451,0.827451,0.827451}%
\pgfsetfillcolor{currentfill}%
\pgfsetfillopacity{0.800000}%
\pgfsetlinewidth{1.003750pt}%
\definecolor{currentstroke}{rgb}{0.000000,0.000000,0.000000}%
\pgfsetstrokecolor{currentstroke}%
\pgfsetstrokeopacity{0.800000}%
\pgfsetdash{}{0pt}%
\pgfpathmoveto{\pgfqpoint{1.330615in}{0.508375in}}%
\pgfpathlineto{\pgfqpoint{1.351028in}{0.506272in}}%
\pgfpathlineto{\pgfqpoint{1.371197in}{0.503339in}}%
\pgfpathlineto{\pgfqpoint{1.391044in}{0.499588in}}%
\pgfpathlineto{\pgfqpoint{1.410490in}{0.495034in}}%
\pgfpathlineto{\pgfqpoint{1.429457in}{0.489695in}}%
\pgfpathlineto{\pgfqpoint{1.447872in}{0.483591in}}%
\pgfpathlineto{\pgfqpoint{1.465661in}{0.476747in}}%
\pgfpathlineto{\pgfqpoint{1.482755in}{0.469190in}}%
\pgfpathlineto{\pgfqpoint{1.499086in}{0.460950in}}%
\pgfpathlineto{\pgfqpoint{1.514588in}{0.452059in}}%
\pgfpathlineto{\pgfqpoint{1.529202in}{0.442552in}}%
\pgfpathlineto{\pgfqpoint{1.542870in}{0.432467in}}%
\pgfpathlineto{\pgfqpoint{1.555537in}{0.421844in}}%
\pgfpathlineto{\pgfqpoint{1.567154in}{0.410724in}}%
\pgfpathlineto{\pgfqpoint{1.577675in}{0.399152in}}%
\pgfpathlineto{\pgfqpoint{1.587057in}{0.387173in}}%
\pgfpathlineto{\pgfqpoint{1.595265in}{0.374834in}}%
\pgfpathlineto{\pgfqpoint{1.602266in}{0.362184in}}%
\pgfpathlineto{\pgfqpoint{1.608032in}{0.349273in}}%
\pgfpathlineto{\pgfqpoint{1.612541in}{0.336152in}}%
\pgfpathlineto{\pgfqpoint{1.615774in}{0.322873in}}%
\pgfpathlineto{\pgfqpoint{1.617719in}{0.309488in}}%
\pgfpathlineto{\pgfqpoint{1.618368in}{0.296051in}}%
\pgfpathlineto{\pgfqpoint{1.617719in}{0.282613in}}%
\pgfpathlineto{\pgfqpoint{1.615774in}{0.269228in}}%
\pgfpathlineto{\pgfqpoint{1.612541in}{0.255949in}}%
\pgfpathlineto{\pgfqpoint{1.608032in}{0.242828in}}%
\pgfpathlineto{\pgfqpoint{1.602266in}{0.229917in}}%
\pgfpathlineto{\pgfqpoint{1.595265in}{0.217267in}}%
\pgfpathlineto{\pgfqpoint{1.587057in}{0.204928in}}%
\pgfpathlineto{\pgfqpoint{1.577675in}{0.192949in}}%
\pgfpathlineto{\pgfqpoint{1.567154in}{0.181377in}}%
\pgfpathlineto{\pgfqpoint{1.555537in}{0.170257in}}%
\pgfpathlineto{\pgfqpoint{1.542870in}{0.159634in}}%
\pgfpathlineto{\pgfqpoint{1.529202in}{0.149549in}}%
\pgfpathlineto{\pgfqpoint{1.514588in}{0.140042in}}%
\pgfpathlineto{\pgfqpoint{1.499086in}{0.131151in}}%
\pgfpathlineto{\pgfqpoint{1.482755in}{0.122911in}}%
\pgfpathlineto{\pgfqpoint{1.465661in}{0.115354in}}%
\pgfpathlineto{\pgfqpoint{1.447872in}{0.108510in}}%
\pgfpathlineto{\pgfqpoint{1.429457in}{0.102406in}}%
\pgfpathlineto{\pgfqpoint{1.410490in}{0.097067in}}%
\pgfpathlineto{\pgfqpoint{1.391044in}{0.092513in}}%
\pgfpathlineto{\pgfqpoint{1.371197in}{0.088762in}}%
\pgfpathlineto{\pgfqpoint{1.351028in}{0.085829in}}%
\pgfpathlineto{\pgfqpoint{1.330615in}{0.083726in}}%
\pgfpathlineto{\pgfqpoint{1.310039in}{0.082460in}}%
\pgfpathlineto{\pgfqpoint{1.289382in}{0.082038in}}%
\pgfpathlineto{\pgfqpoint{1.268725in}{0.082460in}}%
\pgfpathlineto{\pgfqpoint{1.248149in}{0.083726in}}%
\pgfpathlineto{\pgfqpoint{1.227736in}{0.085829in}}%
\pgfpathlineto{\pgfqpoint{1.207567in}{0.088762in}}%
\pgfpathlineto{\pgfqpoint{1.187720in}{0.092513in}}%
\pgfpathlineto{\pgfqpoint{1.168274in}{0.097067in}}%
\pgfpathlineto{\pgfqpoint{1.149307in}{0.102406in}}%
\pgfpathlineto{\pgfqpoint{1.130892in}{0.108510in}}%
\pgfpathlineto{\pgfqpoint{1.118136in}{0.113417in}}%
\pgfpathlineto{\pgfqpoint{1.118097in}{0.113333in}}%
\pgfpathlineto{\pgfqpoint{0.444692in}{0.380473in}}%
\pgfpathlineto{\pgfqpoint{0.466917in}{0.360254in}}%
\pgfpathlineto{\pgfqpoint{0.466917in}{0.253248in}}%
\pgfpathlineto{\pgfqpoint{0.220178in}{0.253248in}}%
\pgfpathlineto{\pgfqpoint{0.220178in}{0.511760in}}%
\pgfpathlineto{\pgfqpoint{0.055685in}{0.895285in}}%
\pgfpathlineto{\pgfqpoint{0.055685in}{1.002292in}}%
\pgfpathlineto{\pgfqpoint{0.302424in}{1.002292in}}%
\pgfpathlineto{\pgfqpoint{0.302424in}{0.580721in}}%
\pgfpathlineto{\pgfqpoint{0.500203in}{0.855387in}}%
\pgfpathlineto{\pgfqpoint{0.500200in}{0.855360in}}%
\pgfpathlineto{\pgfqpoint{0.500719in}{0.856102in}}%
\pgfpathlineto{\pgfqpoint{0.515096in}{0.875405in}}%
\pgfpathlineto{\pgfqpoint{0.529933in}{0.894111in}}%
\pgfpathlineto{\pgfqpoint{0.545169in}{0.912145in}}%
\pgfpathlineto{\pgfqpoint{0.560747in}{0.929437in}}%
\pgfpathlineto{\pgfqpoint{0.576603in}{0.945919in}}%
\pgfpathlineto{\pgfqpoint{0.592675in}{0.961525in}}%
\pgfpathlineto{\pgfqpoint{0.608900in}{0.976194in}}%
\pgfpathlineto{\pgfqpoint{0.625214in}{0.989869in}}%
\pgfpathlineto{\pgfqpoint{0.641553in}{1.002494in}}%
\pgfpathlineto{\pgfqpoint{0.657851in}{1.014020in}}%
\pgfpathlineto{\pgfqpoint{0.674046in}{1.024402in}}%
\pgfpathlineto{\pgfqpoint{0.690071in}{1.033599in}}%
\pgfpathlineto{\pgfqpoint{0.705866in}{1.041574in}}%
\pgfpathlineto{\pgfqpoint{0.721366in}{1.048297in}}%
\pgfpathlineto{\pgfqpoint{0.736512in}{1.053740in}}%
\pgfpathlineto{\pgfqpoint{0.751243in}{1.057881in}}%
\pgfpathlineto{\pgfqpoint{0.765500in}{1.060706in}}%
\pgfpathlineto{\pgfqpoint{0.779229in}{1.062201in}}%
\pgfpathlineto{\pgfqpoint{0.792374in}{1.062363in}}%
\pgfpathlineto{\pgfqpoint{0.804884in}{1.061189in}}%
\pgfpathlineto{\pgfqpoint{0.816710in}{1.058685in}}%
\pgfpathlineto{\pgfqpoint{0.827466in}{1.054977in}}%
\pgfpathlineto{\pgfqpoint{0.905025in}{1.127122in}}%
\pgfpathlineto{\pgfqpoint{0.905036in}{1.126908in}}%
\pgfpathlineto{\pgfqpoint{0.909565in}{1.131151in}}%
\pgfpathlineto{\pgfqpoint{0.919488in}{1.138653in}}%
\pgfpathlineto{\pgfqpoint{0.929573in}{1.144519in}}%
\pgfpathlineto{\pgfqpoint{0.939780in}{1.148725in}}%
\pgfpathlineto{\pgfqpoint{0.950068in}{1.151256in}}%
\pgfpathlineto{\pgfqpoint{0.960396in}{1.152100in}}%
\pgfpathlineto{\pgfqpoint{0.970725in}{1.151256in}}%
\pgfpathlineto{\pgfqpoint{0.981013in}{1.148725in}}%
\pgfpathlineto{\pgfqpoint{0.991219in}{1.144519in}}%
\pgfpathlineto{\pgfqpoint{1.001304in}{1.138653in}}%
\pgfpathlineto{\pgfqpoint{1.011227in}{1.131151in}}%
\pgfpathlineto{\pgfqpoint{1.020950in}{1.122043in}}%
\pgfpathlineto{\pgfqpoint{1.030434in}{1.111364in}}%
\pgfpathlineto{\pgfqpoint{1.039641in}{1.099156in}}%
\pgfpathlineto{\pgfqpoint{1.048536in}{1.085469in}}%
\pgfpathlineto{\pgfqpoint{1.057083in}{1.070355in}}%
\pgfpathlineto{\pgfqpoint{1.065248in}{1.053874in}}%
\pgfpathlineto{\pgfqpoint{1.072999in}{1.036092in}}%
\pgfpathlineto{\pgfqpoint{1.080306in}{1.017079in}}%
\pgfpathlineto{\pgfqpoint{1.087140in}{0.996909in}}%
\pgfpathlineto{\pgfqpoint{1.093474in}{0.975662in}}%
\pgfpathlineto{\pgfqpoint{1.099282in}{0.953423in}}%
\pgfpathlineto{\pgfqpoint{1.104542in}{0.930278in}}%
\pgfpathlineto{\pgfqpoint{1.109234in}{0.906320in}}%
\pgfpathlineto{\pgfqpoint{1.113338in}{0.881642in}}%
\pgfpathlineto{\pgfqpoint{1.116838in}{0.856342in}}%
\pgfpathlineto{\pgfqpoint{1.119721in}{0.830521in}}%
\pgfpathlineto{\pgfqpoint{1.121975in}{0.804279in}}%
\pgfpathlineto{\pgfqpoint{1.123592in}{0.777721in}}%
\pgfpathlineto{\pgfqpoint{1.124564in}{0.750951in}}%
\pgfpathlineto{\pgfqpoint{1.124889in}{0.724075in}}%
\pgfpathlineto{\pgfqpoint{1.124564in}{0.697199in}}%
\pgfpathlineto{\pgfqpoint{1.123592in}{0.670430in}}%
\pgfpathlineto{\pgfqpoint{1.121975in}{0.643872in}}%
\pgfpathlineto{\pgfqpoint{1.119721in}{0.617630in}}%
\pgfpathlineto{\pgfqpoint{1.116838in}{0.591808in}}%
\pgfpathlineto{\pgfqpoint{1.113338in}{0.566509in}}%
\pgfpathlineto{\pgfqpoint{1.109234in}{0.541831in}}%
\pgfpathlineto{\pgfqpoint{1.104542in}{0.517873in}}%
\pgfpathlineto{\pgfqpoint{1.103956in}{0.515292in}}%
\pgfpathlineto{\pgfqpoint{1.303122in}{0.509876in}}%
\pgfpathlineto{\pgfqpoint{1.303078in}{0.509783in}}%
\pgfpathlineto{\pgfqpoint{1.310039in}{0.509641in}}%
\pgfpathclose%
\pgfusepath{stroke,fill}%
\end{pgfscope}%
\begin{pgfscope}%
\pgfpathrectangle{\pgfqpoint{0.039236in}{0.039236in}}{\pgfqpoint{1.595582in}{1.177068in}}%
\pgfusepath{clip}%
\pgfsetbuttcap%
\pgfsetmiterjoin%
\definecolor{currentfill}{rgb}{0.254902,0.411765,0.882353}%
\pgfsetfillcolor{currentfill}%
\pgfsetfillopacity{0.600000}%
\pgfsetlinewidth{1.003750pt}%
\definecolor{currentstroke}{rgb}{0.000000,0.000000,0.000000}%
\pgfsetstrokecolor{currentstroke}%
\pgfsetstrokeopacity{0.600000}%
\pgfsetdash{{1.000000pt}{1.650000pt}}{0.000000pt}%
\pgfpathmoveto{\pgfqpoint{0.271609in}{0.537925in}}%
\pgfpathlineto{\pgfqpoint{0.301217in}{0.460881in}}%
\pgfpathlineto{\pgfqpoint{0.242000in}{0.460881in}}%
\pgfpathclose%
\pgfusepath{stroke,fill}%
\end{pgfscope}%
\begin{pgfscope}%
\pgfpathrectangle{\pgfqpoint{0.039236in}{0.039236in}}{\pgfqpoint{1.595582in}{1.177068in}}%
\pgfusepath{clip}%
\pgfsetbuttcap%
\pgfsetmiterjoin%
\definecolor{currentfill}{rgb}{1.000000,0.000000,0.000000}%
\pgfsetfillcolor{currentfill}%
\pgfsetlinewidth{1.003750pt}%
\definecolor{currentstroke}{rgb}{1.000000,0.000000,0.000000}%
\pgfsetstrokecolor{currentstroke}%
\pgfsetdash{}{0pt}%
\pgfsys@defobject{currentmarker}{\pgfqpoint{-0.016667in}{-0.016667in}}{\pgfqpoint{0.016667in}{0.016667in}}{%
\pgfpathmoveto{\pgfqpoint{0.000000in}{0.016667in}}%
\pgfpathlineto{\pgfqpoint{-0.016667in}{-0.016667in}}%
\pgfpathlineto{\pgfqpoint{0.016667in}{-0.016667in}}%
\pgfpathclose%
\pgfusepath{stroke,fill}%
}%
\begin{pgfscope}%
\pgfsys@transformshift{0.351772in}{0.681273in}%
\pgfsys@useobject{currentmarker}{}%
\end{pgfscope}%
\begin{pgfscope}%
\pgfsys@transformshift{0.631410in}{0.274649in}%
\pgfsys@useobject{currentmarker}{}%
\end{pgfscope}%
\end{pgfscope}%
\end{pgfpicture}%
\makeatother%
\endgroup%

%% file: figures/hull_intersecting_star_world_polygons.pgf
\begingroup%
\makeatletter%
\begin{pgfpicture}%
\pgfpathrectangle{\pgfpointorigin}{\pgfqpoint{1.674053in}{1.255540in}}%
\pgfusepath{use as bounding box, clip}%
\begin{pgfscope}%
\pgfsetbuttcap%
\pgfsetmiterjoin%
\definecolor{currentfill}{rgb}{1.000000,1.000000,1.000000}%
\pgfsetfillcolor{currentfill}%
\pgfsetlinewidth{0.000000pt}%
\definecolor{currentstroke}{rgb}{1.000000,1.000000,1.000000}%
\pgfsetstrokecolor{currentstroke}%
\pgfsetdash{}{0pt}%
\pgfpathmoveto{\pgfqpoint{0.000000in}{0.000000in}}%
\pgfpathlineto{\pgfqpoint{1.674053in}{0.000000in}}%
\pgfpathlineto{\pgfqpoint{1.674053in}{1.255540in}}%
\pgfpathlineto{\pgfqpoint{0.000000in}{1.255540in}}%
\pgfpathclose%
\pgfusepath{fill}%
\end{pgfscope}%
\begin{pgfscope}%
\pgfpathrectangle{\pgfqpoint{0.039236in}{0.039236in}}{\pgfqpoint{1.595582in}{1.177068in}}%
\pgfusepath{clip}%
\pgfsetbuttcap%
\pgfsetmiterjoin%
\definecolor{currentfill}{rgb}{0.827451,0.827451,0.827451}%
\pgfsetfillcolor{currentfill}%
\pgfsetfillopacity{0.800000}%
\pgfsetlinewidth{1.003750pt}%
\definecolor{currentstroke}{rgb}{0.000000,0.000000,0.000000}%
\pgfsetstrokecolor{currentstroke}%
\pgfsetstrokeopacity{0.800000}%
\pgfsetdash{}{0pt}%
\pgfpathmoveto{\pgfqpoint{0.172201in}{0.510063in}}%
\pgfpathlineto{\pgfqpoint{0.305166in}{0.510063in}}%
\pgfpathlineto{\pgfqpoint{0.305166in}{0.863183in}}%
\pgfpathlineto{\pgfqpoint{1.102957in}{0.863183in}}%
\pgfpathlineto{\pgfqpoint{1.102957in}{0.510063in}}%
\pgfpathlineto{\pgfqpoint{1.235922in}{0.510063in}}%
\pgfpathlineto{\pgfqpoint{1.235922in}{0.980890in}}%
\pgfpathlineto{\pgfqpoint{0.172201in}{0.980890in}}%
\pgfpathclose%
\pgfusepath{stroke,fill}%
\end{pgfscope}%
\begin{pgfscope}%
\pgfpathrectangle{\pgfqpoint{0.039236in}{0.039236in}}{\pgfqpoint{1.595582in}{1.177068in}}%
\pgfusepath{clip}%
\pgfsetbuttcap%
\pgfsetmiterjoin%
\definecolor{currentfill}{rgb}{0.827451,0.827451,0.827451}%
\pgfsetfillcolor{currentfill}%
\pgfsetfillopacity{0.800000}%
\pgfsetlinewidth{1.003750pt}%
\definecolor{currentstroke}{rgb}{0.000000,0.000000,0.000000}%
\pgfsetstrokecolor{currentstroke}%
\pgfsetstrokeopacity{0.800000}%
\pgfsetdash{}{0pt}%
\pgfpathmoveto{\pgfqpoint{0.438131in}{0.745477in}}%
\pgfpathlineto{\pgfqpoint{0.571096in}{0.745477in}}%
\pgfpathlineto{\pgfqpoint{0.571096in}{0.274649in}}%
\pgfpathlineto{\pgfqpoint{1.368887in}{0.274649in}}%
\pgfpathlineto{\pgfqpoint{1.368887in}{0.745477in}}%
\pgfpathlineto{\pgfqpoint{1.501852in}{0.745477in}}%
\pgfpathlineto{\pgfqpoint{1.501852in}{0.156942in}}%
\pgfpathlineto{\pgfqpoint{0.438131in}{0.156942in}}%
\pgfpathclose%
\pgfusepath{stroke,fill}%
\end{pgfscope}%
\begin{pgfscope}%
\pgfpathrectangle{\pgfqpoint{0.039236in}{0.039236in}}{\pgfqpoint{1.595582in}{1.177068in}}%
\pgfusepath{clip}%
\pgfsetbuttcap%
\pgfsetmiterjoin%
\definecolor{currentfill}{rgb}{0.000000,0.501961,0.000000}%
\pgfsetfillcolor{currentfill}%
\pgfsetfillopacity{0.600000}%
\pgfsetlinewidth{0.000000pt}%
\definecolor{currentstroke}{rgb}{0.000000,0.000000,0.000000}%
\pgfsetstrokecolor{currentstroke}%
\pgfsetstrokeopacity{0.600000}%
\pgfsetdash{}{0pt}%
\pgfpathmoveto{\pgfqpoint{0.305166in}{0.510063in}}%
\pgfpathlineto{\pgfqpoint{0.305166in}{0.863183in}}%
\pgfpathlineto{\pgfqpoint{0.749928in}{0.863183in}}%
\pgfpathclose%
\pgfusepath{fill}%
\end{pgfscope}%
\begin{pgfscope}%
\pgfpathrectangle{\pgfqpoint{0.039236in}{0.039236in}}{\pgfqpoint{1.595582in}{1.177068in}}%
\pgfusepath{clip}%
\pgfsetbuttcap%
\pgfsetmiterjoin%
\definecolor{currentfill}{rgb}{0.000000,0.501961,0.000000}%
\pgfsetfillcolor{currentfill}%
\pgfsetfillopacity{0.600000}%
\pgfsetlinewidth{0.000000pt}%
\definecolor{currentstroke}{rgb}{0.000000,0.000000,0.000000}%
\pgfsetstrokecolor{currentstroke}%
\pgfsetstrokeopacity{0.600000}%
\pgfsetdash{}{0pt}%
\pgfpathmoveto{\pgfqpoint{0.809623in}{0.863183in}}%
\pgfpathlineto{\pgfqpoint{1.102957in}{0.863183in}}%
\pgfpathlineto{\pgfqpoint{1.102957in}{0.510063in}}%
\pgfpathclose%
\pgfusepath{fill}%
\end{pgfscope}%
\begin{pgfscope}%
\pgfpathrectangle{\pgfqpoint{0.039236in}{0.039236in}}{\pgfqpoint{1.595582in}{1.177068in}}%
\pgfusepath{clip}%
\pgfsetbuttcap%
\pgfsetmiterjoin%
\definecolor{currentfill}{rgb}{0.000000,0.501961,0.000000}%
\pgfsetfillcolor{currentfill}%
\pgfsetfillopacity{0.600000}%
\pgfsetlinewidth{0.000000pt}%
\definecolor{currentstroke}{rgb}{0.000000,0.000000,0.000000}%
\pgfsetstrokecolor{currentstroke}%
\pgfsetstrokeopacity{0.600000}%
\pgfsetdash{}{0pt}%
\pgfpathmoveto{\pgfqpoint{1.368887in}{0.745477in}}%
\pgfpathlineto{\pgfqpoint{1.368887in}{0.274649in}}%
\pgfpathlineto{\pgfqpoint{0.903280in}{0.274649in}}%
\pgfpathclose%
\pgfusepath{fill}%
\end{pgfscope}%
\begin{pgfscope}%
\pgfpathrectangle{\pgfqpoint{0.039236in}{0.039236in}}{\pgfqpoint{1.595582in}{1.177068in}}%
\pgfusepath{clip}%
\pgfsetbuttcap%
\pgfsetmiterjoin%
\definecolor{currentfill}{rgb}{0.000000,0.501961,0.000000}%
\pgfsetfillcolor{currentfill}%
\pgfsetfillopacity{0.600000}%
\pgfsetlinewidth{0.000000pt}%
\definecolor{currentstroke}{rgb}{0.000000,0.000000,0.000000}%
\pgfsetstrokecolor{currentstroke}%
\pgfsetstrokeopacity{0.600000}%
\pgfsetdash{}{0pt}%
\pgfpathmoveto{\pgfqpoint{0.900667in}{0.274649in}}%
\pgfpathlineto{\pgfqpoint{0.571096in}{0.274649in}}%
\pgfpathlineto{\pgfqpoint{0.571096in}{0.745477in}}%
\pgfpathclose%
\pgfusepath{fill}%
\end{pgfscope}%
\begin{pgfscope}%
\pgfpathrectangle{\pgfqpoint{0.039236in}{0.039236in}}{\pgfqpoint{1.595582in}{1.177068in}}%
\pgfusepath{clip}%
\pgfsetbuttcap%
\pgfsetmiterjoin%
\definecolor{currentfill}{rgb}{1.000000,0.000000,0.000000}%
\pgfsetfillcolor{currentfill}%
\pgfsetlinewidth{1.003750pt}%
\definecolor{currentstroke}{rgb}{1.000000,0.000000,0.000000}%
\pgfsetstrokecolor{currentstroke}%
\pgfsetdash{}{0pt}%
\pgfsys@defobject{currentmarker}{\pgfqpoint{-0.016667in}{-0.016667in}}{\pgfqpoint{0.016667in}{0.016667in}}{%
\pgfpathmoveto{\pgfqpoint{0.000000in}{0.016667in}}%
\pgfpathlineto{\pgfqpoint{-0.016667in}{-0.016667in}}%
\pgfpathlineto{\pgfqpoint{0.016667in}{-0.016667in}}%
\pgfpathclose%
\pgfusepath{stroke,fill}%
}%
\begin{pgfscope}%
\pgfsys@transformshift{0.837026in}{0.627770in}%
\pgfsys@useobject{currentmarker}{}%
\end{pgfscope}%
\begin{pgfscope}%
\pgfsys@transformshift{0.837026in}{1.098597in}%
\pgfsys@useobject{currentmarker}{}%
\end{pgfscope}%
\end{pgfscope}%
\end{pgfpicture}%
\makeatother%
\endgroup%

%% file: figures/intersecting_star_world_polygons.pgf
\begingroup%
\makeatletter%
\begin{pgfpicture}%
\pgfpathrectangle{\pgfpointorigin}{\pgfqpoint{1.674053in}{1.255540in}}%
\pgfusepath{use as bounding box, clip}%
\begin{pgfscope}%
\pgfsetbuttcap%
\pgfsetmiterjoin%
\definecolor{currentfill}{rgb}{1.000000,1.000000,1.000000}%
\pgfsetfillcolor{currentfill}%
\pgfsetlinewidth{0.000000pt}%
\definecolor{currentstroke}{rgb}{1.000000,1.000000,1.000000}%
\pgfsetstrokecolor{currentstroke}%
\pgfsetdash{}{0pt}%
\pgfpathmoveto{\pgfqpoint{0.000000in}{0.000000in}}%
\pgfpathlineto{\pgfqpoint{1.674053in}{0.000000in}}%
\pgfpathlineto{\pgfqpoint{1.674053in}{1.255540in}}%
\pgfpathlineto{\pgfqpoint{0.000000in}{1.255540in}}%
\pgfpathclose%
\pgfusepath{fill}%
\end{pgfscope}%
\begin{pgfscope}%
\pgfpathrectangle{\pgfqpoint{0.039236in}{0.039236in}}{\pgfqpoint{1.595582in}{1.177068in}}%
\pgfusepath{clip}%
\pgfsetbuttcap%
\pgfsetmiterjoin%
\definecolor{currentfill}{rgb}{0.827451,0.827451,0.827451}%
\pgfsetfillcolor{currentfill}%
\pgfsetfillopacity{0.800000}%
\pgfsetlinewidth{1.003750pt}%
\definecolor{currentstroke}{rgb}{0.000000,0.000000,0.000000}%
\pgfsetstrokecolor{currentstroke}%
\pgfsetstrokeopacity{0.800000}%
\pgfsetdash{}{0pt}%
\pgfpathmoveto{\pgfqpoint{0.172201in}{0.980890in}}%
\pgfpathlineto{\pgfqpoint{0.172201in}{0.510063in}}%
\pgfpathlineto{\pgfqpoint{0.305166in}{0.510063in}}%
\pgfpathlineto{\pgfqpoint{0.305166in}{0.863183in}}%
\pgfpathclose%
\pgfusepath{stroke,fill}%
\end{pgfscope}%
\begin{pgfscope}%
\pgfpathrectangle{\pgfqpoint{0.039236in}{0.039236in}}{\pgfqpoint{1.595582in}{1.177068in}}%
\pgfusepath{clip}%
\pgfsetbuttcap%
\pgfsetmiterjoin%
\definecolor{currentfill}{rgb}{0.827451,0.827451,0.827451}%
\pgfsetfillcolor{currentfill}%
\pgfsetfillopacity{0.800000}%
\pgfsetlinewidth{1.003750pt}%
\definecolor{currentstroke}{rgb}{0.000000,0.000000,0.000000}%
\pgfsetstrokecolor{currentstroke}%
\pgfsetstrokeopacity{0.800000}%
\pgfsetdash{}{0pt}%
\pgfpathmoveto{\pgfqpoint{1.235922in}{0.980890in}}%
\pgfpathlineto{\pgfqpoint{0.172201in}{0.980890in}}%
\pgfpathlineto{\pgfqpoint{0.305166in}{0.863183in}}%
\pgfpathlineto{\pgfqpoint{1.102957in}{0.863183in}}%
\pgfpathclose%
\pgfusepath{stroke,fill}%
\end{pgfscope}%
\begin{pgfscope}%
\pgfpathrectangle{\pgfqpoint{0.039236in}{0.039236in}}{\pgfqpoint{1.595582in}{1.177068in}}%
\pgfusepath{clip}%
\pgfsetbuttcap%
\pgfsetmiterjoin%
\definecolor{currentfill}{rgb}{0.827451,0.827451,0.827451}%
\pgfsetfillcolor{currentfill}%
\pgfsetfillopacity{0.800000}%
\pgfsetlinewidth{1.003750pt}%
\definecolor{currentstroke}{rgb}{0.000000,0.000000,0.000000}%
\pgfsetstrokecolor{currentstroke}%
\pgfsetstrokeopacity{0.800000}%
\pgfsetdash{}{0pt}%
\pgfpathmoveto{\pgfqpoint{1.102957in}{0.863183in}}%
\pgfpathlineto{\pgfqpoint{1.102957in}{0.510063in}}%
\pgfpathlineto{\pgfqpoint{1.235922in}{0.510063in}}%
\pgfpathlineto{\pgfqpoint{1.235922in}{0.980890in}}%
\pgfpathclose%
\pgfusepath{stroke,fill}%
\end{pgfscope}%
\begin{pgfscope}%
\pgfpathrectangle{\pgfqpoint{0.039236in}{0.039236in}}{\pgfqpoint{1.595582in}{1.177068in}}%
\pgfusepath{clip}%
\pgfsetbuttcap%
\pgfsetmiterjoin%
\definecolor{currentfill}{rgb}{0.827451,0.827451,0.827451}%
\pgfsetfillcolor{currentfill}%
\pgfsetfillopacity{0.800000}%
\pgfsetlinewidth{1.003750pt}%
\definecolor{currentstroke}{rgb}{0.000000,0.000000,0.000000}%
\pgfsetstrokecolor{currentstroke}%
\pgfsetstrokeopacity{0.800000}%
\pgfsetdash{}{0pt}%
\pgfpathmoveto{\pgfqpoint{1.501852in}{0.156942in}}%
\pgfpathlineto{\pgfqpoint{1.501852in}{0.745477in}}%
\pgfpathlineto{\pgfqpoint{1.368887in}{0.745477in}}%
\pgfpathlineto{\pgfqpoint{1.368887in}{0.274649in}}%
\pgfpathclose%
\pgfusepath{stroke,fill}%
\end{pgfscope}%
\begin{pgfscope}%
\pgfpathrectangle{\pgfqpoint{0.039236in}{0.039236in}}{\pgfqpoint{1.595582in}{1.177068in}}%
\pgfusepath{clip}%
\pgfsetbuttcap%
\pgfsetmiterjoin%
\definecolor{currentfill}{rgb}{0.827451,0.827451,0.827451}%
\pgfsetfillcolor{currentfill}%
\pgfsetfillopacity{0.800000}%
\pgfsetlinewidth{1.003750pt}%
\definecolor{currentstroke}{rgb}{0.000000,0.000000,0.000000}%
\pgfsetstrokecolor{currentstroke}%
\pgfsetstrokeopacity{0.800000}%
\pgfsetdash{}{0pt}%
\pgfpathmoveto{\pgfqpoint{0.438131in}{0.156942in}}%
\pgfpathlineto{\pgfqpoint{1.501852in}{0.156942in}}%
\pgfpathlineto{\pgfqpoint{1.368887in}{0.274649in}}%
\pgfpathlineto{\pgfqpoint{0.571096in}{0.274649in}}%
\pgfpathclose%
\pgfusepath{stroke,fill}%
\end{pgfscope}%
\begin{pgfscope}%
\pgfpathrectangle{\pgfqpoint{0.039236in}{0.039236in}}{\pgfqpoint{1.595582in}{1.177068in}}%
\pgfusepath{clip}%
\pgfsetbuttcap%
\pgfsetmiterjoin%
\definecolor{currentfill}{rgb}{0.827451,0.827451,0.827451}%
\pgfsetfillcolor{currentfill}%
\pgfsetfillopacity{0.800000}%
\pgfsetlinewidth{1.003750pt}%
\definecolor{currentstroke}{rgb}{0.000000,0.000000,0.000000}%
\pgfsetstrokecolor{currentstroke}%
\pgfsetstrokeopacity{0.800000}%
\pgfsetdash{}{0pt}%
\pgfpathmoveto{\pgfqpoint{0.438131in}{0.156942in}}%
\pgfpathlineto{\pgfqpoint{0.571096in}{0.274649in}}%
\pgfpathlineto{\pgfqpoint{0.571096in}{0.745477in}}%
\pgfpathlineto{\pgfqpoint{0.438131in}{0.745477in}}%
\pgfpathclose%
\pgfusepath{stroke,fill}%
\end{pgfscope}%
\begin{pgfscope}%
\pgfpathrectangle{\pgfqpoint{0.039236in}{0.039236in}}{\pgfqpoint{1.595582in}{1.177068in}}%
\pgfusepath{clip}%
\pgfsetbuttcap%
\pgfsetmiterjoin%
\definecolor{currentfill}{rgb}{1.000000,0.000000,0.000000}%
\pgfsetfillcolor{currentfill}%
\pgfsetlinewidth{1.003750pt}%
\definecolor{currentstroke}{rgb}{1.000000,0.000000,0.000000}%
\pgfsetstrokecolor{currentstroke}%
\pgfsetdash{}{0pt}%
\pgfsys@defobject{currentmarker}{\pgfqpoint{-0.016667in}{-0.016667in}}{\pgfqpoint{0.016667in}{0.016667in}}{%
\pgfpathmoveto{\pgfqpoint{0.000000in}{0.016667in}}%
\pgfpathlineto{\pgfqpoint{-0.016667in}{-0.016667in}}%
\pgfpathlineto{\pgfqpoint{0.016667in}{-0.016667in}}%
\pgfpathclose%
\pgfusepath{stroke,fill}%
}%
\begin{pgfscope}%
\pgfsys@transformshift{0.837026in}{0.627770in}%
\pgfsys@useobject{currentmarker}{}%
\end{pgfscope}%
\begin{pgfscope}%
\pgfsys@transformshift{0.837026in}{1.098597in}%
\pgfsys@useobject{currentmarker}{}%
\end{pgfscope}%
\end{pgfscope}%
\end{pgfpicture}%
\makeatother%
\endgroup%

%% file: figures/disjoint_star_world_polygons.pgf
\begingroup%
\makeatletter%
\begin{pgfpicture}%
\pgfpathrectangle{\pgfpointorigin}{\pgfqpoint{1.674053in}{1.255540in}}%
\pgfusepath{use as bounding box, clip}%
\begin{pgfscope}%
\pgfsetbuttcap%
\pgfsetmiterjoin%
\definecolor{currentfill}{rgb}{1.000000,1.000000,1.000000}%
\pgfsetfillcolor{currentfill}%
\pgfsetlinewidth{0.000000pt}%
\definecolor{currentstroke}{rgb}{1.000000,1.000000,1.000000}%
\pgfsetstrokecolor{currentstroke}%
\pgfsetdash{}{0pt}%
\pgfpathmoveto{\pgfqpoint{0.000000in}{0.000000in}}%
\pgfpathlineto{\pgfqpoint{1.674053in}{0.000000in}}%
\pgfpathlineto{\pgfqpoint{1.674053in}{1.255540in}}%
\pgfpathlineto{\pgfqpoint{0.000000in}{1.255540in}}%
\pgfpathclose%
\pgfusepath{fill}%
\end{pgfscope}%
\begin{pgfscope}%
\pgfpathrectangle{\pgfqpoint{0.039236in}{0.039236in}}{\pgfqpoint{1.595582in}{1.177068in}}%
\pgfusepath{clip}%
\pgfsetbuttcap%
\pgfsetmiterjoin%
\definecolor{currentfill}{rgb}{0.827451,0.827451,0.827451}%
\pgfsetfillcolor{currentfill}%
\pgfsetfillopacity{0.800000}%
\pgfsetlinewidth{1.003750pt}%
\definecolor{currentstroke}{rgb}{0.000000,0.000000,0.000000}%
\pgfsetstrokecolor{currentstroke}%
\pgfsetstrokeopacity{0.800000}%
\pgfsetdash{}{0pt}%
\pgfpathmoveto{\pgfqpoint{0.172201in}{0.510063in}}%
\pgfpathlineto{\pgfqpoint{0.305166in}{0.510063in}}%
\pgfpathlineto{\pgfqpoint{0.305166in}{0.863183in}}%
\pgfpathlineto{\pgfqpoint{1.102957in}{0.863183in}}%
\pgfpathlineto{\pgfqpoint{1.102957in}{0.510063in}}%
\pgfpathlineto{\pgfqpoint{1.235922in}{0.510063in}}%
\pgfpathlineto{\pgfqpoint{1.235922in}{0.980890in}}%
\pgfpathlineto{\pgfqpoint{0.172201in}{0.980890in}}%
\pgfpathclose%
\pgfusepath{stroke,fill}%
\end{pgfscope}%
\begin{pgfscope}%
\pgfpathrectangle{\pgfqpoint{0.039236in}{0.039236in}}{\pgfqpoint{1.595582in}{1.177068in}}%
\pgfusepath{clip}%
\pgfsetbuttcap%
\pgfsetmiterjoin%
\definecolor{currentfill}{rgb}{0.827451,0.827451,0.827451}%
\pgfsetfillcolor{currentfill}%
\pgfsetfillopacity{0.800000}%
\pgfsetlinewidth{1.003750pt}%
\definecolor{currentstroke}{rgb}{0.000000,0.000000,0.000000}%
\pgfsetstrokecolor{currentstroke}%
\pgfsetstrokeopacity{0.800000}%
\pgfsetdash{}{0pt}%
\pgfpathmoveto{\pgfqpoint{0.438131in}{0.745477in}}%
\pgfpathlineto{\pgfqpoint{0.571096in}{0.745477in}}%
\pgfpathlineto{\pgfqpoint{0.571096in}{0.274649in}}%
\pgfpathlineto{\pgfqpoint{1.368887in}{0.274649in}}%
\pgfpathlineto{\pgfqpoint{1.368887in}{0.745477in}}%
\pgfpathlineto{\pgfqpoint{1.501852in}{0.745477in}}%
\pgfpathlineto{\pgfqpoint{1.501852in}{0.156942in}}%
\pgfpathlineto{\pgfqpoint{0.438131in}{0.156942in}}%
\pgfpathclose%
\pgfusepath{stroke,fill}%
\end{pgfscope}%
\begin{pgfscope}%
\pgfpathrectangle{\pgfqpoint{0.039236in}{0.039236in}}{\pgfqpoint{1.595582in}{1.177068in}}%
\pgfusepath{clip}%
\pgfsetbuttcap%
\pgfsetmiterjoin%
\definecolor{currentfill}{rgb}{0.000000,0.501961,0.000000}%
\pgfsetfillcolor{currentfill}%
\pgfsetfillopacity{0.600000}%
\pgfsetlinewidth{0.000000pt}%
\definecolor{currentstroke}{rgb}{0.000000,0.000000,0.000000}%
\pgfsetstrokecolor{currentstroke}%
\pgfsetstrokeopacity{0.600000}%
\pgfsetdash{}{0pt}%
\pgfpathmoveto{\pgfqpoint{0.305166in}{0.863183in}}%
\pgfpathlineto{\pgfqpoint{0.365662in}{0.863183in}}%
\pgfpathlineto{\pgfqpoint{0.305166in}{0.510063in}}%
\pgfpathclose%
\pgfusepath{fill}%
\end{pgfscope}%
\begin{pgfscope}%
\pgfpathrectangle{\pgfqpoint{0.039236in}{0.039236in}}{\pgfqpoint{1.595582in}{1.177068in}}%
\pgfusepath{clip}%
\pgfsetbuttcap%
\pgfsetmiterjoin%
\definecolor{currentfill}{rgb}{0.000000,0.501961,0.000000}%
\pgfsetfillcolor{currentfill}%
\pgfsetfillopacity{0.600000}%
\pgfsetlinewidth{0.000000pt}%
\definecolor{currentstroke}{rgb}{0.000000,0.000000,0.000000}%
\pgfsetstrokecolor{currentstroke}%
\pgfsetstrokeopacity{0.600000}%
\pgfsetdash{}{0pt}%
\pgfpathmoveto{\pgfqpoint{1.102957in}{0.863183in}}%
\pgfpathlineto{\pgfqpoint{1.102957in}{0.510063in}}%
\pgfpathlineto{\pgfqpoint{0.471512in}{0.863183in}}%
\pgfpathclose%
\pgfusepath{fill}%
\end{pgfscope}%
\begin{pgfscope}%
\pgfpathrectangle{\pgfqpoint{0.039236in}{0.039236in}}{\pgfqpoint{1.595582in}{1.177068in}}%
\pgfusepath{clip}%
\pgfsetbuttcap%
\pgfsetmiterjoin%
\definecolor{currentfill}{rgb}{0.000000,0.501961,0.000000}%
\pgfsetfillcolor{currentfill}%
\pgfsetfillopacity{0.600000}%
\pgfsetlinewidth{0.000000pt}%
\definecolor{currentstroke}{rgb}{0.000000,0.000000,0.000000}%
\pgfsetstrokecolor{currentstroke}%
\pgfsetstrokeopacity{0.600000}%
\pgfsetdash{}{0pt}%
\pgfpathmoveto{\pgfqpoint{1.264100in}{0.274649in}}%
\pgfpathlineto{\pgfqpoint{0.571096in}{0.274649in}}%
\pgfpathlineto{\pgfqpoint{0.571096in}{0.745477in}}%
\pgfpathclose%
\pgfusepath{fill}%
\end{pgfscope}%
\begin{pgfscope}%
\pgfpathrectangle{\pgfqpoint{0.039236in}{0.039236in}}{\pgfqpoint{1.595582in}{1.177068in}}%
\pgfusepath{clip}%
\pgfsetbuttcap%
\pgfsetmiterjoin%
\definecolor{currentfill}{rgb}{0.000000,0.501961,0.000000}%
\pgfsetfillcolor{currentfill}%
\pgfsetfillopacity{0.600000}%
\pgfsetlinewidth{0.000000pt}%
\definecolor{currentstroke}{rgb}{0.000000,0.000000,0.000000}%
\pgfsetstrokecolor{currentstroke}%
\pgfsetstrokeopacity{0.600000}%
\pgfsetdash{}{0pt}%
\pgfpathmoveto{\pgfqpoint{1.368887in}{0.745477in}}%
\pgfpathlineto{\pgfqpoint{1.368887in}{0.274649in}}%
\pgfpathlineto{\pgfqpoint{1.310939in}{0.274649in}}%
\pgfpathclose%
\pgfusepath{fill}%
\end{pgfscope}%
\begin{pgfscope}%
\pgfpathrectangle{\pgfqpoint{0.039236in}{0.039236in}}{\pgfqpoint{1.595582in}{1.177068in}}%
\pgfusepath{clip}%
\pgfsetbuttcap%
\pgfsetmiterjoin%
\definecolor{currentfill}{rgb}{1.000000,0.000000,0.000000}%
\pgfsetfillcolor{currentfill}%
\pgfsetlinewidth{1.003750pt}%
\definecolor{currentstroke}{rgb}{1.000000,0.000000,0.000000}%
\pgfsetstrokecolor{currentstroke}%
\pgfsetdash{}{0pt}%
\pgfsys@defobject{currentmarker}{\pgfqpoint{-0.016667in}{-0.016667in}}{\pgfqpoint{0.016667in}{0.016667in}}{%
\pgfpathmoveto{\pgfqpoint{0.000000in}{0.016667in}}%
\pgfpathlineto{\pgfqpoint{-0.016667in}{-0.016667in}}%
\pgfpathlineto{\pgfqpoint{0.016667in}{-0.016667in}}%
\pgfpathclose%
\pgfusepath{stroke,fill}%
}%
\begin{pgfscope}%
\pgfsys@transformshift{0.837026in}{0.627770in}%
\pgfsys@useobject{currentmarker}{}%
\end{pgfscope}%
\begin{pgfscope}%
\pgfsys@transformshift{0.837026in}{1.098597in}%
\pgfsys@useobject{currentmarker}{}%
\end{pgfscope}%
\end{pgfscope}%
\end{pgfpicture}%
\makeatother%
\endgroup%

%% file: figures/step1_admker3_ex_obs.pgf
\begingroup%
\makeatletter%
\begin{pgfpicture}%
\pgfpathrectangle{\pgfpointorigin}{\pgfqpoint{1.674053in}{1.255540in}}%
\pgfusepath{use as bounding box, clip}%
\begin{pgfscope}%
\pgfsetbuttcap%
\pgfsetmiterjoin%
\definecolor{currentfill}{rgb}{1.000000,1.000000,1.000000}%
\pgfsetfillcolor{currentfill}%
\pgfsetlinewidth{0.000000pt}%
\definecolor{currentstroke}{rgb}{1.000000,1.000000,1.000000}%
\pgfsetstrokecolor{currentstroke}%
\pgfsetdash{}{0pt}%
\pgfpathmoveto{\pgfqpoint{0.000000in}{0.000000in}}%
\pgfpathlineto{\pgfqpoint{1.674053in}{0.000000in}}%
\pgfpathlineto{\pgfqpoint{1.674053in}{1.255540in}}%
\pgfpathlineto{\pgfqpoint{0.000000in}{1.255540in}}%
\pgfpathclose%
\pgfusepath{fill}%
\end{pgfscope}%
\begin{pgfscope}%
\pgfpathrectangle{\pgfqpoint{0.039236in}{0.039236in}}{\pgfqpoint{1.595582in}{1.177068in}}%
\pgfusepath{clip}%
\pgfsetbuttcap%
\pgfsetmiterjoin%
\definecolor{currentfill}{rgb}{0.827451,0.827451,0.827451}%
\pgfsetfillcolor{currentfill}%
\pgfsetfillopacity{0.800000}%
\pgfsetlinewidth{1.003750pt}%
\definecolor{currentstroke}{rgb}{0.000000,0.000000,0.000000}%
\pgfsetstrokecolor{currentstroke}%
\pgfsetstrokeopacity{0.800000}%
\pgfsetdash{}{0pt}%
\pgfpathmoveto{\pgfqpoint{0.747724in}{0.572746in}}%
\pgfpathcurveto{\pgfqpoint{0.809418in}{0.653012in}}{\pgfqpoint{0.856338in}{0.745946in}}{\pgfqpoint{0.878150in}{0.831082in}}%
\pgfpathcurveto{\pgfqpoint{0.899962in}{0.916217in}}{\pgfqpoint{0.894885in}{0.986602in}}{\pgfqpoint{0.864038in}{1.026735in}}%
\pgfpathcurveto{\pgfqpoint{0.833192in}{1.066868in}}{\pgfqpoint{0.779093in}{1.073472in}}{\pgfqpoint{0.713657in}{1.045094in}}%
\pgfpathcurveto{\pgfqpoint{0.648221in}{1.016716in}}{\pgfqpoint{0.576790in}{0.955671in}}{\pgfqpoint{0.515096in}{0.875405in}}%
\pgfpathcurveto{\pgfqpoint{0.453402in}{0.795139in}}{\pgfqpoint{0.406483in}{0.702204in}}{\pgfqpoint{0.384671in}{0.617069in}}%
\pgfpathcurveto{\pgfqpoint{0.362859in}{0.531934in}}{\pgfqpoint{0.367935in}{0.461549in}}{\pgfqpoint{0.398782in}{0.421416in}}%
\pgfpathcurveto{\pgfqpoint{0.429629in}{0.381283in}}{\pgfqpoint{0.483728in}{0.374678in}}{\pgfqpoint{0.549164in}{0.403057in}}%
\pgfpathcurveto{\pgfqpoint{0.614600in}{0.431435in}}{\pgfqpoint{0.686031in}{0.492480in}}{\pgfqpoint{0.747724in}{0.572746in}}%
\pgfpathclose%
\pgfusepath{stroke,fill}%
\end{pgfscope}%
\begin{pgfscope}%
\pgfpathrectangle{\pgfqpoint{0.039236in}{0.039236in}}{\pgfqpoint{1.595582in}{1.177068in}}%
\pgfusepath{clip}%
\pgfsetbuttcap%
\pgfsetmiterjoin%
\definecolor{currentfill}{rgb}{0.627451,0.321569,0.176471}%
\pgfsetfillcolor{currentfill}%
\pgfsetfillopacity{0.800000}%
\pgfsetlinewidth{1.003750pt}%
\definecolor{currentstroke}{rgb}{0.000000,0.000000,0.000000}%
\pgfsetstrokecolor{currentstroke}%
\pgfsetstrokeopacity{0.800000}%
\pgfsetdash{}{0pt}%
\pgfpathmoveto{\pgfqpoint{1.124889in}{0.724075in}}%
\pgfpathcurveto{\pgfqpoint{1.124889in}{0.837589in}}{\pgfqpoint{1.107557in}{0.946468in}}{\pgfqpoint{1.076710in}{1.026735in}}%
\pgfpathcurveto{\pgfqpoint{1.045863in}{1.107001in}}{\pgfqpoint{1.004020in}{1.152100in}}{\pgfqpoint{0.960396in}{1.152100in}}%
\pgfpathcurveto{\pgfqpoint{0.916772in}{1.152100in}}{\pgfqpoint{0.874929in}{1.107001in}}{\pgfqpoint{0.844082in}{1.026735in}}%
\pgfpathcurveto{\pgfqpoint{0.813235in}{0.946468in}}{\pgfqpoint{0.795903in}{0.837589in}}{\pgfqpoint{0.795903in}{0.724075in}}%
\pgfpathcurveto{\pgfqpoint{0.795903in}{0.610562in}}{\pgfqpoint{0.813235in}{0.501682in}}{\pgfqpoint{0.844082in}{0.421416in}}%
\pgfpathcurveto{\pgfqpoint{0.874929in}{0.341150in}}{\pgfqpoint{0.916772in}{0.296051in}}{\pgfqpoint{0.960396in}{0.296051in}}%
\pgfpathcurveto{\pgfqpoint{1.004020in}{0.296051in}}{\pgfqpoint{1.045863in}{0.341150in}}{\pgfqpoint{1.076710in}{0.421416in}}%
\pgfpathcurveto{\pgfqpoint{1.107557in}{0.501682in}}{\pgfqpoint{1.124889in}{0.610562in}}{\pgfqpoint{1.124889in}{0.724075in}}%
\pgfpathclose%
\pgfusepath{stroke,fill}%
\end{pgfscope}%
\begin{pgfscope}%
\pgfpathrectangle{\pgfqpoint{0.039236in}{0.039236in}}{\pgfqpoint{1.595582in}{1.177068in}}%
\pgfusepath{clip}%
\pgfsetbuttcap%
\pgfsetmiterjoin%
\definecolor{currentfill}{rgb}{1.000000,0.000000,1.000000}%
\pgfsetfillcolor{currentfill}%
\pgfsetfillopacity{0.800000}%
\pgfsetlinewidth{1.003750pt}%
\definecolor{currentstroke}{rgb}{0.000000,0.000000,0.000000}%
\pgfsetstrokecolor{currentstroke}%
\pgfsetstrokeopacity{0.800000}%
\pgfsetdash{}{0pt}%
\pgfpathmoveto{\pgfqpoint{1.289382in}{0.082038in}}%
\pgfpathcurveto{\pgfqpoint{1.376630in}{0.082038in}}{\pgfqpoint{1.460316in}{0.104588in}}{\pgfqpoint{1.522010in}{0.144721in}}%
\pgfpathcurveto{\pgfqpoint{1.583704in}{0.184854in}}{\pgfqpoint{1.618368in}{0.239294in}}{\pgfqpoint{1.618368in}{0.296051in}}%
\pgfpathcurveto{\pgfqpoint{1.618368in}{0.352807in}}{\pgfqpoint{1.583704in}{0.407247in}}{\pgfqpoint{1.522010in}{0.447380in}}%
\pgfpathcurveto{\pgfqpoint{1.460316in}{0.487513in}}{\pgfqpoint{1.376630in}{0.510063in}}{\pgfqpoint{1.289382in}{0.510063in}}%
\pgfpathcurveto{\pgfqpoint{1.202134in}{0.510063in}}{\pgfqpoint{1.118448in}{0.487513in}}{\pgfqpoint{1.056754in}{0.447380in}}%
\pgfpathcurveto{\pgfqpoint{0.995060in}{0.407247in}}{\pgfqpoint{0.960396in}{0.352807in}}{\pgfqpoint{0.960396in}{0.296051in}}%
\pgfpathcurveto{\pgfqpoint{0.960396in}{0.239294in}}{\pgfqpoint{0.995060in}{0.184854in}}{\pgfqpoint{1.056754in}{0.144721in}}%
\pgfpathcurveto{\pgfqpoint{1.118448in}{0.104588in}}{\pgfqpoint{1.202134in}{0.082038in}}{\pgfqpoint{1.289382in}{0.082038in}}%
\pgfpathclose%
\pgfusepath{stroke,fill}%
\end{pgfscope}%
\begin{pgfscope}%
\pgfpathrectangle{\pgfqpoint{0.039236in}{0.039236in}}{\pgfqpoint{1.595582in}{1.177068in}}%
\pgfusepath{clip}%
\pgfsetbuttcap%
\pgfsetmiterjoin%
\definecolor{currentfill}{rgb}{1.000000,0.388235,0.278431}%
\pgfsetfillcolor{currentfill}%
\pgfsetfillopacity{0.800000}%
\pgfsetlinewidth{1.003750pt}%
\definecolor{currentstroke}{rgb}{0.000000,0.000000,0.000000}%
\pgfsetstrokecolor{currentstroke}%
\pgfsetstrokeopacity{0.800000}%
\pgfsetdash{}{0pt}%
\pgfpathmoveto{\pgfqpoint{0.055685in}{1.002292in}}%
\pgfpathlineto{\pgfqpoint{0.302424in}{1.002292in}}%
\pgfpathlineto{\pgfqpoint{0.302424in}{0.360254in}}%
\pgfpathlineto{\pgfqpoint{0.466917in}{0.360254in}}%
\pgfpathlineto{\pgfqpoint{0.466917in}{0.253248in}}%
\pgfpathlineto{\pgfqpoint{0.220178in}{0.253248in}}%
\pgfpathlineto{\pgfqpoint{0.220178in}{0.895285in}}%
\pgfpathlineto{\pgfqpoint{0.055685in}{0.895285in}}%
\pgfpathclose%
\pgfusepath{stroke,fill}%
\end{pgfscope}%
\begin{pgfscope}%
\pgfpathrectangle{\pgfqpoint{0.039236in}{0.039236in}}{\pgfqpoint{1.595582in}{1.177068in}}%
\pgfusepath{clip}%
\pgfsetbuttcap%
\pgfsetmiterjoin%
\definecolor{currentfill}{rgb}{1.000000,1.000000,0.000000}%
\pgfsetfillcolor{currentfill}%
\pgfsetfillopacity{0.600000}%
\pgfsetlinewidth{1.003750pt}%
\definecolor{currentstroke}{rgb}{0.000000,0.000000,0.000000}%
\pgfsetstrokecolor{currentstroke}%
\pgfsetstrokeopacity{0.600000}%
\pgfsetdash{{3.700000pt}{1.600000pt}}{0.000000pt}%
\pgfpathmoveto{\pgfqpoint{0.390587in}{0.428772in}}%
\pgfpathlineto{\pgfqpoint{0.483367in}{-0.174777in}}%
\pgfpathlineto{\pgfqpoint{-0.125257in}{-0.174777in}}%
\pgfpathlineto{\pgfqpoint{-0.125257in}{0.891822in}}%
\pgfpathclose%
\pgfusepath{stroke,fill}%
\end{pgfscope}%
\begin{pgfscope}%
\pgfpathrectangle{\pgfqpoint{0.039236in}{0.039236in}}{\pgfqpoint{1.595582in}{1.177068in}}%
\pgfusepath{clip}%
\pgfsetbuttcap%
\pgfsetmiterjoin%
\definecolor{currentfill}{rgb}{1.000000,0.000000,0.000000}%
\pgfsetfillcolor{currentfill}%
\pgfsetlinewidth{1.003750pt}%
\definecolor{currentstroke}{rgb}{1.000000,0.000000,0.000000}%
\pgfsetstrokecolor{currentstroke}%
\pgfsetdash{}{0pt}%
\pgfsys@defobject{currentmarker}{\pgfqpoint{-0.016667in}{-0.016667in}}{\pgfqpoint{0.016667in}{0.016667in}}{%
\pgfpathmoveto{\pgfqpoint{0.000000in}{0.016667in}}%
\pgfpathlineto{\pgfqpoint{-0.016667in}{-0.016667in}}%
\pgfpathlineto{\pgfqpoint{0.016667in}{-0.016667in}}%
\pgfpathclose%
\pgfusepath{stroke,fill}%
}%
\begin{pgfscope}%
\pgfsys@transformshift{0.351772in}{0.681273in}%
\pgfsys@useobject{currentmarker}{}%
\end{pgfscope}%
\begin{pgfscope}%
\pgfsys@transformshift{0.631410in}{0.274649in}%
\pgfsys@useobject{currentmarker}{}%
\end{pgfscope}%
\end{pgfscope}%
\begin{pgfscope}%
\pgfpathrectangle{\pgfqpoint{0.039236in}{0.039236in}}{\pgfqpoint{1.595582in}{1.177068in}}%
\pgfusepath{clip}%
\pgfsetbuttcap%
\pgfsetmiterjoin%
\definecolor{currentfill}{rgb}{1.000000,0.000000,0.000000}%
\pgfsetfillcolor{currentfill}%
\pgfsetlinewidth{1.003750pt}%
\definecolor{currentstroke}{rgb}{1.000000,0.000000,0.000000}%
\pgfsetstrokecolor{currentstroke}%
\pgfsetdash{}{0pt}%
\pgfsys@defobject{currentmarker}{\pgfqpoint{-0.016667in}{-0.016667in}}{\pgfqpoint{0.016667in}{0.016667in}}{%
\pgfpathmoveto{\pgfqpoint{0.000000in}{0.016667in}}%
\pgfpathlineto{\pgfqpoint{-0.016667in}{-0.016667in}}%
\pgfpathlineto{\pgfqpoint{0.016667in}{-0.016667in}}%
\pgfpathclose%
\pgfusepath{stroke,fill}%
}%
\begin{pgfscope}%
\pgfsys@transformshift{0.398782in}{0.421416in}%
\pgfsys@useobject{currentmarker}{}%
\end{pgfscope}%
\end{pgfscope}%
\begin{pgfscope}%
\pgfpathrectangle{\pgfqpoint{0.039236in}{0.039236in}}{\pgfqpoint{1.595582in}{1.177068in}}%
\pgfusepath{clip}%
\pgfsetbuttcap%
\pgfsetmiterjoin%
\definecolor{currentfill}{rgb}{1.000000,0.000000,0.000000}%
\pgfsetfillcolor{currentfill}%
\pgfsetlinewidth{1.003750pt}%
\definecolor{currentstroke}{rgb}{1.000000,0.000000,0.000000}%
\pgfsetstrokecolor{currentstroke}%
\pgfsetdash{}{0pt}%
\pgfsys@defobject{currentmarker}{\pgfqpoint{-0.016667in}{-0.016667in}}{\pgfqpoint{0.016667in}{0.016667in}}{%
\pgfpathmoveto{\pgfqpoint{0.000000in}{0.016667in}}%
\pgfpathlineto{\pgfqpoint{-0.016667in}{-0.016667in}}%
\pgfpathlineto{\pgfqpoint{0.016667in}{-0.016667in}}%
\pgfpathclose%
\pgfusepath{stroke,fill}%
}%
\begin{pgfscope}%
\pgfsys@transformshift{0.864038in}{1.026735in}%
\pgfsys@useobject{currentmarker}{}%
\end{pgfscope}%
\end{pgfscope}%
\begin{pgfscope}%
\pgfpathrectangle{\pgfqpoint{0.039236in}{0.039236in}}{\pgfqpoint{1.595582in}{1.177068in}}%
\pgfusepath{clip}%
\pgfsetbuttcap%
\pgfsetmiterjoin%
\definecolor{currentfill}{rgb}{1.000000,0.000000,0.000000}%
\pgfsetfillcolor{currentfill}%
\pgfsetlinewidth{1.003750pt}%
\definecolor{currentstroke}{rgb}{1.000000,0.000000,0.000000}%
\pgfsetstrokecolor{currentstroke}%
\pgfsetdash{}{0pt}%
\pgfsys@defobject{currentmarker}{\pgfqpoint{-0.016667in}{-0.016667in}}{\pgfqpoint{0.016667in}{0.016667in}}{%
\pgfpathmoveto{\pgfqpoint{0.000000in}{0.016667in}}%
\pgfpathlineto{\pgfqpoint{-0.016667in}{-0.016667in}}%
\pgfpathlineto{\pgfqpoint{0.016667in}{-0.016667in}}%
\pgfpathclose%
\pgfusepath{stroke,fill}%
}%
\begin{pgfscope}%
\pgfsys@transformshift{0.747724in}{0.572746in}%
\pgfsys@useobject{currentmarker}{}%
\end{pgfscope}%
\end{pgfscope}%
\begin{pgfscope}%
\pgfpathrectangle{\pgfqpoint{0.039236in}{0.039236in}}{\pgfqpoint{1.595582in}{1.177068in}}%
\pgfusepath{clip}%
\pgfsetbuttcap%
\pgfsetmiterjoin%
\definecolor{currentfill}{rgb}{1.000000,0.000000,0.000000}%
\pgfsetfillcolor{currentfill}%
\pgfsetlinewidth{1.003750pt}%
\definecolor{currentstroke}{rgb}{1.000000,0.000000,0.000000}%
\pgfsetstrokecolor{currentstroke}%
\pgfsetdash{}{0pt}%
\pgfsys@defobject{currentmarker}{\pgfqpoint{-0.016667in}{-0.016667in}}{\pgfqpoint{0.016667in}{0.016667in}}{%
\pgfpathmoveto{\pgfqpoint{0.000000in}{0.016667in}}%
\pgfpathlineto{\pgfqpoint{-0.016667in}{-0.016667in}}%
\pgfpathlineto{\pgfqpoint{0.016667in}{-0.016667in}}%
\pgfpathclose%
\pgfusepath{stroke,fill}%
}%
\begin{pgfscope}%
\pgfsys@transformshift{0.515096in}{0.875405in}%
\pgfsys@useobject{currentmarker}{}%
\end{pgfscope}%
\end{pgfscope}%
\begin{pgfscope}%
\pgfpathrectangle{\pgfqpoint{0.039236in}{0.039236in}}{\pgfqpoint{1.595582in}{1.177068in}}%
\pgfusepath{clip}%
\pgfsetbuttcap%
\pgfsetmiterjoin%
\definecolor{currentfill}{rgb}{1.000000,0.000000,0.000000}%
\pgfsetfillcolor{currentfill}%
\pgfsetlinewidth{1.003750pt}%
\definecolor{currentstroke}{rgb}{1.000000,0.000000,0.000000}%
\pgfsetstrokecolor{currentstroke}%
\pgfsetdash{}{0pt}%
\pgfsys@defobject{currentmarker}{\pgfqpoint{-0.016667in}{-0.016667in}}{\pgfqpoint{0.016667in}{0.016667in}}{%
\pgfpathmoveto{\pgfqpoint{0.000000in}{0.016667in}}%
\pgfpathlineto{\pgfqpoint{-0.016667in}{-0.016667in}}%
\pgfpathlineto{\pgfqpoint{0.016667in}{-0.016667in}}%
\pgfpathclose%
\pgfusepath{stroke,fill}%
}%
\begin{pgfscope}%
\pgfsys@transformshift{0.960396in}{0.296051in}%
\pgfsys@useobject{currentmarker}{}%
\end{pgfscope}%
\end{pgfscope}%
\begin{pgfscope}%
\pgfpathrectangle{\pgfqpoint{0.039236in}{0.039236in}}{\pgfqpoint{1.595582in}{1.177068in}}%
\pgfusepath{clip}%
\pgfsetbuttcap%
\pgfsetmiterjoin%
\definecolor{currentfill}{rgb}{1.000000,0.000000,0.000000}%
\pgfsetfillcolor{currentfill}%
\pgfsetlinewidth{1.003750pt}%
\definecolor{currentstroke}{rgb}{1.000000,0.000000,0.000000}%
\pgfsetstrokecolor{currentstroke}%
\pgfsetdash{}{0pt}%
\pgfsys@defobject{currentmarker}{\pgfqpoint{-0.016667in}{-0.016667in}}{\pgfqpoint{0.016667in}{0.016667in}}{%
\pgfpathmoveto{\pgfqpoint{0.000000in}{0.016667in}}%
\pgfpathlineto{\pgfqpoint{-0.016667in}{-0.016667in}}%
\pgfpathlineto{\pgfqpoint{0.016667in}{-0.016667in}}%
\pgfpathclose%
\pgfusepath{stroke,fill}%
}%
\begin{pgfscope}%
\pgfsys@transformshift{0.960396in}{1.152100in}%
\pgfsys@useobject{currentmarker}{}%
\end{pgfscope}%
\end{pgfscope}%
\begin{pgfscope}%
\pgfpathrectangle{\pgfqpoint{0.039236in}{0.039236in}}{\pgfqpoint{1.595582in}{1.177068in}}%
\pgfusepath{clip}%
\pgfsetbuttcap%
\pgfsetmiterjoin%
\definecolor{currentfill}{rgb}{1.000000,0.000000,0.000000}%
\pgfsetfillcolor{currentfill}%
\pgfsetlinewidth{1.003750pt}%
\definecolor{currentstroke}{rgb}{1.000000,0.000000,0.000000}%
\pgfsetstrokecolor{currentstroke}%
\pgfsetdash{}{0pt}%
\pgfsys@defobject{currentmarker}{\pgfqpoint{-0.016667in}{-0.016667in}}{\pgfqpoint{0.016667in}{0.016667in}}{%
\pgfpathmoveto{\pgfqpoint{0.000000in}{0.016667in}}%
\pgfpathlineto{\pgfqpoint{-0.016667in}{-0.016667in}}%
\pgfpathlineto{\pgfqpoint{0.016667in}{-0.016667in}}%
\pgfpathclose%
\pgfusepath{stroke,fill}%
}%
\begin{pgfscope}%
\pgfsys@transformshift{1.124889in}{0.724075in}%
\pgfsys@useobject{currentmarker}{}%
\end{pgfscope}%
\end{pgfscope}%
\begin{pgfscope}%
\pgfpathrectangle{\pgfqpoint{0.039236in}{0.039236in}}{\pgfqpoint{1.595582in}{1.177068in}}%
\pgfusepath{clip}%
\pgfsetbuttcap%
\pgfsetmiterjoin%
\definecolor{currentfill}{rgb}{1.000000,0.000000,0.000000}%
\pgfsetfillcolor{currentfill}%
\pgfsetlinewidth{1.003750pt}%
\definecolor{currentstroke}{rgb}{1.000000,0.000000,0.000000}%
\pgfsetstrokecolor{currentstroke}%
\pgfsetdash{}{0pt}%
\pgfsys@defobject{currentmarker}{\pgfqpoint{-0.016667in}{-0.016667in}}{\pgfqpoint{0.016667in}{0.016667in}}{%
\pgfpathmoveto{\pgfqpoint{0.000000in}{0.016667in}}%
\pgfpathlineto{\pgfqpoint{-0.016667in}{-0.016667in}}%
\pgfpathlineto{\pgfqpoint{0.016667in}{-0.016667in}}%
\pgfpathclose%
\pgfusepath{stroke,fill}%
}%
\begin{pgfscope}%
\pgfsys@transformshift{0.795903in}{0.724075in}%
\pgfsys@useobject{currentmarker}{}%
\end{pgfscope}%
\end{pgfscope}%
\begin{pgfscope}%
\pgfpathrectangle{\pgfqpoint{0.039236in}{0.039236in}}{\pgfqpoint{1.595582in}{1.177068in}}%
\pgfusepath{clip}%
\pgfsetbuttcap%
\pgfsetmiterjoin%
\definecolor{currentfill}{rgb}{1.000000,0.000000,0.000000}%
\pgfsetfillcolor{currentfill}%
\pgfsetlinewidth{1.003750pt}%
\definecolor{currentstroke}{rgb}{1.000000,0.000000,0.000000}%
\pgfsetstrokecolor{currentstroke}%
\pgfsetdash{}{0pt}%
\pgfsys@defobject{currentmarker}{\pgfqpoint{-0.016667in}{-0.016667in}}{\pgfqpoint{0.016667in}{0.016667in}}{%
\pgfpathmoveto{\pgfqpoint{0.000000in}{0.016667in}}%
\pgfpathlineto{\pgfqpoint{-0.016667in}{-0.016667in}}%
\pgfpathlineto{\pgfqpoint{0.016667in}{-0.016667in}}%
\pgfpathclose%
\pgfusepath{stroke,fill}%
}%
\begin{pgfscope}%
\pgfsys@transformshift{0.960396in}{0.296051in}%
\pgfsys@useobject{currentmarker}{}%
\end{pgfscope}%
\end{pgfscope}%
\begin{pgfscope}%
\pgfpathrectangle{\pgfqpoint{0.039236in}{0.039236in}}{\pgfqpoint{1.595582in}{1.177068in}}%
\pgfusepath{clip}%
\pgfsetbuttcap%
\pgfsetmiterjoin%
\definecolor{currentfill}{rgb}{1.000000,0.000000,0.000000}%
\pgfsetfillcolor{currentfill}%
\pgfsetlinewidth{1.003750pt}%
\definecolor{currentstroke}{rgb}{1.000000,0.000000,0.000000}%
\pgfsetstrokecolor{currentstroke}%
\pgfsetdash{}{0pt}%
\pgfsys@defobject{currentmarker}{\pgfqpoint{-0.016667in}{-0.016667in}}{\pgfqpoint{0.016667in}{0.016667in}}{%
\pgfpathmoveto{\pgfqpoint{0.000000in}{0.016667in}}%
\pgfpathlineto{\pgfqpoint{-0.016667in}{-0.016667in}}%
\pgfpathlineto{\pgfqpoint{0.016667in}{-0.016667in}}%
\pgfpathclose%
\pgfusepath{stroke,fill}%
}%
\begin{pgfscope}%
\pgfsys@transformshift{1.618368in}{0.296051in}%
\pgfsys@useobject{currentmarker}{}%
\end{pgfscope}%
\end{pgfscope}%
\begin{pgfscope}%
\pgfpathrectangle{\pgfqpoint{0.039236in}{0.039236in}}{\pgfqpoint{1.595582in}{1.177068in}}%
\pgfusepath{clip}%
\pgfsetbuttcap%
\pgfsetmiterjoin%
\definecolor{currentfill}{rgb}{1.000000,0.000000,0.000000}%
\pgfsetfillcolor{currentfill}%
\pgfsetlinewidth{1.003750pt}%
\definecolor{currentstroke}{rgb}{1.000000,0.000000,0.000000}%
\pgfsetstrokecolor{currentstroke}%
\pgfsetdash{}{0pt}%
\pgfsys@defobject{currentmarker}{\pgfqpoint{-0.016667in}{-0.016667in}}{\pgfqpoint{0.016667in}{0.016667in}}{%
\pgfpathmoveto{\pgfqpoint{0.000000in}{0.016667in}}%
\pgfpathlineto{\pgfqpoint{-0.016667in}{-0.016667in}}%
\pgfpathlineto{\pgfqpoint{0.016667in}{-0.016667in}}%
\pgfpathclose%
\pgfusepath{stroke,fill}%
}%
\begin{pgfscope}%
\pgfsys@transformshift{1.289382in}{0.082038in}%
\pgfsys@useobject{currentmarker}{}%
\end{pgfscope}%
\end{pgfscope}%
\begin{pgfscope}%
\pgfpathrectangle{\pgfqpoint{0.039236in}{0.039236in}}{\pgfqpoint{1.595582in}{1.177068in}}%
\pgfusepath{clip}%
\pgfsetbuttcap%
\pgfsetmiterjoin%
\definecolor{currentfill}{rgb}{1.000000,0.000000,0.000000}%
\pgfsetfillcolor{currentfill}%
\pgfsetlinewidth{1.003750pt}%
\definecolor{currentstroke}{rgb}{1.000000,0.000000,0.000000}%
\pgfsetstrokecolor{currentstroke}%
\pgfsetdash{}{0pt}%
\pgfsys@defobject{currentmarker}{\pgfqpoint{-0.016667in}{-0.016667in}}{\pgfqpoint{0.016667in}{0.016667in}}{%
\pgfpathmoveto{\pgfqpoint{0.000000in}{0.016667in}}%
\pgfpathlineto{\pgfqpoint{-0.016667in}{-0.016667in}}%
\pgfpathlineto{\pgfqpoint{0.016667in}{-0.016667in}}%
\pgfpathclose%
\pgfusepath{stroke,fill}%
}%
\begin{pgfscope}%
\pgfsys@transformshift{1.289382in}{0.510063in}%
\pgfsys@useobject{currentmarker}{}%
\end{pgfscope}%
\end{pgfscope}%
\end{pgfpicture}%
\makeatother%
\endgroup%

%% file: figures/step1_star_ex_obs.pgf
\begingroup%
\makeatletter%
\begin{pgfpicture}%
\pgfpathrectangle{\pgfpointorigin}{\pgfqpoint{1.674053in}{1.255540in}}%
\pgfusepath{use as bounding box, clip}%
\begin{pgfscope}%
\pgfsetbuttcap%
\pgfsetmiterjoin%
\definecolor{currentfill}{rgb}{1.000000,1.000000,1.000000}%
\pgfsetfillcolor{currentfill}%
\pgfsetlinewidth{0.000000pt}%
\definecolor{currentstroke}{rgb}{1.000000,1.000000,1.000000}%
\pgfsetstrokecolor{currentstroke}%
\pgfsetdash{}{0pt}%
\pgfpathmoveto{\pgfqpoint{0.000000in}{0.000000in}}%
\pgfpathlineto{\pgfqpoint{1.674053in}{0.000000in}}%
\pgfpathlineto{\pgfqpoint{1.674053in}{1.255540in}}%
\pgfpathlineto{\pgfqpoint{0.000000in}{1.255540in}}%
\pgfpathclose%
\pgfusepath{fill}%
\end{pgfscope}%
\begin{pgfscope}%
\pgfpathrectangle{\pgfqpoint{0.039236in}{0.039236in}}{\pgfqpoint{1.595582in}{1.177068in}}%
\pgfusepath{clip}%
\pgfsetbuttcap%
\pgfsetmiterjoin%
\definecolor{currentfill}{rgb}{0.827451,0.827451,0.827451}%
\pgfsetfillcolor{currentfill}%
\pgfsetfillopacity{0.800000}%
\pgfsetlinewidth{1.003750pt}%
\definecolor{currentstroke}{rgb}{0.000000,0.000000,0.000000}%
\pgfsetstrokecolor{currentstroke}%
\pgfsetstrokeopacity{0.800000}%
\pgfsetdash{}{0pt}%
\pgfpathmoveto{\pgfqpoint{0.864038in}{1.026735in}}%
\pgfpathlineto{\pgfqpoint{0.856276in}{1.035640in}}%
\pgfpathlineto{\pgfqpoint{0.847626in}{1.043315in}}%
\pgfpathlineto{\pgfqpoint{0.838123in}{1.049730in}}%
\pgfpathlineto{\pgfqpoint{0.827804in}{1.054860in}}%
\pgfpathlineto{\pgfqpoint{0.816710in}{1.058685in}}%
\pgfpathlineto{\pgfqpoint{0.804884in}{1.061189in}}%
\pgfpathlineto{\pgfqpoint{0.792374in}{1.062363in}}%
\pgfpathlineto{\pgfqpoint{0.779229in}{1.062201in}}%
\pgfpathlineto{\pgfqpoint{0.765500in}{1.060706in}}%
\pgfpathlineto{\pgfqpoint{0.751243in}{1.057881in}}%
\pgfpathlineto{\pgfqpoint{0.736512in}{1.053740in}}%
\pgfpathlineto{\pgfqpoint{0.721366in}{1.048297in}}%
\pgfpathlineto{\pgfqpoint{0.705866in}{1.041574in}}%
\pgfpathlineto{\pgfqpoint{0.690071in}{1.033599in}}%
\pgfpathlineto{\pgfqpoint{0.674046in}{1.024402in}}%
\pgfpathlineto{\pgfqpoint{0.657851in}{1.014020in}}%
\pgfpathlineto{\pgfqpoint{0.641553in}{1.002494in}}%
\pgfpathlineto{\pgfqpoint{0.625214in}{0.989869in}}%
\pgfpathlineto{\pgfqpoint{0.608900in}{0.976194in}}%
\pgfpathlineto{\pgfqpoint{0.592675in}{0.961525in}}%
\pgfpathlineto{\pgfqpoint{0.576603in}{0.945919in}}%
\pgfpathlineto{\pgfqpoint{0.560747in}{0.929437in}}%
\pgfpathlineto{\pgfqpoint{0.545169in}{0.912145in}}%
\pgfpathlineto{\pgfqpoint{0.529933in}{0.894111in}}%
\pgfpathlineto{\pgfqpoint{0.515096in}{0.875405in}}%
\pgfpathlineto{\pgfqpoint{0.500719in}{0.856102in}}%
\pgfpathlineto{\pgfqpoint{0.486857in}{0.836278in}}%
\pgfpathlineto{\pgfqpoint{0.473566in}{0.816012in}}%
\pgfpathlineto{\pgfqpoint{0.460898in}{0.795382in}}%
\pgfpathlineto{\pgfqpoint{0.448903in}{0.774472in}}%
\pgfpathlineto{\pgfqpoint{0.437628in}{0.753362in}}%
\pgfpathlineto{\pgfqpoint{0.427118in}{0.732136in}}%
\pgfpathlineto{\pgfqpoint{0.417414in}{0.710879in}}%
\pgfpathlineto{\pgfqpoint{0.408555in}{0.689674in}}%
\pgfpathlineto{\pgfqpoint{0.400575in}{0.668605in}}%
\pgfpathlineto{\pgfqpoint{0.393506in}{0.647755in}}%
\pgfpathlineto{\pgfqpoint{0.387376in}{0.627205in}}%
\pgfpathlineto{\pgfqpoint{0.382209in}{0.607039in}}%
\pgfpathlineto{\pgfqpoint{0.378026in}{0.587334in}}%
\pgfpathlineto{\pgfqpoint{0.374842in}{0.568168in}}%
\pgfpathlineto{\pgfqpoint{0.372672in}{0.549618in}}%
\pgfpathlineto{\pgfqpoint{0.371522in}{0.531757in}}%
\pgfpathlineto{\pgfqpoint{0.371398in}{0.514654in}}%
\pgfpathlineto{\pgfqpoint{0.372300in}{0.498378in}}%
\pgfpathlineto{\pgfqpoint{0.374225in}{0.482993in}}%
\pgfpathlineto{\pgfqpoint{0.377164in}{0.468559in}}%
\pgfpathlineto{\pgfqpoint{0.381107in}{0.455133in}}%
\pgfpathlineto{\pgfqpoint{0.386038in}{0.442769in}}%
\pgfpathlineto{\pgfqpoint{0.391938in}{0.431515in}}%
\pgfpathlineto{\pgfqpoint{0.398782in}{0.421416in}}%
\pgfpathlineto{\pgfqpoint{0.406545in}{0.412511in}}%
\pgfpathlineto{\pgfqpoint{0.415194in}{0.404836in}}%
\pgfpathlineto{\pgfqpoint{0.424698in}{0.398421in}}%
\pgfpathlineto{\pgfqpoint{0.435017in}{0.393291in}}%
\pgfpathlineto{\pgfqpoint{0.446111in}{0.389466in}}%
\pgfpathlineto{\pgfqpoint{0.457936in}{0.386962in}}%
\pgfpathlineto{\pgfqpoint{0.470446in}{0.385788in}}%
\pgfpathlineto{\pgfqpoint{0.483591in}{0.385949in}}%
\pgfpathlineto{\pgfqpoint{0.497320in}{0.387445in}}%
\pgfpathlineto{\pgfqpoint{0.511578in}{0.390270in}}%
\pgfpathlineto{\pgfqpoint{0.526309in}{0.394411in}}%
\pgfpathlineto{\pgfqpoint{0.541454in}{0.399854in}}%
\pgfpathlineto{\pgfqpoint{0.556955in}{0.406576in}}%
\pgfpathlineto{\pgfqpoint{0.572749in}{0.414552in}}%
\pgfpathlineto{\pgfqpoint{0.588775in}{0.423748in}}%
\pgfpathlineto{\pgfqpoint{0.604969in}{0.434131in}}%
\pgfpathlineto{\pgfqpoint{0.621268in}{0.445657in}}%
\pgfpathlineto{\pgfqpoint{0.637606in}{0.458282in}}%
\pgfpathlineto{\pgfqpoint{0.653920in}{0.471956in}}%
\pgfpathlineto{\pgfqpoint{0.670145in}{0.486625in}}%
\pgfpathlineto{\pgfqpoint{0.686218in}{0.502232in}}%
\pgfpathlineto{\pgfqpoint{0.702074in}{0.518714in}}%
\pgfpathlineto{\pgfqpoint{0.717651in}{0.536006in}}%
\pgfpathlineto{\pgfqpoint{0.732888in}{0.554040in}}%
\pgfpathlineto{\pgfqpoint{0.747724in}{0.572746in}}%
\pgfpathlineto{\pgfqpoint{0.762102in}{0.592049in}}%
\pgfpathlineto{\pgfqpoint{0.775963in}{0.611872in}}%
\pgfpathlineto{\pgfqpoint{0.789254in}{0.632139in}}%
\pgfpathlineto{\pgfqpoint{0.801922in}{0.652768in}}%
\pgfpathlineto{\pgfqpoint{0.813918in}{0.673679in}}%
\pgfpathlineto{\pgfqpoint{0.825192in}{0.694789in}}%
\pgfpathlineto{\pgfqpoint{0.835703in}{0.716014in}}%
\pgfpathlineto{\pgfqpoint{0.845406in}{0.737271in}}%
\pgfpathlineto{\pgfqpoint{0.854266in}{0.758477in}}%
\pgfpathlineto{\pgfqpoint{0.862246in}{0.779546in}}%
\pgfpathlineto{\pgfqpoint{0.869315in}{0.800396in}}%
\pgfpathlineto{\pgfqpoint{0.875444in}{0.820945in}}%
\pgfpathlineto{\pgfqpoint{0.880611in}{0.841112in}}%
\pgfpathlineto{\pgfqpoint{0.884795in}{0.860817in}}%
\pgfpathlineto{\pgfqpoint{0.887978in}{0.879983in}}%
\pgfpathlineto{\pgfqpoint{0.890149in}{0.898533in}}%
\pgfpathlineto{\pgfqpoint{0.891299in}{0.916394in}}%
\pgfpathlineto{\pgfqpoint{0.891423in}{0.933497in}}%
\pgfpathlineto{\pgfqpoint{0.890520in}{0.949773in}}%
\pgfpathlineto{\pgfqpoint{0.888596in}{0.965158in}}%
\pgfpathlineto{\pgfqpoint{0.885656in}{0.979592in}}%
\pgfpathlineto{\pgfqpoint{0.881713in}{0.993017in}}%
\pgfpathlineto{\pgfqpoint{0.876782in}{1.005381in}}%
\pgfpathlineto{\pgfqpoint{0.870883in}{1.016635in}}%
\pgfpathclose%
\pgfusepath{stroke,fill}%
\end{pgfscope}%
\begin{pgfscope}%
\pgfpathrectangle{\pgfqpoint{0.039236in}{0.039236in}}{\pgfqpoint{1.595582in}{1.177068in}}%
\pgfusepath{clip}%
\pgfsetbuttcap%
\pgfsetmiterjoin%
\definecolor{currentfill}{rgb}{0.254902,0.411765,0.882353}%
\pgfsetfillcolor{currentfill}%
\pgfsetfillopacity{0.600000}%
\pgfsetlinewidth{1.003750pt}%
\definecolor{currentstroke}{rgb}{0.000000,0.000000,0.000000}%
\pgfsetstrokecolor{currentstroke}%
\pgfsetstrokeopacity{0.600000}%
\pgfsetdash{{1.000000pt}{1.650000pt}}{0.000000pt}%
\pgfpathmoveto{\pgfqpoint{0.631410in}{0.781145in}}%
\pgfpathlineto{\pgfqpoint{0.664309in}{0.695540in}}%
\pgfpathlineto{\pgfqpoint{0.598512in}{0.695540in}}%
\pgfpathclose%
\pgfusepath{stroke,fill}%
\end{pgfscope}%
\begin{pgfscope}%
\pgfpathrectangle{\pgfqpoint{0.039236in}{0.039236in}}{\pgfqpoint{1.595582in}{1.177068in}}%
\pgfusepath{clip}%
\pgfsetbuttcap%
\pgfsetmiterjoin%
\definecolor{currentfill}{rgb}{0.627451,0.321569,0.176471}%
\pgfsetfillcolor{currentfill}%
\pgfsetfillopacity{0.800000}%
\pgfsetlinewidth{1.003750pt}%
\definecolor{currentstroke}{rgb}{0.000000,0.000000,0.000000}%
\pgfsetstrokecolor{currentstroke}%
\pgfsetstrokeopacity{0.800000}%
\pgfsetdash{}{0pt}%
\pgfpathmoveto{\pgfqpoint{0.960396in}{1.152100in}}%
\pgfpathlineto{\pgfqpoint{0.950068in}{1.151256in}}%
\pgfpathlineto{\pgfqpoint{0.939780in}{1.148725in}}%
\pgfpathlineto{\pgfqpoint{0.929573in}{1.144519in}}%
\pgfpathlineto{\pgfqpoint{0.919488in}{1.138653in}}%
\pgfpathlineto{\pgfqpoint{0.909565in}{1.131151in}}%
\pgfpathlineto{\pgfqpoint{0.899842in}{1.122043in}}%
\pgfpathlineto{\pgfqpoint{0.890358in}{1.111364in}}%
\pgfpathlineto{\pgfqpoint{0.881151in}{1.099156in}}%
\pgfpathlineto{\pgfqpoint{0.872256in}{1.085469in}}%
\pgfpathlineto{\pgfqpoint{0.863710in}{1.070355in}}%
\pgfpathlineto{\pgfqpoint{0.855544in}{1.053874in}}%
\pgfpathlineto{\pgfqpoint{0.847793in}{1.036092in}}%
\pgfpathlineto{\pgfqpoint{0.840486in}{1.017079in}}%
\pgfpathlineto{\pgfqpoint{0.833652in}{0.996909in}}%
\pgfpathlineto{\pgfqpoint{0.827319in}{0.975662in}}%
\pgfpathlineto{\pgfqpoint{0.821510in}{0.953423in}}%
\pgfpathlineto{\pgfqpoint{0.816250in}{0.930278in}}%
\pgfpathlineto{\pgfqpoint{0.811558in}{0.906320in}}%
\pgfpathlineto{\pgfqpoint{0.807454in}{0.881642in}}%
\pgfpathlineto{\pgfqpoint{0.803954in}{0.856342in}}%
\pgfpathlineto{\pgfqpoint{0.801071in}{0.830521in}}%
\pgfpathlineto{\pgfqpoint{0.798817in}{0.804279in}}%
\pgfpathlineto{\pgfqpoint{0.797200in}{0.777721in}}%
\pgfpathlineto{\pgfqpoint{0.796228in}{0.750951in}}%
\pgfpathlineto{\pgfqpoint{0.795903in}{0.724075in}}%
\pgfpathlineto{\pgfqpoint{0.796228in}{0.697199in}}%
\pgfpathlineto{\pgfqpoint{0.797200in}{0.670430in}}%
\pgfpathlineto{\pgfqpoint{0.798817in}{0.643872in}}%
\pgfpathlineto{\pgfqpoint{0.801071in}{0.617630in}}%
\pgfpathlineto{\pgfqpoint{0.803954in}{0.591808in}}%
\pgfpathlineto{\pgfqpoint{0.807454in}{0.566509in}}%
\pgfpathlineto{\pgfqpoint{0.811558in}{0.541831in}}%
\pgfpathlineto{\pgfqpoint{0.816250in}{0.517873in}}%
\pgfpathlineto{\pgfqpoint{0.821510in}{0.494728in}}%
\pgfpathlineto{\pgfqpoint{0.827319in}{0.472489in}}%
\pgfpathlineto{\pgfqpoint{0.833652in}{0.451242in}}%
\pgfpathlineto{\pgfqpoint{0.840486in}{0.431072in}}%
\pgfpathlineto{\pgfqpoint{0.847793in}{0.412059in}}%
\pgfpathlineto{\pgfqpoint{0.855544in}{0.394277in}}%
\pgfpathlineto{\pgfqpoint{0.863710in}{0.377796in}}%
\pgfpathlineto{\pgfqpoint{0.872256in}{0.362682in}}%
\pgfpathlineto{\pgfqpoint{0.881151in}{0.348994in}}%
\pgfpathlineto{\pgfqpoint{0.890358in}{0.336787in}}%
\pgfpathlineto{\pgfqpoint{0.899842in}{0.326108in}}%
\pgfpathlineto{\pgfqpoint{0.909565in}{0.317000in}}%
\pgfpathlineto{\pgfqpoint{0.919488in}{0.309498in}}%
\pgfpathlineto{\pgfqpoint{0.929573in}{0.303632in}}%
\pgfpathlineto{\pgfqpoint{0.939780in}{0.299426in}}%
\pgfpathlineto{\pgfqpoint{0.950068in}{0.296895in}}%
\pgfpathlineto{\pgfqpoint{0.960396in}{0.296051in}}%
\pgfpathlineto{\pgfqpoint{0.970725in}{0.296895in}}%
\pgfpathlineto{\pgfqpoint{0.981013in}{0.299426in}}%
\pgfpathlineto{\pgfqpoint{0.991219in}{0.303632in}}%
\pgfpathlineto{\pgfqpoint{1.001304in}{0.309498in}}%
\pgfpathlineto{\pgfqpoint{1.011227in}{0.317000in}}%
\pgfpathlineto{\pgfqpoint{1.020950in}{0.326108in}}%
\pgfpathlineto{\pgfqpoint{1.030434in}{0.336787in}}%
\pgfpathlineto{\pgfqpoint{1.039641in}{0.348994in}}%
\pgfpathlineto{\pgfqpoint{1.048536in}{0.362682in}}%
\pgfpathlineto{\pgfqpoint{1.057083in}{0.377796in}}%
\pgfpathlineto{\pgfqpoint{1.065248in}{0.394277in}}%
\pgfpathlineto{\pgfqpoint{1.072999in}{0.412059in}}%
\pgfpathlineto{\pgfqpoint{1.080306in}{0.431072in}}%
\pgfpathlineto{\pgfqpoint{1.087140in}{0.451242in}}%
\pgfpathlineto{\pgfqpoint{1.093474in}{0.472489in}}%
\pgfpathlineto{\pgfqpoint{1.099282in}{0.494728in}}%
\pgfpathlineto{\pgfqpoint{1.104542in}{0.517873in}}%
\pgfpathlineto{\pgfqpoint{1.109234in}{0.541831in}}%
\pgfpathlineto{\pgfqpoint{1.113338in}{0.566509in}}%
\pgfpathlineto{\pgfqpoint{1.116838in}{0.591808in}}%
\pgfpathlineto{\pgfqpoint{1.119721in}{0.617630in}}%
\pgfpathlineto{\pgfqpoint{1.121975in}{0.643872in}}%
\pgfpathlineto{\pgfqpoint{1.123592in}{0.670430in}}%
\pgfpathlineto{\pgfqpoint{1.124564in}{0.697199in}}%
\pgfpathlineto{\pgfqpoint{1.124889in}{0.724075in}}%
\pgfpathlineto{\pgfqpoint{1.124564in}{0.750951in}}%
\pgfpathlineto{\pgfqpoint{1.123592in}{0.777721in}}%
\pgfpathlineto{\pgfqpoint{1.121975in}{0.804279in}}%
\pgfpathlineto{\pgfqpoint{1.119721in}{0.830521in}}%
\pgfpathlineto{\pgfqpoint{1.116838in}{0.856342in}}%
\pgfpathlineto{\pgfqpoint{1.113338in}{0.881642in}}%
\pgfpathlineto{\pgfqpoint{1.109234in}{0.906320in}}%
\pgfpathlineto{\pgfqpoint{1.104542in}{0.930278in}}%
\pgfpathlineto{\pgfqpoint{1.099282in}{0.953423in}}%
\pgfpathlineto{\pgfqpoint{1.093474in}{0.975662in}}%
\pgfpathlineto{\pgfqpoint{1.087140in}{0.996909in}}%
\pgfpathlineto{\pgfqpoint{1.080306in}{1.017079in}}%
\pgfpathlineto{\pgfqpoint{1.072999in}{1.036092in}}%
\pgfpathlineto{\pgfqpoint{1.065248in}{1.053874in}}%
\pgfpathlineto{\pgfqpoint{1.057083in}{1.070355in}}%
\pgfpathlineto{\pgfqpoint{1.048536in}{1.085469in}}%
\pgfpathlineto{\pgfqpoint{1.039641in}{1.099156in}}%
\pgfpathlineto{\pgfqpoint{1.030434in}{1.111364in}}%
\pgfpathlineto{\pgfqpoint{1.020950in}{1.122043in}}%
\pgfpathlineto{\pgfqpoint{1.011227in}{1.131151in}}%
\pgfpathlineto{\pgfqpoint{1.001304in}{1.138653in}}%
\pgfpathlineto{\pgfqpoint{0.991219in}{1.144519in}}%
\pgfpathlineto{\pgfqpoint{0.981013in}{1.148725in}}%
\pgfpathlineto{\pgfqpoint{0.970725in}{1.151256in}}%
\pgfpathclose%
\pgfusepath{stroke,fill}%
\end{pgfscope}%
\begin{pgfscope}%
\pgfpathrectangle{\pgfqpoint{0.039236in}{0.039236in}}{\pgfqpoint{1.595582in}{1.177068in}}%
\pgfusepath{clip}%
\pgfsetbuttcap%
\pgfsetmiterjoin%
\definecolor{currentfill}{rgb}{0.254902,0.411765,0.882353}%
\pgfsetfillcolor{currentfill}%
\pgfsetfillopacity{0.600000}%
\pgfsetlinewidth{1.003750pt}%
\definecolor{currentstroke}{rgb}{0.000000,0.000000,0.000000}%
\pgfsetstrokecolor{currentstroke}%
\pgfsetstrokeopacity{0.600000}%
\pgfsetdash{{1.000000pt}{1.650000pt}}{0.000000pt}%
\pgfpathmoveto{\pgfqpoint{0.960396in}{0.781145in}}%
\pgfpathlineto{\pgfqpoint{0.993295in}{0.695540in}}%
\pgfpathlineto{\pgfqpoint{0.927498in}{0.695540in}}%
\pgfpathclose%
\pgfusepath{stroke,fill}%
\end{pgfscope}%
\begin{pgfscope}%
\pgfpathrectangle{\pgfqpoint{0.039236in}{0.039236in}}{\pgfqpoint{1.595582in}{1.177068in}}%
\pgfusepath{clip}%
\pgfsetbuttcap%
\pgfsetmiterjoin%
\definecolor{currentfill}{rgb}{1.000000,0.000000,1.000000}%
\pgfsetfillcolor{currentfill}%
\pgfsetfillopacity{0.800000}%
\pgfsetlinewidth{1.003750pt}%
\definecolor{currentstroke}{rgb}{0.000000,0.000000,0.000000}%
\pgfsetstrokecolor{currentstroke}%
\pgfsetstrokeopacity{0.800000}%
\pgfsetdash{}{0pt}%
\pgfpathmoveto{\pgfqpoint{1.618368in}{0.296051in}}%
\pgfpathlineto{\pgfqpoint{1.617719in}{0.309488in}}%
\pgfpathlineto{\pgfqpoint{1.615774in}{0.322873in}}%
\pgfpathlineto{\pgfqpoint{1.612541in}{0.336152in}}%
\pgfpathlineto{\pgfqpoint{1.608032in}{0.349273in}}%
\pgfpathlineto{\pgfqpoint{1.602266in}{0.362184in}}%
\pgfpathlineto{\pgfqpoint{1.595265in}{0.374834in}}%
\pgfpathlineto{\pgfqpoint{1.587057in}{0.387173in}}%
\pgfpathlineto{\pgfqpoint{1.577675in}{0.399152in}}%
\pgfpathlineto{\pgfqpoint{1.567154in}{0.410724in}}%
\pgfpathlineto{\pgfqpoint{1.555537in}{0.421844in}}%
\pgfpathlineto{\pgfqpoint{1.542870in}{0.432467in}}%
\pgfpathlineto{\pgfqpoint{1.529202in}{0.442552in}}%
\pgfpathlineto{\pgfqpoint{1.514588in}{0.452059in}}%
\pgfpathlineto{\pgfqpoint{1.499086in}{0.460950in}}%
\pgfpathlineto{\pgfqpoint{1.482755in}{0.469190in}}%
\pgfpathlineto{\pgfqpoint{1.465661in}{0.476747in}}%
\pgfpathlineto{\pgfqpoint{1.447872in}{0.483591in}}%
\pgfpathlineto{\pgfqpoint{1.429457in}{0.489695in}}%
\pgfpathlineto{\pgfqpoint{1.410490in}{0.495034in}}%
\pgfpathlineto{\pgfqpoint{1.391044in}{0.499588in}}%
\pgfpathlineto{\pgfqpoint{1.371197in}{0.503339in}}%
\pgfpathlineto{\pgfqpoint{1.351028in}{0.506272in}}%
\pgfpathlineto{\pgfqpoint{1.330615in}{0.508375in}}%
\pgfpathlineto{\pgfqpoint{1.310039in}{0.509641in}}%
\pgfpathlineto{\pgfqpoint{1.289382in}{0.510063in}}%
\pgfpathlineto{\pgfqpoint{1.268725in}{0.509641in}}%
\pgfpathlineto{\pgfqpoint{1.248149in}{0.508375in}}%
\pgfpathlineto{\pgfqpoint{1.227736in}{0.506272in}}%
\pgfpathlineto{\pgfqpoint{1.207567in}{0.503339in}}%
\pgfpathlineto{\pgfqpoint{1.187720in}{0.499588in}}%
\pgfpathlineto{\pgfqpoint{1.168274in}{0.495034in}}%
\pgfpathlineto{\pgfqpoint{1.149307in}{0.489695in}}%
\pgfpathlineto{\pgfqpoint{1.130892in}{0.483591in}}%
\pgfpathlineto{\pgfqpoint{1.113103in}{0.476747in}}%
\pgfpathlineto{\pgfqpoint{1.096009in}{0.469190in}}%
\pgfpathlineto{\pgfqpoint{1.079679in}{0.460950in}}%
\pgfpathlineto{\pgfqpoint{1.064176in}{0.452059in}}%
\pgfpathlineto{\pgfqpoint{1.049562in}{0.442552in}}%
\pgfpathlineto{\pgfqpoint{1.035894in}{0.432467in}}%
\pgfpathlineto{\pgfqpoint{1.023227in}{0.421844in}}%
\pgfpathlineto{\pgfqpoint{1.011610in}{0.410724in}}%
\pgfpathlineto{\pgfqpoint{1.001089in}{0.399152in}}%
\pgfpathlineto{\pgfqpoint{0.991707in}{0.387173in}}%
\pgfpathlineto{\pgfqpoint{0.983499in}{0.374834in}}%
\pgfpathlineto{\pgfqpoint{0.976498in}{0.362184in}}%
\pgfpathlineto{\pgfqpoint{0.970732in}{0.349273in}}%
\pgfpathlineto{\pgfqpoint{0.966223in}{0.336152in}}%
\pgfpathlineto{\pgfqpoint{0.962990in}{0.322873in}}%
\pgfpathlineto{\pgfqpoint{0.961045in}{0.309488in}}%
\pgfpathlineto{\pgfqpoint{0.960396in}{0.296051in}}%
\pgfpathlineto{\pgfqpoint{0.961045in}{0.282613in}}%
\pgfpathlineto{\pgfqpoint{0.962990in}{0.269228in}}%
\pgfpathlineto{\pgfqpoint{0.966223in}{0.255949in}}%
\pgfpathlineto{\pgfqpoint{0.970732in}{0.242828in}}%
\pgfpathlineto{\pgfqpoint{0.976498in}{0.229917in}}%
\pgfpathlineto{\pgfqpoint{0.983499in}{0.217267in}}%
\pgfpathlineto{\pgfqpoint{0.991707in}{0.204928in}}%
\pgfpathlineto{\pgfqpoint{1.001089in}{0.192949in}}%
\pgfpathlineto{\pgfqpoint{1.011610in}{0.181377in}}%
\pgfpathlineto{\pgfqpoint{1.023227in}{0.170257in}}%
\pgfpathlineto{\pgfqpoint{1.035894in}{0.159634in}}%
\pgfpathlineto{\pgfqpoint{1.049562in}{0.149549in}}%
\pgfpathlineto{\pgfqpoint{1.064176in}{0.140042in}}%
\pgfpathlineto{\pgfqpoint{1.079679in}{0.131151in}}%
\pgfpathlineto{\pgfqpoint{1.096009in}{0.122911in}}%
\pgfpathlineto{\pgfqpoint{1.113103in}{0.115354in}}%
\pgfpathlineto{\pgfqpoint{1.130892in}{0.108510in}}%
\pgfpathlineto{\pgfqpoint{1.149307in}{0.102406in}}%
\pgfpathlineto{\pgfqpoint{1.168274in}{0.097067in}}%
\pgfpathlineto{\pgfqpoint{1.187720in}{0.092513in}}%
\pgfpathlineto{\pgfqpoint{1.207567in}{0.088762in}}%
\pgfpathlineto{\pgfqpoint{1.227736in}{0.085829in}}%
\pgfpathlineto{\pgfqpoint{1.248149in}{0.083726in}}%
\pgfpathlineto{\pgfqpoint{1.268725in}{0.082460in}}%
\pgfpathlineto{\pgfqpoint{1.289382in}{0.082038in}}%
\pgfpathlineto{\pgfqpoint{1.310039in}{0.082460in}}%
\pgfpathlineto{\pgfqpoint{1.330615in}{0.083726in}}%
\pgfpathlineto{\pgfqpoint{1.351028in}{0.085829in}}%
\pgfpathlineto{\pgfqpoint{1.371197in}{0.088762in}}%
\pgfpathlineto{\pgfqpoint{1.391044in}{0.092513in}}%
\pgfpathlineto{\pgfqpoint{1.410490in}{0.097067in}}%
\pgfpathlineto{\pgfqpoint{1.429457in}{0.102406in}}%
\pgfpathlineto{\pgfqpoint{1.447872in}{0.108510in}}%
\pgfpathlineto{\pgfqpoint{1.465661in}{0.115354in}}%
\pgfpathlineto{\pgfqpoint{1.482755in}{0.122911in}}%
\pgfpathlineto{\pgfqpoint{1.499086in}{0.131151in}}%
\pgfpathlineto{\pgfqpoint{1.514588in}{0.140042in}}%
\pgfpathlineto{\pgfqpoint{1.529202in}{0.149549in}}%
\pgfpathlineto{\pgfqpoint{1.542870in}{0.159634in}}%
\pgfpathlineto{\pgfqpoint{1.555537in}{0.170257in}}%
\pgfpathlineto{\pgfqpoint{1.567154in}{0.181377in}}%
\pgfpathlineto{\pgfqpoint{1.577675in}{0.192949in}}%
\pgfpathlineto{\pgfqpoint{1.587057in}{0.204928in}}%
\pgfpathlineto{\pgfqpoint{1.595265in}{0.217267in}}%
\pgfpathlineto{\pgfqpoint{1.602266in}{0.229917in}}%
\pgfpathlineto{\pgfqpoint{1.608032in}{0.242828in}}%
\pgfpathlineto{\pgfqpoint{1.612541in}{0.255949in}}%
\pgfpathlineto{\pgfqpoint{1.615774in}{0.269228in}}%
\pgfpathlineto{\pgfqpoint{1.617719in}{0.282613in}}%
\pgfpathclose%
\pgfusepath{stroke,fill}%
\end{pgfscope}%
\begin{pgfscope}%
\pgfpathrectangle{\pgfqpoint{0.039236in}{0.039236in}}{\pgfqpoint{1.595582in}{1.177068in}}%
\pgfusepath{clip}%
\pgfsetbuttcap%
\pgfsetmiterjoin%
\definecolor{currentfill}{rgb}{0.254902,0.411765,0.882353}%
\pgfsetfillcolor{currentfill}%
\pgfsetfillopacity{0.600000}%
\pgfsetlinewidth{1.003750pt}%
\definecolor{currentstroke}{rgb}{0.000000,0.000000,0.000000}%
\pgfsetstrokecolor{currentstroke}%
\pgfsetstrokeopacity{0.600000}%
\pgfsetdash{{1.000000pt}{1.650000pt}}{0.000000pt}%
\pgfpathmoveto{\pgfqpoint{1.289382in}{0.353121in}}%
\pgfpathlineto{\pgfqpoint{1.322281in}{0.267516in}}%
\pgfpathlineto{\pgfqpoint{1.256483in}{0.267516in}}%
\pgfpathclose%
\pgfusepath{stroke,fill}%
\end{pgfscope}%
\begin{pgfscope}%
\pgfpathrectangle{\pgfqpoint{0.039236in}{0.039236in}}{\pgfqpoint{1.595582in}{1.177068in}}%
\pgfusepath{clip}%
\pgfsetbuttcap%
\pgfsetmiterjoin%
\definecolor{currentfill}{rgb}{1.000000,0.388235,0.278431}%
\pgfsetfillcolor{currentfill}%
\pgfsetfillopacity{0.800000}%
\pgfsetlinewidth{1.003750pt}%
\definecolor{currentstroke}{rgb}{0.000000,0.000000,0.000000}%
\pgfsetstrokecolor{currentstroke}%
\pgfsetstrokeopacity{0.800000}%
\pgfsetdash{}{0pt}%
\pgfpathmoveto{\pgfqpoint{0.466917in}{0.360254in}}%
\pgfpathlineto{\pgfqpoint{0.302424in}{0.372144in}}%
\pgfpathlineto{\pgfqpoint{0.302424in}{1.002292in}}%
\pgfpathlineto{\pgfqpoint{0.055685in}{1.002292in}}%
\pgfpathlineto{\pgfqpoint{0.055685in}{0.895285in}}%
\pgfpathlineto{\pgfqpoint{0.220178in}{0.344041in}}%
\pgfpathlineto{\pgfqpoint{0.220178in}{0.253248in}}%
\pgfpathlineto{\pgfqpoint{0.466917in}{0.253248in}}%
\pgfpathclose%
\pgfusepath{stroke,fill}%
\end{pgfscope}%
\begin{pgfscope}%
\pgfpathrectangle{\pgfqpoint{0.039236in}{0.039236in}}{\pgfqpoint{1.595582in}{1.177068in}}%
\pgfusepath{clip}%
\pgfsetbuttcap%
\pgfsetmiterjoin%
\definecolor{currentfill}{rgb}{0.254902,0.411765,0.882353}%
\pgfsetfillcolor{currentfill}%
\pgfsetfillopacity{0.600000}%
\pgfsetlinewidth{1.003750pt}%
\definecolor{currentstroke}{rgb}{0.000000,0.000000,0.000000}%
\pgfsetstrokecolor{currentstroke}%
\pgfsetstrokeopacity{0.600000}%
\pgfsetdash{{1.000000pt}{1.650000pt}}{0.000000pt}%
\pgfpathmoveto{\pgfqpoint{0.269526in}{0.374522in}}%
\pgfpathlineto{\pgfqpoint{0.302424in}{0.288917in}}%
\pgfpathlineto{\pgfqpoint{0.236627in}{0.288917in}}%
\pgfpathclose%
\pgfusepath{stroke,fill}%
\end{pgfscope}%
\begin{pgfscope}%
\pgfpathrectangle{\pgfqpoint{0.039236in}{0.039236in}}{\pgfqpoint{1.595582in}{1.177068in}}%
\pgfusepath{clip}%
\pgfsetbuttcap%
\pgfsetmiterjoin%
\definecolor{currentfill}{rgb}{1.000000,0.000000,0.000000}%
\pgfsetfillcolor{currentfill}%
\pgfsetlinewidth{1.003750pt}%
\definecolor{currentstroke}{rgb}{1.000000,0.000000,0.000000}%
\pgfsetstrokecolor{currentstroke}%
\pgfsetdash{}{0pt}%
\pgfsys@defobject{currentmarker}{\pgfqpoint{-0.016667in}{-0.016667in}}{\pgfqpoint{0.016667in}{0.016667in}}{%
\pgfpathmoveto{\pgfqpoint{0.000000in}{0.016667in}}%
\pgfpathlineto{\pgfqpoint{-0.016667in}{-0.016667in}}%
\pgfpathlineto{\pgfqpoint{0.016667in}{-0.016667in}}%
\pgfpathclose%
\pgfusepath{stroke,fill}%
}%
\begin{pgfscope}%
\pgfsys@transformshift{0.351772in}{0.681273in}%
\pgfsys@useobject{currentmarker}{}%
\end{pgfscope}%
\begin{pgfscope}%
\pgfsys@transformshift{0.631410in}{0.274649in}%
\pgfsys@useobject{currentmarker}{}%
\end{pgfscope}%
\end{pgfscope}%
\end{pgfpicture}%
\makeatother%
\endgroup%

%% file: figures/step2_admker0_ex_obs.pgf
\begingroup%
\makeatletter%
\begin{pgfpicture}%
\pgfpathrectangle{\pgfpointorigin}{\pgfqpoint{1.674053in}{1.255540in}}%
\pgfusepath{use as bounding box, clip}%
\begin{pgfscope}%
\pgfsetbuttcap%
\pgfsetmiterjoin%
\definecolor{currentfill}{rgb}{1.000000,1.000000,1.000000}%
\pgfsetfillcolor{currentfill}%
\pgfsetlinewidth{0.000000pt}%
\definecolor{currentstroke}{rgb}{1.000000,1.000000,1.000000}%
\pgfsetstrokecolor{currentstroke}%
\pgfsetdash{}{0pt}%
\pgfpathmoveto{\pgfqpoint{0.000000in}{0.000000in}}%
\pgfpathlineto{\pgfqpoint{1.674053in}{0.000000in}}%
\pgfpathlineto{\pgfqpoint{1.674053in}{1.255540in}}%
\pgfpathlineto{\pgfqpoint{0.000000in}{1.255540in}}%
\pgfpathclose%
\pgfusepath{fill}%
\end{pgfscope}%
\begin{pgfscope}%
\pgfpathrectangle{\pgfqpoint{0.039236in}{0.039236in}}{\pgfqpoint{1.595582in}{1.177068in}}%
\pgfusepath{clip}%
\pgfsetbuttcap%
\pgfsetmiterjoin%
\definecolor{currentfill}{rgb}{0.827451,0.827451,0.827451}%
\pgfsetfillcolor{currentfill}%
\pgfsetfillopacity{0.800000}%
\pgfsetlinewidth{1.003750pt}%
\definecolor{currentstroke}{rgb}{0.000000,0.000000,0.000000}%
\pgfsetstrokecolor{currentstroke}%
\pgfsetstrokeopacity{0.800000}%
\pgfsetdash{}{0pt}%
\pgfpathmoveto{\pgfqpoint{0.747724in}{0.572746in}}%
\pgfpathcurveto{\pgfqpoint{0.809418in}{0.653012in}}{\pgfqpoint{0.856338in}{0.745946in}}{\pgfqpoint{0.878150in}{0.831082in}}%
\pgfpathcurveto{\pgfqpoint{0.899962in}{0.916217in}}{\pgfqpoint{0.894885in}{0.986602in}}{\pgfqpoint{0.864038in}{1.026735in}}%
\pgfpathcurveto{\pgfqpoint{0.833192in}{1.066868in}}{\pgfqpoint{0.779093in}{1.073472in}}{\pgfqpoint{0.713657in}{1.045094in}}%
\pgfpathcurveto{\pgfqpoint{0.648221in}{1.016716in}}{\pgfqpoint{0.576790in}{0.955671in}}{\pgfqpoint{0.515096in}{0.875405in}}%
\pgfpathcurveto{\pgfqpoint{0.453402in}{0.795139in}}{\pgfqpoint{0.406483in}{0.702204in}}{\pgfqpoint{0.384671in}{0.617069in}}%
\pgfpathcurveto{\pgfqpoint{0.362859in}{0.531934in}}{\pgfqpoint{0.367935in}{0.461549in}}{\pgfqpoint{0.398782in}{0.421416in}}%
\pgfpathcurveto{\pgfqpoint{0.429629in}{0.381283in}}{\pgfqpoint{0.483728in}{0.374678in}}{\pgfqpoint{0.549164in}{0.403057in}}%
\pgfpathcurveto{\pgfqpoint{0.614600in}{0.431435in}}{\pgfqpoint{0.686031in}{0.492480in}}{\pgfqpoint{0.747724in}{0.572746in}}%
\pgfpathclose%
\pgfusepath{stroke,fill}%
\end{pgfscope}%
\begin{pgfscope}%
\pgfpathrectangle{\pgfqpoint{0.039236in}{0.039236in}}{\pgfqpoint{1.595582in}{1.177068in}}%
\pgfusepath{clip}%
\pgfsetbuttcap%
\pgfsetmiterjoin%
\definecolor{currentfill}{rgb}{0.827451,0.827451,0.827451}%
\pgfsetfillcolor{currentfill}%
\pgfsetfillopacity{0.800000}%
\pgfsetlinewidth{1.003750pt}%
\definecolor{currentstroke}{rgb}{0.000000,0.000000,0.000000}%
\pgfsetstrokecolor{currentstroke}%
\pgfsetstrokeopacity{0.800000}%
\pgfsetdash{}{0pt}%
\pgfpathmoveto{\pgfqpoint{1.124889in}{0.724075in}}%
\pgfpathcurveto{\pgfqpoint{1.124889in}{0.837589in}}{\pgfqpoint{1.107557in}{0.946468in}}{\pgfqpoint{1.076710in}{1.026735in}}%
\pgfpathcurveto{\pgfqpoint{1.045863in}{1.107001in}}{\pgfqpoint{1.004020in}{1.152100in}}{\pgfqpoint{0.960396in}{1.152100in}}%
\pgfpathcurveto{\pgfqpoint{0.916772in}{1.152100in}}{\pgfqpoint{0.874929in}{1.107001in}}{\pgfqpoint{0.844082in}{1.026735in}}%
\pgfpathcurveto{\pgfqpoint{0.813235in}{0.946468in}}{\pgfqpoint{0.795903in}{0.837589in}}{\pgfqpoint{0.795903in}{0.724075in}}%
\pgfpathcurveto{\pgfqpoint{0.795903in}{0.610562in}}{\pgfqpoint{0.813235in}{0.501682in}}{\pgfqpoint{0.844082in}{0.421416in}}%
\pgfpathcurveto{\pgfqpoint{0.874929in}{0.341150in}}{\pgfqpoint{0.916772in}{0.296051in}}{\pgfqpoint{0.960396in}{0.296051in}}%
\pgfpathcurveto{\pgfqpoint{1.004020in}{0.296051in}}{\pgfqpoint{1.045863in}{0.341150in}}{\pgfqpoint{1.076710in}{0.421416in}}%
\pgfpathcurveto{\pgfqpoint{1.107557in}{0.501682in}}{\pgfqpoint{1.124889in}{0.610562in}}{\pgfqpoint{1.124889in}{0.724075in}}%
\pgfpathclose%
\pgfusepath{stroke,fill}%
\end{pgfscope}%
\begin{pgfscope}%
\pgfpathrectangle{\pgfqpoint{0.039236in}{0.039236in}}{\pgfqpoint{1.595582in}{1.177068in}}%
\pgfusepath{clip}%
\pgfsetbuttcap%
\pgfsetmiterjoin%
\definecolor{currentfill}{rgb}{0.827451,0.827451,0.827451}%
\pgfsetfillcolor{currentfill}%
\pgfsetfillopacity{0.800000}%
\pgfsetlinewidth{1.003750pt}%
\definecolor{currentstroke}{rgb}{0.000000,0.000000,0.000000}%
\pgfsetstrokecolor{currentstroke}%
\pgfsetstrokeopacity{0.800000}%
\pgfsetdash{}{0pt}%
\pgfpathmoveto{\pgfqpoint{1.289382in}{0.082038in}}%
\pgfpathcurveto{\pgfqpoint{1.376630in}{0.082038in}}{\pgfqpoint{1.460316in}{0.104588in}}{\pgfqpoint{1.522010in}{0.144721in}}%
\pgfpathcurveto{\pgfqpoint{1.583704in}{0.184854in}}{\pgfqpoint{1.618368in}{0.239294in}}{\pgfqpoint{1.618368in}{0.296051in}}%
\pgfpathcurveto{\pgfqpoint{1.618368in}{0.352807in}}{\pgfqpoint{1.583704in}{0.407247in}}{\pgfqpoint{1.522010in}{0.447380in}}%
\pgfpathcurveto{\pgfqpoint{1.460316in}{0.487513in}}{\pgfqpoint{1.376630in}{0.510063in}}{\pgfqpoint{1.289382in}{0.510063in}}%
\pgfpathcurveto{\pgfqpoint{1.202134in}{0.510063in}}{\pgfqpoint{1.118448in}{0.487513in}}{\pgfqpoint{1.056754in}{0.447380in}}%
\pgfpathcurveto{\pgfqpoint{0.995060in}{0.407247in}}{\pgfqpoint{0.960396in}{0.352807in}}{\pgfqpoint{0.960396in}{0.296051in}}%
\pgfpathcurveto{\pgfqpoint{0.960396in}{0.239294in}}{\pgfqpoint{0.995060in}{0.184854in}}{\pgfqpoint{1.056754in}{0.144721in}}%
\pgfpathcurveto{\pgfqpoint{1.118448in}{0.104588in}}{\pgfqpoint{1.202134in}{0.082038in}}{\pgfqpoint{1.289382in}{0.082038in}}%
\pgfpathclose%
\pgfusepath{stroke,fill}%
\end{pgfscope}%
\begin{pgfscope}%
\pgfpathrectangle{\pgfqpoint{0.039236in}{0.039236in}}{\pgfqpoint{1.595582in}{1.177068in}}%
\pgfusepath{clip}%
\pgfsetbuttcap%
\pgfsetmiterjoin%
\definecolor{currentfill}{rgb}{0.627451,0.321569,0.176471}%
\pgfsetfillcolor{currentfill}%
\pgfsetfillopacity{0.800000}%
\pgfsetlinewidth{1.003750pt}%
\definecolor{currentstroke}{rgb}{0.000000,0.000000,0.000000}%
\pgfsetstrokecolor{currentstroke}%
\pgfsetstrokeopacity{0.800000}%
\pgfsetdash{}{0pt}%
\pgfpathmoveto{\pgfqpoint{0.055685in}{1.002292in}}%
\pgfpathlineto{\pgfqpoint{0.302424in}{1.002292in}}%
\pgfpathlineto{\pgfqpoint{0.302424in}{0.360254in}}%
\pgfpathlineto{\pgfqpoint{0.466917in}{0.360254in}}%
\pgfpathlineto{\pgfqpoint{0.466917in}{0.253248in}}%
\pgfpathlineto{\pgfqpoint{0.220178in}{0.253248in}}%
\pgfpathlineto{\pgfqpoint{0.220178in}{0.895285in}}%
\pgfpathlineto{\pgfqpoint{0.055685in}{0.895285in}}%
\pgfpathclose%
\pgfusepath{stroke,fill}%
\end{pgfscope}%
\begin{pgfscope}%
\pgfpathrectangle{\pgfqpoint{0.039236in}{0.039236in}}{\pgfqpoint{1.595582in}{1.177068in}}%
\pgfusepath{clip}%
\pgfsetbuttcap%
\pgfsetmiterjoin%
\definecolor{currentfill}{rgb}{1.000000,1.000000,0.000000}%
\pgfsetfillcolor{currentfill}%
\pgfsetfillopacity{0.600000}%
\pgfsetlinewidth{1.003750pt}%
\definecolor{currentstroke}{rgb}{0.000000,0.000000,0.000000}%
\pgfsetstrokecolor{currentstroke}%
\pgfsetstrokeopacity{0.600000}%
\pgfsetdash{{3.700000pt}{1.600000pt}}{0.000000pt}%
\pgfpathmoveto{\pgfqpoint{0.175600in}{0.470741in}}%
\pgfpathlineto{\pgfqpoint{0.351772in}{0.681273in}}%
\pgfpathlineto{\pgfqpoint{0.267936in}{1.430316in}}%
\pgfpathlineto{\pgfqpoint{1.799310in}{1.430316in}}%
\pgfpathlineto{\pgfqpoint{1.799310in}{-0.174777in}}%
\pgfpathlineto{\pgfqpoint{1.382048in}{-0.174777in}}%
\pgfpathlineto{\pgfqpoint{0.631410in}{0.274649in}}%
\pgfpathlineto{\pgfqpoint{0.525540in}{0.309887in}}%
\pgfpathlineto{\pgfqpoint{0.466917in}{0.360254in}}%
\pgfpathclose%
\pgfusepath{stroke,fill}%
\end{pgfscope}%
\begin{pgfscope}%
\pgfpathrectangle{\pgfqpoint{0.039236in}{0.039236in}}{\pgfqpoint{1.595582in}{1.177068in}}%
\pgfusepath{clip}%
\pgfsetbuttcap%
\pgfsetmiterjoin%
\definecolor{currentfill}{rgb}{1.000000,0.000000,0.000000}%
\pgfsetfillcolor{currentfill}%
\pgfsetlinewidth{1.003750pt}%
\definecolor{currentstroke}{rgb}{1.000000,0.000000,0.000000}%
\pgfsetstrokecolor{currentstroke}%
\pgfsetdash{}{0pt}%
\pgfsys@defobject{currentmarker}{\pgfqpoint{-0.016667in}{-0.016667in}}{\pgfqpoint{0.016667in}{0.016667in}}{%
\pgfpathmoveto{\pgfqpoint{0.000000in}{0.016667in}}%
\pgfpathlineto{\pgfqpoint{-0.016667in}{-0.016667in}}%
\pgfpathlineto{\pgfqpoint{0.016667in}{-0.016667in}}%
\pgfpathclose%
\pgfusepath{stroke,fill}%
}%
\begin{pgfscope}%
\pgfsys@transformshift{0.351772in}{0.681273in}%
\pgfsys@useobject{currentmarker}{}%
\end{pgfscope}%
\begin{pgfscope}%
\pgfsys@transformshift{0.631410in}{0.274649in}%
\pgfsys@useobject{currentmarker}{}%
\end{pgfscope}%
\end{pgfscope}%
\begin{pgfscope}%
\pgfpathrectangle{\pgfqpoint{0.039236in}{0.039236in}}{\pgfqpoint{1.595582in}{1.177068in}}%
\pgfusepath{clip}%
\pgfsetbuttcap%
\pgfsetmiterjoin%
\definecolor{currentfill}{rgb}{1.000000,0.000000,0.000000}%
\pgfsetfillcolor{currentfill}%
\pgfsetlinewidth{1.003750pt}%
\definecolor{currentstroke}{rgb}{1.000000,0.000000,0.000000}%
\pgfsetstrokecolor{currentstroke}%
\pgfsetdash{}{0pt}%
\pgfsys@defobject{currentmarker}{\pgfqpoint{-0.016667in}{-0.016667in}}{\pgfqpoint{0.016667in}{0.016667in}}{%
\pgfpathmoveto{\pgfqpoint{0.000000in}{0.016667in}}%
\pgfpathlineto{\pgfqpoint{-0.016667in}{-0.016667in}}%
\pgfpathlineto{\pgfqpoint{0.016667in}{-0.016667in}}%
\pgfpathclose%
\pgfusepath{stroke,fill}%
}%
\begin{pgfscope}%
\pgfsys@transformshift{0.055685in}{1.002292in}%
\pgfsys@useobject{currentmarker}{}%
\end{pgfscope}%
\end{pgfscope}%
\begin{pgfscope}%
\pgfpathrectangle{\pgfqpoint{0.039236in}{0.039236in}}{\pgfqpoint{1.595582in}{1.177068in}}%
\pgfusepath{clip}%
\pgfsetbuttcap%
\pgfsetmiterjoin%
\definecolor{currentfill}{rgb}{1.000000,0.000000,0.000000}%
\pgfsetfillcolor{currentfill}%
\pgfsetlinewidth{1.003750pt}%
\definecolor{currentstroke}{rgb}{1.000000,0.000000,0.000000}%
\pgfsetstrokecolor{currentstroke}%
\pgfsetdash{}{0pt}%
\pgfsys@defobject{currentmarker}{\pgfqpoint{-0.016667in}{-0.016667in}}{\pgfqpoint{0.016667in}{0.016667in}}{%
\pgfpathmoveto{\pgfqpoint{0.000000in}{0.016667in}}%
\pgfpathlineto{\pgfqpoint{-0.016667in}{-0.016667in}}%
\pgfpathlineto{\pgfqpoint{0.016667in}{-0.016667in}}%
\pgfpathclose%
\pgfusepath{stroke,fill}%
}%
\begin{pgfscope}%
\pgfsys@transformshift{0.302424in}{1.002292in}%
\pgfsys@useobject{currentmarker}{}%
\end{pgfscope}%
\end{pgfscope}%
\begin{pgfscope}%
\pgfpathrectangle{\pgfqpoint{0.039236in}{0.039236in}}{\pgfqpoint{1.595582in}{1.177068in}}%
\pgfusepath{clip}%
\pgfsetbuttcap%
\pgfsetmiterjoin%
\definecolor{currentfill}{rgb}{1.000000,0.000000,0.000000}%
\pgfsetfillcolor{currentfill}%
\pgfsetlinewidth{1.003750pt}%
\definecolor{currentstroke}{rgb}{1.000000,0.000000,0.000000}%
\pgfsetstrokecolor{currentstroke}%
\pgfsetdash{}{0pt}%
\pgfsys@defobject{currentmarker}{\pgfqpoint{-0.016667in}{-0.016667in}}{\pgfqpoint{0.016667in}{0.016667in}}{%
\pgfpathmoveto{\pgfqpoint{0.000000in}{0.016667in}}%
\pgfpathlineto{\pgfqpoint{-0.016667in}{-0.016667in}}%
\pgfpathlineto{\pgfqpoint{0.016667in}{-0.016667in}}%
\pgfpathclose%
\pgfusepath{stroke,fill}%
}%
\begin{pgfscope}%
\pgfsys@transformshift{0.302424in}{0.360254in}%
\pgfsys@useobject{currentmarker}{}%
\end{pgfscope}%
\end{pgfscope}%
\begin{pgfscope}%
\pgfpathrectangle{\pgfqpoint{0.039236in}{0.039236in}}{\pgfqpoint{1.595582in}{1.177068in}}%
\pgfusepath{clip}%
\pgfsetbuttcap%
\pgfsetmiterjoin%
\definecolor{currentfill}{rgb}{1.000000,0.000000,0.000000}%
\pgfsetfillcolor{currentfill}%
\pgfsetlinewidth{1.003750pt}%
\definecolor{currentstroke}{rgb}{1.000000,0.000000,0.000000}%
\pgfsetstrokecolor{currentstroke}%
\pgfsetdash{}{0pt}%
\pgfsys@defobject{currentmarker}{\pgfqpoint{-0.016667in}{-0.016667in}}{\pgfqpoint{0.016667in}{0.016667in}}{%
\pgfpathmoveto{\pgfqpoint{0.000000in}{0.016667in}}%
\pgfpathlineto{\pgfqpoint{-0.016667in}{-0.016667in}}%
\pgfpathlineto{\pgfqpoint{0.016667in}{-0.016667in}}%
\pgfpathclose%
\pgfusepath{stroke,fill}%
}%
\begin{pgfscope}%
\pgfsys@transformshift{0.466917in}{0.360254in}%
\pgfsys@useobject{currentmarker}{}%
\end{pgfscope}%
\end{pgfscope}%
\begin{pgfscope}%
\pgfpathrectangle{\pgfqpoint{0.039236in}{0.039236in}}{\pgfqpoint{1.595582in}{1.177068in}}%
\pgfusepath{clip}%
\pgfsetbuttcap%
\pgfsetmiterjoin%
\definecolor{currentfill}{rgb}{1.000000,0.000000,0.000000}%
\pgfsetfillcolor{currentfill}%
\pgfsetlinewidth{1.003750pt}%
\definecolor{currentstroke}{rgb}{1.000000,0.000000,0.000000}%
\pgfsetstrokecolor{currentstroke}%
\pgfsetdash{}{0pt}%
\pgfsys@defobject{currentmarker}{\pgfqpoint{-0.016667in}{-0.016667in}}{\pgfqpoint{0.016667in}{0.016667in}}{%
\pgfpathmoveto{\pgfqpoint{0.000000in}{0.016667in}}%
\pgfpathlineto{\pgfqpoint{-0.016667in}{-0.016667in}}%
\pgfpathlineto{\pgfqpoint{0.016667in}{-0.016667in}}%
\pgfpathclose%
\pgfusepath{stroke,fill}%
}%
\begin{pgfscope}%
\pgfsys@transformshift{0.466917in}{0.253248in}%
\pgfsys@useobject{currentmarker}{}%
\end{pgfscope}%
\end{pgfscope}%
\begin{pgfscope}%
\pgfpathrectangle{\pgfqpoint{0.039236in}{0.039236in}}{\pgfqpoint{1.595582in}{1.177068in}}%
\pgfusepath{clip}%
\pgfsetbuttcap%
\pgfsetmiterjoin%
\definecolor{currentfill}{rgb}{1.000000,0.000000,0.000000}%
\pgfsetfillcolor{currentfill}%
\pgfsetlinewidth{1.003750pt}%
\definecolor{currentstroke}{rgb}{1.000000,0.000000,0.000000}%
\pgfsetstrokecolor{currentstroke}%
\pgfsetdash{}{0pt}%
\pgfsys@defobject{currentmarker}{\pgfqpoint{-0.016667in}{-0.016667in}}{\pgfqpoint{0.016667in}{0.016667in}}{%
\pgfpathmoveto{\pgfqpoint{0.000000in}{0.016667in}}%
\pgfpathlineto{\pgfqpoint{-0.016667in}{-0.016667in}}%
\pgfpathlineto{\pgfqpoint{0.016667in}{-0.016667in}}%
\pgfpathclose%
\pgfusepath{stroke,fill}%
}%
\begin{pgfscope}%
\pgfsys@transformshift{0.220178in}{0.253248in}%
\pgfsys@useobject{currentmarker}{}%
\end{pgfscope}%
\end{pgfscope}%
\begin{pgfscope}%
\pgfpathrectangle{\pgfqpoint{0.039236in}{0.039236in}}{\pgfqpoint{1.595582in}{1.177068in}}%
\pgfusepath{clip}%
\pgfsetbuttcap%
\pgfsetmiterjoin%
\definecolor{currentfill}{rgb}{1.000000,0.000000,0.000000}%
\pgfsetfillcolor{currentfill}%
\pgfsetlinewidth{1.003750pt}%
\definecolor{currentstroke}{rgb}{1.000000,0.000000,0.000000}%
\pgfsetstrokecolor{currentstroke}%
\pgfsetdash{}{0pt}%
\pgfsys@defobject{currentmarker}{\pgfqpoint{-0.016667in}{-0.016667in}}{\pgfqpoint{0.016667in}{0.016667in}}{%
\pgfpathmoveto{\pgfqpoint{0.000000in}{0.016667in}}%
\pgfpathlineto{\pgfqpoint{-0.016667in}{-0.016667in}}%
\pgfpathlineto{\pgfqpoint{0.016667in}{-0.016667in}}%
\pgfpathclose%
\pgfusepath{stroke,fill}%
}%
\begin{pgfscope}%
\pgfsys@transformshift{0.220178in}{0.895285in}%
\pgfsys@useobject{currentmarker}{}%
\end{pgfscope}%
\end{pgfscope}%
\begin{pgfscope}%
\pgfpathrectangle{\pgfqpoint{0.039236in}{0.039236in}}{\pgfqpoint{1.595582in}{1.177068in}}%
\pgfusepath{clip}%
\pgfsetbuttcap%
\pgfsetmiterjoin%
\definecolor{currentfill}{rgb}{1.000000,0.000000,0.000000}%
\pgfsetfillcolor{currentfill}%
\pgfsetlinewidth{1.003750pt}%
\definecolor{currentstroke}{rgb}{1.000000,0.000000,0.000000}%
\pgfsetstrokecolor{currentstroke}%
\pgfsetdash{}{0pt}%
\pgfsys@defobject{currentmarker}{\pgfqpoint{-0.016667in}{-0.016667in}}{\pgfqpoint{0.016667in}{0.016667in}}{%
\pgfpathmoveto{\pgfqpoint{0.000000in}{0.016667in}}%
\pgfpathlineto{\pgfqpoint{-0.016667in}{-0.016667in}}%
\pgfpathlineto{\pgfqpoint{0.016667in}{-0.016667in}}%
\pgfpathclose%
\pgfusepath{stroke,fill}%
}%
\begin{pgfscope}%
\pgfsys@transformshift{0.055685in}{0.895285in}%
\pgfsys@useobject{currentmarker}{}%
\end{pgfscope}%
\end{pgfscope}%
\end{pgfpicture}%
\makeatother%
\endgroup%

%% file: figures/step2_star_ex_obs.pgf
\begingroup%
\makeatletter%
\begin{pgfpicture}%
\pgfpathrectangle{\pgfpointorigin}{\pgfqpoint{1.674053in}{1.255540in}}%
\pgfusepath{use as bounding box, clip}%
\begin{pgfscope}%
\pgfsetbuttcap%
\pgfsetmiterjoin%
\definecolor{currentfill}{rgb}{1.000000,1.000000,1.000000}%
\pgfsetfillcolor{currentfill}%
\pgfsetlinewidth{0.000000pt}%
\definecolor{currentstroke}{rgb}{1.000000,1.000000,1.000000}%
\pgfsetstrokecolor{currentstroke}%
\pgfsetdash{}{0pt}%
\pgfpathmoveto{\pgfqpoint{0.000000in}{0.000000in}}%
\pgfpathlineto{\pgfqpoint{1.674053in}{0.000000in}}%
\pgfpathlineto{\pgfqpoint{1.674053in}{1.255540in}}%
\pgfpathlineto{\pgfqpoint{0.000000in}{1.255540in}}%
\pgfpathclose%
\pgfusepath{fill}%
\end{pgfscope}%
\begin{pgfscope}%
\pgfpathrectangle{\pgfqpoint{0.039236in}{0.039236in}}{\pgfqpoint{1.595582in}{1.177068in}}%
\pgfusepath{clip}%
\pgfsetbuttcap%
\pgfsetmiterjoin%
\definecolor{currentfill}{rgb}{0.827451,0.827451,0.827451}%
\pgfsetfillcolor{currentfill}%
\pgfsetfillopacity{0.800000}%
\pgfsetlinewidth{1.003750pt}%
\definecolor{currentstroke}{rgb}{0.000000,0.000000,0.000000}%
\pgfsetstrokecolor{currentstroke}%
\pgfsetstrokeopacity{0.800000}%
\pgfsetdash{}{0pt}%
\pgfpathmoveto{\pgfqpoint{1.447872in}{0.483591in}}%
\pgfpathlineto{\pgfqpoint{1.465661in}{0.476747in}}%
\pgfpathlineto{\pgfqpoint{1.482755in}{0.469190in}}%
\pgfpathlineto{\pgfqpoint{1.499086in}{0.460950in}}%
\pgfpathlineto{\pgfqpoint{1.514588in}{0.452059in}}%
\pgfpathlineto{\pgfqpoint{1.529202in}{0.442552in}}%
\pgfpathlineto{\pgfqpoint{1.542870in}{0.432467in}}%
\pgfpathlineto{\pgfqpoint{1.555537in}{0.421844in}}%
\pgfpathlineto{\pgfqpoint{1.567154in}{0.410724in}}%
\pgfpathlineto{\pgfqpoint{1.577675in}{0.399152in}}%
\pgfpathlineto{\pgfqpoint{1.587057in}{0.387173in}}%
\pgfpathlineto{\pgfqpoint{1.595265in}{0.374834in}}%
\pgfpathlineto{\pgfqpoint{1.602266in}{0.362184in}}%
\pgfpathlineto{\pgfqpoint{1.608032in}{0.349273in}}%
\pgfpathlineto{\pgfqpoint{1.612541in}{0.336152in}}%
\pgfpathlineto{\pgfqpoint{1.615774in}{0.322873in}}%
\pgfpathlineto{\pgfqpoint{1.617719in}{0.309488in}}%
\pgfpathlineto{\pgfqpoint{1.618368in}{0.296051in}}%
\pgfpathlineto{\pgfqpoint{1.617719in}{0.282613in}}%
\pgfpathlineto{\pgfqpoint{1.615774in}{0.269228in}}%
\pgfpathlineto{\pgfqpoint{1.612541in}{0.255949in}}%
\pgfpathlineto{\pgfqpoint{1.608032in}{0.242828in}}%
\pgfpathlineto{\pgfqpoint{1.602266in}{0.229917in}}%
\pgfpathlineto{\pgfqpoint{1.595265in}{0.217267in}}%
\pgfpathlineto{\pgfqpoint{1.587057in}{0.204928in}}%
\pgfpathlineto{\pgfqpoint{1.577675in}{0.192949in}}%
\pgfpathlineto{\pgfqpoint{1.567154in}{0.181377in}}%
\pgfpathlineto{\pgfqpoint{1.555537in}{0.170257in}}%
\pgfpathlineto{\pgfqpoint{1.542870in}{0.159634in}}%
\pgfpathlineto{\pgfqpoint{1.529202in}{0.149549in}}%
\pgfpathlineto{\pgfqpoint{1.514588in}{0.140042in}}%
\pgfpathlineto{\pgfqpoint{1.499086in}{0.131151in}}%
\pgfpathlineto{\pgfqpoint{1.482755in}{0.122911in}}%
\pgfpathlineto{\pgfqpoint{1.465661in}{0.115354in}}%
\pgfpathlineto{\pgfqpoint{1.447872in}{0.108510in}}%
\pgfpathlineto{\pgfqpoint{1.429457in}{0.102406in}}%
\pgfpathlineto{\pgfqpoint{1.410490in}{0.097067in}}%
\pgfpathlineto{\pgfqpoint{1.391044in}{0.092513in}}%
\pgfpathlineto{\pgfqpoint{1.371197in}{0.088762in}}%
\pgfpathlineto{\pgfqpoint{1.351028in}{0.085829in}}%
\pgfpathlineto{\pgfqpoint{1.330615in}{0.083726in}}%
\pgfpathlineto{\pgfqpoint{1.310039in}{0.082460in}}%
\pgfpathlineto{\pgfqpoint{1.289382in}{0.082038in}}%
\pgfpathlineto{\pgfqpoint{1.268725in}{0.082460in}}%
\pgfpathlineto{\pgfqpoint{1.248149in}{0.083726in}}%
\pgfpathlineto{\pgfqpoint{1.227736in}{0.085829in}}%
\pgfpathlineto{\pgfqpoint{1.207567in}{0.088762in}}%
\pgfpathlineto{\pgfqpoint{1.187720in}{0.092513in}}%
\pgfpathlineto{\pgfqpoint{1.168274in}{0.097067in}}%
\pgfpathlineto{\pgfqpoint{1.149307in}{0.102406in}}%
\pgfpathlineto{\pgfqpoint{1.130892in}{0.108510in}}%
\pgfpathlineto{\pgfqpoint{1.113103in}{0.115354in}}%
\pgfpathlineto{\pgfqpoint{1.096009in}{0.122911in}}%
\pgfpathlineto{\pgfqpoint{1.079679in}{0.131151in}}%
\pgfpathlineto{\pgfqpoint{1.064176in}{0.140042in}}%
\pgfpathlineto{\pgfqpoint{1.049562in}{0.149549in}}%
\pgfpathlineto{\pgfqpoint{1.035894in}{0.159634in}}%
\pgfpathlineto{\pgfqpoint{1.023227in}{0.170257in}}%
\pgfpathlineto{\pgfqpoint{1.011610in}{0.181377in}}%
\pgfpathlineto{\pgfqpoint{1.001089in}{0.192949in}}%
\pgfpathlineto{\pgfqpoint{0.991707in}{0.204928in}}%
\pgfpathlineto{\pgfqpoint{0.983499in}{0.217267in}}%
\pgfpathlineto{\pgfqpoint{0.976498in}{0.229917in}}%
\pgfpathlineto{\pgfqpoint{0.970732in}{0.242828in}}%
\pgfpathlineto{\pgfqpoint{0.966223in}{0.255949in}}%
\pgfpathlineto{\pgfqpoint{0.962990in}{0.269228in}}%
\pgfpathlineto{\pgfqpoint{0.961404in}{0.280142in}}%
\pgfpathlineto{\pgfqpoint{0.961311in}{0.280100in}}%
\pgfpathlineto{\pgfqpoint{0.959470in}{0.296126in}}%
\pgfpathlineto{\pgfqpoint{0.950068in}{0.296895in}}%
\pgfpathlineto{\pgfqpoint{0.939780in}{0.299426in}}%
\pgfpathlineto{\pgfqpoint{0.929573in}{0.303632in}}%
\pgfpathlineto{\pgfqpoint{0.919488in}{0.309498in}}%
\pgfpathlineto{\pgfqpoint{0.909565in}{0.317000in}}%
\pgfpathlineto{\pgfqpoint{0.899842in}{0.326108in}}%
\pgfpathlineto{\pgfqpoint{0.890358in}{0.336787in}}%
\pgfpathlineto{\pgfqpoint{0.881151in}{0.348994in}}%
\pgfpathlineto{\pgfqpoint{0.872256in}{0.362682in}}%
\pgfpathlineto{\pgfqpoint{0.863710in}{0.377796in}}%
\pgfpathlineto{\pgfqpoint{0.855544in}{0.394277in}}%
\pgfpathlineto{\pgfqpoint{0.847793in}{0.412059in}}%
\pgfpathlineto{\pgfqpoint{0.840486in}{0.431072in}}%
\pgfpathlineto{\pgfqpoint{0.833652in}{0.451242in}}%
\pgfpathlineto{\pgfqpoint{0.827319in}{0.472489in}}%
\pgfpathlineto{\pgfqpoint{0.822711in}{0.490130in}}%
\pgfpathlineto{\pgfqpoint{0.526716in}{0.394543in}}%
\pgfpathlineto{\pgfqpoint{0.526728in}{0.394562in}}%
\pgfpathlineto{\pgfqpoint{0.526309in}{0.394411in}}%
\pgfpathlineto{\pgfqpoint{0.511578in}{0.390270in}}%
\pgfpathlineto{\pgfqpoint{0.497320in}{0.387445in}}%
\pgfpathlineto{\pgfqpoint{0.483591in}{0.385949in}}%
\pgfpathlineto{\pgfqpoint{0.470446in}{0.385788in}}%
\pgfpathlineto{\pgfqpoint{0.457936in}{0.386962in}}%
\pgfpathlineto{\pgfqpoint{0.446111in}{0.389466in}}%
\pgfpathlineto{\pgfqpoint{0.435017in}{0.393291in}}%
\pgfpathlineto{\pgfqpoint{0.424698in}{0.398421in}}%
\pgfpathlineto{\pgfqpoint{0.415194in}{0.404836in}}%
\pgfpathlineto{\pgfqpoint{0.406545in}{0.412511in}}%
\pgfpathlineto{\pgfqpoint{0.398782in}{0.421416in}}%
\pgfpathlineto{\pgfqpoint{0.391938in}{0.431515in}}%
\pgfpathlineto{\pgfqpoint{0.386038in}{0.442769in}}%
\pgfpathlineto{\pgfqpoint{0.381107in}{0.455133in}}%
\pgfpathlineto{\pgfqpoint{0.377164in}{0.468559in}}%
\pgfpathlineto{\pgfqpoint{0.374225in}{0.482993in}}%
\pgfpathlineto{\pgfqpoint{0.372300in}{0.498378in}}%
\pgfpathlineto{\pgfqpoint{0.371398in}{0.514654in}}%
\pgfpathlineto{\pgfqpoint{0.371522in}{0.531757in}}%
\pgfpathlineto{\pgfqpoint{0.372672in}{0.549618in}}%
\pgfpathlineto{\pgfqpoint{0.374842in}{0.568168in}}%
\pgfpathlineto{\pgfqpoint{0.378026in}{0.587334in}}%
\pgfpathlineto{\pgfqpoint{0.382209in}{0.607039in}}%
\pgfpathlineto{\pgfqpoint{0.387376in}{0.627205in}}%
\pgfpathlineto{\pgfqpoint{0.393506in}{0.647755in}}%
\pgfpathlineto{\pgfqpoint{0.400575in}{0.668605in}}%
\pgfpathlineto{\pgfqpoint{0.408555in}{0.689674in}}%
\pgfpathlineto{\pgfqpoint{0.417414in}{0.710879in}}%
\pgfpathlineto{\pgfqpoint{0.427118in}{0.732136in}}%
\pgfpathlineto{\pgfqpoint{0.437628in}{0.753362in}}%
\pgfpathlineto{\pgfqpoint{0.448903in}{0.774472in}}%
\pgfpathlineto{\pgfqpoint{0.460898in}{0.795382in}}%
\pgfpathlineto{\pgfqpoint{0.473566in}{0.816012in}}%
\pgfpathlineto{\pgfqpoint{0.486857in}{0.836278in}}%
\pgfpathlineto{\pgfqpoint{0.500719in}{0.856102in}}%
\pgfpathlineto{\pgfqpoint{0.515096in}{0.875405in}}%
\pgfpathlineto{\pgfqpoint{0.529933in}{0.894111in}}%
\pgfpathlineto{\pgfqpoint{0.545169in}{0.912145in}}%
\pgfpathlineto{\pgfqpoint{0.560747in}{0.929437in}}%
\pgfpathlineto{\pgfqpoint{0.576603in}{0.945919in}}%
\pgfpathlineto{\pgfqpoint{0.592675in}{0.961525in}}%
\pgfpathlineto{\pgfqpoint{0.608900in}{0.976194in}}%
\pgfpathlineto{\pgfqpoint{0.625214in}{0.989869in}}%
\pgfpathlineto{\pgfqpoint{0.641553in}{1.002494in}}%
\pgfpathlineto{\pgfqpoint{0.657851in}{1.014020in}}%
\pgfpathlineto{\pgfqpoint{0.674046in}{1.024402in}}%
\pgfpathlineto{\pgfqpoint{0.690071in}{1.033599in}}%
\pgfpathlineto{\pgfqpoint{0.705866in}{1.041574in}}%
\pgfpathlineto{\pgfqpoint{0.721366in}{1.048297in}}%
\pgfpathlineto{\pgfqpoint{0.736512in}{1.053740in}}%
\pgfpathlineto{\pgfqpoint{0.751243in}{1.057881in}}%
\pgfpathlineto{\pgfqpoint{0.765500in}{1.060706in}}%
\pgfpathlineto{\pgfqpoint{0.779229in}{1.062201in}}%
\pgfpathlineto{\pgfqpoint{0.792374in}{1.062363in}}%
\pgfpathlineto{\pgfqpoint{0.804884in}{1.061189in}}%
\pgfpathlineto{\pgfqpoint{0.816710in}{1.058685in}}%
\pgfpathlineto{\pgfqpoint{0.827804in}{1.054860in}}%
\pgfpathlineto{\pgfqpoint{0.838123in}{1.049730in}}%
\pgfpathlineto{\pgfqpoint{0.847626in}{1.043315in}}%
\pgfpathlineto{\pgfqpoint{0.850017in}{1.041193in}}%
\pgfpathlineto{\pgfqpoint{0.855544in}{1.053874in}}%
\pgfpathlineto{\pgfqpoint{0.863710in}{1.070355in}}%
\pgfpathlineto{\pgfqpoint{0.872256in}{1.085469in}}%
\pgfpathlineto{\pgfqpoint{0.881151in}{1.099156in}}%
\pgfpathlineto{\pgfqpoint{0.890358in}{1.111364in}}%
\pgfpathlineto{\pgfqpoint{0.899842in}{1.122043in}}%
\pgfpathlineto{\pgfqpoint{0.909565in}{1.131151in}}%
\pgfpathlineto{\pgfqpoint{0.919488in}{1.138653in}}%
\pgfpathlineto{\pgfqpoint{0.929573in}{1.144519in}}%
\pgfpathlineto{\pgfqpoint{0.939780in}{1.148725in}}%
\pgfpathlineto{\pgfqpoint{0.950068in}{1.151256in}}%
\pgfpathlineto{\pgfqpoint{0.960396in}{1.152100in}}%
\pgfpathlineto{\pgfqpoint{0.970725in}{1.151256in}}%
\pgfpathlineto{\pgfqpoint{0.981013in}{1.148725in}}%
\pgfpathlineto{\pgfqpoint{0.991219in}{1.144519in}}%
\pgfpathlineto{\pgfqpoint{1.001304in}{1.138653in}}%
\pgfpathlineto{\pgfqpoint{1.011227in}{1.131151in}}%
\pgfpathlineto{\pgfqpoint{1.020950in}{1.122043in}}%
\pgfpathlineto{\pgfqpoint{1.030434in}{1.111364in}}%
\pgfpathlineto{\pgfqpoint{1.039641in}{1.099156in}}%
\pgfpathlineto{\pgfqpoint{1.048536in}{1.085469in}}%
\pgfpathlineto{\pgfqpoint{1.057083in}{1.070355in}}%
\pgfpathlineto{\pgfqpoint{1.065248in}{1.053874in}}%
\pgfpathlineto{\pgfqpoint{1.072999in}{1.036092in}}%
\pgfpathlineto{\pgfqpoint{1.080306in}{1.017079in}}%
\pgfpathlineto{\pgfqpoint{1.087140in}{0.996909in}}%
\pgfpathlineto{\pgfqpoint{1.093474in}{0.975662in}}%
\pgfpathlineto{\pgfqpoint{1.099282in}{0.953423in}}%
\pgfpathlineto{\pgfqpoint{1.104542in}{0.930278in}}%
\pgfpathlineto{\pgfqpoint{1.109234in}{0.906320in}}%
\pgfpathlineto{\pgfqpoint{1.113338in}{0.881642in}}%
\pgfpathlineto{\pgfqpoint{1.116838in}{0.856342in}}%
\pgfpathlineto{\pgfqpoint{1.119721in}{0.830521in}}%
\pgfpathlineto{\pgfqpoint{1.121975in}{0.804279in}}%
\pgfpathlineto{\pgfqpoint{1.123592in}{0.777721in}}%
\pgfpathlineto{\pgfqpoint{1.124564in}{0.750951in}}%
\pgfpathlineto{\pgfqpoint{1.124889in}{0.724075in}}%
\pgfpathlineto{\pgfqpoint{1.124564in}{0.697199in}}%
\pgfpathlineto{\pgfqpoint{1.123592in}{0.670430in}}%
\pgfpathlineto{\pgfqpoint{1.121975in}{0.643872in}}%
\pgfpathlineto{\pgfqpoint{1.119721in}{0.617630in}}%
\pgfpathlineto{\pgfqpoint{1.116838in}{0.591808in}}%
\pgfpathlineto{\pgfqpoint{1.115949in}{0.585378in}}%
\pgfpathlineto{\pgfqpoint{1.429116in}{0.489799in}}%
\pgfpathlineto{\pgfqpoint{1.429105in}{0.489794in}}%
\pgfpathlineto{\pgfqpoint{1.429457in}{0.489695in}}%
\pgfpathclose%
\pgfusepath{stroke,fill}%
\end{pgfscope}%
\begin{pgfscope}%
\pgfpathrectangle{\pgfqpoint{0.039236in}{0.039236in}}{\pgfqpoint{1.595582in}{1.177068in}}%
\pgfusepath{clip}%
\pgfsetbuttcap%
\pgfsetmiterjoin%
\definecolor{currentfill}{rgb}{0.254902,0.411765,0.882353}%
\pgfsetfillcolor{currentfill}%
\pgfsetfillopacity{0.600000}%
\pgfsetlinewidth{1.003750pt}%
\definecolor{currentstroke}{rgb}{0.000000,0.000000,0.000000}%
\pgfsetstrokecolor{currentstroke}%
\pgfsetstrokeopacity{0.600000}%
\pgfsetdash{{1.000000pt}{1.650000pt}}{0.000000pt}%
\pgfpathmoveto{\pgfqpoint{0.963630in}{0.631866in}}%
\pgfpathlineto{\pgfqpoint{0.996529in}{0.546261in}}%
\pgfpathlineto{\pgfqpoint{0.930732in}{0.546261in}}%
\pgfpathclose%
\pgfusepath{stroke,fill}%
\end{pgfscope}%
\begin{pgfscope}%
\pgfpathrectangle{\pgfqpoint{0.039236in}{0.039236in}}{\pgfqpoint{1.595582in}{1.177068in}}%
\pgfusepath{clip}%
\pgfsetbuttcap%
\pgfsetmiterjoin%
\definecolor{currentfill}{rgb}{0.627451,0.321569,0.176471}%
\pgfsetfillcolor{currentfill}%
\pgfsetfillopacity{0.800000}%
\pgfsetlinewidth{1.003750pt}%
\definecolor{currentstroke}{rgb}{0.000000,0.000000,0.000000}%
\pgfsetstrokecolor{currentstroke}%
\pgfsetstrokeopacity{0.800000}%
\pgfsetdash{}{0pt}%
\pgfpathmoveto{\pgfqpoint{0.466917in}{0.360254in}}%
\pgfpathlineto{\pgfqpoint{0.302424in}{0.387288in}}%
\pgfpathlineto{\pgfqpoint{0.302424in}{1.002292in}}%
\pgfpathlineto{\pgfqpoint{0.055685in}{1.002292in}}%
\pgfpathlineto{\pgfqpoint{0.055685in}{0.895285in}}%
\pgfpathlineto{\pgfqpoint{0.220178in}{0.500950in}}%
\pgfpathlineto{\pgfqpoint{0.220178in}{0.253248in}}%
\pgfpathlineto{\pgfqpoint{0.466917in}{0.253248in}}%
\pgfpathclose%
\pgfusepath{stroke,fill}%
\end{pgfscope}%
\begin{pgfscope}%
\pgfpathrectangle{\pgfqpoint{0.039236in}{0.039236in}}{\pgfqpoint{1.595582in}{1.177068in}}%
\pgfusepath{clip}%
\pgfsetbuttcap%
\pgfsetmiterjoin%
\definecolor{currentfill}{rgb}{0.254902,0.411765,0.882353}%
\pgfsetfillcolor{currentfill}%
\pgfsetfillopacity{0.600000}%
\pgfsetlinewidth{1.003750pt}%
\definecolor{currentstroke}{rgb}{0.000000,0.000000,0.000000}%
\pgfsetstrokecolor{currentstroke}%
\pgfsetstrokeopacity{0.600000}%
\pgfsetdash{{1.000000pt}{1.650000pt}}{0.000000pt}%
\pgfpathmoveto{\pgfqpoint{0.290710in}{0.389213in}}%
\pgfpathlineto{\pgfqpoint{0.302181in}{0.359364in}}%
\pgfpathlineto{\pgfqpoint{0.279239in}{0.359364in}}%
\pgfpathclose%
\pgfusepath{stroke,fill}%
\end{pgfscope}%
\begin{pgfscope}%
\pgfpathrectangle{\pgfqpoint{0.039236in}{0.039236in}}{\pgfqpoint{1.595582in}{1.177068in}}%
\pgfusepath{clip}%
\pgfsetbuttcap%
\pgfsetmiterjoin%
\definecolor{currentfill}{rgb}{1.000000,0.000000,0.000000}%
\pgfsetfillcolor{currentfill}%
\pgfsetlinewidth{1.003750pt}%
\definecolor{currentstroke}{rgb}{1.000000,0.000000,0.000000}%
\pgfsetstrokecolor{currentstroke}%
\pgfsetdash{}{0pt}%
\pgfsys@defobject{currentmarker}{\pgfqpoint{-0.016667in}{-0.016667in}}{\pgfqpoint{0.016667in}{0.016667in}}{%
\pgfpathmoveto{\pgfqpoint{0.000000in}{0.016667in}}%
\pgfpathlineto{\pgfqpoint{-0.016667in}{-0.016667in}}%
\pgfpathlineto{\pgfqpoint{0.016667in}{-0.016667in}}%
\pgfpathclose%
\pgfusepath{stroke,fill}%
}%
\begin{pgfscope}%
\pgfsys@transformshift{0.351772in}{0.681273in}%
\pgfsys@useobject{currentmarker}{}%
\end{pgfscope}%
\begin{pgfscope}%
\pgfsys@transformshift{0.631410in}{0.274649in}%
\pgfsys@useobject{currentmarker}{}%
\end{pgfscope}%
\end{pgfscope}%
\end{pgfpicture}%
\makeatother%
\endgroup%

%% file: figures/kernel_update_1.pgf
\begingroup%
\makeatletter%
\begin{pgfpicture}%
\pgfpathrectangle{\pgfpointorigin}{\pgfqpoint{1.674053in}{1.255540in}}%
\pgfusepath{use as bounding box, clip}%
\begin{pgfscope}%
\pgfsetbuttcap%
\pgfsetmiterjoin%
\definecolor{currentfill}{rgb}{1.000000,1.000000,1.000000}%
\pgfsetfillcolor{currentfill}%
\pgfsetlinewidth{0.000000pt}%
\definecolor{currentstroke}{rgb}{1.000000,1.000000,1.000000}%
\pgfsetstrokecolor{currentstroke}%
\pgfsetdash{}{0pt}%
\pgfpathmoveto{\pgfqpoint{0.000000in}{0.000000in}}%
\pgfpathlineto{\pgfqpoint{1.674053in}{0.000000in}}%
\pgfpathlineto{\pgfqpoint{1.674053in}{1.255540in}}%
\pgfpathlineto{\pgfqpoint{0.000000in}{1.255540in}}%
\pgfpathclose%
\pgfusepath{fill}%
\end{pgfscope}%
\begin{pgfscope}%
\pgfpathrectangle{\pgfqpoint{0.039236in}{0.039236in}}{\pgfqpoint{1.595582in}{1.177068in}}%
\pgfusepath{clip}%
\pgfsetbuttcap%
\pgfsetmiterjoin%
\definecolor{currentfill}{rgb}{0.827451,0.827451,0.827451}%
\pgfsetfillcolor{currentfill}%
\pgfsetfillopacity{0.800000}%
\pgfsetlinewidth{1.003750pt}%
\definecolor{currentstroke}{rgb}{0.000000,0.000000,0.000000}%
\pgfsetstrokecolor{currentstroke}%
\pgfsetstrokeopacity{0.800000}%
\pgfsetdash{}{0pt}%
\pgfpathmoveto{\pgfqpoint{0.571096in}{0.490445in}}%
\pgfpathcurveto{\pgfqpoint{0.641622in}{0.490445in}}{\pgfqpoint{0.709268in}{0.511116in}}{\pgfqpoint{0.759137in}{0.547904in}}%
\pgfpathcurveto{\pgfqpoint{0.809006in}{0.584693in}}{\pgfqpoint{0.837026in}{0.634596in}}{\pgfqpoint{0.837026in}{0.686623in}}%
\pgfpathcurveto{\pgfqpoint{0.837026in}{0.738650in}}{\pgfqpoint{0.809006in}{0.788553in}}{\pgfqpoint{0.759137in}{0.825342in}}%
\pgfpathcurveto{\pgfqpoint{0.709268in}{0.862131in}}{\pgfqpoint{0.641622in}{0.882801in}}{\pgfqpoint{0.571096in}{0.882801in}}%
\pgfpathcurveto{\pgfqpoint{0.500571in}{0.882801in}}{\pgfqpoint{0.432924in}{0.862131in}}{\pgfqpoint{0.383055in}{0.825342in}}%
\pgfpathcurveto{\pgfqpoint{0.333186in}{0.788553in}}{\pgfqpoint{0.305166in}{0.738650in}}{\pgfqpoint{0.305166in}{0.686623in}}%
\pgfpathcurveto{\pgfqpoint{0.305166in}{0.634596in}}{\pgfqpoint{0.333186in}{0.584693in}}{\pgfqpoint{0.383055in}{0.547904in}}%
\pgfpathcurveto{\pgfqpoint{0.432924in}{0.511116in}}{\pgfqpoint{0.500571in}{0.490445in}}{\pgfqpoint{0.571096in}{0.490445in}}%
\pgfpathclose%
\pgfusepath{stroke,fill}%
\end{pgfscope}%
\begin{pgfscope}%
\pgfpathrectangle{\pgfqpoint{0.039236in}{0.039236in}}{\pgfqpoint{1.595582in}{1.177068in}}%
\pgfusepath{clip}%
\pgfsetbuttcap%
\pgfsetmiterjoin%
\definecolor{currentfill}{rgb}{0.827451,0.827451,0.827451}%
\pgfsetfillcolor{currentfill}%
\pgfsetfillopacity{0.800000}%
\pgfsetlinewidth{1.003750pt}%
\definecolor{currentstroke}{rgb}{0.000000,0.000000,0.000000}%
\pgfsetstrokecolor{currentstroke}%
\pgfsetstrokeopacity{0.800000}%
\pgfsetdash{}{0pt}%
\pgfpathmoveto{\pgfqpoint{0.837026in}{0.490445in}}%
\pgfpathcurveto{\pgfqpoint{0.907552in}{0.490445in}}{\pgfqpoint{0.975198in}{0.511116in}}{\pgfqpoint{1.025067in}{0.547904in}}%
\pgfpathcurveto{\pgfqpoint{1.074937in}{0.584693in}}{\pgfqpoint{1.102957in}{0.634596in}}{\pgfqpoint{1.102957in}{0.686623in}}%
\pgfpathcurveto{\pgfqpoint{1.102957in}{0.738650in}}{\pgfqpoint{1.074937in}{0.788553in}}{\pgfqpoint{1.025067in}{0.825342in}}%
\pgfpathcurveto{\pgfqpoint{0.975198in}{0.862131in}}{\pgfqpoint{0.907552in}{0.882801in}}{\pgfqpoint{0.837026in}{0.882801in}}%
\pgfpathcurveto{\pgfqpoint{0.766501in}{0.882801in}}{\pgfqpoint{0.698854in}{0.862131in}}{\pgfqpoint{0.648985in}{0.825342in}}%
\pgfpathcurveto{\pgfqpoint{0.599116in}{0.788553in}}{\pgfqpoint{0.571096in}{0.738650in}}{\pgfqpoint{0.571096in}{0.686623in}}%
\pgfpathcurveto{\pgfqpoint{0.571096in}{0.634596in}}{\pgfqpoint{0.599116in}{0.584693in}}{\pgfqpoint{0.648985in}{0.547904in}}%
\pgfpathcurveto{\pgfqpoint{0.698854in}{0.511116in}}{\pgfqpoint{0.766501in}{0.490445in}}{\pgfqpoint{0.837026in}{0.490445in}}%
\pgfpathclose%
\pgfusepath{stroke,fill}%
\end{pgfscope}%
\begin{pgfscope}%
\pgfpathrectangle{\pgfqpoint{0.039236in}{0.039236in}}{\pgfqpoint{1.595582in}{1.177068in}}%
\pgfusepath{clip}%
\pgfsetbuttcap%
\pgfsetmiterjoin%
\definecolor{currentfill}{rgb}{0.827451,0.827451,0.827451}%
\pgfsetfillcolor{currentfill}%
\pgfsetfillopacity{0.800000}%
\pgfsetlinewidth{1.003750pt}%
\definecolor{currentstroke}{rgb}{0.000000,0.000000,0.000000}%
\pgfsetstrokecolor{currentstroke}%
\pgfsetstrokeopacity{0.800000}%
\pgfsetdash{}{0pt}%
\pgfpathmoveto{\pgfqpoint{1.102957in}{0.490445in}}%
\pgfpathcurveto{\pgfqpoint{1.173482in}{0.490445in}}{\pgfqpoint{1.241129in}{0.511116in}}{\pgfqpoint{1.290998in}{0.547904in}}%
\pgfpathcurveto{\pgfqpoint{1.340867in}{0.584693in}}{\pgfqpoint{1.368887in}{0.634596in}}{\pgfqpoint{1.368887in}{0.686623in}}%
\pgfpathcurveto{\pgfqpoint{1.368887in}{0.738650in}}{\pgfqpoint{1.340867in}{0.788553in}}{\pgfqpoint{1.290998in}{0.825342in}}%
\pgfpathcurveto{\pgfqpoint{1.241129in}{0.862131in}}{\pgfqpoint{1.173482in}{0.882801in}}{\pgfqpoint{1.102957in}{0.882801in}}%
\pgfpathcurveto{\pgfqpoint{1.032431in}{0.882801in}}{\pgfqpoint{0.964785in}{0.862131in}}{\pgfqpoint{0.914916in}{0.825342in}}%
\pgfpathcurveto{\pgfqpoint{0.865046in}{0.788553in}}{\pgfqpoint{0.837026in}{0.738650in}}{\pgfqpoint{0.837026in}{0.686623in}}%
\pgfpathcurveto{\pgfqpoint{0.837026in}{0.634596in}}{\pgfqpoint{0.865046in}{0.584693in}}{\pgfqpoint{0.914916in}{0.547904in}}%
\pgfpathcurveto{\pgfqpoint{0.964785in}{0.511116in}}{\pgfqpoint{1.032431in}{0.490445in}}{\pgfqpoint{1.102957in}{0.490445in}}%
\pgfpathclose%
\pgfusepath{stroke,fill}%
\end{pgfscope}%
\begin{pgfscope}%
\pgfpathrectangle{\pgfqpoint{0.039236in}{0.039236in}}{\pgfqpoint{1.595582in}{1.177068in}}%
\pgfusepath{clip}%
\pgfsetbuttcap%
\pgfsetroundjoin%
\pgfsetlinewidth{2.007500pt}%
\definecolor{currentstroke}{rgb}{0.750000,0.750000,0.000000}%
\pgfsetstrokecolor{currentstroke}%
\pgfsetdash{{7.400000pt}{3.200000pt}}{0.000000pt}%
\pgfpathmoveto{\pgfqpoint{0.863619in}{0.137325in}}%
\pgfpathlineto{\pgfqpoint{0.866710in}{0.249591in}}%
\pgfpathlineto{\pgfqpoint{0.870443in}{0.313624in}}%
\pgfpathlineto{\pgfqpoint{0.874591in}{0.351384in}}%
\pgfpathlineto{\pgfqpoint{0.879759in}{0.378372in}}%
\pgfpathlineto{\pgfqpoint{0.885677in}{0.396940in}}%
\pgfpathlineto{\pgfqpoint{0.891789in}{0.409360in}}%
\pgfpathlineto{\pgfqpoint{0.899414in}{0.419872in}}%
\pgfpathlineto{\pgfqpoint{0.908841in}{0.428839in}}%
\pgfpathlineto{\pgfqpoint{0.920379in}{0.436726in}}%
\pgfpathlineto{\pgfqpoint{0.938251in}{0.445916in}}%
\pgfpathlineto{\pgfqpoint{1.006200in}{0.475866in}}%
\pgfpathlineto{\pgfqpoint{1.042937in}{0.467410in}}%
\pgfpathlineto{\pgfqpoint{1.083370in}{0.460992in}}%
\pgfpathlineto{\pgfqpoint{1.111498in}{0.458530in}}%
\pgfpathlineto{\pgfqpoint{1.140077in}{0.457758in}}%
\pgfpathlineto{\pgfqpoint{1.168763in}{0.458770in}}%
\pgfpathlineto{\pgfqpoint{1.197265in}{0.461601in}}%
\pgfpathlineto{\pgfqpoint{1.225325in}{0.466248in}}%
\pgfpathlineto{\pgfqpoint{1.252718in}{0.472682in}}%
\pgfpathlineto{\pgfqpoint{1.279236in}{0.480862in}}%
\pgfpathlineto{\pgfqpoint{1.304687in}{0.490730in}}%
\pgfpathlineto{\pgfqpoint{1.328895in}{0.502225in}}%
\pgfpathlineto{\pgfqpoint{1.351688in}{0.515276in}}%
\pgfpathlineto{\pgfqpoint{1.372907in}{0.529808in}}%
\pgfpathlineto{\pgfqpoint{1.392393in}{0.545739in}}%
\pgfpathlineto{\pgfqpoint{1.409996in}{0.562985in}}%
\pgfpathlineto{\pgfqpoint{1.425566in}{0.581453in}}%
\pgfpathlineto{\pgfqpoint{1.438955in}{0.601047in}}%
\pgfpathlineto{\pgfqpoint{1.450017in}{0.621663in}}%
\pgfpathlineto{\pgfqpoint{1.458607in}{0.643192in}}%
\pgfpathlineto{\pgfqpoint{1.464578in}{0.665513in}}%
\pgfpathlineto{\pgfqpoint{1.467786in}{0.688499in}}%
\pgfpathlineto{\pgfqpoint{1.468082in}{0.712010in}}%
\pgfpathlineto{\pgfqpoint{1.465321in}{0.735892in}}%
\pgfpathlineto{\pgfqpoint{1.459356in}{0.759977in}}%
\pgfpathlineto{\pgfqpoint{1.450042in}{0.784080in}}%
\pgfpathlineto{\pgfqpoint{1.437239in}{0.807991in}}%
\pgfpathlineto{\pgfqpoint{1.420817in}{0.831481in}}%
\pgfpathlineto{\pgfqpoint{1.400659in}{0.854288in}}%
\pgfpathlineto{\pgfqpoint{1.376680in}{0.876121in}}%
\pgfpathlineto{\pgfqpoint{1.348840in}{0.896658in}}%
\pgfpathlineto{\pgfqpoint{1.317180in}{0.915544in}}%
\pgfpathlineto{\pgfqpoint{1.281868in}{0.932409in}}%
\pgfpathlineto{\pgfqpoint{1.226191in}{0.955042in}}%
\pgfpathlineto{\pgfqpoint{1.170596in}{0.974982in}}%
\pgfpathlineto{\pgfqpoint{1.085167in}{1.003845in}}%
\pgfpathlineto{\pgfqpoint{1.000165in}{1.044780in}}%
\pgfpathlineto{\pgfqpoint{0.918552in}{1.083983in}}%
\pgfpathlineto{\pgfqpoint{0.881282in}{1.104305in}}%
\pgfpathlineto{\pgfqpoint{0.870045in}{1.111588in}}%
\pgfpathlineto{\pgfqpoint{0.870045in}{1.111588in}}%
\pgfusepath{stroke}%
\end{pgfscope}%
\begin{pgfscope}%
\pgfpathrectangle{\pgfqpoint{0.039236in}{0.039236in}}{\pgfqpoint{1.595582in}{1.177068in}}%
\pgfusepath{clip}%
\pgfsetbuttcap%
\pgfsetroundjoin%
\definecolor{currentfill}{rgb}{0.000000,0.000000,0.000000}%
\pgfsetfillcolor{currentfill}%
\pgfsetlinewidth{1.003750pt}%
\definecolor{currentstroke}{rgb}{0.000000,0.000000,0.000000}%
\pgfsetstrokecolor{currentstroke}%
\pgfsetdash{}{0pt}%
\pgfsys@defobject{currentmarker}{\pgfqpoint{-0.041667in}{-0.041667in}}{\pgfqpoint{0.041667in}{0.041667in}}{%
\pgfpathmoveto{\pgfqpoint{0.000000in}{-0.041667in}}%
\pgfpathcurveto{\pgfqpoint{0.011050in}{-0.041667in}}{\pgfqpoint{0.021649in}{-0.037276in}}{\pgfqpoint{0.029463in}{-0.029463in}}%
\pgfpathcurveto{\pgfqpoint{0.037276in}{-0.021649in}}{\pgfqpoint{0.041667in}{-0.011050in}}{\pgfqpoint{0.041667in}{0.000000in}}%
\pgfpathcurveto{\pgfqpoint{0.041667in}{0.011050in}}{\pgfqpoint{0.037276in}{0.021649in}}{\pgfqpoint{0.029463in}{0.029463in}}%
\pgfpathcurveto{\pgfqpoint{0.021649in}{0.037276in}}{\pgfqpoint{0.011050in}{0.041667in}}{\pgfqpoint{0.000000in}{0.041667in}}%
\pgfpathcurveto{\pgfqpoint{-0.011050in}{0.041667in}}{\pgfqpoint{-0.021649in}{0.037276in}}{\pgfqpoint{-0.029463in}{0.029463in}}%
\pgfpathcurveto{\pgfqpoint{-0.037276in}{0.021649in}}{\pgfqpoint{-0.041667in}{0.011050in}}{\pgfqpoint{-0.041667in}{0.000000in}}%
\pgfpathcurveto{\pgfqpoint{-0.041667in}{-0.011050in}}{\pgfqpoint{-0.037276in}{-0.021649in}}{\pgfqpoint{-0.029463in}{-0.029463in}}%
\pgfpathcurveto{\pgfqpoint{-0.021649in}{-0.037276in}}{\pgfqpoint{-0.011050in}{-0.041667in}}{\pgfqpoint{0.000000in}{-0.041667in}}%
\pgfpathclose%
\pgfusepath{stroke,fill}%
}%
\begin{pgfscope}%
\pgfsys@transformshift{0.863619in}{0.137325in}%
\pgfsys@useobject{currentmarker}{}%
\end{pgfscope}%
\end{pgfscope}%
\begin{pgfscope}%
\pgfpathrectangle{\pgfqpoint{0.039236in}{0.039236in}}{\pgfqpoint{1.595582in}{1.177068in}}%
\pgfusepath{clip}%
\pgfsetbuttcap%
\pgfsetbeveljoin%
\definecolor{currentfill}{rgb}{0.000000,0.000000,0.000000}%
\pgfsetfillcolor{currentfill}%
\pgfsetlinewidth{1.003750pt}%
\definecolor{currentstroke}{rgb}{0.000000,0.000000,0.000000}%
\pgfsetstrokecolor{currentstroke}%
\pgfsetdash{}{0pt}%
\pgfsys@defobject{currentmarker}{\pgfqpoint{-0.047553in}{-0.040451in}}{\pgfqpoint{0.047553in}{0.050000in}}{%
\pgfpathmoveto{\pgfqpoint{0.000000in}{0.050000in}}%
\pgfpathlineto{\pgfqpoint{-0.011226in}{0.015451in}}%
\pgfpathlineto{\pgfqpoint{-0.047553in}{0.015451in}}%
\pgfpathlineto{\pgfqpoint{-0.018164in}{-0.005902in}}%
\pgfpathlineto{\pgfqpoint{-0.029389in}{-0.040451in}}%
\pgfpathlineto{\pgfqpoint{-0.000000in}{-0.019098in}}%
\pgfpathlineto{\pgfqpoint{0.029389in}{-0.040451in}}%
\pgfpathlineto{\pgfqpoint{0.018164in}{-0.005902in}}%
\pgfpathlineto{\pgfqpoint{0.047553in}{0.015451in}}%
\pgfpathlineto{\pgfqpoint{0.011226in}{0.015451in}}%
\pgfpathclose%
\pgfusepath{stroke,fill}%
}%
\begin{pgfscope}%
\pgfsys@transformshift{0.863619in}{1.118215in}%
\pgfsys@useobject{currentmarker}{}%
\end{pgfscope}%
\end{pgfscope}%
\begin{pgfscope}%
\pgfpathrectangle{\pgfqpoint{0.039236in}{0.039236in}}{\pgfqpoint{1.595582in}{1.177068in}}%
\pgfusepath{clip}%
\pgfsetbuttcap%
\pgfsetmiterjoin%
\definecolor{currentfill}{rgb}{0.000000,0.500000,0.000000}%
\pgfsetfillcolor{currentfill}%
\pgfsetlinewidth{1.003750pt}%
\definecolor{currentstroke}{rgb}{0.000000,0.500000,0.000000}%
\pgfsetstrokecolor{currentstroke}%
\pgfsetdash{}{0pt}%
\pgfsys@defobject{currentmarker}{\pgfqpoint{-0.035355in}{-0.058926in}}{\pgfqpoint{0.035355in}{0.058926in}}{%
\pgfpathmoveto{\pgfqpoint{-0.000000in}{-0.058926in}}%
\pgfpathlineto{\pgfqpoint{0.035355in}{0.000000in}}%
\pgfpathlineto{\pgfqpoint{0.000000in}{0.058926in}}%
\pgfpathlineto{\pgfqpoint{-0.035355in}{0.000000in}}%
\pgfpathclose%
\pgfusepath{stroke,fill}%
}%
\begin{pgfscope}%
\pgfsys@transformshift{0.837026in}{0.686623in}%
\pgfsys@useobject{currentmarker}{}%
\end{pgfscope}%
\end{pgfscope}%
\begin{pgfscope}%
\pgfpathrectangle{\pgfqpoint{0.039236in}{0.039236in}}{\pgfqpoint{1.595582in}{1.177068in}}%
\pgfusepath{clip}%
\pgfsetbuttcap%
\pgfsetroundjoin%
\pgfsetlinewidth{1.505625pt}%
\definecolor{currentstroke}{rgb}{0.121569,0.466667,0.705882}%
\pgfsetstrokecolor{currentstroke}%
\pgfsetdash{{5.550000pt}{2.400000pt}}{0.000000pt}%
\pgfpathmoveto{\pgfqpoint{0.863619in}{0.137325in}}%
\pgfpathlineto{\pgfqpoint{0.863619in}{1.118215in}}%
\pgfusepath{stroke}%
\end{pgfscope}%
\end{pgfpicture}%
\makeatother%
\endgroup%

%% file: figures/kernel_update_2.pgf
\begingroup%
\makeatletter%
\begin{pgfpicture}%
\pgfpathrectangle{\pgfpointorigin}{\pgfqpoint{1.674053in}{1.255540in}}%
\pgfusepath{use as bounding box, clip}%
\begin{pgfscope}%
\pgfsetbuttcap%
\pgfsetmiterjoin%
\definecolor{currentfill}{rgb}{1.000000,1.000000,1.000000}%
\pgfsetfillcolor{currentfill}%
\pgfsetlinewidth{0.000000pt}%
\definecolor{currentstroke}{rgb}{1.000000,1.000000,1.000000}%
\pgfsetstrokecolor{currentstroke}%
\pgfsetdash{}{0pt}%
\pgfpathmoveto{\pgfqpoint{0.000000in}{0.000000in}}%
\pgfpathlineto{\pgfqpoint{1.674053in}{0.000000in}}%
\pgfpathlineto{\pgfqpoint{1.674053in}{1.255540in}}%
\pgfpathlineto{\pgfqpoint{0.000000in}{1.255540in}}%
\pgfpathclose%
\pgfusepath{fill}%
\end{pgfscope}%
\begin{pgfscope}%
\pgfpathrectangle{\pgfqpoint{0.039236in}{0.039236in}}{\pgfqpoint{1.595582in}{1.177068in}}%
\pgfusepath{clip}%
\pgfsetbuttcap%
\pgfsetmiterjoin%
\definecolor{currentfill}{rgb}{0.827451,0.827451,0.827451}%
\pgfsetfillcolor{currentfill}%
\pgfsetfillopacity{0.800000}%
\pgfsetlinewidth{1.003750pt}%
\definecolor{currentstroke}{rgb}{0.000000,0.000000,0.000000}%
\pgfsetstrokecolor{currentstroke}%
\pgfsetstrokeopacity{0.800000}%
\pgfsetdash{}{0pt}%
\pgfpathmoveto{\pgfqpoint{0.571096in}{0.235414in}}%
\pgfpathcurveto{\pgfqpoint{0.641622in}{0.235414in}}{\pgfqpoint{0.709268in}{0.256084in}}{\pgfqpoint{0.759137in}{0.292873in}}%
\pgfpathcurveto{\pgfqpoint{0.809006in}{0.329662in}}{\pgfqpoint{0.837026in}{0.379565in}}{\pgfqpoint{0.837026in}{0.431592in}}%
\pgfpathcurveto{\pgfqpoint{0.837026in}{0.483619in}}{\pgfqpoint{0.809006in}{0.533522in}}{\pgfqpoint{0.759137in}{0.570311in}}%
\pgfpathcurveto{\pgfqpoint{0.709268in}{0.607099in}}{\pgfqpoint{0.641622in}{0.627770in}}{\pgfqpoint{0.571096in}{0.627770in}}%
\pgfpathcurveto{\pgfqpoint{0.500571in}{0.627770in}}{\pgfqpoint{0.432924in}{0.607099in}}{\pgfqpoint{0.383055in}{0.570311in}}%
\pgfpathcurveto{\pgfqpoint{0.333186in}{0.533522in}}{\pgfqpoint{0.305166in}{0.483619in}}{\pgfqpoint{0.305166in}{0.431592in}}%
\pgfpathcurveto{\pgfqpoint{0.305166in}{0.379565in}}{\pgfqpoint{0.333186in}{0.329662in}}{\pgfqpoint{0.383055in}{0.292873in}}%
\pgfpathcurveto{\pgfqpoint{0.432924in}{0.256084in}}{\pgfqpoint{0.500571in}{0.235414in}}{\pgfqpoint{0.571096in}{0.235414in}}%
\pgfpathclose%
\pgfusepath{stroke,fill}%
\end{pgfscope}%
\begin{pgfscope}%
\pgfpathrectangle{\pgfqpoint{0.039236in}{0.039236in}}{\pgfqpoint{1.595582in}{1.177068in}}%
\pgfusepath{clip}%
\pgfsetbuttcap%
\pgfsetmiterjoin%
\definecolor{currentfill}{rgb}{0.827451,0.827451,0.827451}%
\pgfsetfillcolor{currentfill}%
\pgfsetfillopacity{0.800000}%
\pgfsetlinewidth{1.003750pt}%
\definecolor{currentstroke}{rgb}{0.000000,0.000000,0.000000}%
\pgfsetstrokecolor{currentstroke}%
\pgfsetstrokeopacity{0.800000}%
\pgfsetdash{}{0pt}%
\pgfpathmoveto{\pgfqpoint{0.837026in}{0.235414in}}%
\pgfpathcurveto{\pgfqpoint{0.907552in}{0.235414in}}{\pgfqpoint{0.975198in}{0.256084in}}{\pgfqpoint{1.025067in}{0.292873in}}%
\pgfpathcurveto{\pgfqpoint{1.074937in}{0.329662in}}{\pgfqpoint{1.102957in}{0.379565in}}{\pgfqpoint{1.102957in}{0.431592in}}%
\pgfpathcurveto{\pgfqpoint{1.102957in}{0.483619in}}{\pgfqpoint{1.074937in}{0.533522in}}{\pgfqpoint{1.025067in}{0.570311in}}%
\pgfpathcurveto{\pgfqpoint{0.975198in}{0.607099in}}{\pgfqpoint{0.907552in}{0.627770in}}{\pgfqpoint{0.837026in}{0.627770in}}%
\pgfpathcurveto{\pgfqpoint{0.766501in}{0.627770in}}{\pgfqpoint{0.698854in}{0.607099in}}{\pgfqpoint{0.648985in}{0.570311in}}%
\pgfpathcurveto{\pgfqpoint{0.599116in}{0.533522in}}{\pgfqpoint{0.571096in}{0.483619in}}{\pgfqpoint{0.571096in}{0.431592in}}%
\pgfpathcurveto{\pgfqpoint{0.571096in}{0.379565in}}{\pgfqpoint{0.599116in}{0.329662in}}{\pgfqpoint{0.648985in}{0.292873in}}%
\pgfpathcurveto{\pgfqpoint{0.698854in}{0.256084in}}{\pgfqpoint{0.766501in}{0.235414in}}{\pgfqpoint{0.837026in}{0.235414in}}%
\pgfpathclose%
\pgfusepath{stroke,fill}%
\end{pgfscope}%
\begin{pgfscope}%
\pgfpathrectangle{\pgfqpoint{0.039236in}{0.039236in}}{\pgfqpoint{1.595582in}{1.177068in}}%
\pgfusepath{clip}%
\pgfsetbuttcap%
\pgfsetmiterjoin%
\definecolor{currentfill}{rgb}{0.827451,0.827451,0.827451}%
\pgfsetfillcolor{currentfill}%
\pgfsetfillopacity{0.800000}%
\pgfsetlinewidth{1.003750pt}%
\definecolor{currentstroke}{rgb}{0.000000,0.000000,0.000000}%
\pgfsetstrokecolor{currentstroke}%
\pgfsetstrokeopacity{0.800000}%
\pgfsetdash{}{0pt}%
\pgfpathmoveto{\pgfqpoint{1.102957in}{0.235414in}}%
\pgfpathcurveto{\pgfqpoint{1.173482in}{0.235414in}}{\pgfqpoint{1.241129in}{0.256084in}}{\pgfqpoint{1.290998in}{0.292873in}}%
\pgfpathcurveto{\pgfqpoint{1.340867in}{0.329662in}}{\pgfqpoint{1.368887in}{0.379565in}}{\pgfqpoint{1.368887in}{0.431592in}}%
\pgfpathcurveto{\pgfqpoint{1.368887in}{0.483619in}}{\pgfqpoint{1.340867in}{0.533522in}}{\pgfqpoint{1.290998in}{0.570311in}}%
\pgfpathcurveto{\pgfqpoint{1.241129in}{0.607099in}}{\pgfqpoint{1.173482in}{0.627770in}}{\pgfqpoint{1.102957in}{0.627770in}}%
\pgfpathcurveto{\pgfqpoint{1.032431in}{0.627770in}}{\pgfqpoint{0.964785in}{0.607099in}}{\pgfqpoint{0.914916in}{0.570311in}}%
\pgfpathcurveto{\pgfqpoint{0.865046in}{0.533522in}}{\pgfqpoint{0.837026in}{0.483619in}}{\pgfqpoint{0.837026in}{0.431592in}}%
\pgfpathcurveto{\pgfqpoint{0.837026in}{0.379565in}}{\pgfqpoint{0.865046in}{0.329662in}}{\pgfqpoint{0.914916in}{0.292873in}}%
\pgfpathcurveto{\pgfqpoint{0.964785in}{0.256084in}}{\pgfqpoint{1.032431in}{0.235414in}}{\pgfqpoint{1.102957in}{0.235414in}}%
\pgfpathclose%
\pgfusepath{stroke,fill}%
\end{pgfscope}%
\begin{pgfscope}%
\pgfpathrectangle{\pgfqpoint{0.039236in}{0.039236in}}{\pgfqpoint{1.595582in}{1.177068in}}%
\pgfusepath{clip}%
\pgfsetbuttcap%
\pgfsetmiterjoin%
\pgfsetlinewidth{0.000000pt}%
\definecolor{currentstroke}{rgb}{0.000000,0.000000,0.000000}%
\pgfsetstrokecolor{currentstroke}%
\pgfsetdash{}{0pt}%
\pgfpathmoveto{\pgfqpoint{0.571096in}{0.235414in}}%
\pgfpathlineto{\pgfqpoint{0.488919in}{0.245015in}}%
\pgfpathlineto{\pgfqpoint{0.414786in}{0.272880in}}%
\pgfpathlineto{\pgfqpoint{0.355954in}{0.316281in}}%
\pgfpathlineto{\pgfqpoint{0.318181in}{0.370969in}}%
\pgfpathlineto{\pgfqpoint{0.305166in}{0.431592in}}%
\pgfpathlineto{\pgfqpoint{0.318181in}{0.492214in}}%
\pgfpathlineto{\pgfqpoint{0.355954in}{0.546902in}}%
\pgfpathlineto{\pgfqpoint{0.414786in}{0.590303in}}%
\pgfpathlineto{\pgfqpoint{0.488919in}{0.618168in}}%
\pgfpathlineto{\pgfqpoint{0.571096in}{0.627770in}}%
\pgfpathlineto{\pgfqpoint{0.653273in}{0.618168in}}%
\pgfpathlineto{\pgfqpoint{0.704061in}{0.599078in}}%
\pgfpathlineto{\pgfqpoint{0.754849in}{0.618168in}}%
\pgfpathlineto{\pgfqpoint{0.837026in}{0.627770in}}%
\pgfpathlineto{\pgfqpoint{0.863619in}{0.624663in}}%
\pgfpathlineto{\pgfqpoint{0.863619in}{0.238521in}}%
\pgfpathlineto{\pgfqpoint{0.837026in}{0.235414in}}%
\pgfpathlineto{\pgfqpoint{0.754849in}{0.245015in}}%
\pgfpathlineto{\pgfqpoint{0.704061in}{0.264106in}}%
\pgfpathlineto{\pgfqpoint{0.653273in}{0.245015in}}%
\pgfpathclose%
\pgfusepath{}%
\end{pgfscope}%
\begin{pgfscope}%
\pgfsetbuttcap%
\pgfsetmiterjoin%
\pgfsetlinewidth{0.000000pt}%
\definecolor{currentstroke}{rgb}{0.000000,0.000000,0.000000}%
\pgfsetstrokecolor{currentstroke}%
\pgfsetdash{}{0pt}%
\pgfpathrectangle{\pgfqpoint{0.039236in}{0.039236in}}{\pgfqpoint{1.595582in}{1.177068in}}%
\pgfusepath{clip}%
\pgfpathmoveto{\pgfqpoint{0.571096in}{0.235414in}}%
\pgfpathlineto{\pgfqpoint{0.488919in}{0.245015in}}%
\pgfpathlineto{\pgfqpoint{0.414786in}{0.272880in}}%
\pgfpathlineto{\pgfqpoint{0.355954in}{0.316281in}}%
\pgfpathlineto{\pgfqpoint{0.318181in}{0.370969in}}%
\pgfpathlineto{\pgfqpoint{0.305166in}{0.431592in}}%
\pgfpathlineto{\pgfqpoint{0.318181in}{0.492214in}}%
\pgfpathlineto{\pgfqpoint{0.355954in}{0.546902in}}%
\pgfpathlineto{\pgfqpoint{0.414786in}{0.590303in}}%
\pgfpathlineto{\pgfqpoint{0.488919in}{0.618168in}}%
\pgfpathlineto{\pgfqpoint{0.571096in}{0.627770in}}%
\pgfpathlineto{\pgfqpoint{0.653273in}{0.618168in}}%
\pgfpathlineto{\pgfqpoint{0.704061in}{0.599078in}}%
\pgfpathlineto{\pgfqpoint{0.754849in}{0.618168in}}%
\pgfpathlineto{\pgfqpoint{0.837026in}{0.627770in}}%
\pgfpathlineto{\pgfqpoint{0.863619in}{0.624663in}}%
\pgfpathlineto{\pgfqpoint{0.863619in}{0.238521in}}%
\pgfpathlineto{\pgfqpoint{0.837026in}{0.235414in}}%
\pgfpathlineto{\pgfqpoint{0.754849in}{0.245015in}}%
\pgfpathlineto{\pgfqpoint{0.704061in}{0.264106in}}%
\pgfpathlineto{\pgfqpoint{0.653273in}{0.245015in}}%
\pgfpathclose%
\pgfusepath{clip}%
\pgfsys@defobject{currentpattern}{\pgfqpoint{0in}{0in}}{\pgfqpoint{1in}{1in}}{%
\begin{pgfscope}%
\pgfpathrectangle{\pgfqpoint{0in}{0in}}{\pgfqpoint{1in}{1in}}%
\pgfusepath{clip}%
\pgfpathmoveto{\pgfqpoint{-0.500000in}{0.500000in}}%
\pgfpathlineto{\pgfqpoint{0.500000in}{1.500000in}}%
\pgfpathmoveto{\pgfqpoint{-0.444444in}{0.444444in}}%
\pgfpathlineto{\pgfqpoint{0.555556in}{1.444444in}}%
\pgfpathmoveto{\pgfqpoint{-0.388889in}{0.388889in}}%
\pgfpathlineto{\pgfqpoint{0.611111in}{1.388889in}}%
\pgfpathmoveto{\pgfqpoint{-0.333333in}{0.333333in}}%
\pgfpathlineto{\pgfqpoint{0.666667in}{1.333333in}}%
\pgfpathmoveto{\pgfqpoint{-0.277778in}{0.277778in}}%
\pgfpathlineto{\pgfqpoint{0.722222in}{1.277778in}}%
\pgfpathmoveto{\pgfqpoint{-0.222222in}{0.222222in}}%
\pgfpathlineto{\pgfqpoint{0.777778in}{1.222222in}}%
\pgfpathmoveto{\pgfqpoint{-0.166667in}{0.166667in}}%
\pgfpathlineto{\pgfqpoint{0.833333in}{1.166667in}}%
\pgfpathmoveto{\pgfqpoint{-0.111111in}{0.111111in}}%
\pgfpathlineto{\pgfqpoint{0.888889in}{1.111111in}}%
\pgfpathmoveto{\pgfqpoint{-0.055556in}{0.055556in}}%
\pgfpathlineto{\pgfqpoint{0.944444in}{1.055556in}}%
\pgfpathmoveto{\pgfqpoint{0.000000in}{0.000000in}}%
\pgfpathlineto{\pgfqpoint{1.000000in}{1.000000in}}%
\pgfpathmoveto{\pgfqpoint{0.055556in}{-0.055556in}}%
\pgfpathlineto{\pgfqpoint{1.055556in}{0.944444in}}%
\pgfpathmoveto{\pgfqpoint{0.111111in}{-0.111111in}}%
\pgfpathlineto{\pgfqpoint{1.111111in}{0.888889in}}%
\pgfpathmoveto{\pgfqpoint{0.166667in}{-0.166667in}}%
\pgfpathlineto{\pgfqpoint{1.166667in}{0.833333in}}%
\pgfpathmoveto{\pgfqpoint{0.222222in}{-0.222222in}}%
\pgfpathlineto{\pgfqpoint{1.222222in}{0.777778in}}%
\pgfpathmoveto{\pgfqpoint{0.277778in}{-0.277778in}}%
\pgfpathlineto{\pgfqpoint{1.277778in}{0.722222in}}%
\pgfpathmoveto{\pgfqpoint{0.333333in}{-0.333333in}}%
\pgfpathlineto{\pgfqpoint{1.333333in}{0.666667in}}%
\pgfpathmoveto{\pgfqpoint{0.388889in}{-0.388889in}}%
\pgfpathlineto{\pgfqpoint{1.388889in}{0.611111in}}%
\pgfpathmoveto{\pgfqpoint{0.444444in}{-0.444444in}}%
\pgfpathlineto{\pgfqpoint{1.444444in}{0.555556in}}%
\pgfpathmoveto{\pgfqpoint{0.500000in}{-0.500000in}}%
\pgfpathlineto{\pgfqpoint{1.500000in}{0.500000in}}%
\pgfusepath{stroke}%
\end{pgfscope}%
}%
\pgfsys@transformshift{0.305166in}{0.235414in}%
\pgfsys@useobject{currentpattern}{}%
\pgfsys@transformshift{1in}{0in}%
\pgfsys@transformshift{-1in}{0in}%
\pgfsys@transformshift{0in}{1in}%
\end{pgfscope}%
\begin{pgfscope}%
\pgfpathrectangle{\pgfqpoint{0.039236in}{0.039236in}}{\pgfqpoint{1.595582in}{1.177068in}}%
\pgfusepath{clip}%
\pgfsetbuttcap%
\pgfsetroundjoin%
\pgfsetlinewidth{2.007500pt}%
\definecolor{currentstroke}{rgb}{0.750000,0.750000,0.000000}%
\pgfsetstrokecolor{currentstroke}%
\pgfsetdash{{7.400000pt}{3.200000pt}}{0.000000pt}%
\pgfpathmoveto{\pgfqpoint{0.863619in}{0.137325in}}%
\pgfpathlineto{\pgfqpoint{0.890934in}{0.149261in}}%
\pgfpathlineto{\pgfqpoint{0.946866in}{0.169459in}}%
\pgfpathlineto{\pgfqpoint{0.998694in}{0.189334in}}%
\pgfpathlineto{\pgfqpoint{1.025768in}{0.201496in}}%
\pgfpathlineto{\pgfqpoint{1.052895in}{0.215500in}}%
\pgfpathlineto{\pgfqpoint{1.066299in}{0.223250in}}%
\pgfpathlineto{\pgfqpoint{1.114525in}{0.216868in}}%
\pgfpathlineto{\pgfqpoint{1.153413in}{0.215795in}}%
\pgfpathlineto{\pgfqpoint{1.187067in}{0.217697in}}%
\pgfpathlineto{\pgfqpoint{1.217103in}{0.221641in}}%
\pgfpathlineto{\pgfqpoint{1.244343in}{0.227137in}}%
\pgfpathlineto{\pgfqpoint{1.269270in}{0.233884in}}%
\pgfpathlineto{\pgfqpoint{1.292196in}{0.241676in}}%
\pgfpathlineto{\pgfqpoint{1.313337in}{0.250363in}}%
\pgfpathlineto{\pgfqpoint{1.332846in}{0.259831in}}%
\pgfpathlineto{\pgfqpoint{1.350837in}{0.269989in}}%
\pgfpathlineto{\pgfqpoint{1.367396in}{0.280761in}}%
\pgfpathlineto{\pgfqpoint{1.382589in}{0.292084in}}%
\pgfpathlineto{\pgfqpoint{1.396466in}{0.303903in}}%
\pgfpathlineto{\pgfqpoint{1.409066in}{0.316170in}}%
\pgfpathlineto{\pgfqpoint{1.420419in}{0.328841in}}%
\pgfpathlineto{\pgfqpoint{1.430546in}{0.341878in}}%
\pgfpathlineto{\pgfqpoint{1.439464in}{0.355242in}}%
\pgfpathlineto{\pgfqpoint{1.447184in}{0.368899in}}%
\pgfpathlineto{\pgfqpoint{1.453712in}{0.382817in}}%
\pgfpathlineto{\pgfqpoint{1.459051in}{0.396962in}}%
\pgfpathlineto{\pgfqpoint{1.463201in}{0.411305in}}%
\pgfpathlineto{\pgfqpoint{1.466157in}{0.425814in}}%
\pgfpathlineto{\pgfqpoint{1.467914in}{0.440458in}}%
\pgfpathlineto{\pgfqpoint{1.468461in}{0.455205in}}%
\pgfpathlineto{\pgfqpoint{1.467789in}{0.470025in}}%
\pgfpathlineto{\pgfqpoint{1.465882in}{0.484884in}}%
\pgfpathlineto{\pgfqpoint{1.462726in}{0.499748in}}%
\pgfpathlineto{\pgfqpoint{1.458304in}{0.514581in}}%
\pgfpathlineto{\pgfqpoint{1.452596in}{0.529347in}}%
\pgfpathlineto{\pgfqpoint{1.445583in}{0.544004in}}%
\pgfpathlineto{\pgfqpoint{1.437245in}{0.558511in}}%
\pgfpathlineto{\pgfqpoint{1.427561in}{0.572821in}}%
\pgfpathlineto{\pgfqpoint{1.416512in}{0.586884in}}%
\pgfpathlineto{\pgfqpoint{1.404083in}{0.600649in}}%
\pgfpathlineto{\pgfqpoint{1.375837in}{0.627229in}}%
\pgfpathlineto{\pgfqpoint{1.345650in}{0.652779in}}%
\pgfpathlineto{\pgfqpoint{1.314191in}{0.677020in}}%
\pgfpathlineto{\pgfqpoint{1.266722in}{0.710427in}}%
\pgfpathlineto{\pgfqpoint{1.237137in}{0.730518in}}%
\pgfpathlineto{\pgfqpoint{1.182790in}{0.780925in}}%
\pgfpathlineto{\pgfqpoint{1.041690in}{0.909678in}}%
\pgfpathlineto{\pgfqpoint{0.998010in}{0.953198in}}%
\pgfpathlineto{\pgfqpoint{0.957830in}{0.996412in}}%
\pgfpathlineto{\pgfqpoint{0.928041in}{1.031216in}}%
\pgfpathlineto{\pgfqpoint{0.886389in}{1.084504in}}%
\pgfpathlineto{\pgfqpoint{0.867287in}{1.111506in}}%
\pgfpathlineto{\pgfqpoint{0.861771in}{1.120732in}}%
\pgfpathlineto{\pgfqpoint{0.868959in}{1.112651in}}%
\pgfpathlineto{\pgfqpoint{0.860588in}{1.120821in}}%
\pgfpathlineto{\pgfqpoint{0.870161in}{1.113732in}}%
\pgfpathlineto{\pgfqpoint{0.859635in}{1.120595in}}%
\pgfpathlineto{\pgfqpoint{0.871320in}{1.115224in}}%
\pgfpathlineto{\pgfqpoint{0.858698in}{1.119999in}}%
\pgfpathlineto{\pgfqpoint{0.872037in}{1.116247in}}%
\pgfpathlineto{\pgfqpoint{0.858314in}{1.119473in}}%
\pgfpathlineto{\pgfqpoint{0.858492in}{1.119426in}}%
\pgfusepath{stroke}%
\end{pgfscope}%
\begin{pgfscope}%
\pgfpathrectangle{\pgfqpoint{0.039236in}{0.039236in}}{\pgfqpoint{1.595582in}{1.177068in}}%
\pgfusepath{clip}%
\pgfsetbuttcap%
\pgfsetroundjoin%
\definecolor{currentfill}{rgb}{0.000000,0.000000,0.000000}%
\pgfsetfillcolor{currentfill}%
\pgfsetlinewidth{1.003750pt}%
\definecolor{currentstroke}{rgb}{0.000000,0.000000,0.000000}%
\pgfsetstrokecolor{currentstroke}%
\pgfsetdash{}{0pt}%
\pgfsys@defobject{currentmarker}{\pgfqpoint{-0.041667in}{-0.041667in}}{\pgfqpoint{0.041667in}{0.041667in}}{%
\pgfpathmoveto{\pgfqpoint{0.000000in}{-0.041667in}}%
\pgfpathcurveto{\pgfqpoint{0.011050in}{-0.041667in}}{\pgfqpoint{0.021649in}{-0.037276in}}{\pgfqpoint{0.029463in}{-0.029463in}}%
\pgfpathcurveto{\pgfqpoint{0.037276in}{-0.021649in}}{\pgfqpoint{0.041667in}{-0.011050in}}{\pgfqpoint{0.041667in}{0.000000in}}%
\pgfpathcurveto{\pgfqpoint{0.041667in}{0.011050in}}{\pgfqpoint{0.037276in}{0.021649in}}{\pgfqpoint{0.029463in}{0.029463in}}%
\pgfpathcurveto{\pgfqpoint{0.021649in}{0.037276in}}{\pgfqpoint{0.011050in}{0.041667in}}{\pgfqpoint{0.000000in}{0.041667in}}%
\pgfpathcurveto{\pgfqpoint{-0.011050in}{0.041667in}}{\pgfqpoint{-0.021649in}{0.037276in}}{\pgfqpoint{-0.029463in}{0.029463in}}%
\pgfpathcurveto{\pgfqpoint{-0.037276in}{0.021649in}}{\pgfqpoint{-0.041667in}{0.011050in}}{\pgfqpoint{-0.041667in}{0.000000in}}%
\pgfpathcurveto{\pgfqpoint{-0.041667in}{-0.011050in}}{\pgfqpoint{-0.037276in}{-0.021649in}}{\pgfqpoint{-0.029463in}{-0.029463in}}%
\pgfpathcurveto{\pgfqpoint{-0.021649in}{-0.037276in}}{\pgfqpoint{-0.011050in}{-0.041667in}}{\pgfqpoint{0.000000in}{-0.041667in}}%
\pgfpathclose%
\pgfusepath{stroke,fill}%
}%
\begin{pgfscope}%
\pgfsys@transformshift{0.863619in}{0.137325in}%
\pgfsys@useobject{currentmarker}{}%
\end{pgfscope}%
\end{pgfscope}%
\begin{pgfscope}%
\pgfpathrectangle{\pgfqpoint{0.039236in}{0.039236in}}{\pgfqpoint{1.595582in}{1.177068in}}%
\pgfusepath{clip}%
\pgfsetbuttcap%
\pgfsetbeveljoin%
\definecolor{currentfill}{rgb}{0.000000,0.000000,0.000000}%
\pgfsetfillcolor{currentfill}%
\pgfsetlinewidth{1.003750pt}%
\definecolor{currentstroke}{rgb}{0.000000,0.000000,0.000000}%
\pgfsetstrokecolor{currentstroke}%
\pgfsetdash{}{0pt}%
\pgfsys@defobject{currentmarker}{\pgfqpoint{-0.047553in}{-0.040451in}}{\pgfqpoint{0.047553in}{0.050000in}}{%
\pgfpathmoveto{\pgfqpoint{0.000000in}{0.050000in}}%
\pgfpathlineto{\pgfqpoint{-0.011226in}{0.015451in}}%
\pgfpathlineto{\pgfqpoint{-0.047553in}{0.015451in}}%
\pgfpathlineto{\pgfqpoint{-0.018164in}{-0.005902in}}%
\pgfpathlineto{\pgfqpoint{-0.029389in}{-0.040451in}}%
\pgfpathlineto{\pgfqpoint{-0.000000in}{-0.019098in}}%
\pgfpathlineto{\pgfqpoint{0.029389in}{-0.040451in}}%
\pgfpathlineto{\pgfqpoint{0.018164in}{-0.005902in}}%
\pgfpathlineto{\pgfqpoint{0.047553in}{0.015451in}}%
\pgfpathlineto{\pgfqpoint{0.011226in}{0.015451in}}%
\pgfpathclose%
\pgfusepath{stroke,fill}%
}%
\begin{pgfscope}%
\pgfsys@transformshift{0.863619in}{1.118215in}%
\pgfsys@useobject{currentmarker}{}%
\end{pgfscope}%
\end{pgfscope}%
\begin{pgfscope}%
\pgfpathrectangle{\pgfqpoint{0.039236in}{0.039236in}}{\pgfqpoint{1.595582in}{1.177068in}}%
\pgfusepath{clip}%
\pgfsetbuttcap%
\pgfsetmiterjoin%
\definecolor{currentfill}{rgb}{0.000000,0.000000,0.000000}%
\pgfsetfillcolor{currentfill}%
\pgfsetlinewidth{1.003750pt}%
\definecolor{currentstroke}{rgb}{0.000000,0.000000,0.000000}%
\pgfsetstrokecolor{currentstroke}%
\pgfsetdash{}{0pt}%
\pgfsys@defobject{currentmarker}{\pgfqpoint{-0.041667in}{-0.041667in}}{\pgfqpoint{0.041667in}{0.041667in}}{%
\pgfpathmoveto{\pgfqpoint{-0.041667in}{-0.041667in}}%
\pgfpathlineto{\pgfqpoint{0.041667in}{-0.041667in}}%
\pgfpathlineto{\pgfqpoint{0.041667in}{0.041667in}}%
\pgfpathlineto{\pgfqpoint{-0.041667in}{0.041667in}}%
\pgfpathclose%
\pgfusepath{stroke,fill}%
}%
\begin{pgfscope}%
\pgfsys@transformshift{0.837026in}{0.686623in}%
\pgfsys@useobject{currentmarker}{}%
\end{pgfscope}%
\end{pgfscope}%
\begin{pgfscope}%
\pgfpathrectangle{\pgfqpoint{0.039236in}{0.039236in}}{\pgfqpoint{1.595582in}{1.177068in}}%
\pgfusepath{clip}%
\pgfsetbuttcap%
\pgfsetmiterjoin%
\definecolor{currentfill}{rgb}{0.000000,0.500000,0.000000}%
\pgfsetfillcolor{currentfill}%
\pgfsetlinewidth{1.003750pt}%
\definecolor{currentstroke}{rgb}{0.000000,0.500000,0.000000}%
\pgfsetstrokecolor{currentstroke}%
\pgfsetdash{}{0pt}%
\pgfsys@defobject{currentmarker}{\pgfqpoint{-0.035355in}{-0.058926in}}{\pgfqpoint{0.035355in}{0.058926in}}{%
\pgfpathmoveto{\pgfqpoint{-0.000000in}{-0.058926in}}%
\pgfpathlineto{\pgfqpoint{0.035355in}{0.000000in}}%
\pgfpathlineto{\pgfqpoint{0.000000in}{0.058926in}}%
\pgfpathlineto{\pgfqpoint{-0.035355in}{0.000000in}}%
\pgfpathclose%
\pgfusepath{stroke,fill}%
}%
\begin{pgfscope}%
\pgfsys@transformshift{0.608044in}{0.431592in}%
\pgfsys@useobject{currentmarker}{}%
\end{pgfscope}%
\end{pgfscope}%
\begin{pgfscope}%
\pgfpathrectangle{\pgfqpoint{0.039236in}{0.039236in}}{\pgfqpoint{1.595582in}{1.177068in}}%
\pgfusepath{clip}%
\pgfsetbuttcap%
\pgfsetroundjoin%
\pgfsetlinewidth{1.505625pt}%
\definecolor{currentstroke}{rgb}{0.121569,0.466667,0.705882}%
\pgfsetstrokecolor{currentstroke}%
\pgfsetdash{{5.550000pt}{2.400000pt}}{0.000000pt}%
\pgfpathmoveto{\pgfqpoint{0.863619in}{0.137325in}}%
\pgfpathlineto{\pgfqpoint{0.863619in}{1.118215in}}%
\pgfusepath{stroke}%
\end{pgfscope}%
\end{pgfpicture}%
\makeatother%
\endgroup%

%% file: figures/kernel_update_3.pgf
\begingroup%
\makeatletter%
\begin{pgfpicture}%
\pgfpathrectangle{\pgfpointorigin}{\pgfqpoint{1.674053in}{1.255540in}}%
\pgfusepath{use as bounding box, clip}%
\begin{pgfscope}%
\pgfsetbuttcap%
\pgfsetmiterjoin%
\definecolor{currentfill}{rgb}{1.000000,1.000000,1.000000}%
\pgfsetfillcolor{currentfill}%
\pgfsetlinewidth{0.000000pt}%
\definecolor{currentstroke}{rgb}{1.000000,1.000000,1.000000}%
\pgfsetstrokecolor{currentstroke}%
\pgfsetdash{}{0pt}%
\pgfpathmoveto{\pgfqpoint{0.000000in}{0.000000in}}%
\pgfpathlineto{\pgfqpoint{1.674053in}{0.000000in}}%
\pgfpathlineto{\pgfqpoint{1.674053in}{1.255540in}}%
\pgfpathlineto{\pgfqpoint{0.000000in}{1.255540in}}%
\pgfpathclose%
\pgfusepath{fill}%
\end{pgfscope}%
\begin{pgfscope}%
\pgfpathrectangle{\pgfqpoint{0.039236in}{0.039236in}}{\pgfqpoint{1.595582in}{1.177068in}}%
\pgfusepath{clip}%
\pgfsetbuttcap%
\pgfsetmiterjoin%
\definecolor{currentfill}{rgb}{0.827451,0.827451,0.827451}%
\pgfsetfillcolor{currentfill}%
\pgfsetfillopacity{0.800000}%
\pgfsetlinewidth{1.003750pt}%
\definecolor{currentstroke}{rgb}{0.000000,0.000000,0.000000}%
\pgfsetstrokecolor{currentstroke}%
\pgfsetstrokeopacity{0.800000}%
\pgfsetdash{}{0pt}%
\pgfpathmoveto{\pgfqpoint{0.571096in}{0.235414in}}%
\pgfpathcurveto{\pgfqpoint{0.641622in}{0.235414in}}{\pgfqpoint{0.709268in}{0.256084in}}{\pgfqpoint{0.759137in}{0.292873in}}%
\pgfpathcurveto{\pgfqpoint{0.809006in}{0.329662in}}{\pgfqpoint{0.837026in}{0.379565in}}{\pgfqpoint{0.837026in}{0.431592in}}%
\pgfpathcurveto{\pgfqpoint{0.837026in}{0.483619in}}{\pgfqpoint{0.809006in}{0.533522in}}{\pgfqpoint{0.759137in}{0.570311in}}%
\pgfpathcurveto{\pgfqpoint{0.709268in}{0.607099in}}{\pgfqpoint{0.641622in}{0.627770in}}{\pgfqpoint{0.571096in}{0.627770in}}%
\pgfpathcurveto{\pgfqpoint{0.500571in}{0.627770in}}{\pgfqpoint{0.432924in}{0.607099in}}{\pgfqpoint{0.383055in}{0.570311in}}%
\pgfpathcurveto{\pgfqpoint{0.333186in}{0.533522in}}{\pgfqpoint{0.305166in}{0.483619in}}{\pgfqpoint{0.305166in}{0.431592in}}%
\pgfpathcurveto{\pgfqpoint{0.305166in}{0.379565in}}{\pgfqpoint{0.333186in}{0.329662in}}{\pgfqpoint{0.383055in}{0.292873in}}%
\pgfpathcurveto{\pgfqpoint{0.432924in}{0.256084in}}{\pgfqpoint{0.500571in}{0.235414in}}{\pgfqpoint{0.571096in}{0.235414in}}%
\pgfpathclose%
\pgfusepath{stroke,fill}%
\end{pgfscope}%
\begin{pgfscope}%
\pgfpathrectangle{\pgfqpoint{0.039236in}{0.039236in}}{\pgfqpoint{1.595582in}{1.177068in}}%
\pgfusepath{clip}%
\pgfsetbuttcap%
\pgfsetmiterjoin%
\definecolor{currentfill}{rgb}{0.827451,0.827451,0.827451}%
\pgfsetfillcolor{currentfill}%
\pgfsetfillopacity{0.800000}%
\pgfsetlinewidth{1.003750pt}%
\definecolor{currentstroke}{rgb}{0.000000,0.000000,0.000000}%
\pgfsetstrokecolor{currentstroke}%
\pgfsetstrokeopacity{0.800000}%
\pgfsetdash{}{0pt}%
\pgfpathmoveto{\pgfqpoint{0.837026in}{0.235414in}}%
\pgfpathcurveto{\pgfqpoint{0.907552in}{0.235414in}}{\pgfqpoint{0.975198in}{0.256084in}}{\pgfqpoint{1.025067in}{0.292873in}}%
\pgfpathcurveto{\pgfqpoint{1.074937in}{0.329662in}}{\pgfqpoint{1.102957in}{0.379565in}}{\pgfqpoint{1.102957in}{0.431592in}}%
\pgfpathcurveto{\pgfqpoint{1.102957in}{0.483619in}}{\pgfqpoint{1.074937in}{0.533522in}}{\pgfqpoint{1.025067in}{0.570311in}}%
\pgfpathcurveto{\pgfqpoint{0.975198in}{0.607099in}}{\pgfqpoint{0.907552in}{0.627770in}}{\pgfqpoint{0.837026in}{0.627770in}}%
\pgfpathcurveto{\pgfqpoint{0.766501in}{0.627770in}}{\pgfqpoint{0.698854in}{0.607099in}}{\pgfqpoint{0.648985in}{0.570311in}}%
\pgfpathcurveto{\pgfqpoint{0.599116in}{0.533522in}}{\pgfqpoint{0.571096in}{0.483619in}}{\pgfqpoint{0.571096in}{0.431592in}}%
\pgfpathcurveto{\pgfqpoint{0.571096in}{0.379565in}}{\pgfqpoint{0.599116in}{0.329662in}}{\pgfqpoint{0.648985in}{0.292873in}}%
\pgfpathcurveto{\pgfqpoint{0.698854in}{0.256084in}}{\pgfqpoint{0.766501in}{0.235414in}}{\pgfqpoint{0.837026in}{0.235414in}}%
\pgfpathclose%
\pgfusepath{stroke,fill}%
\end{pgfscope}%
\begin{pgfscope}%
\pgfpathrectangle{\pgfqpoint{0.039236in}{0.039236in}}{\pgfqpoint{1.595582in}{1.177068in}}%
\pgfusepath{clip}%
\pgfsetbuttcap%
\pgfsetmiterjoin%
\definecolor{currentfill}{rgb}{0.827451,0.827451,0.827451}%
\pgfsetfillcolor{currentfill}%
\pgfsetfillopacity{0.800000}%
\pgfsetlinewidth{1.003750pt}%
\definecolor{currentstroke}{rgb}{0.000000,0.000000,0.000000}%
\pgfsetstrokecolor{currentstroke}%
\pgfsetstrokeopacity{0.800000}%
\pgfsetdash{}{0pt}%
\pgfpathmoveto{\pgfqpoint{1.102957in}{0.235414in}}%
\pgfpathcurveto{\pgfqpoint{1.173482in}{0.235414in}}{\pgfqpoint{1.241129in}{0.256084in}}{\pgfqpoint{1.290998in}{0.292873in}}%
\pgfpathcurveto{\pgfqpoint{1.340867in}{0.329662in}}{\pgfqpoint{1.368887in}{0.379565in}}{\pgfqpoint{1.368887in}{0.431592in}}%
\pgfpathcurveto{\pgfqpoint{1.368887in}{0.483619in}}{\pgfqpoint{1.340867in}{0.533522in}}{\pgfqpoint{1.290998in}{0.570311in}}%
\pgfpathcurveto{\pgfqpoint{1.241129in}{0.607099in}}{\pgfqpoint{1.173482in}{0.627770in}}{\pgfqpoint{1.102957in}{0.627770in}}%
\pgfpathcurveto{\pgfqpoint{1.032431in}{0.627770in}}{\pgfqpoint{0.964785in}{0.607099in}}{\pgfqpoint{0.914916in}{0.570311in}}%
\pgfpathcurveto{\pgfqpoint{0.865046in}{0.533522in}}{\pgfqpoint{0.837026in}{0.483619in}}{\pgfqpoint{0.837026in}{0.431592in}}%
\pgfpathcurveto{\pgfqpoint{0.837026in}{0.379565in}}{\pgfqpoint{0.865046in}{0.329662in}}{\pgfqpoint{0.914916in}{0.292873in}}%
\pgfpathcurveto{\pgfqpoint{0.964785in}{0.256084in}}{\pgfqpoint{1.032431in}{0.235414in}}{\pgfqpoint{1.102957in}{0.235414in}}%
\pgfpathclose%
\pgfusepath{stroke,fill}%
\end{pgfscope}%
\begin{pgfscope}%
\pgfpathrectangle{\pgfqpoint{0.039236in}{0.039236in}}{\pgfqpoint{1.595582in}{1.177068in}}%
\pgfusepath{clip}%
\pgfsetbuttcap%
\pgfsetmiterjoin%
\pgfsetlinewidth{0.000000pt}%
\definecolor{currentstroke}{rgb}{0.000000,0.000000,0.000000}%
\pgfsetstrokecolor{currentstroke}%
\pgfsetdash{}{0pt}%
\pgfpathmoveto{\pgfqpoint{0.571096in}{0.235414in}}%
\pgfpathlineto{\pgfqpoint{0.488919in}{0.245015in}}%
\pgfpathlineto{\pgfqpoint{0.414786in}{0.272880in}}%
\pgfpathlineto{\pgfqpoint{0.355954in}{0.316281in}}%
\pgfpathlineto{\pgfqpoint{0.318181in}{0.370969in}}%
\pgfpathlineto{\pgfqpoint{0.305166in}{0.431592in}}%
\pgfpathlineto{\pgfqpoint{0.318181in}{0.492214in}}%
\pgfpathlineto{\pgfqpoint{0.355954in}{0.546902in}}%
\pgfpathlineto{\pgfqpoint{0.414786in}{0.590303in}}%
\pgfpathlineto{\pgfqpoint{0.488919in}{0.618168in}}%
\pgfpathlineto{\pgfqpoint{0.571096in}{0.627770in}}%
\pgfpathlineto{\pgfqpoint{0.653273in}{0.618168in}}%
\pgfpathlineto{\pgfqpoint{0.704061in}{0.599078in}}%
\pgfpathlineto{\pgfqpoint{0.754849in}{0.618168in}}%
\pgfpathlineto{\pgfqpoint{0.837026in}{0.627770in}}%
\pgfpathlineto{\pgfqpoint{0.863619in}{0.624663in}}%
\pgfpathlineto{\pgfqpoint{0.863619in}{0.238521in}}%
\pgfpathlineto{\pgfqpoint{0.837026in}{0.235414in}}%
\pgfpathlineto{\pgfqpoint{0.754849in}{0.245015in}}%
\pgfpathlineto{\pgfqpoint{0.704061in}{0.264106in}}%
\pgfpathlineto{\pgfqpoint{0.653273in}{0.245015in}}%
\pgfpathclose%
\pgfusepath{}%
\end{pgfscope}%
\begin{pgfscope}%
\pgfsetbuttcap%
\pgfsetmiterjoin%
\pgfsetlinewidth{0.000000pt}%
\definecolor{currentstroke}{rgb}{0.000000,0.000000,0.000000}%
\pgfsetstrokecolor{currentstroke}%
\pgfsetdash{}{0pt}%
\pgfpathrectangle{\pgfqpoint{0.039236in}{0.039236in}}{\pgfqpoint{1.595582in}{1.177068in}}%
\pgfusepath{clip}%
\pgfpathmoveto{\pgfqpoint{0.571096in}{0.235414in}}%
\pgfpathlineto{\pgfqpoint{0.488919in}{0.245015in}}%
\pgfpathlineto{\pgfqpoint{0.414786in}{0.272880in}}%
\pgfpathlineto{\pgfqpoint{0.355954in}{0.316281in}}%
\pgfpathlineto{\pgfqpoint{0.318181in}{0.370969in}}%
\pgfpathlineto{\pgfqpoint{0.305166in}{0.431592in}}%
\pgfpathlineto{\pgfqpoint{0.318181in}{0.492214in}}%
\pgfpathlineto{\pgfqpoint{0.355954in}{0.546902in}}%
\pgfpathlineto{\pgfqpoint{0.414786in}{0.590303in}}%
\pgfpathlineto{\pgfqpoint{0.488919in}{0.618168in}}%
\pgfpathlineto{\pgfqpoint{0.571096in}{0.627770in}}%
\pgfpathlineto{\pgfqpoint{0.653273in}{0.618168in}}%
\pgfpathlineto{\pgfqpoint{0.704061in}{0.599078in}}%
\pgfpathlineto{\pgfqpoint{0.754849in}{0.618168in}}%
\pgfpathlineto{\pgfqpoint{0.837026in}{0.627770in}}%
\pgfpathlineto{\pgfqpoint{0.863619in}{0.624663in}}%
\pgfpathlineto{\pgfqpoint{0.863619in}{0.238521in}}%
\pgfpathlineto{\pgfqpoint{0.837026in}{0.235414in}}%
\pgfpathlineto{\pgfqpoint{0.754849in}{0.245015in}}%
\pgfpathlineto{\pgfqpoint{0.704061in}{0.264106in}}%
\pgfpathlineto{\pgfqpoint{0.653273in}{0.245015in}}%
\pgfpathclose%
\pgfusepath{clip}%
\pgfsys@defobject{currentpattern}{\pgfqpoint{0in}{0in}}{\pgfqpoint{1in}{1in}}{%
\begin{pgfscope}%
\pgfpathrectangle{\pgfqpoint{0in}{0in}}{\pgfqpoint{1in}{1in}}%
\pgfusepath{clip}%
\pgfpathmoveto{\pgfqpoint{-0.500000in}{0.500000in}}%
\pgfpathlineto{\pgfqpoint{0.500000in}{1.500000in}}%
\pgfpathmoveto{\pgfqpoint{-0.444444in}{0.444444in}}%
\pgfpathlineto{\pgfqpoint{0.555556in}{1.444444in}}%
\pgfpathmoveto{\pgfqpoint{-0.388889in}{0.388889in}}%
\pgfpathlineto{\pgfqpoint{0.611111in}{1.388889in}}%
\pgfpathmoveto{\pgfqpoint{-0.333333in}{0.333333in}}%
\pgfpathlineto{\pgfqpoint{0.666667in}{1.333333in}}%
\pgfpathmoveto{\pgfqpoint{-0.277778in}{0.277778in}}%
\pgfpathlineto{\pgfqpoint{0.722222in}{1.277778in}}%
\pgfpathmoveto{\pgfqpoint{-0.222222in}{0.222222in}}%
\pgfpathlineto{\pgfqpoint{0.777778in}{1.222222in}}%
\pgfpathmoveto{\pgfqpoint{-0.166667in}{0.166667in}}%
\pgfpathlineto{\pgfqpoint{0.833333in}{1.166667in}}%
\pgfpathmoveto{\pgfqpoint{-0.111111in}{0.111111in}}%
\pgfpathlineto{\pgfqpoint{0.888889in}{1.111111in}}%
\pgfpathmoveto{\pgfqpoint{-0.055556in}{0.055556in}}%
\pgfpathlineto{\pgfqpoint{0.944444in}{1.055556in}}%
\pgfpathmoveto{\pgfqpoint{0.000000in}{0.000000in}}%
\pgfpathlineto{\pgfqpoint{1.000000in}{1.000000in}}%
\pgfpathmoveto{\pgfqpoint{0.055556in}{-0.055556in}}%
\pgfpathlineto{\pgfqpoint{1.055556in}{0.944444in}}%
\pgfpathmoveto{\pgfqpoint{0.111111in}{-0.111111in}}%
\pgfpathlineto{\pgfqpoint{1.111111in}{0.888889in}}%
\pgfpathmoveto{\pgfqpoint{0.166667in}{-0.166667in}}%
\pgfpathlineto{\pgfqpoint{1.166667in}{0.833333in}}%
\pgfpathmoveto{\pgfqpoint{0.222222in}{-0.222222in}}%
\pgfpathlineto{\pgfqpoint{1.222222in}{0.777778in}}%
\pgfpathmoveto{\pgfqpoint{0.277778in}{-0.277778in}}%
\pgfpathlineto{\pgfqpoint{1.277778in}{0.722222in}}%
\pgfpathmoveto{\pgfqpoint{0.333333in}{-0.333333in}}%
\pgfpathlineto{\pgfqpoint{1.333333in}{0.666667in}}%
\pgfpathmoveto{\pgfqpoint{0.388889in}{-0.388889in}}%
\pgfpathlineto{\pgfqpoint{1.388889in}{0.611111in}}%
\pgfpathmoveto{\pgfqpoint{0.444444in}{-0.444444in}}%
\pgfpathlineto{\pgfqpoint{1.444444in}{0.555556in}}%
\pgfpathmoveto{\pgfqpoint{0.500000in}{-0.500000in}}%
\pgfpathlineto{\pgfqpoint{1.500000in}{0.500000in}}%
\pgfusepath{stroke}%
\end{pgfscope}%
}%
\pgfsys@transformshift{0.305166in}{0.235414in}%
\pgfsys@useobject{currentpattern}{}%
\pgfsys@transformshift{1in}{0in}%
\pgfsys@transformshift{-1in}{0in}%
\pgfsys@transformshift{0in}{1in}%
\end{pgfscope}%
\begin{pgfscope}%
\pgfpathrectangle{\pgfqpoint{0.039236in}{0.039236in}}{\pgfqpoint{1.595582in}{1.177068in}}%
\pgfusepath{clip}%
\pgfsetbuttcap%
\pgfsetroundjoin%
\pgfsetlinewidth{2.007500pt}%
\definecolor{currentstroke}{rgb}{0.750000,0.750000,0.000000}%
\pgfsetstrokecolor{currentstroke}%
\pgfsetdash{{7.400000pt}{3.200000pt}}{0.000000pt}%
\pgfpathmoveto{\pgfqpoint{0.863619in}{0.137325in}}%
\pgfpathlineto{\pgfqpoint{0.850011in}{0.145017in}}%
\pgfpathlineto{\pgfqpoint{0.834216in}{0.151843in}}%
\pgfpathlineto{\pgfqpoint{0.806174in}{0.161003in}}%
\pgfpathlineto{\pgfqpoint{0.694845in}{0.192819in}}%
\pgfpathlineto{\pgfqpoint{0.666881in}{0.203348in}}%
\pgfpathlineto{\pgfqpoint{0.638849in}{0.215803in}}%
\pgfpathlineto{\pgfqpoint{0.611309in}{0.230353in}}%
\pgfpathlineto{\pgfqpoint{0.554284in}{0.220989in}}%
\pgfpathlineto{\pgfqpoint{0.514659in}{0.219868in}}%
\pgfpathlineto{\pgfqpoint{0.481193in}{0.221898in}}%
\pgfpathlineto{\pgfqpoint{0.451625in}{0.225956in}}%
\pgfpathlineto{\pgfqpoint{0.424956in}{0.231519in}}%
\pgfpathlineto{\pgfqpoint{0.400632in}{0.238280in}}%
\pgfpathlineto{\pgfqpoint{0.378308in}{0.246034in}}%
\pgfpathlineto{\pgfqpoint{0.357751in}{0.254639in}}%
\pgfpathlineto{\pgfqpoint{0.338799in}{0.263982in}}%
\pgfpathlineto{\pgfqpoint{0.321331in}{0.273978in}}%
\pgfpathlineto{\pgfqpoint{0.305257in}{0.284556in}}%
\pgfpathlineto{\pgfqpoint{0.290510in}{0.295657in}}%
\pgfpathlineto{\pgfqpoint{0.277037in}{0.307228in}}%
\pgfpathlineto{\pgfqpoint{0.264798in}{0.319224in}}%
\pgfpathlineto{\pgfqpoint{0.253762in}{0.331605in}}%
\pgfpathlineto{\pgfqpoint{0.243907in}{0.344334in}}%
\pgfpathlineto{\pgfqpoint{0.235216in}{0.357377in}}%
\pgfpathlineto{\pgfqpoint{0.227677in}{0.370701in}}%
\pgfpathlineto{\pgfqpoint{0.221283in}{0.384276in}}%
\pgfpathlineto{\pgfqpoint{0.216031in}{0.398071in}}%
\pgfpathlineto{\pgfqpoint{0.211921in}{0.412058in}}%
\pgfpathlineto{\pgfqpoint{0.208956in}{0.426208in}}%
\pgfpathlineto{\pgfqpoint{0.207143in}{0.440492in}}%
\pgfpathlineto{\pgfqpoint{0.206489in}{0.454881in}}%
\pgfpathlineto{\pgfqpoint{0.207006in}{0.469346in}}%
\pgfpathlineto{\pgfqpoint{0.208707in}{0.483857in}}%
\pgfpathlineto{\pgfqpoint{0.211607in}{0.498381in}}%
\pgfpathlineto{\pgfqpoint{0.215723in}{0.512886in}}%
\pgfpathlineto{\pgfqpoint{0.221074in}{0.527337in}}%
\pgfpathlineto{\pgfqpoint{0.227679in}{0.541696in}}%
\pgfpathlineto{\pgfqpoint{0.235559in}{0.555926in}}%
\pgfpathlineto{\pgfqpoint{0.244736in}{0.569982in}}%
\pgfpathlineto{\pgfqpoint{0.255230in}{0.583820in}}%
\pgfpathlineto{\pgfqpoint{0.267061in}{0.597390in}}%
\pgfpathlineto{\pgfqpoint{0.280245in}{0.610639in}}%
\pgfpathlineto{\pgfqpoint{0.309951in}{0.636126in}}%
\pgfpathlineto{\pgfqpoint{0.341537in}{0.660531in}}%
\pgfpathlineto{\pgfqpoint{0.374264in}{0.683588in}}%
\pgfpathlineto{\pgfqpoint{0.423167in}{0.715158in}}%
\pgfpathlineto{\pgfqpoint{0.452605in}{0.733882in}}%
\pgfpathlineto{\pgfqpoint{0.510213in}{0.782614in}}%
\pgfpathlineto{\pgfqpoint{0.679585in}{0.923275in}}%
\pgfpathlineto{\pgfqpoint{0.724967in}{0.964920in}}%
\pgfpathlineto{\pgfqpoint{0.766981in}{1.006703in}}%
\pgfpathlineto{\pgfqpoint{0.821962in}{1.066615in}}%
\pgfpathlineto{\pgfqpoint{0.851361in}{1.101227in}}%
\pgfpathlineto{\pgfqpoint{0.864619in}{1.118959in}}%
\pgfpathlineto{\pgfqpoint{0.854328in}{1.112377in}}%
\pgfpathlineto{\pgfqpoint{0.865278in}{1.118902in}}%
\pgfpathlineto{\pgfqpoint{0.852178in}{1.115148in}}%
\pgfpathlineto{\pgfqpoint{0.865593in}{1.118681in}}%
\pgfpathlineto{\pgfqpoint{0.851760in}{1.115869in}}%
\pgfpathlineto{\pgfqpoint{0.865631in}{1.118619in}}%
\pgfpathlineto{\pgfqpoint{0.851760in}{1.115869in}}%
\pgfusepath{stroke}%
\end{pgfscope}%
\begin{pgfscope}%
\pgfpathrectangle{\pgfqpoint{0.039236in}{0.039236in}}{\pgfqpoint{1.595582in}{1.177068in}}%
\pgfusepath{clip}%
\pgfsetbuttcap%
\pgfsetroundjoin%
\definecolor{currentfill}{rgb}{0.000000,0.000000,0.000000}%
\pgfsetfillcolor{currentfill}%
\pgfsetlinewidth{1.003750pt}%
\definecolor{currentstroke}{rgb}{0.000000,0.000000,0.000000}%
\pgfsetstrokecolor{currentstroke}%
\pgfsetdash{}{0pt}%
\pgfsys@defobject{currentmarker}{\pgfqpoint{-0.041667in}{-0.041667in}}{\pgfqpoint{0.041667in}{0.041667in}}{%
\pgfpathmoveto{\pgfqpoint{0.000000in}{-0.041667in}}%
\pgfpathcurveto{\pgfqpoint{0.011050in}{-0.041667in}}{\pgfqpoint{0.021649in}{-0.037276in}}{\pgfqpoint{0.029463in}{-0.029463in}}%
\pgfpathcurveto{\pgfqpoint{0.037276in}{-0.021649in}}{\pgfqpoint{0.041667in}{-0.011050in}}{\pgfqpoint{0.041667in}{0.000000in}}%
\pgfpathcurveto{\pgfqpoint{0.041667in}{0.011050in}}{\pgfqpoint{0.037276in}{0.021649in}}{\pgfqpoint{0.029463in}{0.029463in}}%
\pgfpathcurveto{\pgfqpoint{0.021649in}{0.037276in}}{\pgfqpoint{0.011050in}{0.041667in}}{\pgfqpoint{0.000000in}{0.041667in}}%
\pgfpathcurveto{\pgfqpoint{-0.011050in}{0.041667in}}{\pgfqpoint{-0.021649in}{0.037276in}}{\pgfqpoint{-0.029463in}{0.029463in}}%
\pgfpathcurveto{\pgfqpoint{-0.037276in}{0.021649in}}{\pgfqpoint{-0.041667in}{0.011050in}}{\pgfqpoint{-0.041667in}{0.000000in}}%
\pgfpathcurveto{\pgfqpoint{-0.041667in}{-0.011050in}}{\pgfqpoint{-0.037276in}{-0.021649in}}{\pgfqpoint{-0.029463in}{-0.029463in}}%
\pgfpathcurveto{\pgfqpoint{-0.021649in}{-0.037276in}}{\pgfqpoint{-0.011050in}{-0.041667in}}{\pgfqpoint{0.000000in}{-0.041667in}}%
\pgfpathclose%
\pgfusepath{stroke,fill}%
}%
\begin{pgfscope}%
\pgfsys@transformshift{0.863619in}{0.137325in}%
\pgfsys@useobject{currentmarker}{}%
\end{pgfscope}%
\end{pgfscope}%
\begin{pgfscope}%
\pgfpathrectangle{\pgfqpoint{0.039236in}{0.039236in}}{\pgfqpoint{1.595582in}{1.177068in}}%
\pgfusepath{clip}%
\pgfsetbuttcap%
\pgfsetbeveljoin%
\definecolor{currentfill}{rgb}{0.000000,0.000000,0.000000}%
\pgfsetfillcolor{currentfill}%
\pgfsetlinewidth{1.003750pt}%
\definecolor{currentstroke}{rgb}{0.000000,0.000000,0.000000}%
\pgfsetstrokecolor{currentstroke}%
\pgfsetdash{}{0pt}%
\pgfsys@defobject{currentmarker}{\pgfqpoint{-0.047553in}{-0.040451in}}{\pgfqpoint{0.047553in}{0.050000in}}{%
\pgfpathmoveto{\pgfqpoint{0.000000in}{0.050000in}}%
\pgfpathlineto{\pgfqpoint{-0.011226in}{0.015451in}}%
\pgfpathlineto{\pgfqpoint{-0.047553in}{0.015451in}}%
\pgfpathlineto{\pgfqpoint{-0.018164in}{-0.005902in}}%
\pgfpathlineto{\pgfqpoint{-0.029389in}{-0.040451in}}%
\pgfpathlineto{\pgfqpoint{-0.000000in}{-0.019098in}}%
\pgfpathlineto{\pgfqpoint{0.029389in}{-0.040451in}}%
\pgfpathlineto{\pgfqpoint{0.018164in}{-0.005902in}}%
\pgfpathlineto{\pgfqpoint{0.047553in}{0.015451in}}%
\pgfpathlineto{\pgfqpoint{0.011226in}{0.015451in}}%
\pgfpathclose%
\pgfusepath{stroke,fill}%
}%
\begin{pgfscope}%
\pgfsys@transformshift{0.863619in}{1.118215in}%
\pgfsys@useobject{currentmarker}{}%
\end{pgfscope}%
\end{pgfscope}%
\begin{pgfscope}%
\pgfpathrectangle{\pgfqpoint{0.039236in}{0.039236in}}{\pgfqpoint{1.595582in}{1.177068in}}%
\pgfusepath{clip}%
\pgfsetbuttcap%
\pgfsetmiterjoin%
\definecolor{currentfill}{rgb}{0.000000,0.000000,0.000000}%
\pgfsetfillcolor{currentfill}%
\pgfsetlinewidth{1.003750pt}%
\definecolor{currentstroke}{rgb}{0.000000,0.000000,0.000000}%
\pgfsetstrokecolor{currentstroke}%
\pgfsetdash{}{0pt}%
\pgfsys@defobject{currentmarker}{\pgfqpoint{-0.041667in}{-0.041667in}}{\pgfqpoint{0.041667in}{0.041667in}}{%
\pgfpathmoveto{\pgfqpoint{-0.041667in}{-0.041667in}}%
\pgfpathlineto{\pgfqpoint{0.041667in}{-0.041667in}}%
\pgfpathlineto{\pgfqpoint{0.041667in}{0.041667in}}%
\pgfpathlineto{\pgfqpoint{-0.041667in}{0.041667in}}%
\pgfpathclose%
\pgfusepath{stroke,fill}%
}%
\begin{pgfscope}%
\pgfsys@transformshift{0.837026in}{0.686623in}%
\pgfsys@useobject{currentmarker}{}%
\end{pgfscope}%
\end{pgfscope}%
\begin{pgfscope}%
\pgfpathrectangle{\pgfqpoint{0.039236in}{0.039236in}}{\pgfqpoint{1.595582in}{1.177068in}}%
\pgfusepath{clip}%
\pgfsetbuttcap%
\pgfsetmiterjoin%
\definecolor{currentfill}{rgb}{0.000000,0.500000,0.000000}%
\pgfsetfillcolor{currentfill}%
\pgfsetlinewidth{1.003750pt}%
\definecolor{currentstroke}{rgb}{0.000000,0.500000,0.000000}%
\pgfsetstrokecolor{currentstroke}%
\pgfsetdash{}{0pt}%
\pgfsys@defobject{currentmarker}{\pgfqpoint{-0.035355in}{-0.058926in}}{\pgfqpoint{0.035355in}{0.058926in}}{%
\pgfpathmoveto{\pgfqpoint{-0.000000in}{-0.058926in}}%
\pgfpathlineto{\pgfqpoint{0.035355in}{0.000000in}}%
\pgfpathlineto{\pgfqpoint{0.000000in}{0.058926in}}%
\pgfpathlineto{\pgfqpoint{-0.035355in}{0.000000in}}%
\pgfpathclose%
\pgfusepath{stroke,fill}%
}%
\begin{pgfscope}%
\pgfsys@transformshift{1.065726in}{0.431592in}%
\pgfsys@useobject{currentmarker}{}%
\end{pgfscope}%
\end{pgfscope}%
\begin{pgfscope}%
\pgfpathrectangle{\pgfqpoint{0.039236in}{0.039236in}}{\pgfqpoint{1.595582in}{1.177068in}}%
\pgfusepath{clip}%
\pgfsetbuttcap%
\pgfsetroundjoin%
\pgfsetlinewidth{1.505625pt}%
\definecolor{currentstroke}{rgb}{0.121569,0.466667,0.705882}%
\pgfsetstrokecolor{currentstroke}%
\pgfsetdash{{5.550000pt}{2.400000pt}}{0.000000pt}%
\pgfpathmoveto{\pgfqpoint{0.863619in}{0.137325in}}%
\pgfpathlineto{\pgfqpoint{0.863619in}{1.118215in}}%
\pgfusepath{stroke}%
\end{pgfscope}%
\end{pgfpicture}%
\makeatother%
\endgroup%